\documentclass[english]{article}

\usepackage[english]{babel}
\usepackage[utf8]{inputenc}

\usepackage{color,xcolor,ucs}
\usepackage{mathtools}
\usepackage{subfig}
\usepackage{wrapfig}
\usepackage{caption}
\usepackage{sidecap}
\usepackage{floatrow}
\usepackage{tabularx}
\usepackage{float}
\usepackage{amsfonts}
\usepackage{helvet}         
\usepackage{courier}        
\usepackage{type1cm}        
\usepackage{amsmath}
\usepackage{amsthm}
\usepackage{amssymb}
\usepackage{makeidx}         
\usepackage{comment}         
\usepackage{graphicx}        
\usepackage{multicol}        
\usepackage[bottom]{footmisc}
\usepackage{bm}

\usepackage{cite}
\usepackage{url}

\usepackage{xr-hyper}
\usepackage[unicode=true, bookmarks=true,colorlinks=true]{hyperref}

\usepackage{algpseudocode,algorithm,algorithmicx}

\usepackage{geometry}
\usepackage{xspace}
\usepackage{rotating}

\usepackage{threeparttable}
\usepackage{multirow}
\usepackage[font=scriptsize,labelfont=bf]{caption}

\usepackage{natbib}
\setcitestyle{authoryear,open={(},close={)}}

\DeclareMathOperator*{\argmin}{arg\,min}
\DeclareMathOperator*{\argmax}{arg\,max}
\DeclareMathOperator*{\Expec}{\mathbb{E}}

\newcommand{\expt}[2]{\ensuremath{\underset{{#1}}{\mathbb E}{\left[{#2}\right]}}}

\newcommand{\prob}[1]{\ensuremath{\mathbb{P}({#1})}}

\algrenewcommand\algorithmicrequire{\textbf{Input:}}
\algrenewcommand\algorithmicensure{\textbf{Input:}}
\algnewcommand{\LineComment}[1]{\State \(\triangleright\) #1}

\newcommand{\ibsp}{\ensuremath{\texttt{iX-BSP}}\xspace}
\newcommand{\fullbsp}{\ensuremath{\texttt{X-BSP}}\xspace}
\newcommand{\mlbsp}{\ensuremath{\texttt{ML-BSP}}\xspace}
\newcommand{\imlbsp}{\ensuremath{\texttt{iML-BSP}}\xspace}

\newcommand{\wf}{\ensuremath{\texttt{wildfire}}\xspace}
\newcommand{\JD}{\ensuremath{\mathbb{D}_{\sqrt{J}}}\xspace}
\newcommand{\DA}{\ensuremath{\mathbb{D}_{DA}}\xspace}
\newcommand{\aH}{\ensuremath{\alpha\!-\!H\overset{..}{o}lder}\xspace}

\newtheorem{innercustomgeneric}{\customgenericname}
\providecommand{\customgenericname}{}
\newcommand{\newcustomtheorem}[2]{%
  \newenvironment{#1}[1]
  {%
   \renewcommand\customgenericname{#2}%
   \renewcommand\theinnercustomgeneric{##1}%
   \innercustomgeneric
  }
  {\endinnercustomgeneric}
}

\newcustomtheorem{forceTheorem}{Theorem}
\newcustomtheorem{forceLemma}{Lemma}
\newcustomtheorem{forceCorollary}{Corollary}

\newtheorem{theorem}{Theorem}
\newtheorem{lemma}{Lemma}

\newtheorem{assumption}{Assumption}
\newtheorem{corollary}{Corollary}

\title{\ibsp: Incremental Belief Space Planning}

\author{\renewcommand\footnotemark{}Elad I. Farhi and Vadim Indelman
	\thanks{Elad I. Farhi is with the Technion Autonomous Systems Program (TASP), Technion - Israel Institute of Technology, Haifa 32000, Israel, {\tt{eladf@campus.technion.ac.il}}. Vadim Indelman is with the Department of Aerospace Engineering, Technion - Israel Institute of Technology, Haifa 32000, Israel, {\tt vadim.indelman@technion.ac.il}.
	}	
	\vspace{-15pt}
}

\date{}

\begin{document}

\maketitle

\vspace{-12pt}
\begin{center}
	\begin{abstract}
Deciding \emph{what's next?} is a fundamental problem in robotics and Artificial Inteligence.   
Under belief space planning (BSP), in a partially observable setting, it involves calculating the expected accumulated belief-dependent reward (cost), where the expectation is with respect to all future measurements.
Since solving this general un-approximated problem quickly becomes intractable, state of the art approaches turn to approximations while still calculating each planning session from scratch.
In this work we propose a novel paradigm, Incremental BSP (\ibsp), based on the key insight that calculations across planning sessions are similar in nature and can thus be appropriately re-used.
We calculate the expectation incrementally by utilizing Multiple Importance Sampling techniques for selective re-sampling and re-use of measurement from previous planning sessions. The formulation of our approach considers general distributions and accounts for data association aspects. 
We demonstrate how \ibsp could benefit existing approximations of the general problem, introducing \imlbsp, which re-uses calculations across planning sessions under the common Maximum Likelihood assumption.
We evaluate both methods and demonstrate a substantial reduction in computation time while statistically preserving accuracy. The evaluation includes both simulation and real-world experiments considering autonomous vision-based navigation and SLAM.
As a further contribution, we introduce to \ibsp the non-integral \wf approximation, allowing one to trade accuracy for computational performance by averting from updating re-used beliefs when they are "close enough". We evaluate \ibsp under \wf demonstrating a substantial reduction in computation time while controlling the accuracy sacrifice. We also provide with analytical and empirical bounds of the effect \wf holds over the objective value. 
\end{abstract}

\end{center}

\vspace{-8pt}

\section{Introduction}
\label{sec:introduction}
Computationally efficient approaches for decision making under uncertainty, are desirable in any autonomous system, and are entrusted with providing the next optimal action(s). 
The common and important setting in which the state is partially observable is usually handled by the Partially Observable Markov Decision Process (POMDP) problem known as PSAPCE-complete \citep{Papadimitriou87math}. Belief Space Planning (BSP) is a particular instantiation of the POMDP problem, and as such finding its globally optimal solution is also computationally intractable.

\begin{figure}
	\centering
        \subfloat[]{\includegraphics[trim={0 0 0 0},clip, width=0.5\columnwidth]{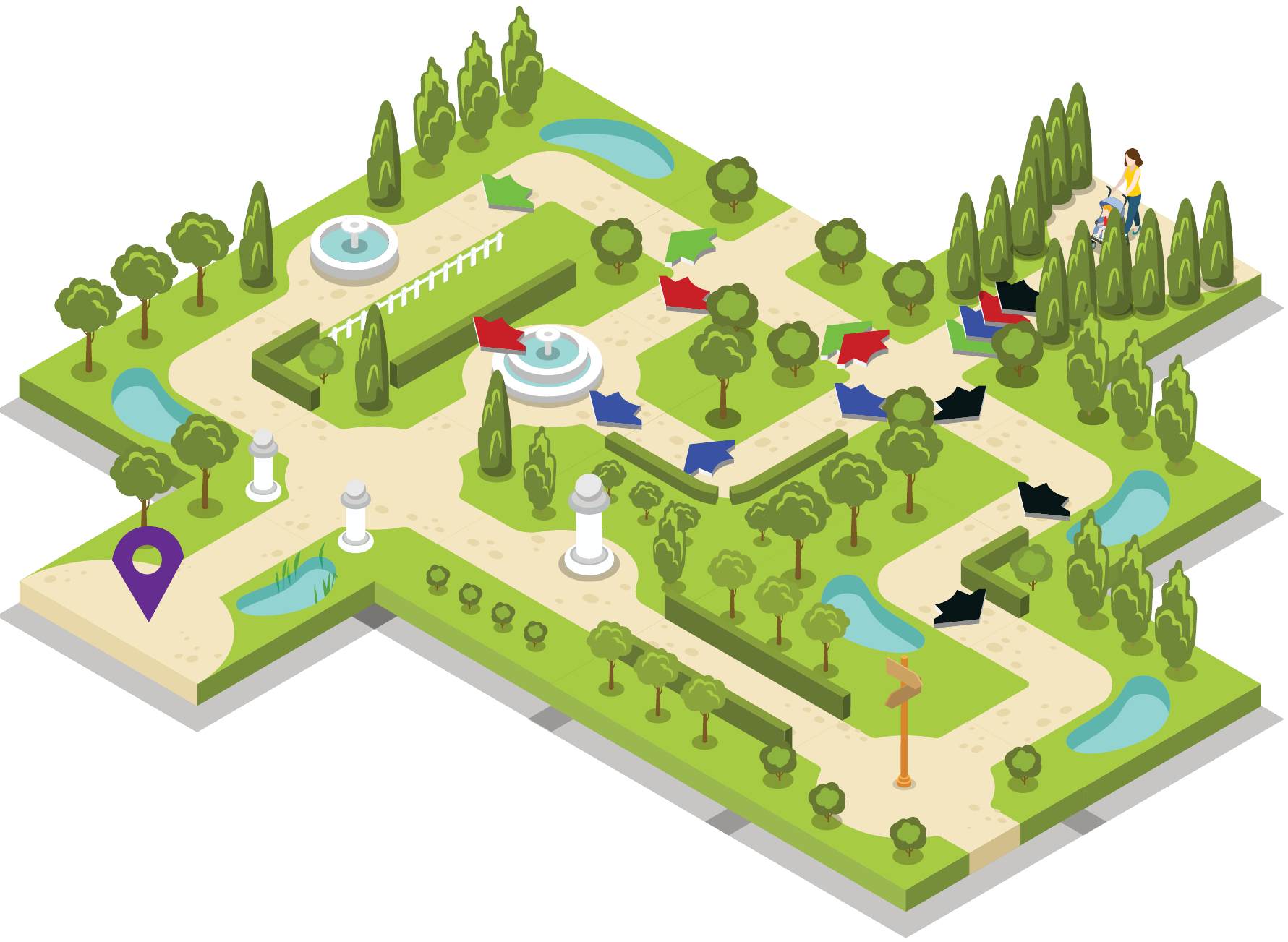}\label{fig:parkBSP:planning_k}}
        \subfloat[]{\includegraphics[trim={0 0 0 0},clip, width=0.5\columnwidth]{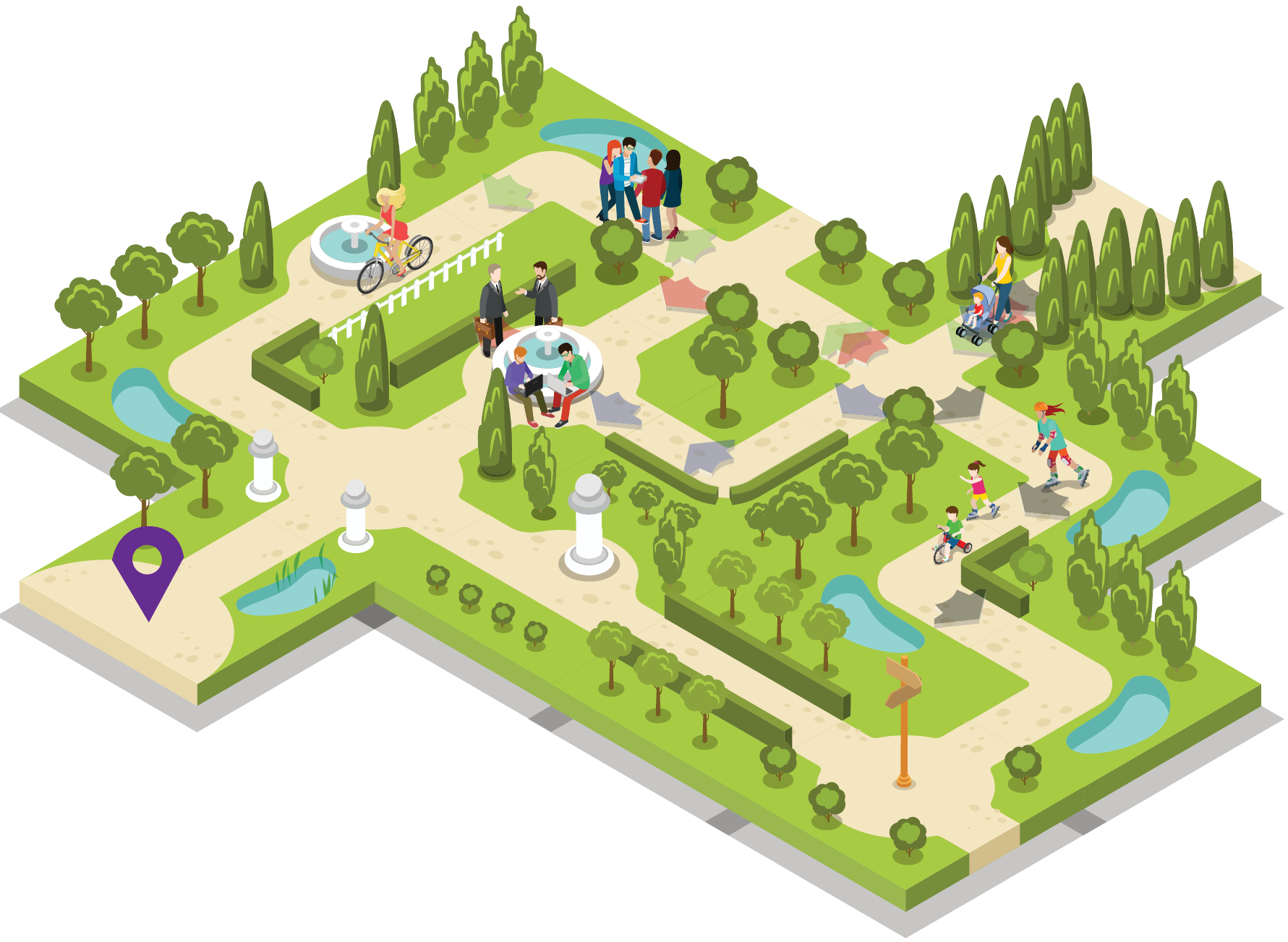}\label{fig:parkBSP:planning_k+1}}
        \caption{Illustrating the difference between \ibsp and \fullbsp using a simple decision making problem. Veronica would like to cross the park with her son as quickly as possible. (a) Before entering the park Veronica, knowing the layout, considers all possible routes, denoted by colored arrows. (b) After entering the park Veronica obtains new information which renders her existing plan suboptimal. Under \fullbsp Veronica would replan from scratch while under \ibsp she would simply incrementally update her existing plan with the new information.}
        \label{fig:parkBSP}
\end{figure}

The main cause for the BSP problem intractability, resides with the use of expectation in the objective function, i.e. reasoning about belief evolution along different candidate actions while considering all possible future measurements 
\begin{equation}
	J(\mathcal{U}) = \expt{z}{ 
	\sum\limits_{i} r_{i}\left(b_i,u_{i-1} \right)
	}.
\end{equation}
The objective over a candidate action sequence $\mathcal{U}$, is obtained by calculating the expected value of all possible rewards (costs) $r$ received from following $\mathcal{U}$. Since the reward (cost) function is a function of the belief $b$ and the action led to it $u$, in practice the objective considers all future beliefs obtained from following $\mathcal{U}$, i.e. the expectation considers the joint measurement likelihood of all future measurements $z$.
We refer to this general problem as the full solution of BSP, denoted by \fullbsp, expectation based BSP.

 The exponential growth of possible measurements and candidate actions, usually denoted as the \emph{curse of history}, is the key aspect targeted by a lot of research efforts. 
Performing inference over multiple future beliefs is the key reason for the costly computation time of \fullbsp. In a planning session with a horizon of 3 steps ahead, 3 candidate actions per step and 3 samples per action, we are required to solve a staggering number of 819 beliefs. Cutting down on the computation time of each belief would benefit the overall computation time of the planning process. 
 
 As in any computational problem, one can either streamline the solution process or change the problem, i.e. take simplifying assumptions or approximations.


Indeed, over the years, numerous approaches have been developed to trade-off suboptimal performance with reduced computational complexity of POMDP, see e.g.~\cite{Pineau06jair, Kurniawati08rss, Hollinger14ijrr,Toussaint09icml}. While the majority of these approaches, including \cite{Prentice09ijrr, Platt10rss,Platt11isrr,Bry11icra, VanDenBerg12ijrr}, assumed some sources of absolute information (GPS, known landmarks) are available or considered the environment to be known, recent research relaxed these assumptions, accounting for the uncertainties in the mapped environment thus far as part of the decision making process (e.g.~\cite{Kim14ijrr, Indelman15ijrr}) at the price of increased state dimensionality. 
Other than assuming available sources of absolute information, some of these approaches use discretization in order to reduce computational complexity.
 Sampling based approaches, e.g.~\cite{AghaMohammadi14ijrr, Bry11icra,Hollinger14ijrr,Prentice09ijrr,Kurniawati08rss}, discretize the state space using randomized exploration strategies to locate the belief's optimal strategy. Other approaches, e.g.~\cite{Richter15isrr}, discretize the action space thus trading optimality with reduced computational load.
 While many sampling based approaches (e.g.~\cite{Kavraki96tra, Lavalle01ijrr, Karaman11ijrr})
 , including probabilistic roadmap (PRM)~\citep{Kavraki96tra}, rapidly exploring random trees (RRT)~\citep{Lavalle01ijrr}, RRT*~\citep{Karaman11icra} and rapidly exploring random graph (RRG)~\citep{Karaman11ijrr} assume perfect knowledge of the state (i.e. MDP framework), along with deterministic control and known environment, efforts have been made to assuage these simplifying assumptions. These efforts vary in the alleviated-assumptions, from the belief roadmap (BRM)~\citep{Prentice09ijrr} and the rapidly exploring random belief trees (RRBT)~\citep{Bry11icra}, through, Partially Observable Monte-Carlo Planning (POMCP)~\citep{Silver10nips}, Determinized Sparse Partially Observable Tree (DESPOT)~\citep{Ye17jair,Luo16ijrr} and up to active full SLAM in discrete \citep{Stachniss04iros} and  continuous \citep{Indelman15ijrr} domains accounting for uncertainties in the environment mapped thus far as part of the decision making process (e.g.~\cite{Stachniss04iros, Valencia13tro, Kim14ijrr, Indelman15ijrr}) at the price of increased state dimensionality.
 
While all the aforementioned research efforts tackle the curse of history through providing various approximations to the \fullbsp problem, a common denominator for some of them is the Maximum Likelihood (ML) assumption \citep{Platt10rss}, which allows to prune \fullbsp by considering only the maximum likelihood measurements rather than all possible ones. We denote the use of ML in BSP as \mlbsp. 

While the aforementioned research efforts mainly focused on approximating the \fullbsp, we suggest a new paradigm, challenging the standard formulation of \fullbsp by re-using calculations across planning sessions thus saving valuable computation time not at the expense of accuracy. 
The following section provides a more in-depth cover of research efforts related to calculation re-use in decision making under uncertainty.

\subsection{Related Work}\label{subsec:relatedWork}

As mentioned above, in strike contrast to the vast amount of research invested in approximating the \fullbsp problem, only few tried re-using calculations, this section shortly reviews some of these efforts.	

POMCP \citep{Silver10nips} is the POMDP generalization of MCTS \citep{Coulom06iccg}. 
POMCP is a point-based POMDP; as such, it is implicitly using particle filtering to express and update future beliefs,  thus breaking the curse of dimensionality. It also maintains a history search tree, containing all possible actions and only sampled measurements, thus breaking the curse of history.
At each planning session POMCP develops the tree using a black box simulator and search for the optimal policy through maintaining an upper bound over each tree node using the PO-UTC algorithm (see section 3.1 in \cite{Silver10nips}). Once an action has been executed, POMCP is pruning the existing tree of all branches but the one containing the current history, thus re-using the statistics obtained so far for the appropriate segment of the tree. Other than parameter tuning, POMCP does not require any offline calculations, although it does require the existence of a black box simulator for sampling future scenarios and their rewards. 

Adaptive Belief Tree (ABT) \citep{Kurniawati16chapter} builds on POMCP but thrives to reuse an existing policy instead of calculating a new one from scratch. Given some offline calculated policy, ABT resamples only the parts affected by newly acquired information, thus re-using the rest of the policy. In addition to offline policy calculation, ABT also requires as an input the set of all states affected by the newly acquired information, which is considered as given. Although re-using unaffected parts of the policy, ABT substitutes the affected parts of the policy with freshly sampled rather than incrementally updating them with the new information.

DESPOT \citep{Ye17jair} is also a point-based POMDP, which implicitly uses particle filtering to express and update the belief. Same as POMCP it also maintains a history search tree containing all possible actions and only sampled measurements. 
Similarly to POMCP, at each planning session DESPOT incrementally develops the tree using a simulator while searching for the optimal policy. Unlike POMCP, DESPOT makes use of upper and lower bounds over an estimate of the regularized weighted return value of each tree node, which is supposed to be less greedy than the use of PO-UCT. The lower bound is calculated over the user defined default policy (calculated offline).
 Once a policy has been chosen, the DESPOT tree is pruned of all other actions; this action-pruned DESPOT is denoted as the DESPOT policy. While following the DESPOT policy, if the agent encounters a scenario not present in the DESPOT policy it will follow the default policy from then on. Similarly it follows the default policy when encountering a leaf node of DESPOT policy. 
DESPOT also makes use of regularization in order to assure the desired performance bound while everting from potentially overfitting due to the sampled scenarios. The regularization constant as well as other heuristic related parameters are calculated offline and are case sensitive. As opposed to POMPC which operates under an MPC setting, DESPOT follows a DESPOT policy to the end of its horizon and does not make any re-use of the statistical results.  

is-DESPOT \citep{Luo19ijrr} builds on DESPOT \citep{Ye17jair} and introduces importance sampling into the scenario sampling procedure. Under is-DESPOT, scenarios are sampled from some importance distribution. The importance sampling distribution is calculated offline and is also case sensitive without explicitly addressing the deployment problem.

DESPOT-$\alpha$ \citep{Garg19rss} also builds on DESPOT \citep{Ye17jair} and introduces $\alpha$ vector approximation to the implicit particle filtering problem of propagating and maintaining future beliefs, thus better representing belief uncertainty under relatively large measurement space. 

	
In contrast to the aforementioned efforts, although under some more simplifying assumptions amongst them ML, both \cite{Chaves16iros} and \cite{Kopitkov17ijrr} re-use computationally expensive calculations during planning.
In their work, Chaves and Eustice \citep{Chaves16iros}, consider a Gaussian belief under \mlbsp in a Bayes tree \citep{Kaess10wafr} representation. All candidate action sequences consider a shared location (entrance pose), thus enabling to re-use a lot of the calculations through state ordering constrains. That work enables to efficiently evaluate a single candidate action across multiple time steps, and is conceptually applicable to multiple candidate actions at a single time step.
While Kopitkov and Indelman \citep{Kopitkov17ijrr,Kopitkov19ijrr}, also consider a Gaussian belief under \mlbsp, they utilize a  factor graph representation of the belief while considering  an information theoretic cost. Using an (augmented) determinant lemma, they are able to avert from belief propagation while re-using calculations throughout the planning session. Although they consider calculation re-use within the same planning session, their work can be augmented to consider re-use also between planning sessions.

To the best of our knowledge, in-spite of the  aforementioned research efforts, calculation re-use has only been done over \mlbsp, with restricting assumptions. While \mlbsp is widely used, the pruning of \fullbsp by considering only the most likely measurements, might mean choosing a sub-optimal action in case the biggest available reward is not the most likely one, in particular in presence of significant estimation uncertainty. 
Although under some conditions few of the aforementioned approaches re-use previously calculated policies, and even selectively re-calculate parts of the policy to account for newly acquired information, they do not address the problem of identifying the parts of the policy affected by the new information nor they incrementally update the policy to account for it.
  
As for today, up to our previous work \cite{Farhi19icra}, \fullbsp approaches do not re-use calculations across planning sessions.

\subsection{Paper Contributions}\label{subsec:contribution}	


In this paper we introduce Incremental eXpectation BSP, or \ibsp, which incrementally updates the expectation related calculations in \fullbsp by re-using previous planning sessions. In contrast to both POMPC and DESPOT which are point-based approximations of \fullbsp, \ibsp is identical to the general intractable formulation of \fullbsp with the single exception of calculation re-use. \ibsp requires no offline calculations and does not rely on any case sensitive heuristics. 
Instead of calculating the expected belief-dependent rewards, which form the objective value, from scratch, \ibsp selectively re-uses previously sampled measurements along with their associated beliefs. 
While POMPC re-uses previously calculated statistics between planning sessions of the pruned search tree it disregards impact of the newly received information over the sampled scenarios. Although ABT accounts for this new information by re-calculating affected parts of the policy, it does not address how to locate the affected segments of the policy or how to incrementally update them rather then discard them.
On the other hand, \ibsp incrementally updates the re-used beliefs with the updated posterior information, using our previous work on efficient belief update \citep{Farhi17icra,Farhi19arxiv}. 

Although in \cite{Farhi19icra}\footnote{A conference version of the current paper.} we already re-used previously sampled measurements in order to facilitate calculation re-use from precursory planning sessions, while addressing the general case, a general distribution, where the state is partially observable and thus has to be inferred within a POMDP framework, it is still based on the simplifying assumption that all samples taken from a single distribution, can be re-used.
By forcing samples, without assuring they constitute an adequate representation of the nominal distribution we should have sampled from, we risk blindly affecting both the estimation value and variance, of the expected objective. Moreover, the use of a single distribution, although not the nominal one, restricts the re-use potential of this paradigm.
Selectively resampling measurements as part of \ibsp, potentially results in estimating the expected objective value through multiple different distributions. We identify this problem within the Multiple Importance Sampling (MIS) problem \citep{Veach95siggraph}, residing in the field of importance sampling \citep{Glynn89ms} and formulate \ibsp accordingly. 

Figure~\ref{fig:parkBSP} illustrates the advantage of \ibsp over \fullbsp through a simple decision making under uncertainty problem, in-which Veronica and her son would like to cross the park in the fastest possible route. Assuming Veronica is already familiar with the park layout, she could solve her decision making problem by considering all four possible routes through the park (denoted by green, red, blue and black arrows in Figure~\ref{fig:parkBSP:planning_k}). As a result, Veronica chooses the blue path and enters the park only to find it populated as illustrated in Figure~\ref{fig:parkBSP:planning_k+1}. This newly acquired information changes the optimality and even validity of Veronica's solution. Under \fullbsp Veronica would calculate a new plan from scratch; instead, under \ibsp Veronica can simply update her existing plan with the newly acquired information, resulting with picking the black path as the fastest one.

While Veronica's planning problem in Figure~\ref{fig:parkBSP} provides some intuition on how \ibsp provides the same solution as \fullbsp but faster, in real life scenarios the desirable accuracy sometime lies within the bounds of a suboptimal solution.
As mentioned, \ibsp is not an approximation of \fullbsp, it updates some precursory planning tree with all of the new posterior information. 

Additional contribution of our work comes to satisfy the desire to controllably sacrifice accuracy for performance, introducing the \wf approximation.
While \ibsp updates some existing planning tree to exactly match current information, \wf introduces a concept of "close enough" defined by a \wf distance threshold. Under \ibsp with \wf, whenever a belief of the existing planning tree is "close enough" to its updated counterpart, this "close enough" belief is considered as already updated, along with all of its decedents in the planning tree.
The choice of \wf threshold value would directly affect the obtainable objective value. When the threshold is taken to its minimum, i.e. zero distance, there is no approximation and the obtainable objective value would statistically match the one obtained by \ibsp without \wf. As we increase the allowable \wf distance threshold, it is as if we consider some or even all of the newly acquired information as irrelevant, which would directly impact the obtainable objective value. In this work we formulate the affect of the \wf threshold over the obtainable objective value, and support it with both analytical proof and empirical results. Moreover, we provide results indicating considerable reduction in computation time under the use of \wf in \ibsp.

As \ibsp is formulated over the original un-approximated problem of \fullbsp, we go further and support our claim made in \citep{Farhi19icra}, that \ibsp can be utilized to also reduce valuable computation time of existing approximations of \fullbsp.
Considering the commonly used approximation \mlbsp, we formulate incremental \mlbsp,  referred to as \imlbsp, by simply enforcing the ML assumption over \ibsp, as being done over \fullbsp.
Under \mlbsp, beliefs are propagated with zero innovation by considering just the most likely measurement for each candidate action, thus averting from expectation and minimizing the \emph{curse of history}. Given access to calculations from precursory planning, at each look ahead step $i$ in the current planning session, \imlbsp considers the appropriate sample from the $i_{th}$ look ahead step in the precursory planning session for re-use. If the sample constitutes an adequate representation, of the measurement likelihood we would have considered at the $i_{th}$ look ahead step in current planning session, then \imlbsp utilizes the associated previously solved belief from the precursory planning session. If the mentioned sample is considered as an inadequate representation of the mentioned measurement likelihood, \ibsp follows the course of \mlbsp , and the most likely measurement of the nominal measurement likelihood is considered instead.

To summarize, our contributions in this paper 
are as follows: 
$(a)$ We present a novel paradigm for incremental expectation belief space planning with selective resampling (\ibsp). Our approach incrementally calculates the expectation over future observations by a set of samples comprising of newly sampled measurements and re-used samples generated at different planning sessions. 
$(b)$ We identify the problem of \ibsp with selective resampling as a Multiple Importance sampling problem, and provide the proper formulation while considering the balance heuristic. 
$(c)$ We evaluate \ibsp in simulation and provide statistical comparison to \fullbsp, which calculates expectation from scratch, while considering the problem of autonomous navigation in unknown environments, across different randomized scenarios.
$(d)$ We introduce the \wf approximation into \ibsp, which enables one to controllably trade accuracy for performance.
$(e)$ We provide an analytical proof of the affect the choice of a \wf threshold would have over the objective value, in the form of bounds over the objective value error.
$(f)$ We provide empirical results of using \wf within \ibsp, as well as the affect \wf holds over the objective value error.
$(g)$ We support our claim, that \ibsp can be used to improve approximations of the general problem of \fullbsp, by introducing to \ibsp the commonly used ML approximation, and denote it as \imlbsp. The novel approach of \imlbsp, incrementally calculates the expectation over future observations, while considering either the most likely observation or some previously sampled observation, given from a precursory planning session.  
$(h)$ We evaluate \imlbsp in simulation as well as in real-world experiments and compare it to the commonly used approximation for the \fullbsp problem, \mlbsp, while considering the problem of autonomous navigation in unknown environments and active visual-SLAM setting with belief over high dimensional state space. 

\subsection{Paper Outline}
This paper is organized as follows. Section~\ref{sec:problem-formulation} formulates the discussed problem, and provides the necessary theoretical background. Section~\ref{sec:approach} presents the suggested approach and its mathematical formulation. Section~\ref{sec:results} presents a thorough analysis of the suggested approach as well as a comparison to related work. Section~\ref{sec:conclusions} captivates the conclusions of our work along with future work and possible usage. To improve coherence some theoretical background as well as proofs are covered in appendices.

\vspace{-5pt}
\section{Background and Problem Formulation}
\label{sec:problem-formulation}

This section provides the theoretical background for belief space planning (BSP), starting with belief definition, followed by the BSP formulation and the common Maximum Likelihood (ML) approximation. While the formulation, as well as the suggested paradigm, are impartial to a specific belief distribution, throughout this paper we also provide the conventional case which deals with Gaussian distributions.

\subsection{Belief Definition}\label{subsec:belief}
Let $x_{t}$ denote the agent's state at time instant $t$ and $\mathcal{L}$ represent the mapped environment thus far. The joint state, up to and including time $k$, is defined as $X_{k}=\{x_{0},...,x_{k},\mathcal{L}\}$. We shall be using the notation ${t|k}$ to refer to some time instant $t$ while considering information up to and including time $k$. 
The unique time notation is required since this paper makes use of both current and future time indices in the same equations.
Let $z_{t|k}$ and $u_{t|k}$ denote, respectively, measurements and the control action at time $t$, while the current time is $k$. The measurements and controls up to time $t$ given current time is $k$, are represented by
\begin{equation}
z_{1:t|k}\doteq \{z_{1|k},...,z_{t|k}\} \ , \ u_{0:t-1|k} \doteq \{u_{0|k},...,u_{t-1|k}\},
\end{equation}
The posterior probability density function (pdf) over the joint state, denoted as the \emph{belief}, is given by
\begin{equation}\label{eq:p_k} 
	b[X_{t|k}] \doteq \prob{X_{t}|z_{1:t|k},u_{0:t-1|k}} = \prob{X_{t}|H_{t|k}}.
\end{equation}
where $H_{t|k} \! \doteq \! \{ u_{0:t-1|k},z_{1:t|k} \} $ represents history at time $t$ given current time $k$. The propagated belief at time $t$, i.e. belief $b[X_{t|k}]$ lacking the measurements of time $t$, is denoted by
\begin{equation}
	b^{-}[X_{t|k}] \! \doteq \! b[X_{t-1|k}]\!\cdot\prob{x_{t}|x_{t-1},u_{t-1|k}} \!= \prob{X_{t}|H^{-}_{t|k}},
\end{equation}
where $H_{t|k}^{-} \! \doteq \! H_{t-1|k} \cup \{u_{t-1|k}\}$.
Using Bayes rule, Eq.~(\ref{eq:p_k}) can be rewritten as 
\begin{equation}
	b[X_{t|k}] \propto \mathbb{P}(X_{0})  \prod_{i=1}^{t} \! \left[ \!\prob{x_{i}|x_{i-1},\!u_{i-1|k}} \prod_{j \in \mathcal{M}_{i|k}}  \prob{z_{i,j|k}|x_{i},l_{j}} \!\right] ,
\end{equation}
where $\prob{X_{0}}$ is the prior on the initial joint state, and $\prob{x_{i}|x_{i-1},u_{i-1|k}}$ and $\prob{z_{i,j|k}|x_{i},l_{j}}$  denote, respectively,  the motion and measurement likelihood models. Here, $z_{i,j|k}$ represents an observation of landmark $l_j$ from robot pose $x_i$, while the set $\mathcal{M}_{i|k}$ contains all landmark indices observed at time $i$, i.e. it denotes data association (DA). The DA of a few time steps is denoted by $\mathcal{M}_{1:i|k} \doteq \{\mathcal{M}_{1|k},\cdots,\mathcal{M}_{i|k}\}$. 

\subsection{Belief Space Planning}\label{subsec:bsp}

The purpose of BSP is to determine an optimal action given an objective function $J$,  belief $b[X_{k|k}]$ at planning time instant $k$ and, considering a discrete action space, a set of candidate actions $\mathcal{U}_k$. While these actions can be with different planning horizons, we consider for simplicity the same horizon of $L$ look ahead steps for all actions, i.e.~$\mathcal{U}_k=\{ u_{k:k+L-1} \}$. The optimal action is given by 
\begin{equation}\label{eq:optAction}
	u_{k:k+L-1|k}^{\star} = \underset{u_{k:k+L-1|k}\in \mathcal{U}_k}{\argmax} J(u_{k:k+L-1|k}),  
\end{equation}
where the general objective function $J(.)$ is defined as
\begin{flalign}\label{eq:objective}
	J(u) \! \doteq 
	\underset{z_{k+1:k+L|k}}{\Expec} \! \left[ \sum_{i=k+1}^{k+L} r_i \left( b[X_{i|k}],u_{i-1|k} \right) \right] \!,\!\!
\end{flalign}
with $u\doteq u_{k:k+L-1|k}$, immediate rewards (or costs) $r_i$ and where the expectation is with respect to future observations $z_{k+1:k+L|k}$ while,
\begin{equation}\label{eq:JmeasLikelihood}
	z_{k+1:k+L|k} \sim \prob{z_{k+1:k+L|k}|H_{k|k},u_{k:k+L-1}}.
\end{equation}
The expectation in (\ref{eq:objective}) can be written explicitly 
\begin{equation}\label{eq:objIntegral_full}
	J(u)  = \!\!\!\!\!
	\int\limits_{z_{k+1|k}} \!\!\!\!\! \prob{z_{k+1|k}|H_{k|k},u_{k|k}} \cdot r_{k+1}(.) + \hdots + \!\!\!\!\! \int\limits_{z_{k+1:i|k}} \!\!\!\!\! \prob{z_{k+1:i|k}|H_{k|k},u_{k:i-1|k}} \cdot r_{i}(.) +  \hdots. 
\end{equation}
Using the chain rule and the Markov assumption, we can re-formulate the joint measurement likelihood (\ref{eq:JmeasLikelihood}), as  
\begin{equation}\label{eq:measLikelihood}
	\prob{z_{k+1:k+L|k}|H_{k|k},u_{k:k+L-1}} = \prod_{i = k+1}^{k+L} \prob{z_{i|k}|H^-_{i|k}}
\end{equation}
where $H_{i|k}^{-}$ is a function of a specific sequence of measurement realization, i.e. 
\begin{equation}\label{eq:H_i_minus}
	H_{i|k}^{-} =H_{k|k} \cup \{z_{k+1:i-1|k},u_{k:i-1|k}\}.
\end{equation}
Using (\ref{eq:measLikelihood}), we can re-formulate (\ref{eq:objIntegral_full}) as 
\begin{equation}
	J(u)  = \!\!\!\!\!
	\int\limits_{z_{k+1|k}} \!\!\!\!\!\!  
	\prob{z_{k+1|k}|H_{k+1|k}^{-}}   \left[ r_{k+1}\left( b[X_{k+1|k}],u_{k|k} \right) + \ldots  \!\!\!
	\int\limits_{z_{i|k}} \!\! \prob{z_{i|k}|H_{i|k}^{-}}  \left[ r_i \left( b[X_{i|k}],u_{i-1|k} \right)+ \ldots \right] \right],   
	\label{eq:objIntegral}
\end{equation}
where each integral accounts for all possible measurement realizations from an appropriate look ahead step, with $i\in (k+1,k+L]$ and $b[X_{i|k}] = \prob{X_{i|k} | H_{i|k}^{-}, z_{i|k}}$.

Evaluating the objective   for each candidate action  in $\mathcal{U}_k$ involves calculating (\ref{eq:objIntegral}), considering all different measurement realizations. As solving these integrals analytically is typically not feasible, in practice these are  approximated by sampling future measurements. Although the measurement likelihood $\prob{z_{i|k}|H^-_{i|k}}$ is unattainable, one can still sample from it.
Specifically, consider the $i$-th future step and the corresponding $H_{i|k}^{-}$ to some realization of measurements from the previous steps. In order to sample from $\prob{z_{i|k}|H_{i|k}^{-}}$, we should marginalize over the future robot pose $x_i$ and landmarks $\mathcal{L}$
\begin{equation}
\prob{z_{i|k}|H_{i|k}^{-}} =  \!\! \int\limits_{{x}_{i}} \int\limits_{ \mathcal{L}} \prob{z_{i|k}|x_{i}, \mathcal{L}} \cdot \prob{x_{i}, \mathcal{L}|H_{i|k}^{-}} dx_i d\mathcal{L},
\label{eq:likelihoodInt}
\end{equation}
where $\prob{x_{i}, \mathcal{L}|H_{i|k}^{-}}$ can be calculated from the  belief $b^{-}[X_{i|k}]\doteq \prob{X_{i|k}|H_{i|k}^{-}}$. We approximate the above integral via sampling as summarized in Alg.~\ref{alg:sampling_z}.
One can also choose to approximate further by considering only landmark estimates $\mathcal{\hat{L}}$ (i.e.~without sampling $ \mathcal{L}$). 

\begin{algorithm}
	\caption{Sampling $z_{i|k} \sim \prob{z_{i|k}|H_{i|k}^{-}}$
		\label{alg:sampling_z}}
	\begin{algorithmic}[1]
		
		\State\label{alg:sampling_z:chi} $\chi_i\doteq \{x_i, \mathcal{L} \} \sim \prob{x_i, \mathcal{L}  | H^-_{i|k}}$

		\State\label{alg:sampling_z:DA} Determine data association $\mathcal{M}_{i|k}(x_i, \mathcal{L})$

		\State\label{alg:sampling_z:sample} $z_{i|k} = \{z_{i,j|k} \}_{j\in \mathcal{M}_{i|k}(\chi_i)}$ with  $z_{i,j|k} \sim \prob{z_{i,j|k} | x_i, l_j}$

		\State \Return{$z_{i|k}$ and $\chi_i$}
	\end{algorithmic}
\end{algorithm}

Each sample $\chi_i$ and the determined DA (lines \ref{alg:sampling_z:chi}-\ref{alg:sampling_z:DA} of Alg.~\ref{alg:sampling_z}) define a measurement likelihood $\prob{z_{i|k} | \chi_i, \mathcal{M}_{i|k}(\chi_i)} = \prod_{j\in \mathcal{M}_{i|k}(\chi_i)} \prob{z_{i,j|k} | x_i, l_j}$ from which observations are sampled in line~\ref{alg:sampling_z:sample}. Considering $n_x$ samples, $\{\chi_i^n\}_{n=1}^{n_x}$, we can approximate Eq.~(\ref{eq:likelihoodInt}) by
\begin{equation}\label{eq:likelihoodApprox}
\prob{z_{i|k}|H_{i|k}^{-}} \approx  \eta_i \!\! 
\sum\limits_{n=1}^{n_x} \! w_i^n \cdot 
\prob{z_{i|k}| \chi_i^n, \mathcal{M}_{i|k}(\chi_i^n)}, 
\end{equation}
where $w_i^n$ represents the $n$-th sample weight, $\chi_i^n$, and $\eta_i^{-1} \doteq \sum_{n=1}^{n_x} w_i^n$. 
Here, since all samples are generated from their original distribution (corresponding to the proposal distribution in importance sampling), see line~\ref{alg:sampling_z:chi}, we have identical weights. 

For each sample $\chi_i^n \in \{\chi_i^n\}_{n=1}^{n_x}$, we can generally consider $n_z$ measurement samples (line~\ref{alg:sampling_z:sample}), providing the set $\{ z_{i|k}^{n,m} \}_{m=1}^{n_z}$. In other words, Alg.~\ref{alg:sampling_z} yields $n_x\cdot n_z$ sampled measurements, denoted by $\{ z_{i|k} \}$, for a given realization of $z_{k+1:i-1|k}$. Thus, considering all such possible realizations, we get $(n_x\cdot n_z)^{i-k}$ sampled measurements for the look ahead step at time instant $i$, i.e. the $(i-k)$-th look ahead step for planning time instant $k$.

We can now write an unbiased estimator for (\ref{eq:objIntegral}), considering the $(n_x\cdot n_z)^{i-k}$ sampled measurements. In particular, for the look ahead step at time $i$, we get  
\begin{equation}\label{eq:ExApprx}
	\Expec_{z_{k+1:i|k}} \! \left[ r_i\left( b[X_{i|k}],u_{i-1|k} \right)  \right] \! \approx \!   \eta_{k+1} \!\!
	\underset{\{z_{k+1|k}\}} 
	{\sum} w_{k+1}^n \left(
	\cdots   
	\left(\eta_{i} \!\!
	\underset{\{z_{i|k}\}}{\sum} \!\!
	w_i^n \cdot r_i\left( b[X_{i|k}],u_{i-1|k} \right) \right) \cdots \right)
\end{equation}
where $H_{i|k}^{-}$  varies with each measurement realization. 
When the measurements that are used to estimate the expectation are being sampled from their nominal distributions, then all weights equal $1$, i.e. $\omega_i^{n}=1 \ \forall i, \ n$ , and evidently each normalizer equals the inverse of the sum of samples, i.e. $\eta_i^{-1} = n_x\cdot x_z \ \forall i$
\begin{equation}\label{eq:fullObj}
	J(u) = \sum_{i=k+1}^{k+L} \left[  \frac{1}{(n_x\cdot x_z)^{i-k}} \!\!\!\sum_{\{z_{k+1|k}\}} \hdots \sum_{\{z_{i|k}\}} r_i\left(b[X_{i|k}],u_{i-1|k} \right) \right].
\end{equation}
The above exponential complexity makes the described calculations quickly infeasible, due to both curse of dimensionality and history. In practice, approximate approaches, e.g.~Monte-Carlo tree search  \cite{Silver10nips}, must be used.  However, in this work we prefer to present our  paradigm considering the above formulation, without any further approximations, referring to it as \fullbsp. We believe our proposed concept can be applied in conjunction with existing approximate approaches; in particular, we demonstrate this on the commonly used approximation for the \fullbsp problem - the Maximum Likelihood approximation.

\subsection{Belief Space Planning under ML}\label{subsec:mlbsp} 
A very common approximation to Eq.~(\ref{eq:objective}) is based on the maximum likelihood (ML) observations assumption (see e.g.~\cite{Platt10rss, Kim14ijrr, Indelman15ijrr}). This approximation, referred to as \mlbsp, is often used in BSP and in particular in the context of active SLAM:  Instead of accounting for different measurement realizations, only the most likely observation is considered at each look ahead step, which corresponds to $n_x=n_z=1$ where the single sample is the most likely one. So under ML, the expectation from Eq.~(\ref{eq:objective}) is omitted, and the new objective formulation is given by
\begin{flalign}\label{eq:objectiveML}
	J^{ML}(u) \! \doteq 
	\sum_{i=k+1}^{k+L} r_i \left( b[X_{i|k}],u_{i-1|k} \right) ,
\end{flalign}
thus drastically reducing complexity at the expense of sacrificing performance. While the future belief $b[X_{i|k}]$ is given by $\prob{X_{0:i}|H_{k|k},u_{k:i-1},z^{ML}_{k+1:i|k}}$, and for the Gaussian case $z^{ML}_{k+1:i|k}$ are the measurement model mean-values.
 
\subsection{Problem Statement}\label{PS}
Consider the planning session at time instant $k$ has been solved by evaluating the objective (\ref{eq:objective}) via appropriate measurement sampling for each action in $\mathcal{U}_k$ and subsequently choosing the optimal action $u_{k:k+L-1|k}^{\star}$. 
A subset of this action, $u_{k:k+l-1|k}^{\star}\in u_{k:k+L-1|k}^{\star}$ with $l\in[1,L)$, is now executed, new measurements $z_{k+1:k+l | k+l}$ are obtained and the posterior belief $b[X_{k+l|k+l}]$ in inference is calculated, upon which a new planning session is initiated. 

Determining the optimal action sequence at time instant $k+l$ involves evaluating the objective function for each candidate action $u'\doteq u_{k+l:k+l+L-1|k+l} \in \mathcal{U}_{k+l}$
\begin{equation}\label{eq:objective_kl}
	J(u') \! \doteq 
	\Expec \! \left[ \sum_{i=k+l+1}^{k+l+L} r_i \left( b[X_{i|k+l}],u'_{i-1|k+l} \right) \right] \!,\!\!
\end{equation}
where the expectation is with respect to future observations $z_{k+l+1:k+l+L | k+l}$. Existing approaches perform these costly evaluations from scratch for each candidate action. 
Our \emph{key observation} is that expectation related calculations from two successive \fullbsp planning sessions at time instances $k$ and $k+l$ are similar and hence can often be re-used.
Our goal in this work is to develop an approach for evaluating the objective function (\ref{eq:objective_kl}) more efficiently by appropriately  re-using calculations from preceding planning sessions. 

At this point, we summarize our assumptions in this work.
\begin{assumption}\label{assumption:savedData}
Calculations from a previous planning session are accessible from the current planning session.
\end{assumption}

\begin{assumption}\label{assumption:overlapHorizon}
The planning horizon of current time  $k+l$, overlaps the planning horizon of the precursory planning time $k$, i.e.~$l \in \left[ 1,L\right)$.
\end{assumption}

\begin{assumption}\label{assumption:actions}
	Action sets $\mathcal{U}_{k+l}$ and $\mathcal{U}_{k}$ overlap in the sense that actions in $\mathcal{U}_{k}$ that overlap in the executed portion of the optimal action also partially reside in $\mathcal{U}_{k+l}$. In other words,
	$\forall u \in \mathcal{U}_{k}$ with $u \doteq \{u_{k:k+l-1|k}, u_{k+l:k+L-1|k}\}$ and 
	$u_{k:k+l|k}  \equiv  u_{k:k+l-1|k}^{\star}$, 
	$\exists u' \in  \mathcal{U}_{k+l}$ such that 
	$u' \doteq  \{u'_{k+l:k+L-1}, u'_{k+L:k+l+L-1} \}$
	and $u'_{k+l:k+L-1} \cap u_{k+l:k+L-1|k} \notin \emptyset $. 
\end{assumption}

Unlike assumption~\ref{assumption:savedData}, which is an integral part of \ibsp, assumptions~\ref{assumption:overlapHorizon}-\ref{assumption:actions} exist only as a mean to create a smaller group of candidate beliefs for re-use. By limiting ourselves to beliefs with a shared history, mostly the same action sequence and of the same future time, we obtain a relatively small set that is likely to produce a viable candidate for re-use. One can relax these assumptions after addressing the problem of efficiently searching a set of candidate beliefs.


\vspace{-5pt}
\section{Approach}
\label{sec:approach}
Based on our key observation, we now present our incremental BSP (\ibsp) approach, which enables to incrementally calculate the objective function by re-using calculations from previous planning sessions, thus saving valuable computation time while at the same time preserving the benefits of the expectation solution provided by \fullbsp.
As explained in Section~\ref{subsec:bsp} the immediate rewards, required for calculating the objective value, are in the general case a function of candidate actions and future posterior beliefs calculated over sampled measurements. 
The way \ibsp re-uses previous calculations is by enforcing specific measurements as opposed to sampling them from the appropriate measurement likelihood distribution. 
The measurements being enforced, were considered and sampled in some precursory planning session(s), in which each of the measurements had corresponding posterior belief and immediate reward. By enforcing some previously considered measurement, we can make use of the previously calculated posterior beliefs, instead of performing inference from scratch. In order to make use of the data acquired since these re-used beliefs have been calculated, when needed, we can incrementally update them to match the information up to current time. 

In the following we first analyze the similarities between two successive planning sessions (Section~\ref{subsec:compPlanning}), and use those insights as foundation to develop the paradigm for \ibsp. In Section~\ref{subsec:overview} we provide an overview of the entire \ibsp paradigm, and continue with covering each of the building blocks of \ibsp: Selecting beliefs for re-use and deciding whether there is sufficient data for calculations re-use (Section~\ref{subsec:bestBranch}), validating samples for re-use, incorporating forced samples and belief update (Section~\ref{subsec:beliefUpdate}), calculating expectation incrementally with forced samples (Section~\ref{subsec:incExp}).
In Section~\ref{subsec:wf} we introduce a non-integral approximation of \ibsp, denoted as \wf. Under the use of \wf, one can sacrifice estimation accuracy for computational performance, by setting a threshold for re-using beliefs without any update. 
Further, to demonstrate how \ibsp can be utilized to improve existing approximations of the original \fullbsp problem, we consider the particular case of \imlbsp (Section~\ref{subsec:iML}), which denotes \ibsp under the ML assumption.

It is worth mentioning that the following sections are accompanied by high-level algorithms, describing key aspects of \ibsp. In an effort to simplify these algorithms for the readers' behalf, some of them are written in a sub-optimal manner (complexity-wise). When coming to implement \ibsp, we trust the readers to adhere to the governing principles of \ibsp while writing the source code in a complexity efficient manner.   

\subsection{Comparing Planning Sessions}\label{subsec:compPlanning}
This section analyzes the similarities between two planing sessions that comply with Assumptions~\ref{assumption:savedData}-~\ref{assumption:actions}. 
In order to do so, let us consider two planning sessions, both with horizon of $L$ steps ahead, the first occurred at time $k$ and the second at time $k+l$. Under Assumption~\ref{assumption:overlapHorizon} both planning horizons overlap, i.e. $l < L$, and under Assumption~\ref{assumption:actions} both planning sessions share some actions. For this comparison let us consider the action chosen at planning time $k$ which also partially resides in a candidate action from planning time $k+l$, and denote both as $u^{\star}_{k:k+L} = \{ u^{\star}_{k} , \hdots , u^{\star}_{k+L}\}$.
\begin{figure}
	\centering
		\includegraphics[trim={0 0 0 0},clip, width=0.55\columnwidth]{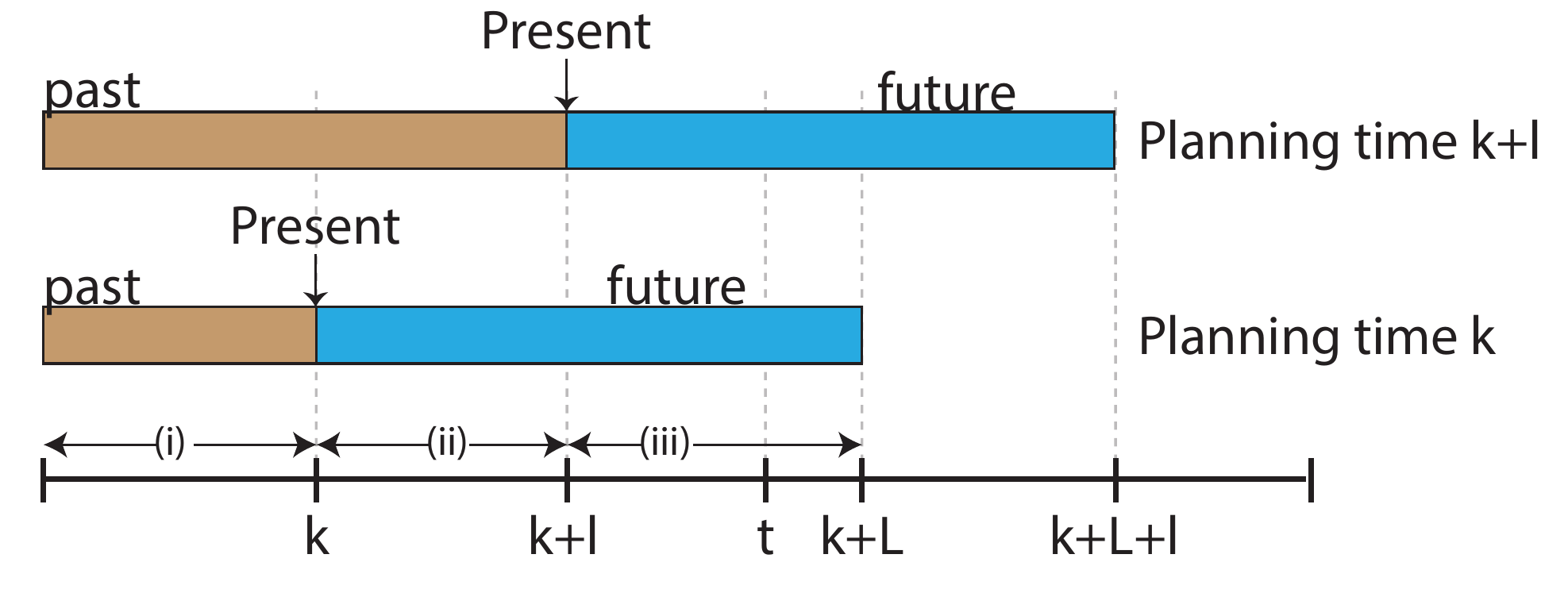}
        \caption{Horizon overlap between planning time k and planning time k+l, both with L steps horizon and same candidate actions: (i) The shared history of both planning sessions (ii) The possibly outdated information of planning time k, since in planning time k+l this time span is considered as known history (iii) Although in both it represents future prediction, it is conditioned over different history hence possibly different.}
        \label{fig:horizonOverlap}
\end{figure}
Figure~\ref{fig:horizonOverlap}, illustrates the aforementioned horizon overlap between two  beliefs at look ahead step at time $t$, given planning time $k$, i.e.~$b[X_{t|k}]$, and given planning time $k+l$, i.e.~$b[X_{t|k+l}]$, while the interesting shared sections, separated by time instances, are denoted as (i) (ii) and (iii). 

At future time $t \in \left[ k+l+1,k+L\right]$, the belief created by the action sequence $\overset{\star}{u}_{k:t-1}$ is given by
%
\begin{multline}\label{eq:belief_t_k}
	b[X_{t|k}] \propto \underbrace{b[X_{k|k}]}_{(a)} \cdot \\
	\underbrace{\prod_{s=k+1}^{k+l} \! \left[ \!\prob{x_{s}|x_{s-1},u^{\star}_{s-1}} \! \prod_{g \in \mathcal{M}_{s|k}} \! \prob{\!z_{s,g|k}|x_{s},l_{g}} \right]}_{(b)} \cdot \! \underbrace{\prod_{i=k+l+1}^{t} \!\! \left[ \!\prob{x_{i}|x_{i-1},u^{\star}_{i-1}} \! \prod_{j \in \mathcal{M}_{i|k}} \!\! \prob{z_{i,j|k}|x_{i},l_{j}} \right]}_{(c)},
\end{multline}
where (\ref{eq:belief_t_k})$_{(a)}$ is the inference posterior at time $k$ corresponding to the lower-bar area $(i)$ in Figure~\ref{fig:horizonOverlap}, (\ref{eq:belief_t_k})$_{(b)}$ are the motion and observation factors of future times $k+1:k+l$ corresponding to the lower-bar area $(ii)$ and (\ref{eq:belief_t_k})$_{(c)}$ are the motion and observation factors of future times $k+l+1:t$ corresponding to the lower-bar area $(iii)$. For the same future time $t$ and the same candidate action, the belief for planning time $k+l$ is given by, 
\begin{multline}\label{eq:belief_t_k+l}
	b[X_{t|k+l}] \propto \underbrace{b[X_{k|k+l}]}_{(a)} \cdot \\
	\underbrace{\prod_{s=k+1}^{k+l} \!\!\! \left[ \!\prob{x_{s}|x_{s-1},u^{\star}_{s-1}} \!\!\!\! \prod_{g \in \mathcal{M}_{s|k+l}} \!\!\! \prob{\!z_{s,g|k+l}|x_{s},l_{g}\!} \!\right]}_{(b)} \cdot \!\!\!\! \underbrace{\prod_{i=k+l+1}^{t} \!\! \left[ \!\prob{x_{i}|x_{i-1},u^{\star}_{i-1}} \!\! \prod_{j \in \mathcal{M}_{i|k+l}} \!\! \prob{z_{i,j|k+l}|x_{i},l_{j}} \right]}_{(c)},
\end{multline}
where (\ref{eq:belief_t_k+l})$_{(a)}$ is the inference posterior at time $k$ corresponding to the upper-bar area $(i)$ in Figure~\ref{fig:horizonOverlap}, (\ref{eq:belief_t_k+l})$_{(b)}$ are the motion and observation factors of past times $k+1:k+l$ corresponding to the upper-bar area $(ii)$  and (\ref{eq:belief_t_k+l})$_{(c)}$ are the motion and observation factors of future times $k+l+1:t$ corresponding to the upper-bar area $(iii)$. 

Although seemingly conditioned on a different history ($k$ vs $k+l$), (\ref{eq:belief_t_k})$_{(a)}$ and (\ref{eq:belief_t_k+l})$_{(a)}$ are identical and denote the same posterior obtained at time $k$ (see Figure~\ref{fig:horizonOverlap} area $(i)$), leaving the difference between (\ref{eq:belief_t_k}) and (\ref{eq:belief_t_k+l}) restricted to (.)$_{(b)}$ and (.)$_{(c)}$. While (\ref{eq:belief_t_k})$_{(b)}$ represents future actions and future measurements predicted at time $k$, (\ref{eq:belief_t_k+l})$_{(b)}$ represents executed actions and previously acquired measurements, this can can be seen more clearly using area (ii) in Figure~\ref{fig:horizonOverlap}. At planning time $k$ (i.e. lower bar), area (ii) denotes future prediction for the time interval $k:k+l$, while at planning time $k+l$ (i.e. upper bar), the same time interval denotes past measurements, and so (\ref{eq:belief_t_k})$_{(b)}$ and (\ref{eq:belief_t_k+l})$_{(b)}$ are potentially different, depending on how accurate was the prediction at planning time $k$. As thoroughly explained in our previous work on \texttt{RUB~inference} \citep{Farhi19arxiv}, the difference between the predictions made during planning and the actual measurements obtained in present time is twofold, the difference in measurement values \citep[Section~3.4]{Farhi19arxiv} and the difference in data association \citep[Section~3.5]{Farhi19arxiv}.   

Even-though both (\ref{eq:belief_t_k})$_{(c)}$ and (\ref{eq:belief_t_k+l})$_{(c)}$ refer to future actions and measurements (see area (iii) in Figure~\ref{fig:horizonOverlap}) they do so with possibly different values and data association since they were sampled from possibly different probability densities. Solving the objective (\ref{eq:objIntegral}), requires sampling from (\ref{eq:likelihoodInt}) (e.g. using Alg.~\ref{alg:sampling_z}), the samples from planning time $k$ were sampled from $\prob{z_{t|k}|H_{t|k}^{-}}$, while the samples from planning time $k+l$ were sampled from $\prob{z_{t|k+l}|H_{t|k+l}^{-}}$. These probabilities would be identical only if conditioned on the same history, i.e. only if the predictions made for time interval $k:k+l$ at planning time $k$ were accurate both in data association and measurement values. 

As such, in order to mind the gap between (\ref{eq:belief_t_k}) and (\ref{eq:belief_t_k+l}), and obtain identical expressions one must update (\ref{eq:belief_t_k})$_{(b)}$ to match (\ref{eq:belief_t_k+l})$_{(b)}$, and second, to adjust the samples from (\ref{eq:belief_t_k})$_{(c)}$ to properly represent the updated measurement probability density.


\subsection{Approach Overview}\label{subsec:overview}
This section presents an overview of \ibsp at planning time $k+l$, whilst the relevant precursory planning session occurred at planning time $k$, as summarized in Alg.~\ref{alg:ibsp}. 
For the reader's convenience all the notations of this section are summarized in Table~\ref{table:overview}.
After executing $l$ steps out of the (sub)optimal action sequence suggested by planning at time $k$, and performing inference over the newly received measurements, we obtain $b[X_{k+l|k+l}]$. Performing planning at time $k+l$ under \ibsp, requires first 
deciding on the planning sub-tree from the precursory planning session to be considered for re-use (Alg.~\ref{alg:ibsp} line~\ref{alg:ibsp:selectBranch}). 
Considering belief roots of candidate planning sub-trees, the selected sub-tree is the one with the "closest" belief root to $b[X_{k+l|k+l}]$, i.e. the one with the minimal distance to it while considering some appropriate probability density function distance. We denote the closest belief root and the appropriate planning sub-tree as $\overset{\sim}{b}[X_{k+l|k}]$ and $\mathcal{B}_{k+l|k}$ respectively. 
%
In case the distance of the closest belief to $b[X_{k+l|k+l}]$ (denoted by \texttt{Dist}) is larger than some critical value $\epsilon_c$, i.e. the closest prediction from the precursory planning session is too far off, \ibsp would presumably have no advantage over the standard \fullbsp so the latter is executed (Alg.~\ref{alg:ibsp} line~\ref{alg:ibsp:fullbsp}). 
On the other hand, if \texttt{Dist} is smaller than the critical value $\epsilon_{wf}$, we consider the difference between the beliefs as insignificant and continue with re-using the precursory planning session without any additional update (Alg.~\ref{alg:ibsp} line~\ref{alg:ibsp:wf}). We denote the aforementioned as \wf.
When the precursory planning is close-enough (see Figure~\ref{fig:distZones}), we can appropriately re-use it to save valuable computation time. 
While we go further and elaborate on specific methods we used in this work, e.g. determining belief distance (see Section~\ref{subsec:bestBranch}) or representative sample (see Section~\ref{ssubsec:repSample}), \ibsp is indifferent to any specific method, as long as it serves its intended purpose. 

\begin{algorithm}
	\caption{iX-BSP: Planning time k+l}\label{alg:ibsp}
	\begin{flushleft}
		\textbf{Input:} \\
		\hspace{30pt} \texttt{data} \Comment{Calculations used for the precursory planning session }\\
		\hspace{30pt} $b[X_{k+l|k+l}]$ \Comment{The up-to-date inference posterior for time $k+l$}\\	
		\hspace{30pt} \texttt{useWF}, $\epsilon_{wf}$, $\epsilon_{c}$ \Comment{User defined flags \& thresholds} \hspace{30pt} 
	\end{flushleft}
	\begin{algorithmic}[1]
		\State Dist , $\mathcal{B}_{k+l|k}$ $\gets$ \Call{SelectClosestBranch}{$b[X_{k+l|k+l}]$, \texttt{data}} \label{alg:ibsp:selectBranch} \Comment{see Section~\ref{ssubsec:candidateSet}}
		\If{Dist $\leq$ $\epsilon_{c}$ } \label{alg:ibsp:distCond} \Comment{belief distance threshold $\epsilon_{c}$}
			\If{\texttt{useWF} $\cap$ (Dist $\leq$ $\epsilon_{wf}$) } \label{alg:ibsp:wfCond} \Comment{\wf threshold $\epsilon_{wf}$}
				\State \texttt{data} $\gets$ $\mathcal{B}_{k+l|k}$  \label{alg:ibsp:wf} \Comment{Reusing the entire selected branch without any update, see Section~\ref{subsec:wf}}
			\Else
				\State \texttt{data} $\gets$ \Call{IncUpdateBeliefTree}{$\mathcal{B}_{k+l|k}$} \label{alg:ibsp:updatePlanning} \Comment{see Section~\ref{subsec:beliefUpdate}}
			\EndIf
			\State \texttt{data} $\gets$ perform \fullbsp over horizon steps $k+L+1:k+L+l$ \label{alg:ibsp:lastSteps} 
			\State Solve Eq.~\ref{eq:objIntegral}, for each candidate action \Comment{see Section~\ref{subsec:incExp}}
			\State $u^{\star}_{k+i:k+L|k+i}$ $\gets$ find best action \label{alg:ibsp:findAction}
		\Else
			\State $u^{\star}_{k+i:k+L|k+i}$ $\gets$ perform \Call{\fullbsp}{$b[X_{k+l|k+l}]$} \label{alg:ibsp:fullbsp}
		\EndIf
		\State \Return{$u^{\star}_{k+i:k+L|k+i}$, \texttt{data}} \label{alg:ibsp:data}
	\end{algorithmic}
\end{algorithm}
The planning sub-tree $\mathcal{B}_{k+l|k}$ is comprised of all future beliefs, i.e. $k+l+1:k+L$, calculated as part of the planning session from time $k$, which originate in $\overset{\sim}{b}[X_{k+l|k}]$. 

We update these beliefs with the information received in inference between time instances $k+1$ and $k+l$, and selectively re-sample predicted measurements (line~\ref{alg:ibsp:updatePlanning}) in an effort to maintain a representative set of samples for the nominal distribution. In case one of the aforementioned beliefs also meets the \wf condition (see Section \ref{subsec:wf}) we consider it, and all of its descendants as already updated. Once the update is complete, we have a planning horizon of just $L-l$ steps, i.e. to the extent of the horizon overlap, hence we need to calculate the rest from scratch, i.e. perform \fullbsp for the final $l$ steps (line~\ref{alg:ibsp:lastSteps}). We are now in position to update the immediate reward function values and calculate their expected value in search of the (sub)optimal action sequence (line~\ref{alg:ibsp:findAction}), thus completing the planning for time $k+l$.

Since we are re-using samples from different planning sessions at planning time $k+l$ we are required to compensate for the different measurement likelihood, through proper formulation. In the sequel we show that our problem falls within the Multiple Importance Sampling problem (see~Appendix~\ref{app:IS}), so we estimate the expected reward values using importance sampling based estimator, thus completing the planning for time $k+l$. 

Differently from \fullbsp which returns only the selected action sequence, \ibsp is also required to return more data from the planning process in order to facilitate re-use (line~\ref{alg:ibsp:data}). 

\begin{table}
	\caption{Notations for Sections~\ref{subsec:overview}-\ref{subsec:bestBranch}} 
	\centering 
	\begin{tabular}{c c} 
		\hline\hline 
		\textbf{Variable} & \textbf{Description}  \\ [0.5ex] 
		\hline 
		$ \Box_{t|k}$ & Of time $t$ while current time is $k$ \\[1ex]
		$\mathcal{M}_{t|k}$ & Data Association at time $t$ while current time is $k$  \\[1ex]
		$b[X_{t|k}]$ & belief at time $t$ while current time is $k$ \\[1ex]
		$b^-[X_{t|k}]$ & belief at time $t-1$ propagated only with action $u_{t-1|k}$ \\[1ex]
		$\mathcal{B}_{k|k}$ & The entire belief tree from planning at time $k$\\[1ex]
		$\overset{\sim}{b}[X_{t|k}]$ & The root of the selected branch for re-use in planning at time $t$ \\[1ex]
		$\mathcal{B}_{t|k}$ & The set of all beliefs from planning time $k$ rooted in $\overset{\sim}{b}[X_{t|k}]$\\[1ex]
		\texttt{Dist} & The distance between $\overset{\sim}{b}[X_{t|k}]$ and the corresponding posterior $b[X_{t|t}]$ \\[1ex]
		\texttt{data} & All available calculations from current and precursory planning session \\[1ex]
		$u^{\star}_{k:k+L|k}$ & The (sub)optimal action sequence of length $L$ chosen in planning at time $k$ \\[1ex]
		$\epsilon_c$ & belief distance critical threshold, above it re-use has no computational advantage \\[1ex]
		$\epsilon_{wf}$ & \wf threshold, bellow it distance is considered close-enough for re-use without any update \\[1ex]
		\texttt{useWF} & a binary flag determining whether or not the \wf condition is considered \\ [1ex]
		$\mathbb{D}(.)$ & belief divergence / metric \\[1ex]
		\hline\hline
	\end{tabular}
	\label{table:overview} 
\end{table}
\subsection{Selecting Beliefs for Re-use}\label{subsec:bestBranch}

\begin{figure}[]
	\centering
	\includegraphics[trim={0 0 0 0},clip, width=0.60\columnwidth]{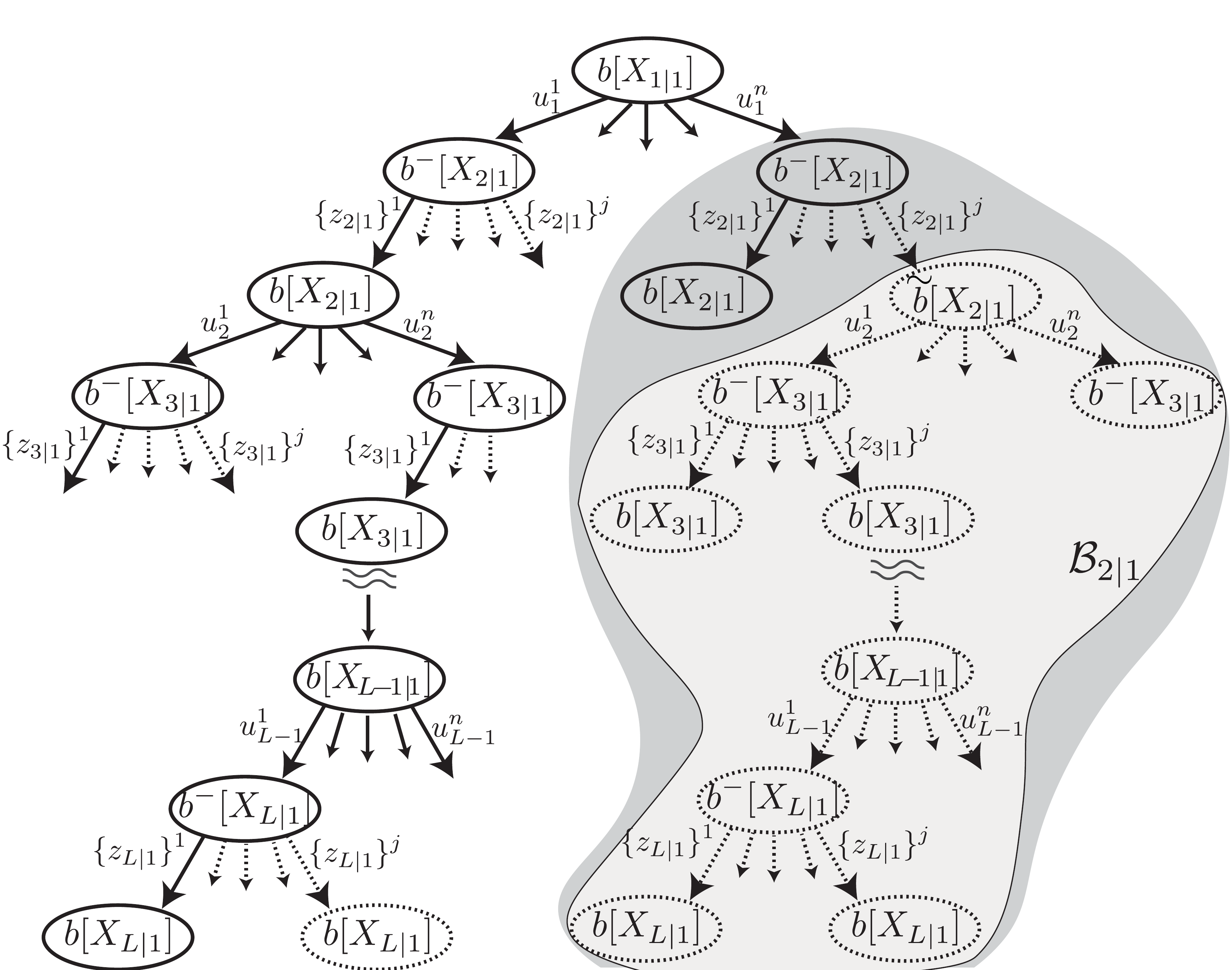}
	\caption{ \fullbsp performs lookahead search on a tree with depth $L$. Each belief tree node represents a belief. For each node, the tree branches either for a candidate action or a sampled measurement. 
		The corresponding belief tree for \mlbsp is marked with solid lines, while the dashed lines represent the parts of \fullbsp that relate to sampled measurements. Under \ibsp, the gray-marked parts of the tree are being re-used for the succeeding planning session.\vspace{-20pt}} 
	\label{fig:selectBranch}
\end{figure}
This section covers an integral part of \ibsp, dealing with how to select which beliefs to re-use, and from where.   
At each step along the planning horizon, \ibsp is required to choose beliefs for re-use. Our goal is to minimize any required updates, i.e. the beliefs we would like to re-use should be as "close" as possible to the beliefs we would have obtained through standard \fullbsp.
In \ibsp, as well as in \fullbsp, the number of beliefs per future time step is derived from the number of samples per action per time step; in order to re-use previous calculations while avoiding a computational load, we need to choose which beliefs to consider for re-use.
The need to obtain the closest belief for re-use entails three fairly complicated problems: Where to find it (Section~\ref{ssubsec:candidateSet}), how to find it (Section~\ref{ssubsec:beliefDist}), and how to determine what is considered "close" in belief space (Section~\ref{ssubsec:closeEnough}). 
For the reader's convenience all the notations of this section are summarized in Table~\ref{table:overview}.

\subsubsection{Selecting the Candidate Set for Re-Use}\label{ssubsec:candidateSet}
While every set of previously calculated beliefs can serve as potential candidates for re-use, in this work we consider previous planning sessions as they are readily available. It is worth mentioning that the problem of searching a set of candidate beliefs can be computationally expensive, thus potentially sabotaging the efforts of \ibsp to relieve the computational load of BSP.
 In order to avert from directly dealing with the aforementioned "search problem" within belief space and maximize the chances of finding an adequate candidate for re-use, we introduce Assumptions~\ref{assumption:overlapHorizon}-\ref{assumption:actions} that are not an integral part of \ibsp. Following Assumption~\ref{assumption:overlapHorizon} we assure that the previous planning session has some overlapping horizon with the current planning session, hence increasing the chances of locating a "close enough" belief for re-use.

Using Assumption~\ref{assumption:actions} we can prune the full belief-tree from previous planning time $k$, denoted as $\mathcal{B}_{k|k}$, and consider only a subset of it while assuring overlapping of some candidate actions. 
More specifically, we prune $\mathcal{B}_{k|k}$ to consider the sub-tree $\mathcal{B}_{k+l|k} \subset \mathcal{B}_{k|k}$, which is rooted in $\overset{\sim}{b}[X_{k+l|k}]$ such that 
\begin{equation}
	\overset{\sim}{b}[X_{k+l|k}] = \argmin_{b[X_{k+l|k}] \in \mathcal{B}_{k|k}} \mathbb{D}(b[X_{k+l|k}] , b[X_{k+l|k+l}]),
\end{equation}
where $\mathbb{D}(.)$ is a metric ( or divergence) quantifying the difference between two beliefs (see Section~\ref{ssubsec:closeEnough}), $b[X_{k+l|k+l}]$ is the posterior from inference at time $k+l$, and $\overset{\sim}{b}[X_{k+l|k}]$ is one of the beliefs for lookahead step $l$ of planning at time $k$. This search for the closest belief is performed by Alg.~\ref{alg:BeliefDist} and discussed in Section~\ref{ssubsec:beliefDist}. 
By considering the sub-belief tree of the closest prediction from planning at time $k$ to the current posterior, we ensure minimal required update along the lookahead steps, i.e. minimizing the difference between the prediction (\ref{eq:belief_t_k})$_{(b)}$ and what eventually happened (\ref{eq:belief_t_k+l})$_{(b)}$. 

Without loss of generality, we now make use of Figure~\ref{fig:selectBranch} to illustrate the branch selection process, i.e. how we choose a candidate set of beliefs given an entire previous planning tree. Figure~\ref{fig:selectBranch} illustrates a belief tree of \fullbsp at planning time $t=1$ for a horizon of $L$ steps, with $n$ candidate actions and $j$ sampled measurements per step, resulting with $\left(n \cdot j \right)^1$ different beliefs for future time $t=2$ and $\left(n \cdot j \right)^{L-1}$ for future time $L$. Let us assume action $u_1^n$ has been determined as optimal at planning time $t=1$ and has been executed. After attaining new measurements for current time $t=2$ and calculating the posterior belief $b[X_{2|2}]$, we perform planning once more. Now, under \ibsp, instead of calculating everything from scratch we would like to re-use previous calculations; specifically, under Assumption~\ref{assumption:overlapHorizon} we consider the beliefs calculated at planning time $t=1$. Instead of considering the entire tree $\mathcal{B}_{1|1}$ for re-use, we look for some sub-tree $\mathcal{B}_{2|1} \subset \mathcal{B}_{1|1}$ rooted in $\overset{\sim}{b}[X_{2|1}]$ such that
\begin{equation}
	\overset{\sim}{b}[X_{2|1}] = \argmin_{b[X_{2|1}] \in \mathcal{B}_{1|1}} \mathbb{D}(b[X_{2|1}] , b[X_{2|2}]).
\end{equation}
We start by considering all beliefs $\{b^i[X_{2|1}]\}_{i=1}^{n\cdot j}$ meeting Assumption~\ref{assumption:actions}, i.e. all beliefs marked by the dark gray area in Figure~\ref{fig:selectBranch} which considered the same action sequence.
Now, from the remaining $j$ beliefs, using Alg.~\ref{alg:BeliefDist} (see Section~\ref{ssubsec:beliefDist}) we denote the closest belief to the posterior $b[X_{2|2}]$ as $\overset{\sim}{b}[X_{2|1}]$. 
Once we determined $\overset{\sim}{b}[X_{2|1}]$, we define the closest branch as consisting of all the beliefs rooted in $\overset{\sim}{b}[X_{2|1}]$, and denote it as $\mathcal{B}_{2|1}$, marked in Figure~\ref{fig:selectBranch} by the light gray area. In case there are no beliefs in the set  $\{b^i[X_{2|1}]\}_{i=1}^{n\cdot j}$ meeting Assumption~\ref{assumption:actions}, we will need to search through the entire set for the closest belief $\overset{\sim}{b}[X_{2|1}]$. In the following section we describe how the closest belief is located.

\subsubsection{Finding the Closest Belief}\label{ssubsec:beliefDist}


This section covers the problem of how to locate the closest belief given a set of candidate beliefs, as required when selecting the closest branch for re-use (Section~\ref{ssubsec:candidateSet}, Alg.~\ref{alg:ibsp} line~\ref{alg:ibsp:selectBranch}) or when incrementally updating the belief tree under \ibsp (Alg.~\ref{alg:updatePrevPlanning} line~\ref{alg:updatePrevPlanning:beliefDist}).
\begin{algorithm}
	\caption{ClosestBelief}\label{alg:BeliefDist}
	\begin{flushleft}
		\textbf{Input:}  \\
		\hspace{30pt} $\mathcal{B}_{k+l|k}$ \Comment{set of candidate beliefs for re-use from planning at time $k$, see Section~\ref{subsec:bestBranch} }\\
		\hspace{30pt} $b[X_{i+1|k+l}]$ \Comment{The belief to check distance to, from planning at time $k+l$}\\
	\end{flushleft}
	\begin{algorithmic}[1]
		\State $\delta_{min} = 0$
		\For{$b[X_{i+1|k}] \in \mathcal{B}_{k+l|k}$}
			\State $\delta$ $\gets$ \Call{$\mathbb{D}$}{$b[X_{i+1|k}]$, $b[X_{i+1|k+l}]$} \Comment{probability metric/ divergence to determine belief distance}\label{alg:BeliefDist:metric}
			\If{$\delta \leq \delta_{min}$} \Comment{keeping track over the shortest distance}
				\State $\delta_{min}$ $\gets$ $\delta$
				\State $b^{'}[X_{i+1|k}]$ $\gets$ $b[X_{i+1|k}]$
			\EndIf
		\EndFor  
		\State \Return{$\delta_{min}$, $b^{'}[X_{i+1|k}]$}
	\end{algorithmic}
\end{algorithm} 
As part of our problem, we have a set of candidate beliefs for re-use, denoted as $\mathcal{B}_{k+l|k}$, and some posterior $b[X_{i|k+l}]$ we wish to be close to. Our goal is to find within $\mathcal{B}_{k+l|k}$ the closest belief to $b[X_{i|k+l}]$, where $i>k+l$ denote some lookahead step.  
Locating the closest belief requires quantifying the differences between two beliefs into a scalar distance. 
We denote the distance function, whether a metric or a divergence, by $\mathbb{D}(.)$. 
Let us consider some candidate belief $b[X_{i|k}] \in \mathcal{B}_{k+l|k}$, although referring to the same future time $i$ as $b[X_{i|k+l}]$, it is conditioned on different history, and therefore is potentially different. While Section~\ref{subsec:compPlanning} discussed the reasons for such difference between $b[X_{i|k}]$ and $b[X_{i|k+l}]$, here we quantify this difference using a belief distance. 
Projecting $b[X_{i|k}]$ into our belief distance space yields a point that suggests how different is $b[X_{i|k}]$ from $b[X_{i|k+l}]$.   
After projecting all candidate beliefs from $\mathcal{B}_{k+l|k}$ into the belief distance space in reference to $b[X_{i|k+l}]$, the problem of locating the closest belief to $b[X_{i|k+l}]$ is reduced to a problem of locating the nearest neighbor. 
 
While there are more efficient ways to determine the closest belief, we made use of a simple realization of \texttt{BeliefDist}(.) in  Alg.~\ref{alg:BeliefDist}.
Any other, more efficient realization, that provides with the same end-result is acceptable, and would benefit the computational load reduction of \ibsp.
Alg.~\ref{alg:BeliefDist} determines the closest belief given a set of beliefs $\mathcal{B}_{k+l|k}$ and a target belief $b[X_{i|k+l}]$, by simply calculating the distance between the target belief to each belief in the set $\mathcal{B}_{k+l|k}$ using $\mathbb{D}(.)$, and picking the closest one. 
 The minimal distance associated with the closest belief is thus given by
 \begin{equation}
 	\delta_{min} = \min_{b[X_{i|k}] \in \mathcal{B}_{k+l|k}} \mathbb{D}(b[X_{i|k}],b[X_{i|k+l}]),
 \end{equation}
 %
while the distinction whether $\delta_{min}$ is acceptable or not, happens outside of Alg.~\ref{alg:BeliefDist} (see Alg.~\ref{alg:ibsp} lines~\ref{alg:ibsp:distCond}-\ref{alg:ibsp:wfCond} and Alg.~\ref{alg:updatePrevPlanning} lines~\ref{alg:updatePrevPlanning:checkEps_c}-\ref{alg:updatePrevPlanning:checkEps_wf}) as discussed next.

\subsubsection{What is Close Enough}\label{ssubsec:closeEnough}
\begin{figure}[]
	\centering
	\includegraphics[trim={0 0 0 0},clip, width=0.30\columnwidth]{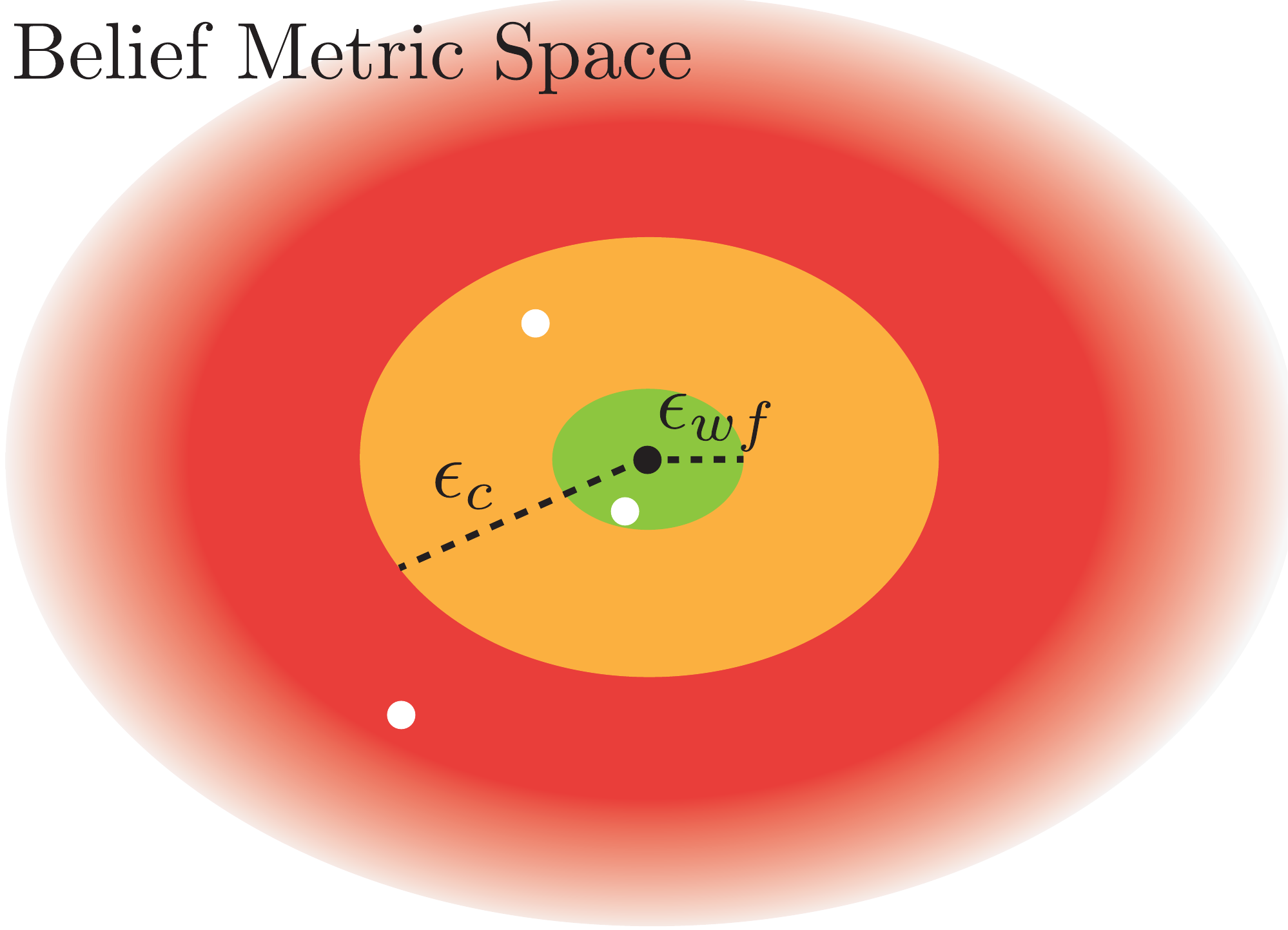}
	\caption{ Illustration of the relative belief distance space. Each point in this space represents some belief $b[X_{t|k}]$, where the black dot denotes $b[X_{t|k+l}]$ as the origin. All beliefs $b[X_{t|k}]$ close to the origin up to $\epsilon_{wf}$, i.e. in the green zone, are being re-used without any update. All beliefs $b[X_{t|k}]$ close to the origin up to $\epsilon_{c}$ but farther than $\epsilon_{wf}$, i.e. in the orange zone, are being re-used with some update.  All beliefs $b[X_{t|k}]$ that are more than $\epsilon_{c}$ away from the origin, i.e. in the red zone, are considered as not close enough to make a re-use worth while.} 
	\label{fig:distZones}
\end{figure}
Choosing a candidate belief with minimal distance is not enough, as this candidate belief could still be very different, and thus require a substantial computational effort in order to update. In order to deal with this issue, we need to set some criteria over the belief distance. 
Let us consider a belief metric space, in-which each point is a unique projection of a candidate belief $b[X_{i|k+l}]$ denoting the distance between the aforementioned candidate belief and $b[X_{i|k+l}]$.
Figure~\ref{fig:distZones} illustrates such space, where the black dot represents the homogeneous projection $\mathbb{D}(b[X_{i|k+l}],b[X_{i|k+l}])$, and the rest of the points denote $\mathbb{D}(b[X_{i|k}],b[X_{i|k+l}])$ e.g. the three white dots in Figure~\ref{fig:distZones}. 
  We divide distances around the homogeneous projection into three areas, 
  \begin{equation}\label{eq:distanceCases}
  	\begin{cases}
  		\mathbb{D}(b[X_{i|k}],b[X_{i|k+l}]) \ \leq \ \epsilon_{wf} & \text{close enough for re-use "as is", see Section~\ref{subsec:wf}}\\
  		\epsilon_{wf} \ < \ \mathbb{D}(b[X_{i|k}],b[X_{i|k+l}]) \ \leq \ \epsilon_{c} & \text{close enough for re-use} \\
  		\epsilon_{c} \ < \ \mathbb{D}(b[X_{i|k}],b[X_{i|k+l}]) & \text{too far off for re-use}\\
  	\end{cases}
  \end{equation}
  denoted respectively in green orange and red.
When \wf is not enabled, i.e. $\texttt{useWF} = \texttt{false}$, the belief metric space is divided into two areas (orange and red), separated by a single parameter $\epsilon_{c}$. The case where \wf is enabled, i.e. $\texttt{useWF} = \texttt{true}$, is covered in Section~\ref{subsec:wf}.
At this point we leave the procedure of choosing $\epsilon_c$ and $\epsilon_{wf}$ for future work, and consider it as a heuristic. In the following we do however show analytically (see Section~\ref{ssubsec:wfBounds}) and empirically (see Section~\ref{ssubsec:wf:sensitivity}) the connection between $\epsilon_{wf}$ and the objective value.

Each belief metric (divergence) would result with a  possibly different projection onto metric space, hence with probably different values for $\epsilon_c$ and $\epsilon_{wf}$. 
As part of our work we considered several alternatives for belief distance, a DA based divergence and another based on Jeffreys divergence. 

\subsubsection*{The \DA distance}
Under $\mathbb{D}_{DA}$ we start by sorting all candidate beliefs according to data association (DA) differences, looking for the smallest available difference. 
For example, the DA differences between Eq.~(\ref{eq:belief_t_k}) and Eq.~(\ref{eq:belief_t_k+l}) are given by matching their DA data denoted by $\mathcal{M}$. Such matching would yield three possible differences: the DA that has been correctly predicted and need not be changed 
\begin{equation}\label{eq:DA:good}
	\mathcal{M}_{k+1:k+l|k} \cap \mathcal{M}_{k+1:k+l|k+l},
\end{equation}
the DA that has been wrongfully predicted and need to be removed
\begin{equation}\label{eq:DA:bad}
	\mathcal{M}_{k+1:k+l|k} \backslash \mathcal{M}_{k+1:k+l|k+l},
\end{equation}
and the DA that has not been predicted and need to be added 
\begin{equation}\label{eq:DA:add}
	\mathcal{M}_{k+1:k+l|k+l} \backslash \mathcal{M}_{k+1:k+l|k}.
\end{equation}
In case there is more than a single belief with minimal DA difference, we continue to sort the remaining beliefs according to the difference between values of corresponding predicted measurements and similarly look for the minimal difference. In case there is more than a single belief with minimal measurement value difference, we select arbitrarily out of the remaining beliefs, and consider the chosen belief as the closest one. A detailed explanation of the DA matching process can be found in \cite[Section 3.5]{Farhi19arxiv}. It is worth stressing that $\mathbb{D}_{DA}$ is just a divergence and not a metric, as it does not meet the symmetry and sub-additivity requirements.

\subsubsection*{The \JD distance}
The \JD distance is a variant of the Jeffreys divergence presented by \cite{jeffreys46rs} (see Appendix~\ref{app:Distance}). This is a symmetric divergence for general probabilities that also has a special form in the case of Gaussian beliefs (for full derivation see Appendix~\ref{app:Distance}).
For two Gaussian beliefs $b[X_{t|k+l}] \sim \mathcal{N}(\mu_p, \Sigma_p)$ and $b[X_{t|k}] \sim \mathcal{N}(\mu_q, \Sigma_q)$, the \JD distance between them is given by,   
 \begin{equation}
	\JD(b[X_{t|k+l}],b[X_{t|k}])= \frac{1}{4}\sqrt{ (\mu_p - \mu_q)^T\left[\Sigma_q^{-1} + \Sigma_p^{-1} \right](\mu_p - \mu_q) +  tr\left( \Sigma_q^{-1}\Sigma_p\right) +  tr\left( \Sigma_p^{-1}\Sigma_q\right) - d_p-d_q },
\end{equation}
 where $d_p$ and $d_q$ are the joint state dimension of $b[X_{t|k+l}]$ and $b[X_{t|k}]$, respectively.

\subsection{Incremental Update of Belief-Tree}\label{subsec:beliefUpdate}

In Section~\ref{ssubsec:candidateSet} we determined the candidate set $\mathcal{B}_{k+l|k}$, and using Section~\ref{ssubsec:beliefDist} have the ability to locate for each posterior belief $b[X_{i|k+l}]$ the closest belief in $\mathcal{B}_{k+l|k}$. Now, we can focus on one of our main contributions, incrementally creating a belief tree, through the re-use of previously calculated beliefs, while accounting for all information differences.    
To this end we supply Alg.~\ref{alg:updatePrevPlanning}, tasked with creating the belief tree of planning time $k+l$ through selective re-use of beliefs from $\mathcal{B}_{k+l|k}$. The process starts with the posterior $b[X_{k+l|k+l}]$, and continues with every new belief $b^s[X_{i|k+l}]$ that is added to the new belief tree up to future time $k+L$, where $s$ accommodates all different sampled beliefs at future time $i$. For the reader's convenience all the notations of this section are summarized in Table~\ref{table:reUseBeliefs}.

\begin{algorithm}
	\caption{IncUpdateBeliefTree}\label{alg:updatePrevPlanning}
	\begin{flushleft}
		\textbf{Input:}   \\
		\hspace{30pt} $\mathcal{B}_{k+l|k}$ \Comment{The selected branch, see Section~\ref{subsec:bestBranch} }\\
		\hspace{30pt} $b[X_{k+l|k+l}]$ \Comment{The posterior from precursory inference}\\
	\end{flushleft}
	\begin{algorithmic}[1]
		\For{ each $i \in [k+l,k+L-1]$}\Comment{each overlapping horizon step}\label{alg:updatePrevPlanning:lookaheadFor}
			\For{ each  $s \in [1,n_u(n_x\!\cdot\!n_z)^{i-k-l}]$} \Comment{each belief in the $i_{th}$ horizon step}\label{alg:updatePrevPlanning:beliefFor}
				\If{\texttt{useWF} $\cap$ \Call{isWildFire}{$b^s[X_{i|k+l}]$}}\label{alg:updatePrevPlanning:isWF}
					\State $r_0$ $\gets$ $(s-1)\cdot n_u\cdot(n_x \cdot n_z)$
					\State $\{b^r[X_{i+1|k+l}]\}_{r = r_0+1}^{r_0 + n_u\cdot(n_x \cdot n_z)}$ $\gets$ all first order children of $b^{s'}[X_{i|k}]$\Comment{see Section~\ref{subsec:wf}}\label{alg:updatePrevPlanning:copyWF} 
					\State mark all $\{b^r[X_{i+1|k+l}]\}_{r = r_0+1}^{r_0 + n_u\cdot(n_x \cdot n_z)}$  as \wf \label{alg:updatePrevPlanning:markWF}
				\Else
					\For{each candidate action $\alpha \in [1,n_u]$ }\label{alg:updatePrevPlanning:actionFor}
						\State $b^{s-}_{\alpha}[X_{i+1|k+l}]$ $\gets$ propagate $b^s[X_{i|k+l}]$ with candidate action $\alpha$\label{alg:updatePrevPlanning:propagate}
						\State \texttt{dist} , $b^{s'-}_{\alpha}[X_{i+1|k}] \gets$ \Call{ClosestBelief}{$\mathcal{B}_{k+l|k}$, $b^{s-}_{\alpha}[X_{i+1|k+l}]$} \label{alg:updatePrevPlanning:beliefDist}\Comment{see Section~\ref{ssubsec:beliefDist}}
						\If{\texttt{dist} $\leq \epsilon_{c}$}\Comment{re-use condition}\label{alg:updatePrevPlanning:checkEps_c}
							\If{\texttt{useWF} $\cap$ (\texttt{dist} $\leq \epsilon_{wf}$)}\Comment{\wf condition}\label{alg:updatePrevPlanning:checkEps_wf}
								\State $\{b^r_{\alpha}[X_{i+1|k+l}]\}_{r = 1}^{n_x \cdot n_z}$ $\gets$ all first order children of $b^{s'-}_{\alpha}[X_{i+1|k}]$\Comment{see Section~\ref{subsec:wf}}\label{alg:updatePrevPlanning:useAsWF}
								\State mark $\{b^r_{\alpha}[X_{i+1|k+l}]\}_{r = 1}^{n_x \cdot n_z}$ as \wf
								\State Continue with next candidate action (i.e. jump to line~\ref{alg:updatePrevPlanning:actionFor})
							\Else
								\State \texttt{samples} $\gets$ all samples taken from $b^{s'-}_{\alpha}[X_{i+1|k}]$\label{alg:updatePrevPlanning:gatherSamples}
								\State $\{$\texttt{repSamples}$\}_{1}^{n_x n_z}$, \texttt{data} $\gets$ \Call{IsRepSample}{\texttt{samples}, $b^{s-}_{\alpha}[X_{i+1|k+l}]$} \Comment{see Section~\ref{ssubsec:repSample}}\label{alg:updatePrevPlanning:repSample}

							\EndIf
						\Else \Comment{not computationally effective to re-use, resample all}
							\State $\{$\texttt{repSamples}$\}_{1}^{n_x n_z}$, \texttt{data} $\gets$ $(n_x \cdot n_z)$ fresh samples based on $b^{s-}_{\alpha}[X_{i+1|k+l}]$\Comment{see Alg.~\ref{alg:sampling_z}}\label{alg:updatePrevPlanning:badBeliefDist}
						\EndIf
						\State \texttt{data} $\gets$ \Call{UpdateBelief}{\texttt{dist}, $\{$\texttt{repSamples}$\}_{1}^{n_xn_z}$, \texttt{data}} \Comment{see Section~\ref{ssubsec:beliefUpdate}}\label{alg:updatePrevPlanning:updateBelief}
						\State \texttt{data} $\gets$ update reward(cost) values for action $\alpha$ \Comment{see Section~\ref{ssubsec:rewardUpdate}}\label{alg:updatePrevPlanning:updateReward}
					\EndFor
				\EndIf
			\EndFor
		\EndFor
		\State \Return{\texttt{data}}
	\end{algorithmic}
\end{algorithm}

\begin{figure}
	\centering
       \subfloat[]{\includegraphics[trim={0 0 -15 5},clip, width=0.34\textwidth]{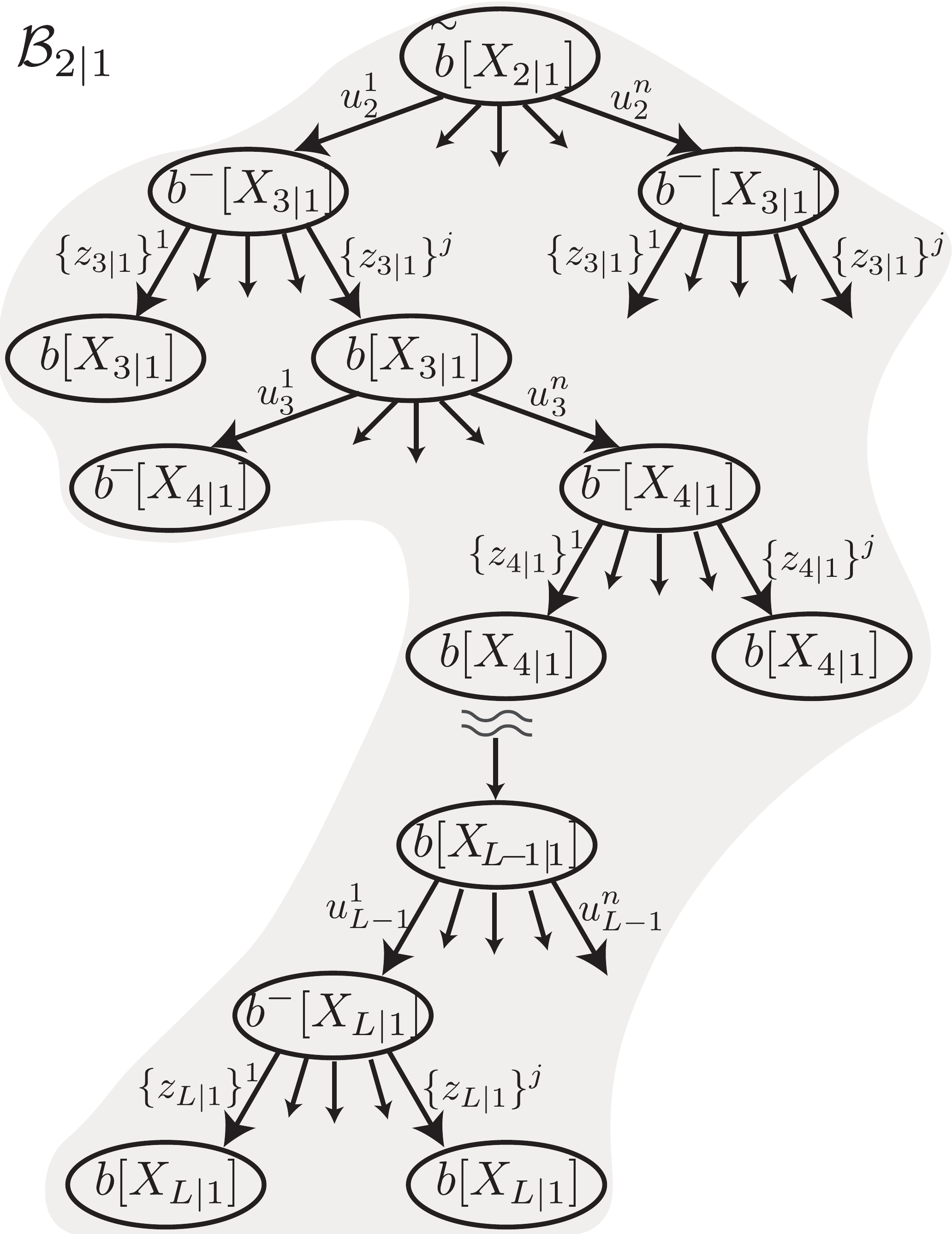}\label{fig:beliefUpdate:k}}
       \subfloat[]{\includegraphics[trim={-15 0 0 5},clip, width=0.34\textwidth]{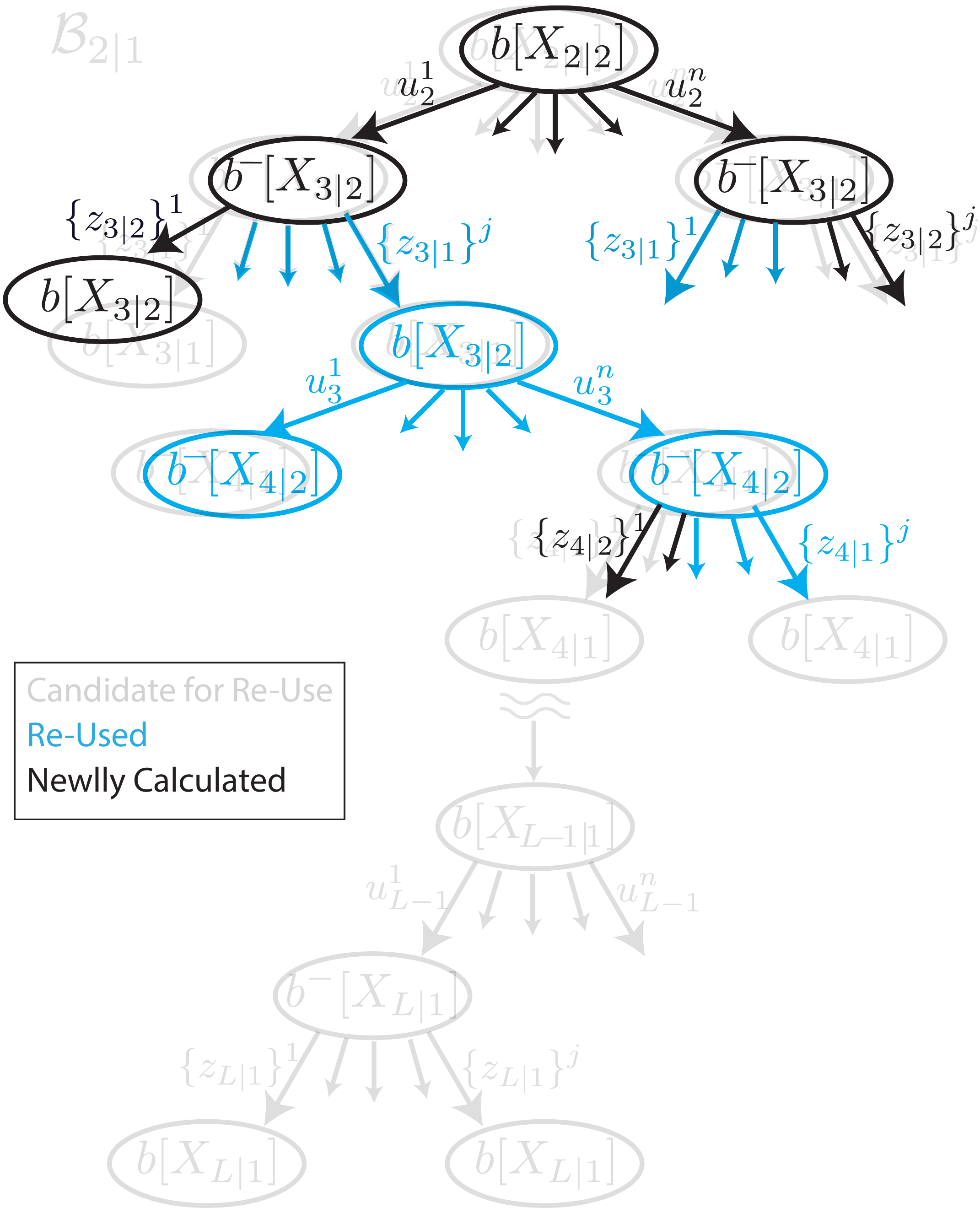}\label{fig:beliefUpdate:k+l}}
        \caption{ The belief update process of \ibsp presented in a belief tree, where each node represents a belief that branches either for one of $n$ candidate actions or $j$ sampled measurements. The selected branch for re-use from Figure~\ref{fig:selectBranch}, denoted by $\mathcal{B}_{2|1}$, is presented in (a) and as a water mark in (b). The succeeding \ibsp planning session at time $t=2$ is illustrated in (b), where the re-used sampled measurements and succeeding beliefs are marked in light blue.}
        \label{fig:beliefUpdate}
\end{figure}
\begin{table}
	\caption{Notations for Section~\ref{subsec:beliefUpdate}} 
	\centering 
	\begin{tabular}{c c} 
		\hline\hline 
		\textbf{Variable} & \textbf{Description}  \\ [0.5ex] 
		\hline 
		$ \Box_{t|k}$ & Of time $t$ while current time is $k$ \\[1ex]
		$b[X_{t|k}]$ & belief at time $t$ while current time is $k$ \\[1ex]
		$b^-[X_{t|k}]$ & belief at time $t-1$ propagated only with action $u_{t-1|k}$ \\[1ex]
		$\overset{\sim}{b}[X_{t|k}]$ & The root of the selected branch for re-use from planning time $k$ \\[1ex]
		$\mathcal{B}_{t|k}$ & The set of all beliefs from planning time $k$ rooted in $\overset{\sim}{b}[X_{t|k}]$\\[1ex]
		$b^s[X_{t|k+l}]$ & The $s_{th}$ sampled belief representing $b[X_{t|k+l}]$\\[1ex]
		$b^{s-}_{\alpha}[X_{t+1|k+l}]$ & The sampled belief $b^s[X_{t|k+l}]$ propagated with the $\alpha$ candidate action\\[1ex]
		$\{b^r_{\alpha}[X_{t|k+l}]\}_{r = 1}^{n}$ & A set of $n$ sampled beliefs that are first order children of $b^{s-}_{\alpha}[X_{t|k+l}]$ and are representing $b[X_{t|k+l}]$\\[1ex]
		$b^{s'-}_{\alpha}[X_{t+i|k}]$ & A propagated belief from $\mathcal{B}_{t|k}$ closest to $b^{s-}_{\alpha}[X_{t+i|k+l}]$\\[1ex]
		\texttt{dist} & The distance between $b^{s'-}_{\alpha}[X_{t|k}]$ and $b^{s-}_{\alpha}[X_{t|k+l}]$ \\[1ex]
		$\{b^r[X_{t|k}]\}_{r=1}^n$ & A set of $n$ sampled beliefs representing $b[X_{t|k}]$\\[1ex]
		$n_u$ & number of candidate actions per step \\[1ex]
		$(n_x \cdot n_z)$ & number of samples for each candidate action \\[1ex]
		\texttt{data} & All available calculations from current and precursory planning session \\[1ex]
		$\epsilon_c$ & belief distance critical threshold, max distance for re-use computational advantage \\[1ex]
		$\epsilon_{wf}$ & \wf threshold, max distance to be considered as close-enough for re-use without any update \\[1ex]
		\texttt{useWF} & a binary flag determining whether or not the \wf condition is considered \\ [1ex]
		$\beta_{\sigma}$ & $\sigma$ acceptance range parameter, for considering measurements as representative \\ [1ex]
		$\JD(p,q)$  & The distance between distributions $p$ and $q$ according to the \JD distance \\ [1ex]
		$D_{DA}(p,q)$  & The divergence between distributions $p$ and $q$ according to the data association difference\\ [1ex]
		\hline\hline
	\end{tabular}
	\label{table:reUseBeliefs} 
\end{table}

We will now describe a single iteration of Alg.~\ref{alg:updatePrevPlanning}, in which we have reached the $s_{th}$ belief (line~\ref{alg:updatePrevPlanning:beliefFor}) at the $i_{th}$ lookahead step (line~\ref{alg:updatePrevPlanning:lookaheadFor}), and already handled all previous steps and beliefs.   
First, we check whether the belief $b^s[X_{i|k+l}]$ has been created under the \wf condition (line~\ref{alg:updatePrevPlanning:isWF}), i.e. directly taken from $\mathcal{B}_{k+l|k}$ without any update (see Section~\ref{subsec:wf}). In case it did, we continue to take its descendants  directly from the appropriate beliefs in $\mathcal{B}_{k+l|k}$ without any update (line~\ref{alg:updatePrevPlanning:copyWF}) and mark them as created under the \wf condition (line~\ref{alg:updatePrevPlanning:markWF}).  
For the case where the belief $b^s[X_{i|k+l}]$ has not been created under the \wf condition we propagate it with the $\alpha$ candidate action, where $\alpha \in [1,n_u]$ and $n_u$ denotes the number of candidate actions per time step (line~\ref{alg:updatePrevPlanning:propagate}); this propagation yeilds $b^{s-}_{\alpha}[X_{i+1|k+l}]$. 
 We then consider all propagated beliefs $b^-[X_{i+1|k}] \subset \mathcal{B}_{k+l|k}$, and search for the closest one to $b^{s-}_{\alpha}[X_{i+1|k+l}]$ in the sense of belief distance, as discussed in Section~\ref{ssubsec:beliefDist}. Once found, we denote the closest propagated belief as $b^{s'-}_{\alpha}[X_{i+1|k}]$ (line~\ref{alg:updatePrevPlanning:beliefDist}).
 
In case there is no such belief close enough to make the update worthwhile, i.e. \texttt{dist}$ > \epsilon_c$, for this candidate action we continue as if using \fullbsp (line~\ref{alg:updatePrevPlanning:badBeliefDist}). In case the distance of the closest belief meets the \wf condition $\epsilon_{wf}$, we consider all beliefs $b[X_{i+1|k}] \subset \mathcal{B}_{k+l|k}$ that are rooted in $b^{s'-}[X_{i+1|k}]$.
Otherwise we continue and check whether the samples generated using $b^{s'-}_{\alpha}[X_{i+1|k}]$ constitute an adequate representation for $\prob{z_{i+1|k+l}|H_{i|k+l},u_{i|k+l}^{\alpha}}$, and re-sample if needed (line~\ref{alg:updatePrevPlanning:repSample}, see Section~\ref{ssubsec:repSample}).
Once we obtain the updated set of samples $\{$\texttt{repSamples}$\}_{1}^{n_x\!\cdot n_z}$, whether all were freshly sampled, entirely re-used or somewhere in between, we can acquire the set of posterior beliefs for look ahead step $i+1$ $\{b_{\alpha}[X_{i+1|k+l}] \}_{1}^{n_x\!\cdot n_z}$ (line~\ref{alg:updatePrevPlanning:updateBelief}) through an update, as discussed subsequently (see Section~\ref{ssubsec:beliefUpdate}). Once we have all updated beliefs for future time $i+1$, we can update the reward values of each of which (see Section~\ref{ssubsec:rewardUpdate}).
 We repeat the entire process for the newly acquired beliefs $\{b_{\alpha}[X_{i+1|k+l}] \}_{1}^{n_x\!\cdot n_z}$, and so forth up to $k+L$.

 Without loss of generality in the supplied formulation of Alg.~\ref{alg:updatePrevPlanning}, the candidate beliefs are considered only to the extent of the planning horizon overlap, i.e. $k+l+1:k+L$, whereas beliefs for the rest of the horizon $k+L+1:k+L+l$, are obtained by performing \fullbsp (Alg.~\ref{alg:ibsp}, line~\ref{alg:ibsp:lastSteps}). 

Next, we provide a walk-through example for Alg.~\ref{alg:updatePrevPlanning} (Section~\ref{ssubsec:eg_updatePrevPlanning})elaborate on belief distance (Section~\ref{ssubsec:beliefDist}), and continue with covering key aspects required by Alg.~\ref{alg:updatePrevPlanning}: determining whether \texttt{samples} are representative or not (Section~\ref{ssubsec:repSample}), the process of belief update given the representative set of samples $\{$\texttt{repSamples}$\}_{1}^{n}$ (Section~\ref{ssubsec:beliefUpdate}) and the incremental calculation of the immediate reward values per sampled belief (Section~\ref{ssubsec:rewardUpdate}).

\subsubsection{Alg.~\ref{alg:updatePrevPlanning} Walk-through Example}\label{ssubsec:eg_updatePrevPlanning}
We will now demonstrate Alg.~\ref{alg:updatePrevPlanning} using Figure~\ref{fig:beliefUpdate}. Figure~\ref{fig:beliefUpdate:k} illustrates the selected branch for re-use from a precursory planning session at time $t = 1$, i.e. $\mathcal{B}_{2|1}$ (see Figure~\ref{fig:selectBranch}), we have $n$ candidate actions each step, and for each candidate action we have $j$ sampled measurements. 
Figure~\ref{fig:beliefUpdate:k+l} illustrates a part of the belief tree at time $t=2$, where the top of the tree is the posterior belief of current time $t=2$, i.e. $b[X_{2|2}]$, from which we start the algorithm. 
Since $b[X_{2|2}]$ is the posterior of current time $t=2$, it was not created under the \wf condition, so we jump directly to line~\ref{alg:updatePrevPlanning:actionFor}. We propagate $b[X_{2|2}]$ with each of the $n$ candidate actions starting with $u^1_2$, and obtain the left most belief in the second level of the tree, $b^-[X_{3|2}]$ (line~\ref{alg:updatePrevPlanning:propagate}). 
Using \texttt{BeliefDist}(.) we obtain the closest belief from $\mathcal{B}_{2|1}$ to $b^-[X_{3|2}]$ as well as their distance \texttt{dist}. For our example the closest belief turns out to be the one which has been also propagated by the same candidate action, i.e. the left most $b^-[X_{3|1}]$ in the second level of $\mathcal{B}_{2|1}$. Since the distance suggests re-use is worthwhile but does not meet the \wf condition, i.e. $\epsilon_{wf} <$ \texttt{dist} $\leq \epsilon_{c}$, we proceed to line~\ref{alg:updatePrevPlanning:gatherSamples} in Alg.~\ref{alg:updatePrevPlanning}. 

We denote the set (of sets) of all $j$ sampled measurements as \texttt{samples}, i.e. \texttt{samples}$ \gets \{\{z_{3|1}\}^1,\hdots,\{z_{3|1} \}^j\}$. Using \texttt{IsRepSample}(.) we obtain a representative set for the measurement likelihood $\prob{z_{3|2}|H_{2|2},u_2^1}$ (see Section~\ref{ssubsec:repSample}). As we can see in Figure~\ref{fig:beliefUpdate:k+l}, other than $\{z_{3|1}\}^1$, which has been re-sampled, all other samples are re-used (denoted by blue arrows). Once we have a representative set of measurements, which in our case all but one are re-used from planning time $t=1$, we can update the appropriate beliefs using \texttt{UpdateBelief}(.) (see Section~\ref{ssubsec:beliefUpdate}). The belief resulting from the newly sampled measurement $\{ z_{3|2}\}^1$ is calculated from scratch by adding the measurement to $b^-[X_{3|2}]$ and performing inference, while the re-used samples allow us to incrementally update the appropriate beliefs from $\mathcal{B}_{2|1}$, rather than calculate them from scratch (see Section~\ref{ssubsec:beliefUpdate}). 

We now have an updated set of beliefs for future time $t=3$, that considers the candidate action $u_2^1$. For each of the aforementioned beliefs we incrementally calculate the appropriate reward value (see Section~\ref{ssubsec:rewardUpdate}) thus completing the incremental update for candidate action $u_2^1$. We repeat the aforementioned for the rest of the candidate actions, thus completing the third level of the belief tree presented in Figure~\ref{fig:beliefUpdate:k+l}. In a similar manner we continue to incrementally calculate the deeper levels of the belief tree up to future time $t=L$, thus concluding Alg.~\ref{alg:updatePrevPlanning}.

\subsubsection{Representative Sample}\label{ssubsec:repSample}
This section covers the problem of obtaining a set of measurement samples that are representative of the measurement likelihood distribution we should have sampled from. 
The motivation for re-using previously sampled measurements lies within the desire to refrain from performing inference by re-using previously calculated beliefs.
As explained in Section~\ref{subsec:compPlanning}, assuming the differences between (\ref{eq:belief_t_k})$_{(b)}$ and (\ref{eq:belief_t_k+l})$_{(b)}$ have been resolved (the predicted factors and their counterparts that already have been obtained respectively), the difference between Eq.~(\ref{eq:belief_t_k}) to Eq.~(\ref{eq:belief_t_k+l}) is limited to the difference between (\ref{eq:belief_t_k})$_{(c)}$ and (\ref{eq:belief_t_k+l})$_{(c)}$. Assuming both use the same action sequence, the difference between (\ref{eq:belief_t_k})$_{(c)}$ and (\ref{eq:belief_t_k+l})$_{(c)}$ is limited to the predicted measurements being considered by each. 
\begin{algorithm}
	\caption{IsRepSample}\label{alg:isRepSample}
	\begin{flushleft}
		\textbf{Input:} \\
		\hspace{30pt} \texttt{samples} \Comment{set of candidate sampled measurements for re-use from planning at time $k$, see Alg~\ref{alg:updatePrevPlanning} line~\ref{alg:updatePrevPlanning:gatherSamples} }\\
		\hspace{30pt} $b^{-}[X_{i|k+l}]$ \Comment{The belief from planning time $k+l$ the samples should be representing, i.e. sampled from}\\
	\end{flushleft}
	\begin{algorithmic}[1]
		\State Given $\beta_{\sigma}$ = 1.5 \Comment{User determined Heuristic, in direct proportion to acceptance}
		\State \texttt{stateSamples} $\gets$ the sampled states that created \texttt{samples} \label{alg:isRepSample:getStates} \Comment{see Alg.~\ref{alg:sampling_z}}
		\For{each \texttt{sample} $\in$ \texttt{stateSamples}}
			\If{\texttt{sample} $\subset$ $\pm \beta_{\sigma} \! \cdot \! \sigma$ of $b^{-}[X_{i|k+l}]$}\Comment{The sample falls within $\pm \beta_{\sigma} \! \cdot \! \sigma$ range, hence accepted}
				\State $\{$\texttt{repSamples}$\}_{1}^{n}$ $\gets$ all measurement samples $\in$ \texttt{samples} that were crated by \texttt{sample} \label{alg:isRepSample:resampleMeas} 
			\Else \Comment{The sample falls outside the $\pm \beta_{\sigma} \! \cdot \! \sigma$ range, hence rejected}
				\State $\{$\texttt{repSamples}$\}_{1}^{n}$ $\gets$ re-sample $n_z$ measurements using $b^{-}[X_{i|k+l}]$ \Comment{freshly sampled, see Alg.~\ref{alg:sampling_z}}
			\EndIf    
		\EndFor
		\State \Return{$\{$\texttt{repSamples}$\}_{1}^{n}$, $\{$q(.)$\}_{1}^{n}$} \Comment{$\{$q(.)$\}_{1}^{n}$ denote the distributions $\{$\texttt{repSamples}$\}_{1}^{n}$ were sampled from}
	\end{algorithmic}
\end{algorithm}

While the field of representative sampling is a rich research area on its own, in order to facilitate \ibsp, we chose a straightforward approach that can be easily substituted with a more sophisticated one in due time.
Under the sampling paradigm presented in Alg.~\ref{alg:sampling_z}, it is sufficient to determine the representativeness of a measurement sample based on the state sample $\chi$ which should be sampled from the propagated belief $b^-[X_{i|k+l}]$.


Considering the known (stochastic) measurement model, the space of measurement model distributions is uniquely defined by the set of state samples. So, in order to simplify the selection of representative measurement samples, we consider only the set of sampled states and assume that a set of sampled states that are representative of the propagated belief they should have been sampled from, yields a set of representative sampled measurements. 
Following the aforementioned, the problem of determining a set of representative measurement samples, turns into a problem of determining a set of state samples representative of some propagated belief. 

 Let us consider \texttt{samples} and $b^-[X_{i|k+l}]$ denoting respectively the candidate measurements for re-use and the propagated belief from planning time $k+l$. Due to the fact that \texttt{samples} were sampled from distributions different from $\prob{z_{i+1|k+l}|H^-_{i+1|k+l}}$, we need to assure they constitute an adequate representation of it.
In Alg.~\ref{alg:isRepSample} we consider the sampled states that led to the acquired sampled measurements (Alg.~\ref{alg:isRepSample} line~\ref{alg:isRepSample:getStates}), and denote it as \texttt{stateSamples}.
We consider each sampled state separately, and determine \texttt{sample} $\in$ \texttt{stateSamples} as representative if it falls within a predetermined $\sigma$ range of the distribution it should have been sampled from. 
It is worth stressing that in order to facilitate the use of importance sampling in solving the expected reward value (as discussed later in Section~\ref{subsec:incExp}) one should have access to the importance sampling distributions denoted by $\{$q(.)$\}_{1}^{n}$ in Alg.~\ref{alg:isRepSample}.  

\begin{figure}
	\centering
		\subfloat[]{\includegraphics[trim={0 0 245 0},clip, width=0.23\columnwidth]{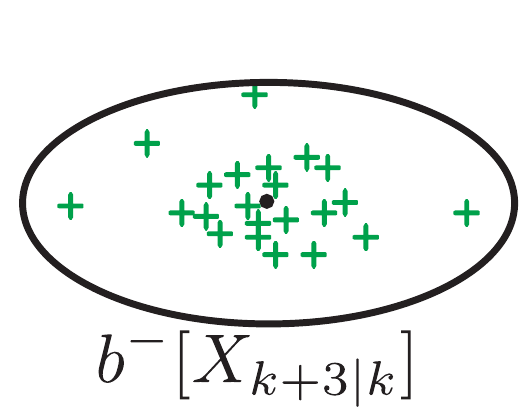}\label{fig:beliefDif:origin}}
		\subfloat[]{\includegraphics[trim={155 0 0 0},clip, width=0.38\columnwidth]{Figures/rest/beliefDiff.pdf}\label{fig:beliefDif:after}}
        \caption{Illustration for adequate and inadequate representation of a belief by samples. (a) Illustrates a belief (denoted by an ellipse) over the propagated joint state at future time $k+3$ calculated as part of planning at time $k$, and the twenty two samples (denoted by green "+" signs) taken from it. (b) Illustrates two instantiations of belief over the propagated joint state at future time $k+3$ calculated as part of planning at time $k+1$ (denoted in red and blue), overlapping the belief of precursory planning time and its samples. While the samples of (a) can be considered as adequate representation for the blue belief, they can also be considered as inadequate representation of the red belief.}
        \label{fig:beliefDif}
\end{figure}

We conclude this section with a toy example for the aforementioned method for determining a representative set of samples. Figure~\ref{fig:beliefDif:origin} illustrates a set of states $\chi$ (denoted by green "+" signs) used in future time $k+3$ under planning time $k$, where  Figure~\ref{fig:beliefDif:after} illustrates how well the same samples represents two instantiations of the same future time $k+3$ in succeeding planning at time $k+1$. By considering some $\pm \beta_{\sigma} \cdot \sigma$ range of each instantiation of $b^-[X_{k+3|k+1}]$ we can determine which of the available samples can be considered as representative. Following Alg.~\ref{alg:isRepSample}, for a value of $\beta_{\sigma} =1$, under the blue belief instantiation in Figure~\ref{fig:beliefDif:after}, all but the left most sample will be considered as representative of $b^{-}[X_{k+3|k+1}]$ since they are within the covariance ellipsoid, representing the $\pm$1$\sigma$ range. While under the red belief instantiation only the three samples within the red covariance ellipsoid will be considered as representative where the rest will be re-sampled from the nominal distribution.  

  
\subsubsection{Belief Update as Part of Immediate Reward  Calculations}\label{ssubsec:beliefUpdate}

Once we determined a set of $n$ samples we wish to use (Alg.~\ref{alg:updatePrevPlanning} line~\ref{alg:updatePrevPlanning:repSample}), whether newly sampled, re-used or a mixture of both, we can update the appropriate beliefs in order to form the set $\{b^r[X_{i+1|k+l}]\}_{r=1}^{n}$, required for calculating the reward  function values at the look ahead step $i+1$. In this section we go through the belief update process, which is case sensitive to whether a sample was newly sampled or re-used. We start with the standard belief update for newly sampled measurements; continue with recalling the difference between some belief $b[X_{i+1|k+l}]$, and its counterpart from planning time k, i.e.~$b[X_{i+1|k}]$; and conclude with belief update for a re-used measurement.
\subsubsection*{Newly Sampled Measurement}
For a newly sampled measurement $z_{i+1|k+l}$, we follow the standard belief update of incorporating the measurement factors to the propagated belief $b^-[X_{i+1|k+l}]$ as in
\begin{equation}\label{eq:stdBeliefUpdate}
	b[X_{i+1|k+l}] \propto b^-[X_{i+1|k+l}] \cdot \!\!\!\!\! \prod_{j \in \mathcal{M}_{i+1|k+l}} \!\!\!\!\!\! \prob{z^j_{i+1|k+l}|x_{i+1},l_{j} },
\end{equation}
and then performing inference; hence no re-use of calculations from planning time $k$.

\subsubsection*{Re-used Measurement}
As mentioned earlier, the motivation for re-using samples is to evert from the costly computation time of performing inference over a belief. Since we already performed inference over beliefs at planning time $k$, if we re-use the same samples, we can evert from performing standard belief update (\ref{eq:stdBeliefUpdate}), and utilize beliefs from planning time $k$. 
As discussed in Section~\ref{subsec:compPlanning}, the factors of two beliefs over the same future time but different planning sessions could be divided into three groups as illustrated in Figure~\ref{fig:horizonOverlap}: (i) representing shared history which is by definition identical between the two; (ii) representing potentially different factors since they are predicted for time $k$ and given for time $k+l$; (iii) represents future time for both, but each is conditioned over different history subject to (ii), so also potentially different.

Let us consider the measurement $z_{i+1|k} \subset$ $\{$repSamples$\}_{1}^{n}$, marked for re-use. The belief we are required to adjust is the one resulted from $z_{i+1|k}$ at planning time $k$, i.e.
\begin{equation}\label{eq:updateBelief:time_k}
	b[X_{i+1|k}] \propto \prob{X_{0:i+1}|H^-_{i+1|k},z_{i+1|k}}.
\end{equation}
Although $b[X_{i+1|k}]$ is given to us from precursory planning, it might require an update to match the new information received up to time $k+l$. 
In contrary to (\ref{eq:stdBeliefUpdate}), we update $b[X_{i+1|k}]$ using our previous work \cite{Farhi17icra,Farhi19arxiv}, which enables us to incrementally update the belief solution without performing inference once more. 
The process of incrementally updating a belief, as thoroughly described in \cite{Farhi19arxiv}, can be divided into two general steps: first updating the DA and then the measurement values. 
Let us consider the belief we wish to re-use from planning time $k$ (\ref{eq:updateBelief:time_k}).
 
In order to update the measurement factors of (\ref{eq:updateBelief:time_k}) to match (\ref{eq:stdBeliefUpdate}) we start with matching their data association (DA).
As described in \cite[Section~3.5.2]{Farhi19arxiv}, this DA matching will provide us with the indices of the factors that their DA should be updated, as well as factors that should be added or removed. The DA update process is being done over the graphical representation of the belief, i.e. the factor graph and bayes-tree (see \cite[Section~3.5.2]{Farhi19arxiv}). Once the DA update is complete we are left with updating the values of all the consistent DA factors. For the special case of Gaussian beliefs, \cite[Section~3.4]{Farhi19arxiv} provides few methods to efficiently update the aforementioned.
Once the update is complete we obtain $n$ beliefs representing $b[X_{i+1|k+l}]$, each corresponding to one of our $n$ samples $\{$\texttt{repSamples}$\}_{1}^{n}$. 

At this point it is worth reiterating the importance of $\epsilon_c$ (Section~\ref{ssubsec:closeEnough}), the computational effort of updating a candidate belief is with direct correlation to the distance between the beliefs.

\subsubsection{Re-using / Calculating Immediate Reward  values}\label{ssubsec:rewardUpdate}
As part of solving the planning problem (\ref{eq:optAction}) we need to get the objective value for various action sequences. The objective value for some action sequence (\ref{eq:objective}) is given by the sum of expected rewards along the planning horizon. The expected reward value is a weighted average of immediate rewards over future belief realizations (\ref{eq:ExApprx}). This section deals with the calculation of those immediate rewards. Since in the general case those immediate rewards are functions of belief and action, we need to perform inference over the beliefs before we can obtain the immediate reward values.  
Once we have a set of beliefs representing the possible futures of executing some action $u_i$ at future time $i+1$ (Alg.~\ref{alg:updatePrevPlanning} line~\ref{alg:updatePrevPlanning:updateBelief}), we can calculate the immediate rewards resulting from each such belief (Alg.~\ref{alg:updatePrevPlanning} line~\ref{alg:updatePrevPlanning:updateReward}).

Given a reward function for the $i+1$ lookahead step $r_{i+1}(b,u)$, a posterior  future belief $b^s[X_{i+1|k+l}]$, and the corresponding action $u_{i|k+l}$, the immediate reward $r^s_{i+1|k+l}$ is given by 
\begin{equation}\label{eq:reward:scratch}
	r^s_{i+1|k+l} = r_{i+1}\left(b^s[X_{i+1|k+l}], u_{i|k+l}\right).
\end{equation}
In this work, depending on the origin of $b^s[X_{i+1|k+l}]$, the immediate reward value is either calculated according to Eq.~(\ref{eq:reward:scratch}), i.e. from scratch, or being taken directly from a precursory planning session. 

Under \ibsp, any future belief is obtained by one of three ways: calculated from scratch using freshly sampled measurements (\ref{eq:stdBeliefUpdate}); through updating a previously calculated belief with the appropriate information; or by one of our main contributions in this work, completely re-using a previously calculated belief without any update (denoted as \wf, Section~\ref{subsec:wf}).
In this work, for the first two cases, the immediate reward value $r^s_{i+1|k+l}$ is obtained through simply solving Eq.~(\ref{eq:reward:scratch}), where for the third case the reward value $r^s_{i+1|k+l}$ is not calculated, but approximated by considering a previously calculated immediate reward value.

Between calculating the immediate reward directly from Eq.~(\ref{eq:reward:scratch}), and approximating it without any calculation based on \wf (see Section~\ref{subsec:wf})
there is the middle ground, not used in this work and left for future work, incrementally updating a previously calculated reward value.

\subsection{Incremental Expectation with Importance Sampling}\label{subsec:incExp}
This section describes one of our main contributions, incorporating Multiple Importance Sampling (MIS) into the objective estimator in order to account for selective re-use of previously calculated future beliefs. 
For the reader's convenience all the notations of this section are summarized in Table~\ref{table:iexWithIS}.
Once we obtain immediate reward values for candidate actions along the planning horizon (Alg.~\ref{alg:ibsp}, lines~\ref{alg:ibsp:updatePlanning}-\ref{alg:ibsp:lastSteps}), we use them to estimate (\ref{eq:objective_kl}). 
However, because we are selectively re-using samples from precursory planning sessions, we estimate (\ref{eq:objective_kl}) using samples not necessarily taken from $\prob{z_{k+l+1:i|k+l}|H^{-}_{k+l|k+l},u_{k+l:i-1|k+l}}$, thus the formulation should be adjusted accordingly. 
\begin{figure}
	\centering
		\subfloat[]{\includegraphics[trim={0 0 0 0},clip, width=0.32\columnwidth]{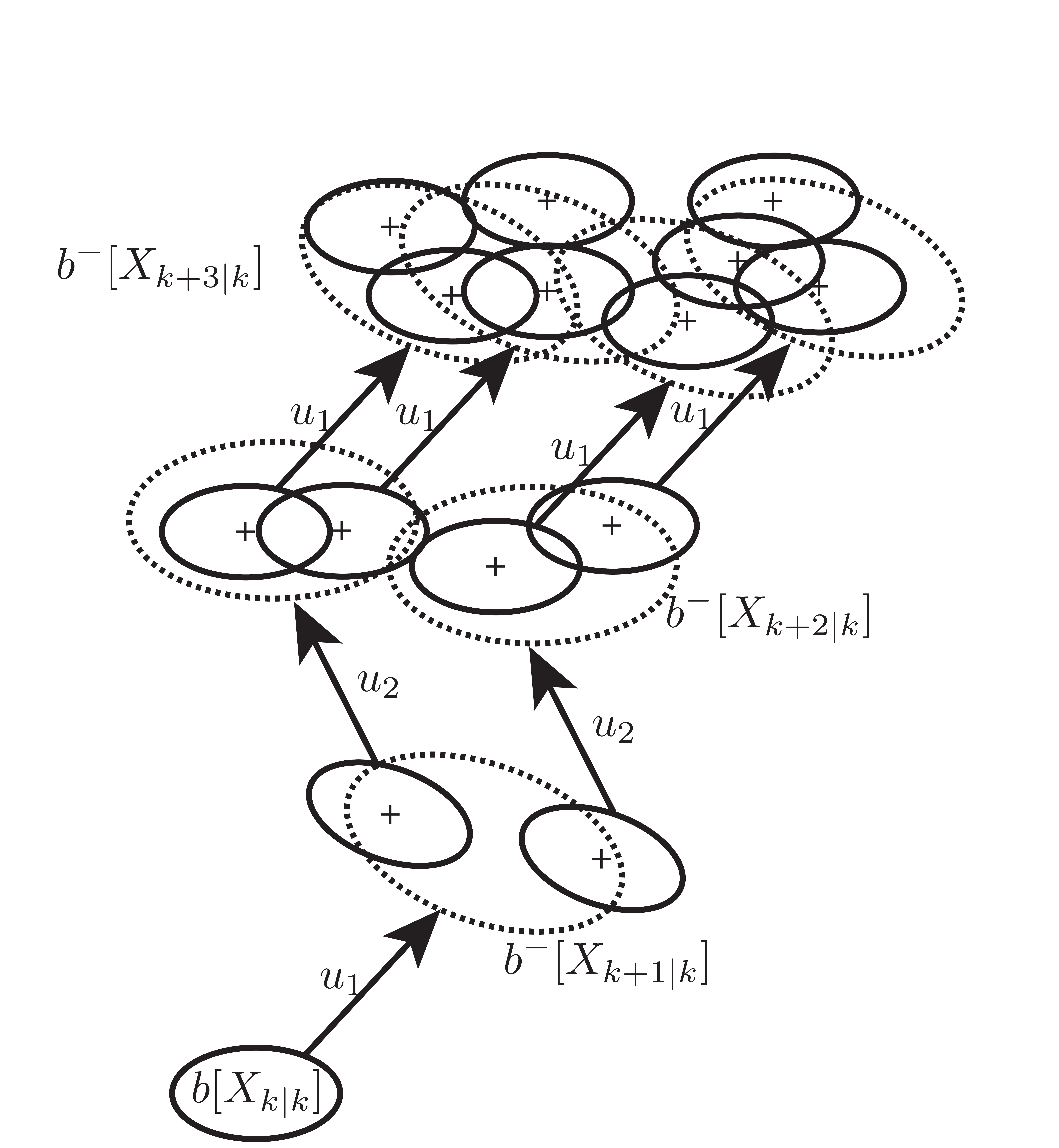}\label{fig:toyTimeK}}
		\subfloat[]{\includegraphics[trim={0 0 0 0},clip, width=0.32\columnwidth]{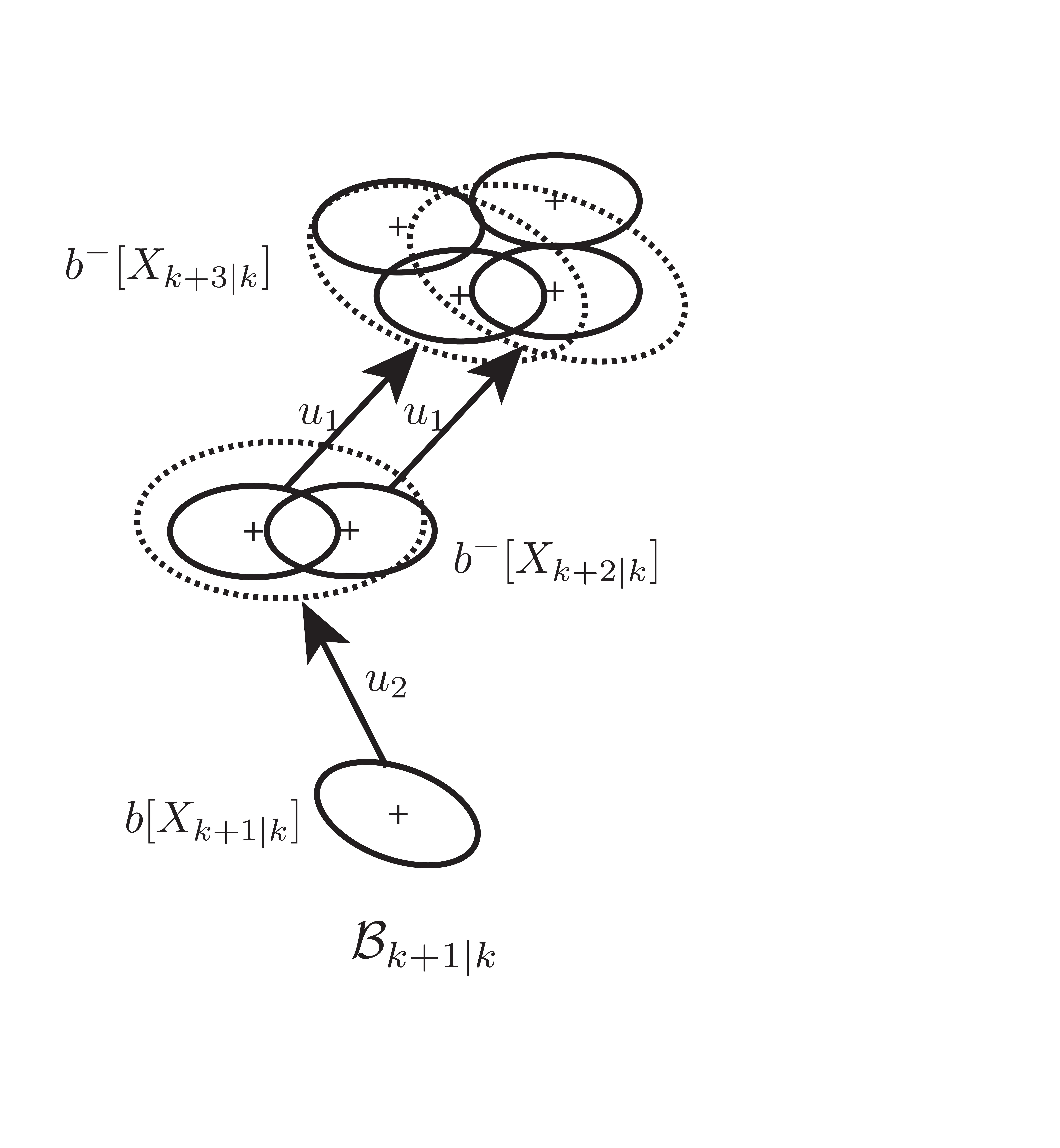}\label{fig:toyCandB}}
        \subfloat[]{\includegraphics[trim={0 0 0 0},clip, width=0.32\columnwidth]{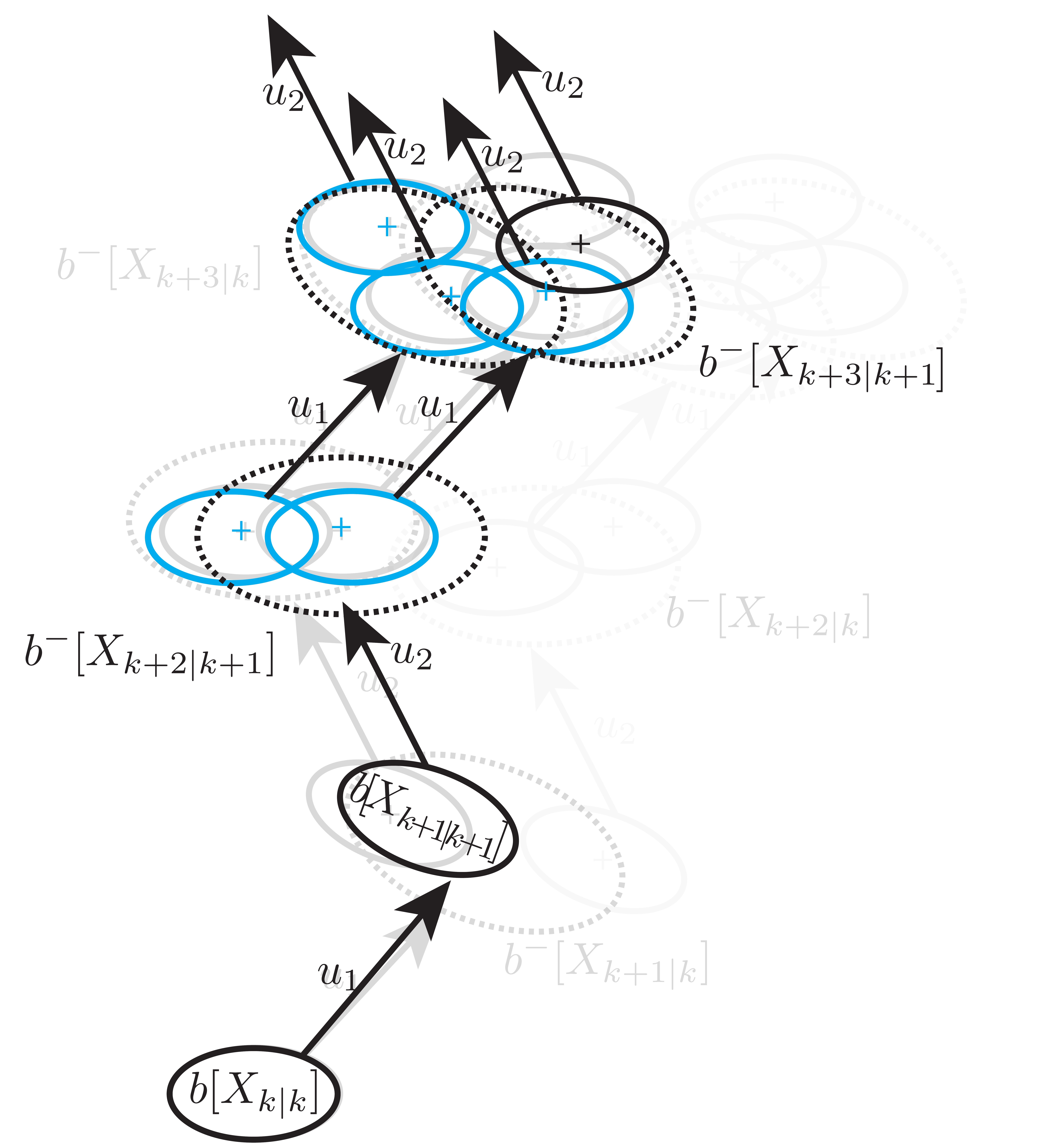}\label{fig:toyTimeK1}}
        \caption{Illustration of two consecutive planning sessions, each with horizon of three steps, and two sampled measurements per action. Beliefs are denoted by solid ellipses, propagated beliefs by dotted ellipses, and beliefs that have been re-used are denoted in blue. (a) Planning over the action sequence $u_1 \rightarrow u_2 \rightarrow u_1$ at time $k$ (b) The selected branch for re-use $\mathcal{B}_{k+1|k}$ from planning at time $k$ (c) The segment of planning at time $k+1$ for the action sequence $u_2 \rightarrow u_1 \rightarrow u_2$ that overlaps with (b).}
        \label{fig:examples}
\end{figure}
In the following, we recall the standard general formulation for the objective function, we consider an assumption simplifying the objective weighting scheme, we then relax this simplifying assumption and characterize our problem under the MIS problem and formulate it accordingly, and finally using a simple example we demonstrate objective calculation under \ibsp.  
 
First let us recall the standard formulation for Eq.~(\ref{eq:objective_kl}), for sampling $(n_x \cdot n_z)$ measurements per candidate step following Alg.~\ref{alg:sampling_z}, 
\begin{equation}\label{eq:stdObjWithSamples}
	J(u') = \sum_{i=k+l+1}^{k+l+L} \left[  \eta_{k+l+1} \!\!\!\!\!\!\!\!\sum_{\{z_{k+l+1|k+l}\}} \!\!\!\!\!\! \omega_{k+l+1}^{n} \left( \hdots \left(\eta_i \sum_{\{z_{i|k+l}\}}\omega_{i}^{n}\cdot r_i\left(b^{n}[X_{i|k+l}],u_{i-1|k+l} \right)  \right) \hdots \right)\right],
\end{equation}
where  $\omega_i^{n}$ denotes the weight of the $n_{th}$ measurement sample for future time $i$, $\eta_i$ denotes the normalizer of the weights at time $i$ such that $\eta_i^{-1} = \sum_{n=1}^{n_x\cdot n_z} \omega_i^{n}$, and $b^{n}[X_{i|k+l}]$ is the belief considering a specific set of samples up to future time $i$, i.e. 
\begin{equation}
	b^{n}[X_{i|k+l}] \doteq \prob{X_i|H_{k+l|k+l},u_{k+l:i-1|k+l},\{z_{k+l+1|k+l}\},\hdots,\{z_{i|k+l}\}}.
\end{equation}
%

In \ibsp we are potentially forcing samples from previous planning sessions; this type of problems that involve expressing one distribution using samples taken from another is referred to as importance sampling (see~Appendix~\ref{app:IS}). 
It is worth stressing that unlike \cite{Luo19ijrr} which uses importance sampling to sample sets of future states actions and measurements, we make use of importance sampling to incorporate re-used beliefs within the objective estimation.

Let us consider an assumption used in \citep{Farhi19icra} to simplify the weighting scheme required for \ibsp under importance sampling.
Assume we are at planning time $k+l$, considering planning session at time $k$ for re-use under \ibsp. At each lookahead step $i \in [k+l+1 : k+L]$ we conclude that all candidate measurements from planning at time $k$ form a representative set of $\prob{z_{i|k+l}|H^-_{i|k+l}}$, so we decide to re-use them all. Moreover, all of these samples from lookahead step $i$ at planning time $k$ were sampled from the same distribution denoted as $q_{i|k}(.)$, whether they were all freshly sampled at planning time $k$ or have been entirely re-used themselves from some past time in-which they were freshly sampled. Under this scenario, we have a set of representative measurements, all sampled from the same distribution that is not the nominal one, thus Eq.~(\ref{eq:stdObjWithSamples}) can be written as
\begin{equation}\label{eq:specialMIS_obj}
	J(u') \sim 
	\sum_{i=k+l+1}^{k+l+L} \left[ \frac{1}{n_i} \sum_{g=1}^{n_i} \omega_{i}(z_{k+l+1:i}^{g}) \cdot r_{i}\left( b^g[X_{i|k+l}],u'_{i-1|k+l} \right) \right],
\end{equation}
where $n_i$ denotes the number of samples used in the $i-k-l$ lookahead step, and the weights are given by a simple probability ratio, which under importance sampling is usually referred to as the likelihood ratio (see~Appendix~\ref{app:IS}),
\begin{equation}
	\omega_{i}(z_{k+l+1:i}^{g}) =  \frac{\prob{z^{g}_{k+l+1:i}|H_{k+l|k+l},u_{k+l:i-1|k+l} }}{q_{i}(z^{g}_{k+l+1:i})}, 
\end{equation}
where $\mathbb{P}(.)$ denotes the measurement likelihood we should have sampled the measurements from, and $q_i(.)$ denotes the probability distribution we actually sampled from. 
For example, assuming we re-used measurements from planning at time $k$, the likelihood ratio will be given by
\begin{equation}\label{eq:omegaExamp}
	\omega_{i}(z_{k+l+1:i}^{g}) =  \frac{\prob{z^{g}_{k+l+1:i}|H_{k+l|k+l},u_{k+l:i-1|k+l} }}{\prob{z^{g}_{k+l+1:i}|H_{k+l|k},u_{k+l:i-1|k} }}. 
\end{equation}
While this formulation under the simplified assumption allows one to easily re-use previous planning sessions, it is only under an "all or nothing" approach, which cripples the full potential of \ibsp for selective re-use of previous planning sessions. 
Our problem as part of \ibsp is more specific: at each look-ahead step we can potentially force samples from $M$ different measurement likelihood distributions (e.g. Section~\ref{ssubsec:ibspExample}), which none of them is necessarily  the nominal one. The possible number of distributions from which measurements are being sampled in every look-ahead step $i$ is bounded by
\begin{equation}\label{eq:pdfNum}
	1 \leq M \leq n_i. 
\end{equation}
The lower bound in (\ref{eq:pdfNum}) would occur either when there are no re-used samples, or when all samples are re-used and were originally sampled from the same distribution, e.g.~from a specific measurement likelihood in a previous planning session, as in (\ref{eq:omegaExamp}). 
The upper bound in (\ref{eq:pdfNum}) would occur when $n_m = 1$ $\forall m$, i.e.~each sample has been obtained using a different distribution, e.g.~from $m$ different measurement likelihoods, potentially, from different planning sessions. 

This problem falls within the multiple importance sampling problem (see~Appendix~\ref{app:IS}), and as such we can reformulate the estimator for (\ref{eq:objective_kl}) accordingly

\begin{table}[h]
	\caption{Notations for Section~\ref{subsec:incExp}} 
	\centering 
	\begin{tabular}{c c} 
		\hline\hline 
		\textbf{Variable} & \textbf{Description}  \\ [0.5ex] 
		\hline 
		$ \Box_{t|k}$ & Of time $t$ while current time is $k$ \\[1ex]
		$b[X_{t|k}]$ & belief at time $t$ while current time is $k$ \\[1ex]
		$b^n[X_{t|k+l}]$ & The $n_{th}$ sampled belief representing $b[X_{t|k+l}]$\\[1ex]
		$\omega_{i}^n$ & the weight corresponding to the $n_{th}$ measurement sample for lookahead step $i$ \\[1ex]
		$q_i(z^g_{t : i})$ & importance sampling distribution at lookahead step $i$, from which $z^g_{t:i}$ were sampled \\[1ex]
		$\prob{z^{\Box}_{t:i}|H,u}$ & the nominal distribution at lookahead steps $t:i$ \\[1ex]
		$n_i$ & the number of samples considered at lookahead step $i$ \\[1ex]
		$M_i$ & number of distributions at look ahead step i from which measurements are being sampled \\[1ex]
		$n_m$ & the number of measurements sampled from the $m_{th}$ distribution at look ahead step $i$ \\[1ex]
		$\omega_{i}(z^g_{t:i})$ & private case of $\omega_{i}(z^{m,g}_{t:i})$ where $m=1$ \\[1ex]
		$\omega_{i}(z^{m,g}_{t:i})$ & Balance Heuristic likelihood ratio at lookahead step $i$ corresponding to $z^{m,g}_{t:i}$ \\[1ex]
		$z_{t:i}^{m,g}$ & the $g_{th}$ set of future measurements at time instances $t:i$ sampled from the $m_{th}$ distribution \\[1ex]
		$b^{m,g}[X_{t|k+l}]$ & the sampled belief representing $b^{m,g}[X_{t|k+l}]$ which consider the measurements $z_{k+l+1:t}^{m,g}$\\ [1ex]
		$q_m()$ & the $m_{th}$ marginal importance sampling distribution at lookahead step $i$, $m \in [1,M_i]$ \\[1ex]
		$b^-[X_{t|k}]$ & belief at time $t-1$ propagated only with action $u_{t-1|k}$ \\[1ex]
		$\overset{\sim}{b}[X_{t|k}]$ & The root of the selected branch for re-use from planning time $k$ \\[1ex]
		$\mathcal{B}_{t|k}$ & The set of all beliefs from planning time $k$ rooted in $\overset{\sim}{b}[X_{t|k}]$\\[1ex]		
		$\{b[X_{t|k}]\}_1^j$ & $j$ sampled beliefs representing $b[X_{t|k}]$\\ [1ex]
		$\{r_i(b[X_{t|k}],u)\}_1^j$ & $j$ immediate rewards of lookahead step $i$\\ [1ex]
		$p_i()$ & the marginal nominal distribution at lookahead step $i$ \\ [1ex]
		$\tilde{p}_i()$ & the nominal distribution at lookahead step $i$ \\ [1ex]
		$\tilde{q}_m()$ & the $m_{th}$ importance sampling distribution at lookahead step $i$, $m \in [1,M_i]$ \\[1ex]
		\hline\hline
	\end{tabular}
	\label{table:iexWithIS} 
\end{table}

%
\begin{equation}\label{eq:MIS_obj}
	J(u') \sim 
	\sum_{i=k+l+1}^{k+l+L} \left[  \sum_{m=1}^{M_i} \frac{1}{n_m} \sum_{g=1}^{n_m} \tilde{\omega}_{m}(z_{k+l+1:i}^{m,g}) \cdot \frac{\prob{z^{m,g}_{k+l+1:i}|H_{k+l|k+l},u_{k+l:i-1|k+l} }}{q_m(z_{k+l+1:i}^{m,g})}\cdot r_{i}\left( b^{m,g}[X_{i|k+l}],u'_{i-1|k+l} \right) \right],
\end{equation}
where $i$ denotes the look-ahead step, $M_i$ is the number of distributions from which measurements are being sampled, $n_m$ is the number of measurements sampled from the $m_{th}$ distribution, $q_m(.)$ is the $m_{th}$ distribution that samples were taken from, and $z^{m,g}_{k+l+1:i}$ are the $g_{th}$ set of future measurements at time instances $k+l+1:i$, sampled from the $m_{th}$ distribution. The $m_{th}$ weight is denoted by $\tilde{\omega}_{m}$ where $\sum \tilde{\omega}_{m} = 1$, and $\tilde{\omega}_{m}>0 \ \forall m$. This estimator (\ref{eq:MIS_obj}) is unbiased under the assumption that $q_m(.)>0$ whenever $\tilde{\omega}_{m}() \cdot \prob{z|H} \cdot r_i(.) \neq 0$. When using previous planning sessions as candidates for re-use under \ibsp, $q_m(.)$ corresponds to a measurement likelihood of those previous planning sessions (e.g. Section~\ref{ssubsec:ibspExample}).

In this work we made use of the unbiased nearly optimal estimator for (\ref{eq:objective_kl}), based on the multiple importance sampling problem with the balance heuristic (see~Appendix~\ref{app:IS}) 
\begin{equation}\label{eq:MIMBH_obj}
	J(u') \sim 
	\sum_{i=k+l+1}^{k+l+L} \left[ \frac{1}{n_i} \sum_{m=1}^{M_i} \sum_{g=1}^{n_m} \omega_{i}(z_{k+l+1:i}^{m,g}) \cdot r_{i}\left( b^{m,g}[X_{i|k+l}],u'_{i-1|k+l} \right) \right],
\end{equation}
where $i$ denotes the look-ahead step, $n_i$ are the number of samples considered, $M_i$ is the number of distributions from which measurements are being sampled, $n_m$ is the number of measurements sampled from the $m_{th}$ distribution, and following the balance heuristic $\omega_{i}(z^{m,g})$ is the likelihood ratio of the $g_{th}$ sample from the $m_{th}$ distribution given by
\begin{equation}\label{eq:MIM_w}
	\omega_{i}(z_{k+l+1:i}^{m,g}) = \frac{\prob{z^{m,g}_{k+l+1:i}|H_{k+l|k+l},u_{k+l:i-1|k+l} }}{\Sigma_{\tilde{m} = 1}^{M_i} \frac{n_{\tilde{m}}}{n_i}q_{\tilde{m}}(z^{m,g}_{k+l+1:i})}, 
\end{equation}
where $z^{m,g}_{k+l+1:i}$ are the $g_{th}$ set of future measurements at time instances $[k+l+1:i]$, sampled from the $m_{th}$ distribution, and $q_{\tilde{m}}(.)$ is the $\tilde{m}_{th}$ importance sampling distribution.

The balance heuristic is considered as nearly optimal in the following sense:
\begin{lemma}\label{th:BHoptimality}
	Let $n_m \geq 1$ be positive integers for $m = 1,...,M_i$. Let $\tilde{\omega}_1,...,\tilde{\omega}_{M_i}$ be a partition of unity and let $\omega^{BH}$ be the balance heuristic. Let $\tilde{J}_{\tilde{\omega}_m}$ and $\tilde{J}_{\omega^{BH}}$ be the estimates of $J$ under $\tilde{\omega}_m$ and $\omega^{BH}$ respectively. Then
	\begin{equation}
		Var(\tilde{J}_{\tilde{\omega^{BH}}}) \leq Var(\tilde{J}_{\tilde{\omega}_m}) + \left( \frac{1}{\min_{m} n_m} - \frac{1}{\sum_m n_m}\right)J^2
	\end{equation}
\end{lemma}
\begin{proof}
	This is Theorem~1. of \cite{Veach95siggraph}.
\end{proof}
When all samples being considered to estimate (\ref{eq:objective_kl}) are sampled from their nominal distributions, (\ref{eq:MIMBH_obj}) is reduced back to Eq.~(\ref{eq:fullObj}), with all the weights degenerating to ones; for such a case, $M_i = 1$ $\forall i$, and $q_1(.) = \prob{z_{k+l+1:i}|H_{k+l|k+l},u_{k+l:i-1|k+l}}$, thus $\omega_i = 1$ $\forall i$.
When for each lookahead step, all samples share the same distribution that is not the nominal one, i.e. $M_i = 1 \ \forall i$, and $p(.) \neq q(.)$, thus Eq.~(\ref{eq:MIMBH_obj}) is reduced back to Eq.~(\ref{eq:specialMIS_obj}).

\subsubsection{\ibsp Walk-through Example}\label{ssubsec:ibspExample}
To better understand the  objective value calculation under \ibsp, let us perform \ibsp over a simple example. Assume we have access to all calculations from planning time $k$, in-which we performed \fullbsp (or \ibsp) for a horizon of three steps, and with $n_x=2$ and $n_z=1$. Figure~\ref{fig:toyTimeK} illustrates a specific action sequence, $u_1 \rightarrow u_2 \rightarrow u_1$, considered as part of planning at time k. Let us assume that the optimal action decided upon as part of planning at time $k$, and was later executed was $u_1$. We are currently at time $k+1$, after performing inference using the measurements we received as a result of executing $u_1$. We perform planning using \ibsp with the same horizon length and number of samples per action, for several action sequences, one of which is the action sequence $u_2 \rightarrow u_1 \rightarrow u_2$, as illustrated in Figure~\ref{fig:toyTimeK1}. 

Following Alg.~\ref{alg:ibsp} line~\ref{alg:ibsp:selectBranch}, out of the two available beliefs from planning time $k$ shown in Figure~\ref{fig:toyTimeK}, $\{b[X_{k+1|k}]\}_1^2$, the left one is determined as closer to $b[X_{k+1|k+1}]$, so we consider all its descendants as the set $\mathcal{B}_{k+1|k}$, as illustrated in Figure~\ref{fig:toyCandB}, and denote the distance between $b[X_{k+1|k}]$ and $b[X_{k+1|k+1}]$ as \texttt{Dist}. For the sake of this example let us say \texttt{Dist} is determined as close enough for re-use; we can therefore continue with re-using the beliefs in the set $\mathcal{B}_{k+1|k}$ (Alg.~\ref{alg:ibsp} line~\ref{alg:ibsp:updatePlanning}). 
First we check which of the two available sampled measurements from planning time $k$ constitutes an  adequate representation for $\prob{z_{k+2}|H_{k+1|k+1},u_2}$. One way to do so, is following Alg.~\ref{alg:isRepSample} and checking whether the two available state samples from planning time $k$ constitute an adequate representation for $b^-[X_{k+2|k+1}]$; since they are, we consider all measurements associated to them as a representative set of $\prob{z_{k+2}|H_{k+1|k+1},u_2}$.
Our representative set of measurement samples for look ahead step $k+2$ now holds two re-used measurements, so we update their corresponding beliefs $\{b[X_{k+2|k}]\}_1^2$ into $\{b[X_{k+2|k+1}]\}_1^2$ (Alg.~\ref{alg:updatePrevPlanning} line~\ref{alg:updatePrevPlanning:updateBelief}), the updated beliefs are denoted in blue in Figure~\ref{fig:toyTimeK1}. After updating the beliefs we can calculate/ update the immediate rewards(costs) associated with them, see Section~\ref{ssubsec:rewardUpdate}, once obtained we can proceed to the next future time step. 

For the next look ahead step, we propagate $\{b[X_{k+2|k+1}]\}_1^2$ with action $u_1$ to obtain $\{b^-[X_{k+3|k+1}]\}_1^2$ (Alg.~\ref{alg:updatePrevPlanning} line~\ref{alg:updatePrevPlanning:propagate}), and check whether the four available measurement samples from planning time $k$ constitute an adequate representation for $\prob{z_{k+3}|H_{k+2|k+1},u_1}$ (Alg.~\ref{alg:updatePrevPlanning} line~\ref{alg:updatePrevPlanning:repSample}); following Alg.~\ref{alg:isRepSample} we find only three of them are, so we mark the associated beliefs for re-use, and sample the forth measurement from the original distribution $\prob{z_{k+3}|H_{k+2|k+1},u_1}$. We then update the beliefs we marked for re-use, $\{b[X_{k+3|k}]\}_1^3$ into $\{b[X_{k+3|k+1}]\}_1^3$ (denoted by the blue colored beliefs at $k+3|k+1$ in Figure~\ref{fig:toyTimeK1}), and $b^-[X_{k+3|k+1}]$ into $b^4[X_{k+3|k+1}]$ (denoted by the black colored belief at $k+3|k+1$ in Figure~\ref{fig:toyTimeK1}) using the newly sampled measurement (Alg.~\ref{alg:updatePrevPlanning} line~\ref{alg:updatePrevPlanning:updateBelief}). After obtaining the beliefs for look ahead step $k+3$, whether through updating a re-used belief or calculation from scratch, we calculate/ update the immediate rewards(costs) of each.
Since we do not have candidate beliefs to be re-used for the next time step, the last step of the horizon $k+4|k+1$ is calculated using \fullbsp (Alg.~\ref{alg:ibsp} line~\ref{alg:ibsp:lastSteps}). 

At this point we have all the immediate rewards for each of the predicted beliefs along the action sequence $u_2 \rightarrow u_1 \rightarrow u_2$, so we can calculate the expected reward value for this action sequence for planning at time $k+1$.
For the look ahead step $k+2$ of planning session at time $k+1$, i.e. $k+2|k+1$, we have two reward values, $\{r_{k+2|k+1}(b[X_{k+2|k+1}],u_2)\}_1^2$, each calculated for a different belief $b[X_{k+2|k+1}]$ considering a different sample $z_{k+2|k}$. Calculating the expected reward value for future time step $k+2|k+1$ would mean in this case, using measurements sampled from $\prob{z_{k+2}|H_{k+1|k},u_2}$ rather then from $\prob{z_{k+2}|H_{k+1|k+1},u_2}$. 
The use of Multiple Importance Sampling (MIS) enables us to calculate expectation while sampling from a mixture of probabilities, where the balance heuristic is used to calculate the weight functions.
Using the formulation of MIS along with the balance heuristic presented in Eq.~(\ref{eq:MIMBH_obj}), we can write down the estimation for the expected reward value of look ahead step $k+2|k+1$,
\begin{equation}\label{eq:toy1}
	\Expec \left[ r_{k+2|k+1}(.)\right] \sim \frac{1}{2}\frac{p_1(z^{1,1}_{k+2|k})}{\frac{2}{2}q_1(z^{1,1}_{k+2|k})}\cdot r^1_{k+2|k+1}(.)  + \frac{1}{2}\frac{p_1(z^{2,1}_{k+2|k})}{\frac{2}{2}q_1(z^{2,1}_{k+2|k})}\cdot r^2_{k+2|k+1}(.),
\end{equation}
where $p_1(.) \doteq \prob{z_{k+2}|H_{k+1|k+1},u_2}$ and $q_1(.) \doteq \prob{z_{k+2}|H_{k+1|k},u_2}$. 
In the same manner, following (\ref{eq:MIMBH_obj}), we can also write down the estimation for the expected reward value at look ahead step $k+3$ of planning at time $k+1$, i.e. $k+3|k+1$,
\begin{align}
	\nonumber
	\Expec \left[ r_{k+3|k+1}(.)\right] \sim 
	&\frac{1}{4}\frac{\tilde{p}_2(z^{1,1}_{k+2:k+3|k})}{\frac{3}{4}\tilde{q}_2(z^{1,1}_{k+2:k+3|k})+\frac{1}{4}\tilde{p}_2(z^{1,1}_{k+2:k+3|k})} r^1_{k+3|k+1}(.) \\ 
	\nonumber
	+ &\frac{1}{4}\frac{\tilde{p}_2(z^{2,1}_{k+2:k+3|k})}{\frac{3}{4}\tilde{q}_2(z^{2,1}_{k+2:k+3|k})+\frac{1}{4}\tilde{p}_2(z^{2,1}_{k+2:k+3|k})} r^2_{k+3|k+1}(.) \\
	\nonumber
	+ &\frac{1}{4}\frac{\tilde{p}_2(z^{3,1}_{k+2:k+3|k})}{\frac{3}{4}\tilde{q}_2(z^{3,1}_{k+2:k+3|k}) \!\!+ \!\!\frac{1}{4}\tilde{p}_2(z^{3,1}_{k+2:k+3|k})} r^3_{k+3|k+1}(.) \\
	\label{eq:toy2}
	  + &\frac{1}{4}\frac{\tilde{p}_2(z^{4,1}_{k+2:k+3|k+1})}{\frac{3}{4}\tilde{q}_2(z^{4,1}_{k+2:k+3|k+1}) + \frac{1}{4}\tilde{p}_2(z^{4,1}_{k+2:+3|k+1})} r^4_{k+3|k+1}(.),
\end{align}
where $\tilde{p}_2(.) \doteq \prob{z_{k+2:k+3}|H_{k+1|k+1},u_2,u_1}$ and $\tilde{q}_2(.) \doteq \prob{z_{k+2:k+3}|H_{k+1|k},u_2,u_1}$. When considering (\ref{eq:measLikelihood}), we can re-write the measurement likelihood from (\ref{eq:toy2}) into a product of measurement likelihoods per look ahead step, e.g. $\tilde{p}_2(z^1_{k+2:k+3|k}) = p_1(z^1_{k+2|k})p_2(z^1_{k+3|k})$, 
\begin{align}
	\nonumber
	\Expec \left[ r_{k+3|k+1}(.)\right] \sim 
	&\frac{1}{4}\frac{p_1(z^{1,1}_{k+2|k})p_2(z^{1,1}_{k+3|k})}{\frac{3}{4}q_1(z^{1,1}_{k+2|k})q_2(z^{1,1}_{k+3|k})+\frac{1}{4}p_1(z^{1,1}_{k+2|k})p_2(z^{1,1}_{k+3|k})} r^1_{k+3|k+1}(.)  \\
	\nonumber
	+  &\frac{1}{4}\frac{p_1(z^{1,1}_{k+2|k})p_2(z^{2,1}_{k+3|k})}{\frac{3}{4}q_1(z^{1,1}_{k+2|k})q_2(z^{2,1}_{k+3|k})+\frac{1}{4}p_1(z^{1,1}_{k+2|k})p_2(z^{2,1}_{k+3|k})} r^2_{k+3|k+1}(.) \\
	\nonumber
	 + &\frac{1}{4}\frac{p_1(z^{2,1}_{k+2|k})p_2(z^{3,1}_{k+3|k})}{\frac{3}{4}q_1(z^{2,1}_{k+2|k})q_2(z^{3,1}_{k+3|k})+\frac{1}{4}p_1(z^{2,1}_{k+2|k})p_2(z^{3,1}_{k+3|k})} r^3_{k+3|k+1}(.) \\
	 \label{eq:toy2+}
	  + &\frac{1}{4}\frac{p_1(z^{2,1}_{k+2|k})p_2(z^{4,1}_{k+3|k})}{\frac{3}{4}q_1(z^{2,1}_{k+2|k})q_2(z^{4,1}_{k+3|k})+\frac{1}{4}p_1(z^{2,1}_{k+2|k})p_2(z^{4,1}_{k+3|k})} r^4_{k+3|k+1}(.),
\end{align}
where $p_1(.)$ need not be calculated at look ahead step $k+3$, since it is already given from (\ref{eq:toy1}).

\subsection{The WildFire condition}\label{subsec:wf}
One of our main contributions in this work is the \wf condition, which allows one to sacrifice statistical accuracy in favor of reduction in computation time by re-using an entire subtree from a previous planning session without any update. 
As explained in the opening of Section~\ref{sec:approach}, at each lookahead step instead of calculating the predicted beliefs from scratch, we locate and update some previously calculated beliefs which are close enough. This belief update is more efficient, in terms of computation time, than performing inference from scratch over future measurements, hence the computation time advantage of \ibsp over \fullbsp. The \wf condition allows to refrain even from updating the belief, thus sacrificing statistical accuracy in favor of further reduction in computation time of the planning process under \ibsp. 

The \wf condition is a distance condition between two beliefs, used to check if one is close to the other up to some predetermined value $\epsilon_{wf}$. When the condition is met, the entire subtree rooted in the considered belief is re-used "as is" without any additional update.
Under our problem of \ibsp, we consider previously calculated beliefs from planning at time $k$ for re-use at time $k+l$ by checking the information gap between the two planning times for each planning horizon $i \in [k+l, k+L]$.

This section covers the \wf condition, starting with the intuition behind it and its working principle (Section~\ref{ssubsec:wf:intuition}), how it is integrated within \ibsp (Section~\ref{ssubsec:wf:inIXBSP}), and concluding with formulating bounds for the objective value error under the use of \wf (Section~\ref{ssubsec:wfBounds}).
For the reader's convenience all the notations of this section are summarized in Table~\ref{table:wf}.

\subsubsection{Intuition and Working Principle}\label{ssubsec:wf:intuition}

Let us assume we found a belief from a precursory planning session $b[X_{i+1|k}] \in \mathcal{B}_{k+l|k}$ that is identical to the belief we would like to calculate $b^s[X_{i+1|k+l}]$;  these beliefs would, of course,  yield zero distance, 
\begin{equation}
	\mathbb{D}(b[X_{i+1|k}],b^s[X_{i+1|k+l}]) = 0.
\end{equation}
For this case, instead of solving Eq.~(\ref{eq:reward:scratch}) in order to obtain $r^s_{i+1|k+l}$, we can simply use the previously calculated immediate reward associated with $b[X_{i+1|k}]$
\begin{equation}
	r^s_{i+1|k+l} = r_{i+1}\left(b^s[X_{i+1|k+l}], u_{i|k+l}\right) \equiv r_{i+1}\left(b[X_{i+1|k}], u_{i|k}\right),
\end{equation}
not only we re-use this immediate reward but we can simply re-use the entire sub-tree rooted in $b[X_{i+1|k}]$ as is. 
Under the \wf condition we consider an approximation to the immediate reward value $r^s_{i+1|k+l}$, by using the immediate reward value of a previously calculated belief $b[X_{i+1|k}]$ which is $\epsilon_{wf}$ close to $b^s[X_{i+1|k+l}]$ in the $\mathbb{D}(.)$ sense, i.e.
\begin{equation}\label{eq:reward:wf}
	r^s_{i+1|k+l} = r_{i+1}\left(b^s[X_{i+1|k+l}], u_{i|k+l}\right) \approxeq  r_{i+1}\left(b[X_{i+1|k}], u_{i|k}\right),
\end{equation}
where
\begin{equation}
	\mathbb{D}(b[X_{i+1|k}],b^s[X_{i+1|k+l}]) \leq \epsilon_{wf}.
\end{equation}
As such the immediate reward approximation error under the use of \wf is given by
\begin{equation}
	err_r = \mid r_{i+1}\left(b^s[X_{i+1|k+l}], u_{i|k+l}\right) - r_{i+1}\left(b[X_{i+1|k}], u_{i|k}\right) \mid
\end{equation}
and as claimed in Theorem~\ref{th:rewardBound}, it can be bounded by our choice of $\epsilon_{wf}$. 

\begin{theorem}[Bounded reward difference]\label{th:rewardBound}
	Let $r(b,u)$ be \aH continuous with $\lambda_{\alpha}$ and $\alpha \in (0,1]$. Let $b$ and $b'$ denote two beliefs. Then the difference between $r(b,u)$ and $r(b',u)$ is bounded by
	\begin{equation}
		\mid r(b,u) - r(b',u) \mid \leq \left(4 \cdot ln 2\right)^{\frac{\alpha}{2}} \cdot \lambda_{\alpha} \cdot  \JD^{\alpha}(b,b'),
	\end{equation}
	where 
	\begin{equation}
		\JD(b,b') = \sqrt{\frac{1}{2}\mathbb{D}_{KL}(b||b') + \frac{1}{2}\mathbb{D}_{KL}(b'||b)},
	\end{equation}
	and $\mathbb{D}_{KL}(.)$ is the KL divergence.
\end{theorem}
\begin{proof}
	See Appendix C.
\end{proof}
Based on Theorem~\ref{th:rewardBound}, for \aH continuous reward function with parameters $\{ \lambda_{\alpha}, \alpha \}$, and the $\JD$ distance, we can bound the immediate reward value error caused due to the use of \wf  
\begin{equation}
	err_r \leq \left(4 \cdot ln 2\right)^{\frac{\alpha}{2}} \cdot \lambda_\alpha \cdot \epsilon_{wf}^{\alpha}.
\end{equation}
%
%
By selecting to re-use a belief with information gap no larger than $\epsilon_{wf}$, we sacrifice some of the statistical accuracy of our estimate in favor of a substantial save in computation time. As explained in Section~\ref{subsec:compPlanning} and illustrated by Figure~\ref{fig:horizonOverlap}, an information gap might occur due to inaccurate predictions for lookahead steps $k+1:k+l$, unrepresentative measurements for lookahead steps $k+l+1:k+L$, or some combination of them. Beliefs from planning at time $k$ that are being re-used under the \wf condition, are not being updated to match the information of planning at time $k+l$, or to have representative future measurements for lookahead steps $k+l+1:k+L$. Instead, they are considered as close enough,  thus taken "as is" and the entire process of belief update (Section~\ref{ssubsec:beliefUpdate}) is everted.
The \wf condition is passed from one belief to all of its descendants along the horizon, e.g.~if $b[X_{i|k}]$ has been flagged as meeting the \wf condition all beliefs originated in $b[X_{i|k}]$ would also be flagged as meeting the \wf condition without calculating any distances, i.e. the entire subtree rooted at $b[X_{i|k}]$  is taken "as is", without any re-calculations, as discussed below.

It is worth stressing that the \wf condition is a non-integral part of \ibsp, and as such, it is up to the user to decide wether to sacrifice statistical accuracy in favor of computation time or not. Moreover it is up to the user to decide how much sacrifice he or she are willing to make by adjusting the \wf condition accordingly $ \epsilon_{wf} \in [0,\epsilon_{c}]$.
Choosing a \wf threshold value of $\epsilon_{wf}=0$, does not yield any immediate reward error, but can still save computation time in instances where an identical belief is available from a previous planning session, e.g. BSP under MPC framework with no new available observations from the environment. 

\begin{table}[h]
	\caption{Notations for Section~\ref{subsec:wf}} 
	\centering 
	\begin{tabular}{c c} 
		\hline\hline 
		\textbf{Variable} & \textbf{Description}  \\ [0.5ex] 
		\hline 
		$ \Box_{t|k}$ & Of time $t$ while current time is $k$ \\[1ex]
		$b[X_{t|k}]$ & belief at time $t$ while current time is $k$ \\[1ex]
		$b^s[X_{t|k}]$ & The $s_{th}$ sampled belief representing $b[X_{t|k}]$\\[1ex]
		$\overset{\sim}{b}[X_{t|k}]$ & The root of the selected branch for re-use from planning time $k$ \\[1ex]
		$\mathcal{B}_{t|k}$ & The set of all beliefs from planning time $k$ rooted in $\overset{\sim}{b}[X_{t|k}]$\\[1ex]	
		$b^{s-}_{\alpha}[X_{t+1|k+l}]$ & The sampled belief $b^s[X_{t|k+l}]$ propagated with the $\alpha$ candidate action\\[1ex]
		$b^{s'-}_{\alpha}[X_{t+i|k}]$ & A propagated belief from $\mathcal{B}_{t|k}$ closest to $b^{s-}_{\alpha}[X_{t+i|k+l}]$\\[1ex]
		$\epsilon_c$ & belief distance critical threshold, max distance for re-use computational advantage \\[1ex]
		$\epsilon_{wf}$ & \wf threshold, max distance to be considered as close-enough for re-use without any update \\[1ex]
		$\mathbb{D}(.)$ & belief divergence / metric \\[1ex]
		$\mathbb{D}^2(.)$ & squared $\mathbb{D}(.)$ \\[1ex]
		$\JD(p,q)$  & The distance between distributions $p$ and $q$ according to the \JD distance \\ [1ex]
		$\lambda_{\alpha}$ & the reward function \aH constant  \\ [1ex]
		$\alpha$ & the reward function \aH exponent  \\ [1ex]
		$r^s_{t|k}$ & the immediate reward at lookahead step $t$, related to $b^s[X_{t|k}]$ \\ [1ex]
		\texttt{Dist} & The distance between $\overset{\sim}{b}[X_{t|k}]$ and the corresponding posterior $b[X_{t|t}]$ \\[1ex]
		\texttt{dist} & The distance between $b^{s'-}_{\alpha}[X_{t|k}]$ and $b^{s-}_{\alpha}[X_{t|k+l}]$ \\[1ex]
		$\Delta$ &  equals $\mathbb{D}^2(b^+_1,b^+_2) - \mathbb{D}^2(b_1,b_2)$, where $b_{ip}$ denotes $b_i$ propagated with motion and measurements \\ [1ex] 
		\hline\hline
	\end{tabular}
	\label{table:wf} 
\end{table}
\subsubsection{Using \wf within \ibsp}\label{ssubsec:wf:inIXBSP}
We will now meticulously demonstrate how the \wf condition is integrated within the \ibsp paradigm. There are two different places within \ibsp in-which the \wf condition is used, the first is just after selecting the closest branch for re-use (Alg.~\ref{alg:ibsp} line~\ref{alg:ibsp:wfCond}), and the second is part of re-using existing beliefs (Alg.~\ref{alg:updatePrevPlanning} line~\ref{alg:updatePrevPlanning:checkEps_wf}).  

Let us assume we have just located the closest branch for re-use from planning at time $k$ (Alg.~\ref{alg:ibsp} line~\ref{alg:ibsp:selectBranch}), where \texttt{Dist} is the distance between our last posterior $b[X_{k+l|k+l}]$ and its counterpart from planning time $k$ $\tilde{b}[X_{k+l|k}]$, which is also the root of the selected branch. The value of \texttt{Dist} represents the information gap between the current posterior belief $b[X_{k+l|k+l}]$, and the appropriate closest prediction to it from a precursory planning session,  $\tilde{b}[X_{k+l|k}]$. In this specific case, the information gap represents how well were the predictions at planning time $k$ for lookahead steps $k+1:k+l$, the closer they were to what actually happened, the smaller the information gap as well as \texttt{Dist} value.
 If the information gap is not too big, i.e. $\texttt{Dist} \leq \epsilon_c$ (Alg.~\ref{alg:ibsp} line~\ref{alg:ibsp:distCond}), we say that the selected candidate branch is re-use worthy; if the information gap also meets the \wf condition, i.e. $\texttt{Dist} \leq \epsilon_{wf}$ (Alg.~\ref{alg:ibsp} line~\ref{alg:ibsp:wfCond}), we say that the information gap is negligible, thus the entire selected branch can be re-used "as is". In case of the latter, we re-use the entire closest branch $\mathcal{B}_{k+l|k}$ rooted at $\tilde{b}[X_{k+l|k}]$, and continue to complete the rest of the lookahead steps with \fullbsp. Because the beliefs were re-used without any update, there is no need to re-calculate the appropriate immediate reward values, available from planning at time $k$. 
 
 Let us now assume that the information gap was not too big, but also did not meet the \wf condition, i.e. $\epsilon_{wf} < \texttt{Dist} \leq \epsilon_c$, that scenario takes us to the second use of the \wf condition in \ibsp, as part of re-using existing beliefs Alg.~\ref{alg:updatePrevPlanning}.
 Because the \wf condition is passed from one belief to its descendants, we always start by checking if a candidate belief has inherited a \wf flag from its ancestor (Alg.~\ref{alg:updatePrevPlanning} line~\ref{alg:updatePrevPlanning:copyWF}). In case it did, we automatically consider it as meeting the \wf condition, and flag its immediate children as such as well. In case a belief is not already flagged as meeting the \wf condition we are required to check it. 
 We consider the $s_{th}$ belief at lookahead step $t=i$, i.e. $b^s[X_{i|k+l}]$, and propagate it with action $\alpha$ to obtain $b^{s-}_{\alpha}[X_{i+1|k+l}]$ (Alg.~\ref{alg:updatePrevPlanning} line~\ref{alg:updatePrevPlanning:propagate}). We locate the closest propagated belief to $b^{s-}_{\alpha}[X_{i+1|k+l}]$, denote it as $b^{s'-}_{\alpha}[X_{i+1|k}]$ and the distance between them as \texttt{dist}.
In this case, the information gap represented by \texttt{dist}, consists of the gap represented by \texttt{Dist} as well as the possibly different predicted measurements for lookahead steps $(k+l+1:i)$, e.g. area~(iii) in Figure~\ref{fig:horizonOverlap}. Because the propagated belief is used to generate predicted measurements (see Alg.~\ref{alg:sampling_z}), a small enough value of \texttt{dist} would improve the chances to obtain a representative set of samples (as discussed in Section~\ref{ssubsec:repSample}).
 
 If the information gap is not too big, i.e. $\texttt{dist} \leq \epsilon_c$ (Alg.~\ref{alg:updatePrevPlanning} line~\ref{alg:updatePrevPlanning:checkEps_c}), we consider the previously sampled measurements associated to $b^{s'-}_{\alpha}[X_{i+1|k}]$ as candidates for a representative set of measurements (Alg.~\ref{alg:updatePrevPlanning} lines~\ref{alg:updatePrevPlanning:gatherSamples}-\ref{alg:updatePrevPlanning:repSample}). If the information gap also meets the \wf condition, i.e. $\texttt{dist} \leq \epsilon_{wf}$ , we consider $b^{s-}_{\alpha}[X_{i+1|k+l}]$ and $b^{s'-}_{\alpha}[X_{i+1|k}]$ as close enough such that all sampled measurements associated with the latter are representative of the former (Alg.~\ref{alg:updatePrevPlanning} line~\ref{alg:updatePrevPlanning:useAsWF}). Consequently, we consider all beliefs descendant of $b^{s'-}_{\alpha}[X_{i+1|k}]$ as meeting the \wf condition, and as such they are re-used "as is" without any update.

\subsubsection{Objective Value Bounds under \wf}\label{ssubsec:wfBounds} 
Under the use of \wf, \ibsp is not necessarily an exact solution of the BSP problem, but a possible approximation. As such, we would like to get a bound over the resulting objective value. In the following we show that under the assumption of \aH rewards and the use of \JD distance (Appendix~\ref{app:Distance}) the immediate reward value under \wf is bounded (Theorem~\ref{th:rewardBound}) by a random variable. We continue with showing that the corresponding objective value is also bounded (Theorem~\ref{th:JacobBound}) by a (different) random variable, where if the distribution of this random variable is explicitly given, a corresponding bound can be formulated (Theorem~\ref{th:JacobBoundProb}). We conclude with showing (Corollary~\ref{cor:linModelBounds}) that for the case of linear Gaussian models, one can explicitly calculate the moments of the aforementioned random variables. 
For the reader's convenience Figure~\ref{fig:wf:proofScheme} illustrates the workflow of the supplied proofs as well as the dependency of each segment over the two involved assumptions, while bolding the final results - Theorem~\ref{th:JacobBound}, Theorem~\ref{th:JacobBoundProb}, and Corollary~\ref{cor:linModelBounds}. 
It is worth reiterating that the purpose of the supplied bounds is to reflect the direct correlation between $\epsilon_{wf}$ and the objective value. More work is required in order to make the bounds convenient enough for online usage, e.g. dynamically updating the $\epsilon_{wf}$ value - we leave this for future work.
\begin{figure}[]
	\centering
	\includegraphics[trim={0 0 0 0},clip, width=0.90\columnwidth]{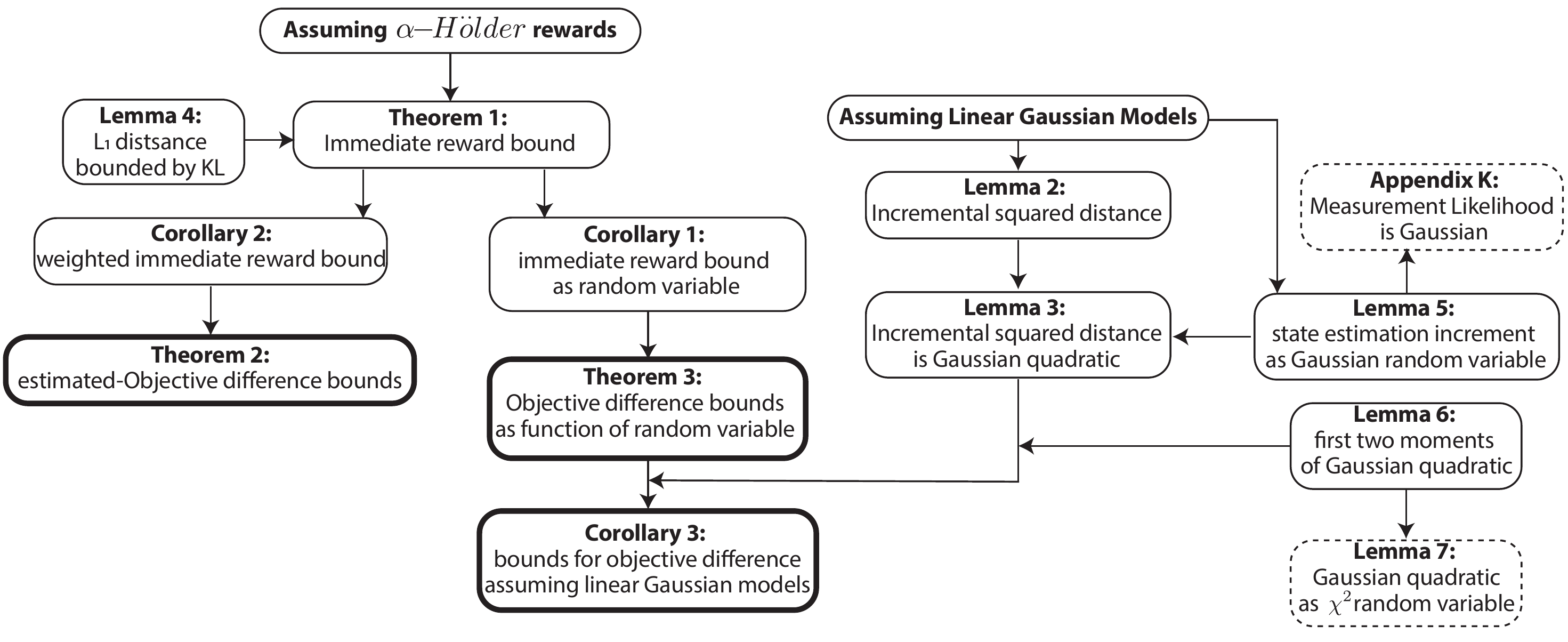}
	\caption{ Illustration of the layout for finding the objective value bounds under \wf. The bolded rectangles denote three variations for bounds over the objective value under \wf. Theorem~\ref{th:JacobBound} provides the objective bound when calculated using samples rather than expectation, whereas Theorem~\ref{th:JacobBoundProb} provides the bound when explicitly solving the expectation (assuming the probability is known). Corollary~\ref{cor:linModelBounds} demonstrate how to explicitly calculate the bound provided in Theorem~\ref{th:JacobBoundProb} assuming linear Gaussian models. The dotted rectangles denote additional non-integral insights, deduced along the way.} 
	\label{fig:wf:proofScheme}
\end{figure}

Under the assumption of general \aH rewards $r(b,u)$, we can get a bound over the difference between two immediate reward functions of the same action and different beliefs, as stated in  Theorem~\ref{th:rewardBound}.

\begin{corollary}[of Theorem~\ref{th:rewardBound}]\label{cor:rvRewardBound}
	Let $r(b,u)$ be \aH continuous with $\lambda_{\alpha}$ and $\alpha \in (0,1]$. Let $b$ and $b'$ denote two future beliefs. Then the bounded difference between $r(b,u)$ and $r(b',u)$ is a random variable.
\end{corollary}
\begin{proof}
	Using Theorem~\ref{th:rewardBound}, the bound is given by 
	\begin{equation}
		\mid r(b,u) - r(b',u) \mid \leq \left(4 \cdot ln 2\right)^{\frac{\alpha}{2}} \cdot \lambda_{\alpha} \cdot  \JD^{\alpha}(b,b'),
	\end{equation}
	The \JD distance is a function of future beliefs $b$ and $b'$.
	The future beliefs $b$ and $b'$ are a function of future measurements. Future measurements are a random variable distributed according to the appropriate measurement likelihood $\prob{z|H^-}$. As a function of random variables, \JD is also a random variable.   
\end{proof}

\begin{corollary}[of Theorem~\ref{th:rewardBound}]\label{cor:weightedBound}
	Let $r(b,u)$ be \aH continuous with $\lambda_{\alpha}$ and $\alpha \in (0,1]$. Let $b$ and $b'$ denote two beliefs. Let $\omega_i$ denote a positive weight, such that $0 \leq \omega_i \leq 1 \ , \ i \in \{1,2\}$. Then the weighted difference between $r(b,u)$ and $r(b',u)$ is given by
	\begin{equation}
		(\omega_1-\omega_2) r(b') - \omega_1 \lambda_{\alpha}\left(4 \cdot ln 2\right)^{\frac{\alpha}{2}} \JD^{\alpha}(b,b') \leq  \omega_1 r(b) - \omega_2r(b') \leq \omega_1 \lambda_{\alpha}\left(4 \cdot ln 2\right)^{\frac{\alpha}{2}} \JD^{\alpha}(b,b')  + (\omega_1-\omega_2) r(b')
 	\end{equation}	
\end{corollary}
\begin{proof}
	see Appendix~\ref{app:weightedBound}.
\end{proof}
As the objective function is defined by the sum of expected rewards along the planning horizon, we are now in a position to provide a bound over the difference between two objective values considering the same action sequence yet different beliefs. 
\begin{theorem}\label{th:JacobBound}
	Let r(b,u) be \aH continuous with $\lambda_{\alpha}$ and $\alpha \in (0,1]$. Let $J_{k+l|k+l}$ and $J_{k+l|k}$ be objective values of the same time step $k+l$, calculated based on information up to time $k+l$ and $k$ respectively. Let $L$ be a planning horizon such that $L \geq l+1$. Let $n_i$ be the number of samples used to estimate the expected reward at lookahead step $i$.  Let $\omega_i^j$ be non-negative weights such that $0 \leq \omega_i^j \leq 1$ and $\sum_{j=1}^{n_i} \omega_i^j = 1$. Then the difference $(J_{k+l|k+l} -J_{k+l|k})$ is bounded by
	\begin{equation}
		\sum_{i = k+l+1}^{k+L} \sum_{j = 1}^{n_i} \omega_{i|k+l}^{j} \left[r^j_{i|k} - \mathcal{D}_{i}^{j} \right] - J_{k+l|k} \leq J_{k+l|k+l} - J_{k+l|k} \leq \sum_{i = k+l+1}^{k+L} \sum_{j = 1}^{n_i} \omega_{i|k+l}^{j} \left[ r^j_{i|k} + \mathcal{D}_{i}^{j}\right] - J_{k+l|k}.
	\end{equation}
	where
	\begin{equation}
		\mathcal{D}_{i}^{j} = \lambda_{\alpha}\left(4 \cdot ln 2\right)^{\frac{\alpha}{2}} \JD^{\alpha}(b^j[X_{i|k+l}],b^j[X_{i|k}]).
	\end{equation}
\end{theorem}
\begin{proof}
	see Appendix~\ref{app:JacobBound}.	
\end{proof}
In order to calculate the bound, we need to sample a set of future measurements, and any specific realization of such a set may result with a different bound value altogether.
In case the probability function of the bound in Corollary~\ref{cor:rvRewardBound} can be obtained, there is no call for using samples as in Theorem~\ref{th:JacobBound}, and the bound over the objective value can be analytically calculated as suggested by Theorem~\ref{th:JacobBoundProb}. In other words, while Theorem~\ref{th:JacobBound} offers a sample based estimation for the objective error bound, Theorem~\ref{th:JacobBoundProb} offers the un-approximated formulation. For $n_i \rightarrow \infty$, both theorems will provide exactly the same bound. 
%
\begin{theorem}\label{th:JacobBoundProb}
	Let r(b,u) be \aH continuous with $\lambda_{\alpha}$ and $\alpha \in (0,1]$. Let $J_{k+l|k+l}$ and $J_{k+l|k}$ be objective values of the same time step $k+l$, calculated based on information up to time $k+l$ and $k$ respectively. Let $L$ be a planning horizon such that $L \geq l+1$. Then the difference $(J_{k+l|k+l} -J_{k+l|k})$ is bounded by
	\begin{equation}
		\phi - \psi \leq J_{k+l|k+l} - J_{k+l|k} \leq \phi + \psi,
	\end{equation}
	where
	\begin{eqnarray}
		\phi &\triangleq& \sum_{i = k+l+1}^{k+L} \Expec_{z \sim p_k} (\omega-1) r_{i|k} 
		\\
		\psi &\triangleq& \lambda_{\alpha}\left(4 \cdot ln 2\right)^{\frac{\alpha}{2}} \left[ (L-l)\epsilon_{wf}^{\alpha} + \sum_{i = k+l+1}^{k+L} \left( \sum_{j=k+l+1}^{i} \Expec_{z \sim p_{k+l}} \Delta_j \right)^{\frac{\alpha}{2}} \right],
	\end{eqnarray}
    and
    \begin{eqnarray}
    	\JD(b[X_{i|k+l}],b[X_{i|k}]) &=& \sqrt{ \JD^2(b[X_{i-1|k+l}],b[X_{i-1|k}]) + \Delta_{i}},
    	\\
    	\epsilon_{wf} &=& \JD(b[X_{k+l|k+l}],b[X_{k+l|k}]),	
    	\\
    		\omega &=& \frac{\prob{z_{k+l+1:k+L}|H_{k+l|k+l},u_{k+l:k+L-1}}}{\prob{z_{k+l+1:k+L}|H_{k+l|k},u_{k+l:k+L-1}}} \triangleq \frac{p_{k+l}}{p_{k}}.
    \end{eqnarray}
\end{theorem}
\begin{proof}
	see Appendix~\ref{app:JacobBoundProb}.	
\end{proof}
We can see that in the bound suggested by Theorem~\ref{th:JacobBoundProb}, even for $\epsilon_{wf}=0$ we are still left with a stochastic expression, depending on the measurement likelihood probabilities, whereas for identical measurement likelihoods, i.e. $\omega =1$, as well as identical beliefs, we expect zero difference between the objective values. 

\subsubsection*{Assuming Linear Gaussian Models}
As discussed earlier $\Delta$ is a function of the future measurements and as such it is a random variable. In the following we explicitly calculate the first moment of $\Delta$ required for calculating the bound of Theorem~\ref{th:JacobBoundProb}, under the simplifying assumption of linear Gaussian models 
	\begin{equation}\label{eq:wf:linMotionModel}
		x' = \mathcal{F}x + \mathcal{J}u + w \quad , \quad  w\sim \mathcal{N}(0,\Sigma_w),
	\end{equation}
	\begin{equation}\label{eq:wf:linMeasurementModel}
		z = \mathcal{H}x + v\quad , \quad  v\sim \mathcal{N}(0,\Sigma_v),
	\end{equation}
	where $\mathcal{F}$ and $\mathcal{J}$ are the motion model jacobian in respect to the state and action appropriately, $\mathcal{H}$ is the measurement jacobian, and $w$ and $v$ are zero mean additive gaussian noises.

 
\begin{lemma}[Incremental \JD distance]\label{lm:incDist}
	Let $b_1$ and $b_2$ be two Gaussian beliefs  $\mathcal{N}(\mu_1 , \Sigma_1)$ and $\mathcal{N}(\mu_{2} , \Sigma_2)$, respectively with state dimension $d$, and their two (differently) propagated counterparts $b_{1p}$ and $b_{2p}$ with $\mathcal{N}(\mu_{1p} , \Sigma_{1p})$  $\mathcal{N}(\mu_{2p} , \Sigma_{2p})$ and with state dimension $d_{p}$. There exist $\zeta_i $ and $A_i$ such that the propagated mean and covariance are given by
	\begin{equation}\label{eq:appB:def}
		\mu_{ip} = \mu_i + \zeta_i \quad , \quad 
		\Sigma_{ip} = \left( \Sigma_i^{-1} + A_i^T A_i \right)^{-1} \quad , \quad 
		i \in [1,2].
	\end{equation}
	Then the squared \JD distance between the propagated beliefs can be written as 
	\begin{equation}
	\JD^2 \left( b_{1p},b_{2p}\right) = \JD^2 \left( b_1,b_2\right) + \Delta,
\end{equation}
where 
\begin{multline}\label{eq:deltaLemma}
	\Delta = \frac{1}{4}(\mu_2 - \mu_1)^T \left[ A_2^T A_2 + A_1^T A_1 \right](\mu_2 - \mu_1) + 
	\frac{1}{2}(\mu_2 - \mu_1)^T \left[ {\Sigma_{1p}}^{-1} + {\Sigma_{2p}}^{-1} \right](\zeta_2 - \zeta_1) \\
	+ \frac{1}{4}(\zeta_2 - \zeta_1)^T \left[ {\Sigma_{1p}}^{-1} + {\Sigma_{2p}}^{-1} \right](\zeta_2 - \zeta_1)
	+ \frac{1}{4}tr\left(  A_2^T A_2 \Sigma_{1p} - \Sigma_2^{-1} \Sigma_1 A_1^T ( I + A_1 \Sigma_1 A_1^T)^{-1} A_1 \Sigma_1 \right) \\ 
	+ \frac{1}{4}tr\left(  A_1^T A_1 \Sigma_{2p} - \Sigma_1^{-1} \Sigma_2 A_2^T ( I + A_2 \Sigma_2 A_2^T)^{-1} A_2 \Sigma_2 \right) - \frac{1}{2}\left( d_{p} - d\right).
\end{multline}
\end{lemma}
\begin{proof}
	see Appendix~\ref{app:iDistance}.	
\end{proof}

Next we make the observation that Eq.~(\ref{eq:deltaLemma}) is a quadratic form of a Multivariate Gaussian vector,
\begin{lemma}[Incremental \JD distance as Gaussian Quadratic]\label{lm:deltaAsGQ}
		Let $b_1$ and $b_2$ be two Gaussian beliefs  $\mathcal{N}(\mu_1 , \Sigma_1)$ and $\mathcal{N}(\mu_{2} , \Sigma_2)$, respectively with state dimension $d$, and their two propagated counterparts $b_{1p}$ and $b_{2p}$ with $\mathcal{N}(\mu_{1p} , \Sigma_{1p})$  $\mathcal{N}(\mu_{2p} , \Sigma_{2p})$ and with state dimension $d_{p}$. 
		There exist $\zeta_i $ and $A_i$ such that the propagated mean and covariance are given by, 
		\begin{equation}
			\mu_{ip} = \mu_i + \zeta_i \quad , \quad 
			\Sigma_{ip} = \left( \Sigma_i^{-1} + A_i^T A_i \right)^{-1} \quad , \quad 
			i \in [1,2].
		\end{equation}
		Then the incremental \JD distance $\Delta \triangleq \JD^2 \left( b_{1p},b_{2p}\right) - \JD^2 \left( b_1,b_2\right)$ is a quadratic form of a gaussian variable.
\end{lemma}
\begin{proof}
	see Appendix~\ref{app:deltaAsGQ}.
\end{proof} 

We are now in position to explicitly calculate the bounds of Theorem~\ref{th:JacobBoundProb} under the assumption of linear Gaussian models,
\begin{corollary}[of Theorem~\ref{th:JacobBoundProb}]\label{cor:linModelBounds}
	Let r(b,u) be \aH continuous with $\lambda_{\alpha}$ and $\alpha \in (0,1]$. Let $J_{k+l|k+l}$ and $J_{k+l|k}$ be objective values of the same time step $k+l$, calculated based on information up to time $k+l$ and $k$ respectively. Let $L$ be a planning horizon such that $L \geq l+1$. Let the motion and measurement models be linear with additive Gaussian noise (\ref{eq:wf:linMotionModel})-(\ref{eq:wf:linMeasurementModel}). Then the bound of $(J_{k+l|k+l} -J_{k+l|k})$ can be explicitly calculated.
	\end{corollary}
\begin{proof}
	see Appendix~\ref{app:linModelJacobBaound}.
\end{proof}

\subsection{Incremental \mlbsp}\label{subsec:iML}
Seeing that our novel approach changes the solution paradigm for the original, un-approximated, problem (\fullbsp), we claim it could be utilized to also reduce computation time of existing approximations of \fullbsp.
To support our claim, in this section we present the implementation of \ibsp principles over a commonly used approximation, \mlbsp, and denote it as \imlbsp. 

Under the ML assumption we consider just the most likely measurements, rather than sampling multiple measurements; hence, there is a single importance sampling distribution at each planning step i.e. $M_i = 1$ $\forall i$, because a single measurement is considered for each action at each time step. Considering the aforementioned, for the case of \imlbsp, Eq.~(\ref{eq:MIMBH_obj}) is reduced to  
\begin{equation}\label{eq:MIMBH_obj_iML}
	J^{iML}(u') \approx 
	\sum_{i=k+l+1}^{k+l+L} \left[ w_{i} \cdot r_{i}\left( b[X_{i|k+l}],u'_{i-1|k+l} \right) \right],
\end{equation}
and Eq.~(\ref{eq:MIM_w}), representing the weight at time $i$ is reduced to 
\begin{equation}\label{eq:MIM_w_iML}
	w_i = \frac{\prob{z_{k+l+1:i}|H_{k+l|k+l},u_{k+l:i-1|k+l} }}{q(z_{k+l+1:i})},
\end{equation}
where $q(.)$ is the importance sampling distribution related to the ML measurement set $z_{k+l+1:i}$, and $\prob{z|H,u}$ is the nominal distribution from which the ML measurement set $z_{k+l+1:i}$ should have been taken from. 
When we take the ML measurements from the nominal distribution, i.e. as in \mlbsp, $\omega_i=1 \ , \forall i$, and Eq.~(\ref{eq:MIMBH_obj_iML}) is identical to Eq.~(\ref{eq:objectiveML}). 
More specifically without loosing generality, when considering the planning tree from planning at time $k$ as candidate for re-use, the possible values of the importance sampling distribution are  
  \begin{equation}\label{eq:distanceCases_iML}
  	q(z_{k+l+1:i}) = \begin{cases}
  		\prob{z_{k+l+1:i}|H_{k+l|k},u_{k+l:i-1|k}} &  \mathbb{D}(b^-[X_{i|k}],b^-[X_{i|k+l}]) \ \leq \ \epsilon_{c} \\
  		\prob{z_{k+l+1:i}|H_{k+l|k+l},u_{k+l:i-1|k+l}} & \epsilon_{c} \ < \ \mathbb{D}(b^-[X_{i|k}],b^-[X_{i|k+l}])\\
  	\end{cases},
  \end{equation}
were $b[X_{i|k}]$ is the belief from planning at time $k$ considered for re-use, $\prob{z_{k+l+1:i}|H_{k+l|k},u_{k+l:i-1|k}}$ is the probability used to sample future measurements considered in $b[X_{i|k}]$, $\prob{z_{k+l+1:i}|H_{k+l|k+l},u_{k+l:i-1|k+l}}$ is the nominal probability used to sample future measurements considered in $b[X_{i|k+l}]$, $\epsilon_c$ is the re-use threshold (see Section~\ref{ssubsec:closeEnough}), and $\mathbb{D}(.)$ denote some belief distance.

\emph{Remark:} By considering an \imlbsp session as a single rollout, it can be extended to rollout-based planners with a belief dependent reward in a straightforward manner; we leave this for future work.

\vspace{-5pt}
\section{Results}
\label{sec:results}
In order to examine the effect of calculation re-use under the \ibsp paradigm, we compare the runtime of \ibsp and \fullbsp using Active full SLAM as a test-bed under Model Predictive Control (MPC) framework. 
To better understand the differences between \fullbsp and \ibsp, let us consider them inside a plan-act-infer system. Figure~\ref{fig:HighLevel_Algo:std} illustrates the high level algorithm for plan-act-infer using \fullbsp, marking the section of the algorithm which is being timed for comparison. Figure~\ref{fig:HighLevel_Algo:iBSP} illustrates the high level algorithm for plan-act-infer using \ibsp, marking the section of the algorithm which is timed for comparison. As can easily be seen in Figure~\ref{fig:HighLevel_Algo} all differences between \fullbsp and \ibsp are confined within the planning block, hence the computation time of the planning process is adequate for fair comparison.
It is important to mention that no offline calculations whatsoever, are involved in any of the comparisons. For simplicity all results consider known motion and observation models with zero mean Gaussian noise as well as motion primitives. Both \fullbsp and \ibsp consider a planning horizon of 3 steps, 3 candidate actions (forward, left and right), with all the possible permutations between them - hence 27 candidate action sequences, $n_x = 5$ $n_z = 1$, 6 DOF robot pose, 3 DOF landmarks and a joint state comprised of both robot poses and landmarks.

In the following we provide a statistical comparison between \mlbsp and \fullbsp (Section~\ref{subsec:results:ML_Ex}); statistical comparison between \fullbsp and \ibsp (Section~\ref{subsec:results:pre_ibsp}) under a simplifying assumption that all previously sampled measurements can be re-used in current planning time; a statistical comparison between \fullbsp and \ibsp with selective re-sampling (Section~\ref{subsec:results:ibsp}); statistical analysis of the \wf addition to \ibsp (Section~\ref{subsec:results:wf}); a statistical comparison between \imlbsp and \mlbsp (Section~\ref{subsec:results:iml}); and real-world experiments comparing \imlbsp and \mlbsp (Section~\ref{subsec:results:real_iml}).


\begin{figure}[h!]
	\centering
        \subfloat[]{\includegraphics[trim={0 0 275 10},clip, width=0.25\columnwidth]{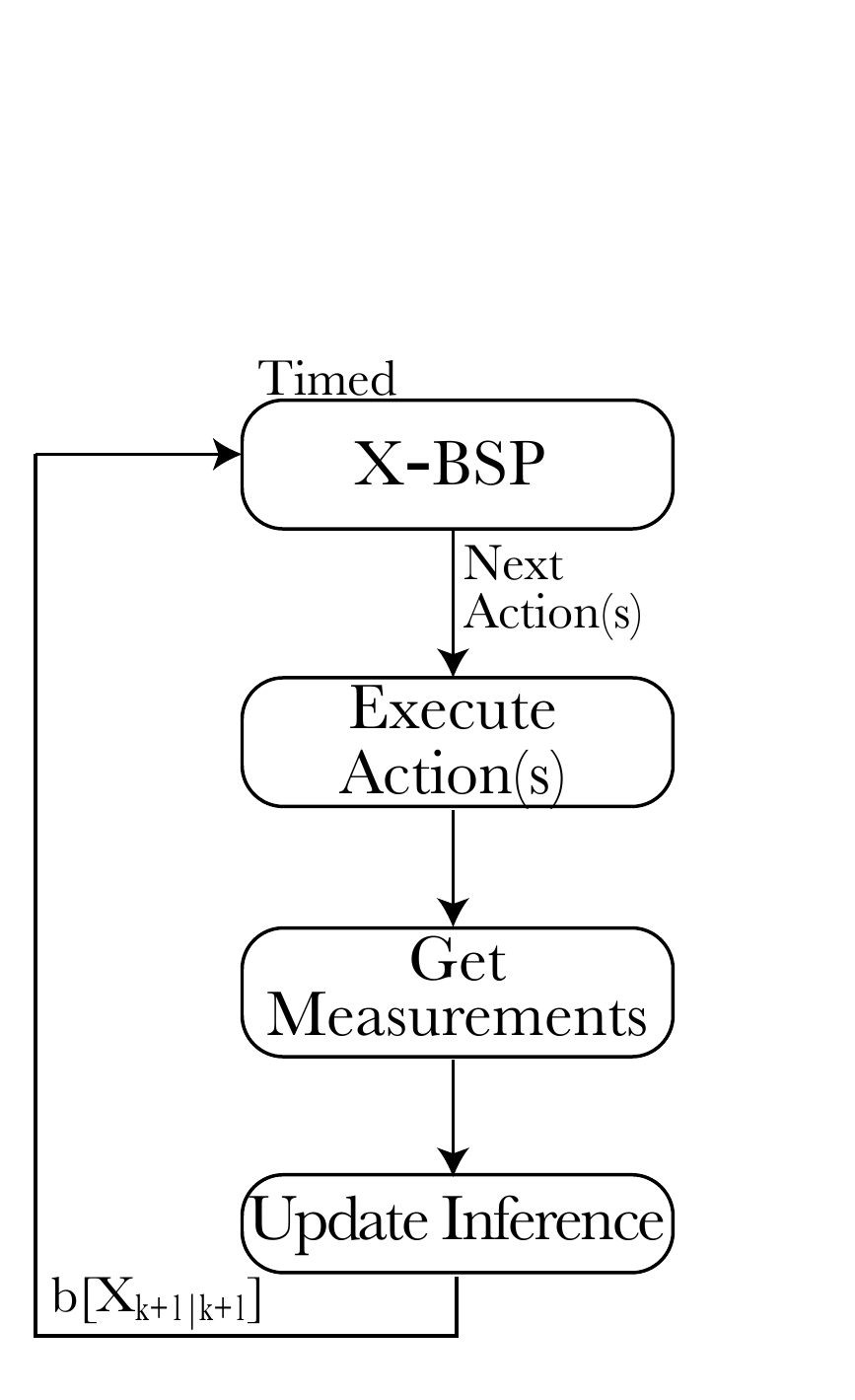}\label{fig:HighLevel_Algo:std}}
        \subfloat[]{\includegraphics[trim={250 0 0 8},clip, width=0.25\columnwidth]{Figures/rest/psudoCode.pdf}\label{fig:HighLevel_Algo:iBSP}}
		\vspace{-5pt}
        \caption{High level algorithm for \ibsp presented in a block diagram: (a)  presents a standard plan-act-infer framework with Bayesian inference and \fullbsp; (b) presents our novel approach for incremental expectation BSP - \ibsp. Instead of calculating planning from scratch we propose to utilize the precursory planning session.}
        \label{fig:HighLevel_Algo}
\end{figure}

\subsection{The ML Assumption}\label{subsec:results:ML_Ex}
In this section we provide a glimpse behind the curtains of \fullbsp and \mlbsp. We show the results of a single planning session, for which expectation and ML produced different optimal actions. Code implemented in MATALB using iSAM2 efficient methodologies, and executed on a MacBookPro 2017, with 2.9GHz Intel Core i7 processor and 16GB of RAM.
Consider a robot with initial estimated location and covariance, given two candidate actions. The world consists out of two types of landmarks, the first with high covariance and the second with low. 
Figures~\ref{fig:MLvsXbsp:Left}-\ref{fig:MLvsXbsp:Forward} present the spatial cost values which are the result of choosing "left" or "forward" actions accordingly, and where warm colors denote higher cost values. Each pixel in Figures~\ref{fig:MLvsXbsp:Left}-\ref{fig:MLvsXbsp:Forward} denotes a possible ground truth location of the robot, where the colored area represents the $1\sigma$ range of the prior covariance and the most likely state is denoted by a black square. While ML considers only the cost value resulting from the most likely state, expectation considers multiple samples from different spatial locations. As a result expectation favored the "left" action while ML favored "forward". 
For $20k$ inference rollouts, each with a different ground truth location, choosing left is statistically favorable in the sense of minimizing cost (uncertainty), $74\%$ of the times.

 \begin{figure*}[!h]
 	\centering
 	\subfloat[]{\includegraphics[trim={0 0 0 0},clip, width=0.3\columnwidth]{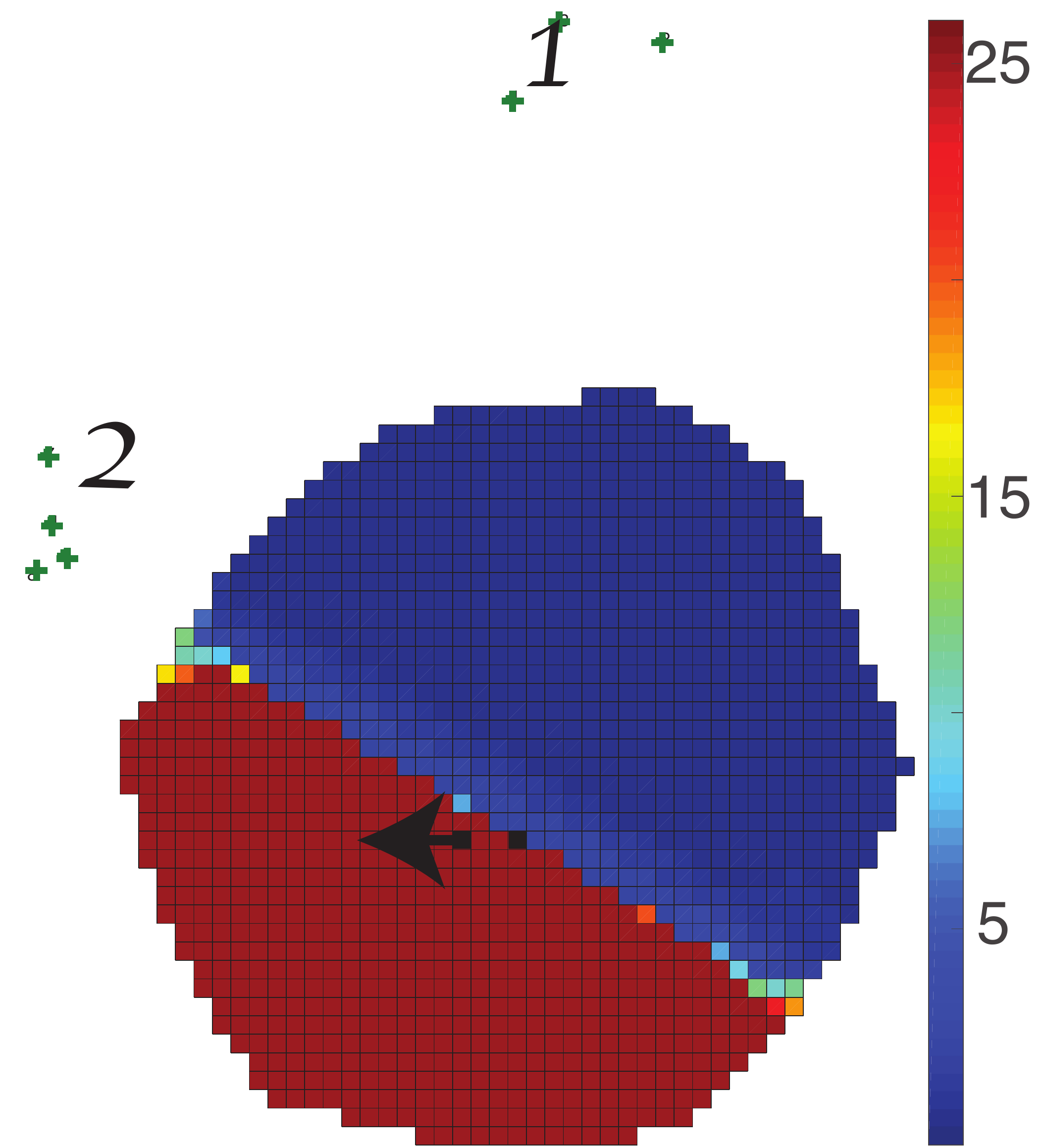}\label{fig:MLvsXbsp:Left}}
 	\subfloat[]{\includegraphics[trim={0 0 0 0},clip, width=0.27\columnwidth]{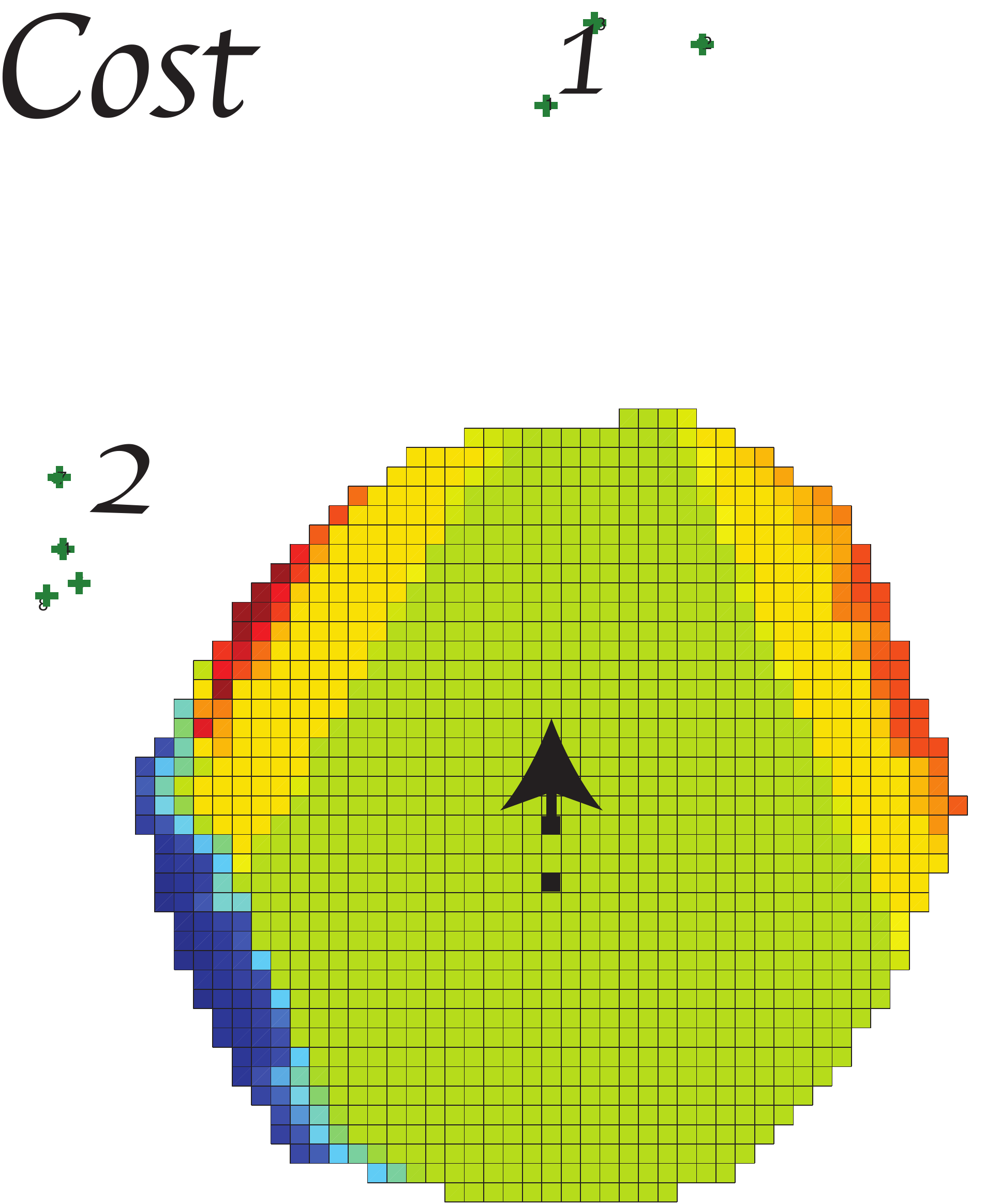}\label{fig:MLvsXbsp:Forward}}
 	\vspace{-10pt}
 	\caption{Spatial sensitivity to the ground truth location in respect to the objective value when considering "left" and "forward" actions accordingly. While \fullbsp considers the weighted average of different possible measurements, denoted by colored area, \mlbsp considers only the most likely measurement, resulted from the black square.}
 	\label{fig:MLvsXbsp}
 \end{figure*}
Under continuous state space, the possibility of the robot location to match exactly the ML location is zero, and as can be seen from the cost values in Figure~\ref{fig:MLvsXbsp:Left} a small shift in robot location could have drastic consequences over the cost value. 
Due to the fact that \fullbsp takes different possible spatial locations into consideration, it provides a weighted estimate of the cost value that might be obtained, while \mlbsp consider a specific instance involving the ML location.
After understanding the advantage \fullbsp holds over \mlbsp one can ponder wether there might be some configuration allowing \mlbsp to match or at least get closer to the estimation performance of \fullbsp. For example under the scenario presented above, adding more candidate actions to \mlbsp should improve the robustness of the estimated cost value. Of course this would result with a heavier computational load and some work is needed in order to determine how much of \mlbsp computational advantage is required to be sacrificed in favor of accuracy, we leave this for future work.
  
\subsection{Re-Using the Precursory Planning Session}\label{subsec:results:pre_ibsp}

In this section we examine \ibsp with the $\mathbb{D}_{DA}$ divergence, without \wf and under the simplifying assumption that previously sampled measurements always constitute an adequate representation of the measurement likelihood, i.e. once the closest belief is selected using $\mathbb{D}_{DA}$, all associated previously sampled measurements are re-used (see parameters in Table~\ref{table:param_pre_ibsp}). 
\begin{table}
	\caption{Parameters for Section~\ref{subsec:results:pre_ibsp} following Alg.~\ref{alg:ibsp}} 
\begin{center} \label{table:param_pre_ibsp}
  \begin{tabular}{ l | c }
    \hline \hline
    Prior belief  standard deviation & $\begin{bmatrix}
    	1^o \cdot I_{3 x 3} & 0 \\
    	0 & 5_{[m]}\cdot I_{3 x 3}
    \end{bmatrix}$ \\ \hline
    Motion Model standard deviation & $\begin{bmatrix}
    	0.5^o \cdot I_{3 x 3} & 0 \\
    	0 & 0.5_{[m]}\cdot I_{3 x 3}
    \end{bmatrix}$  \\ \hline
    Observation Model standard deviation & $\begin{bmatrix}
    	3_{[px]}  & 0 \\
    	0 & 3_{[px]}
    \end{bmatrix}$  \\ \hline
    Camera Aperture & $90^o$  \\ \hline
    Camera acceptable Sensing Range & between $2_{[m]}$ and $40_{[m]}$ \\ \hline
    \texttt{useWF} & \texttt{false} \\ \hline
    $\epsilon_c$ & 250 \\ \hline
    $\beta_{\sigma}$ & $\infty$ \\ \hline
    $n_u$ & 3 \\ \hline
    $n_x$ & 5 \\ \hline
    $n_z$ & 1 \\ \hline
    action primitives & left, right and forward with $1_{[m]}$ translation and  $\pm 90^o$ rotations\\ \hline
    $\mathbb{D}$ & $\mathbb{D}_{DA}$ \\ \hline
    \hline
  \end{tabular}
\end{center}
\end{table}

Moreover we continue the point made in Section~\ref{subsec:results:ML_Ex} and run \mlbsp alongside for comparison. 

To that end we compare \mlbsp, \fullbsp and \ibsp in the sense of planning-session computation time and the posterior estimation error upon reaching the goal.
For comparison we perform 100 rollouts (entire mission run), each with a different sampled ground-truth for the prior state. For each rollout, we time the planning sessions of all three methods. Code implemented in MATALB using iSAM2 efficient methodologies, and executed on a MacBookPro 2017, with 2.9GHz Intel Core i7 processor and 16GB of RAM. 
\begin{figure}[!h]
 	\centering
 	\subfloat[]{\includegraphics[trim={0 10 0 0},clip, width=0.33\columnwidth]{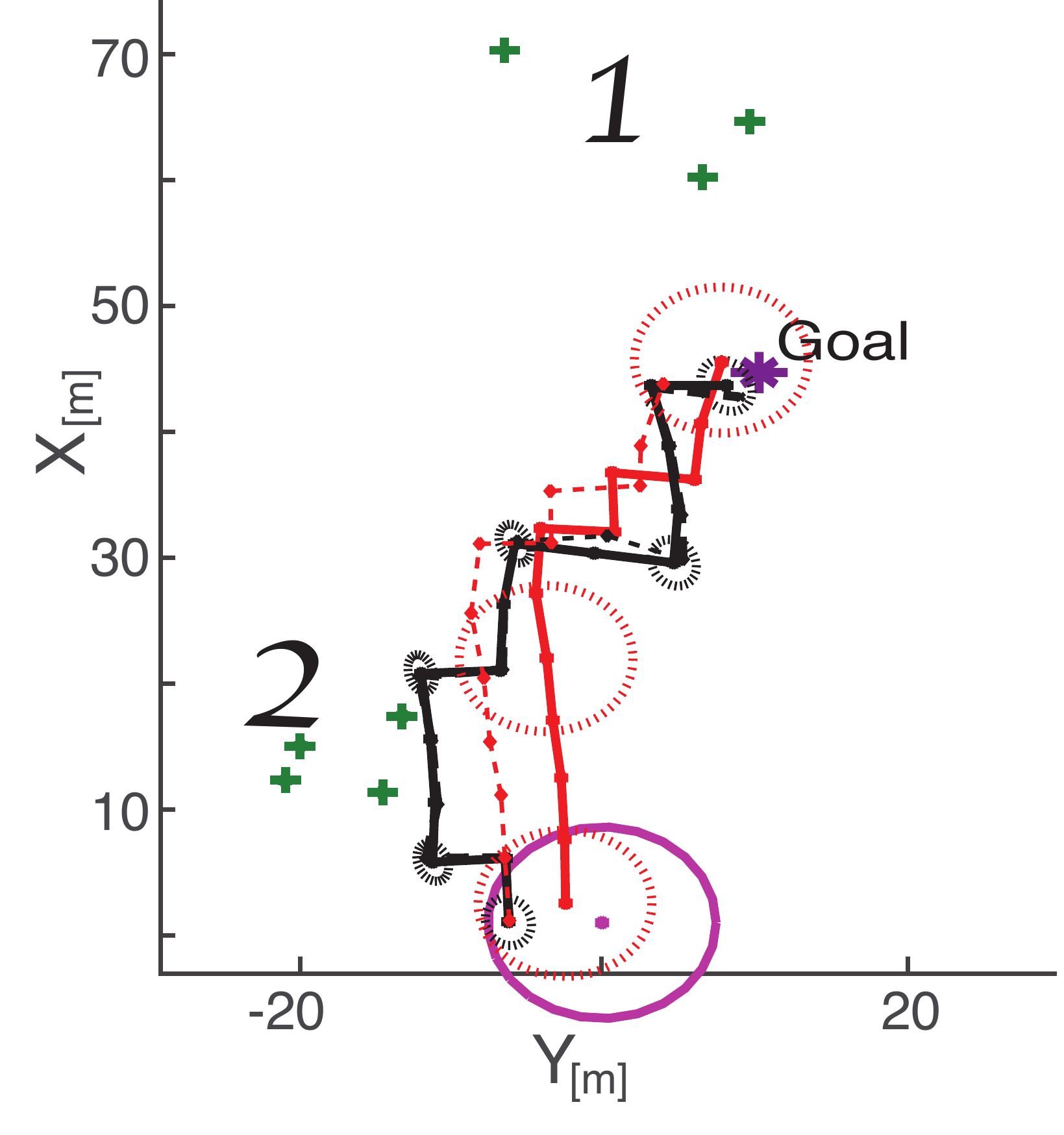}\label{fig:first:Map}}
 	\subfloat[]{\includegraphics[trim={0 -25 0 0},clip, width=0.33\columnwidth]{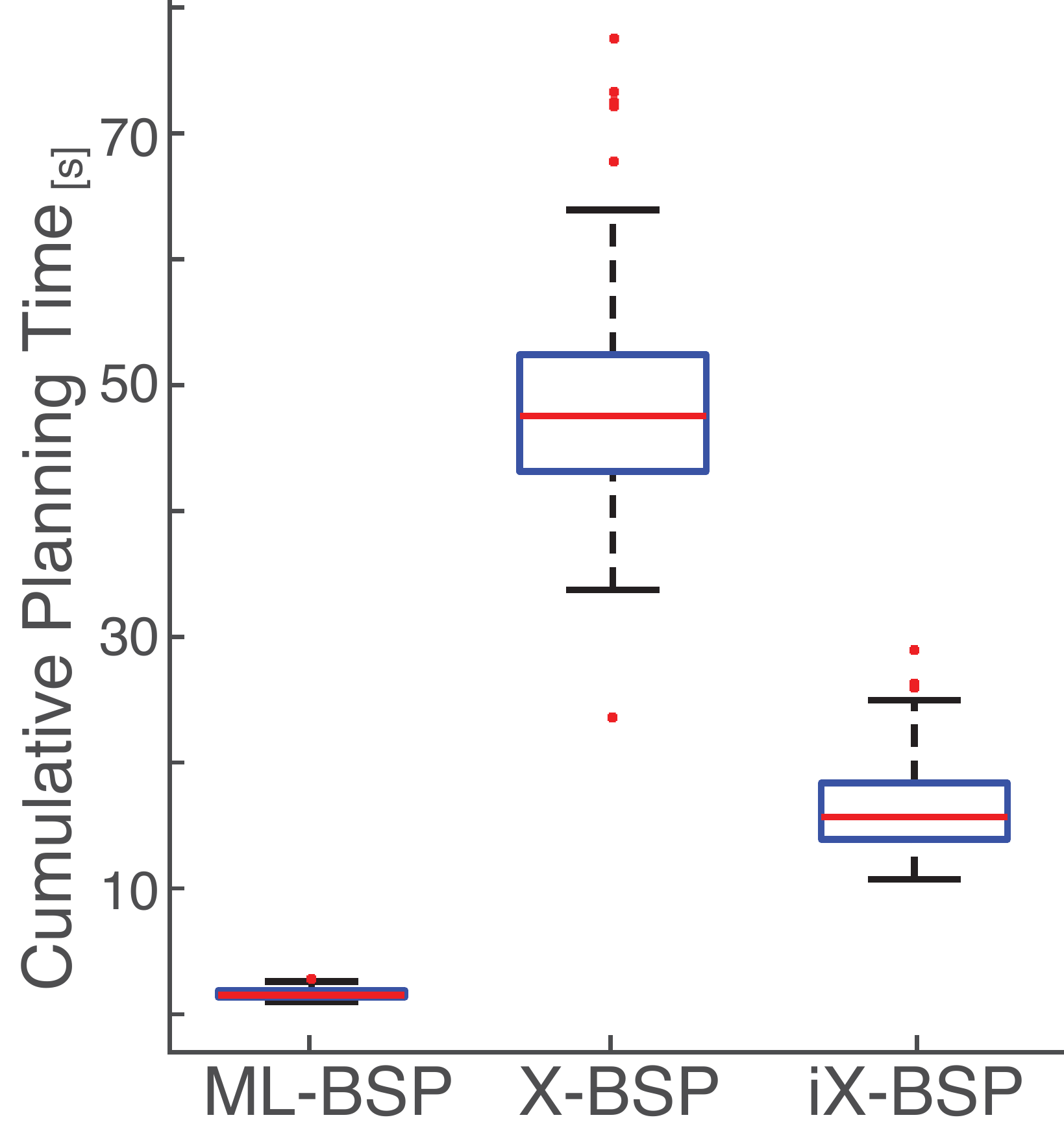}\label{fig:first:TotTime}}
 	\subfloat[]{\includegraphics[trim={0 -25 0 0},clip, width=0.33\columnwidth]{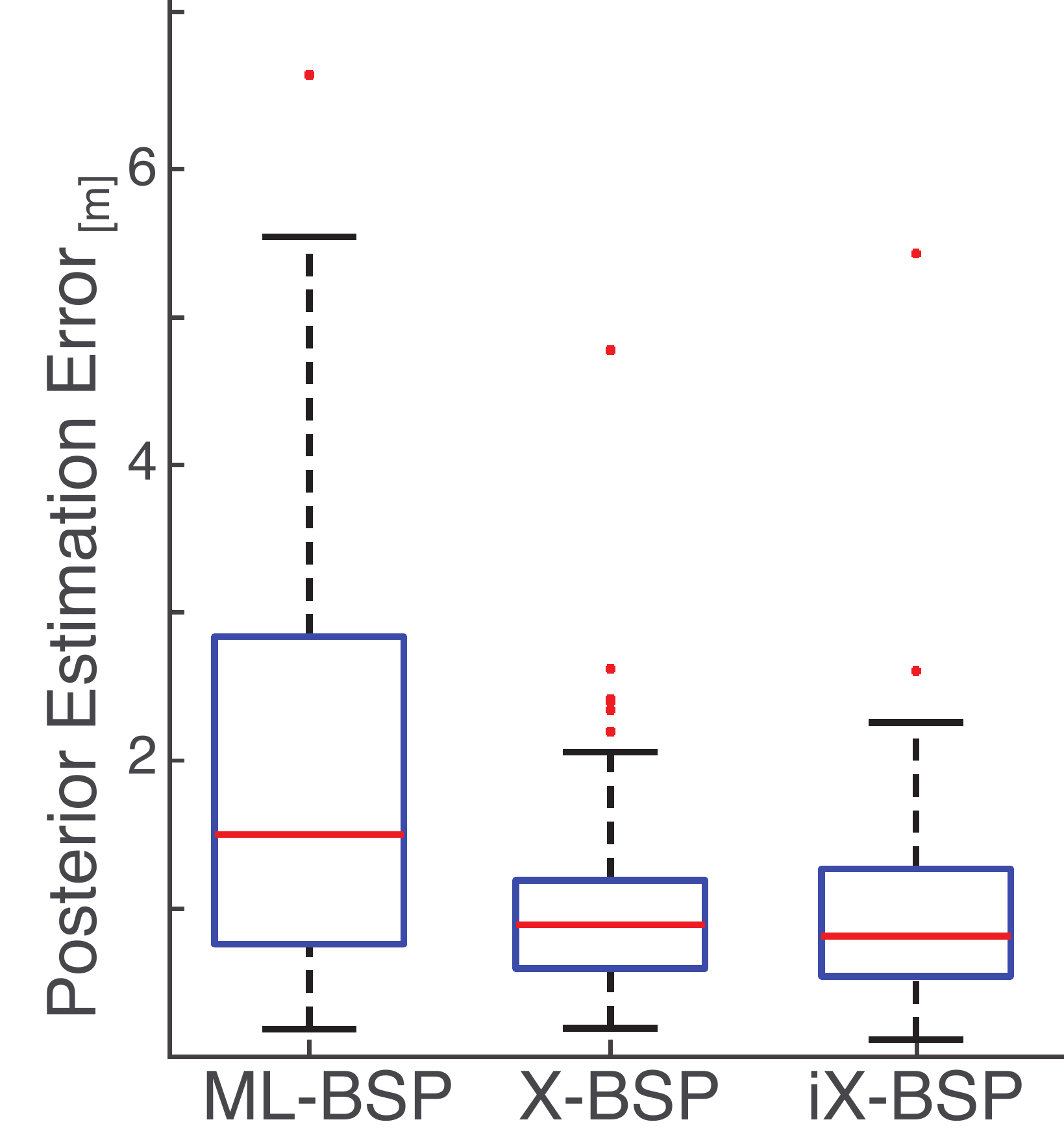}\label{fig:first:gtErr}} \vspace{-10pt}
 	\caption{(a) Testing scenario, landmarks denoted by green "+", prior state and uncertainty in solid purple, \mlbsp denoted by red, \fullbsp and \ibsp denoted by black. (b) and (c) are Box plots of $100$ rollouts for planning session timing (b) and posterior estimation error (c) upon reaching the goal.}
 	\label{fig:first}
 \end{figure}
Figure~\ref{fig:first:Map} presents the scenario on which all rollouts were performed. Considering the same world and same landmarks as in Section~\ref{subsec:results:ML_Ex}. A robot equipped with a stereo camera, is required to reach the goal whilst not crossing a covariance threshold, i.e. cost consisting of distance to goal and a covariance penalty above a certain value. Figure~\ref{fig:first:Map} shows one of the $100$ rollouts that were calculated, in which the estimated trajectory by each method is denoted by a solid line, the ground truth by a dashed line and the posterior covariance by a dashed ellipse.
In Figure~\ref{fig:first:Map} both \fullbsp and \ibsp, in black, chose the same optimal actions along the mission, while \mlbsp, in red, chose differently. We can also see the effect of this difference over each method's covariance, \fullbsp and \ibsp action choice led to a smaller covariance along the entire path.

Figure~\ref{fig:first:TotTime} presents the statistical data of the planning session running time. Since in this example we follow an MPC framework, the last step of each horizon is required to be calculated from scratch. Since doing so is identical to the course of action in \fullbsp, we present the computation time of the entire horizon, excluding the last horizon step.
As expected, for average timing as well as for each separate rollout, both \mlbsp and \ibsp timings are lower than that of \fullbsp. By re-using previous planning session, instead of calculating it from scratch we save valuable computation time, theoretically without effecting the planning solution. We examine the effect on the planning solution in Figure~\ref{fig:first:gtErr}, by comparing the posterior estimation error upon reaching the goal. As expected, the statistical results of $100$ rollouts presented in Figure~\ref{fig:first:gtErr}, shows that \fullbsp is statistically superior to \mlbsp: in $63\%$ of the rollouts it has a smaller estimation error while in $10\%$ they are equal. Importantly, we can also see that \ibsp is statistically similar to \fullbsp, with $41\%$ of the rollouts with smaller estimation error and $15\%$ equal. We note that relaxing the simplifying assumption that all samples are adequately representative, would result with an even better match between \fullbsp and \ibsp.   

\subsection{\ibsp}\label{subsec:results:ibsp}
In this section we examine \ibsp with the \JD distance, without \wf and without the simplifying assumptions used in Section~\ref{subsec:results:pre_ibsp} (see parameters in Table~\ref{table:param_ibsp}).
\begin{table}
	\caption{Parameters for Section~\ref{subsec:results:ibsp} following Alg.~\ref{alg:ibsp}} 
\begin{center} \label{table:param_ibsp}
  \begin{tabular}{ l | c }
    \hline \hline
    Prior belief  standard deviation & $\begin{bmatrix}
    	1^o \cdot I_{3 x 3} & 0 \\
    	0 & 5_{[m]}\cdot I_{3 x 3}
    \end{bmatrix}$ \\ \hline
    Motion Model standard deviation & $\begin{bmatrix}
    	0.5^o \cdot I_{3 x 3} & 0 \\
    	0 & 0.5_{[m]}\cdot I_{3 x 3}
    \end{bmatrix}$  \\ \hline
    Observation Model standard deviation & $\begin{bmatrix}
    	3_{[px]}  & 0 \\
    	0 & 3_{[px]}
    \end{bmatrix}$  \\ \hline
    Camera Aperture & $90^o$  \\ \hline
    Camera acceptable Sensing Range & between $2_{[m]}$ and $40_{[m]}$ \\ \hline
    \texttt{useWF} & \texttt{false} \\ \hline
    $\epsilon_c$ & 250 \\ \hline
    $\beta_{\sigma}$ & 1.5 \\ \hline
    $n_u$ & 3 \\ \hline
    $n_x$ & 5 \\ \hline
    $n_z$ & 1 \\ \hline
    action primitives & left, right and forward with $1_{[m]}$ translation and  $\pm 90^o$ rotations\\ \hline
    $\mathbb{D}$ & \JD \\ \hline
    \hline
  \end{tabular}
\end{center}
\end{table}

We compare \ibsp and \fullbsp in the sense of planning-session computation time, the posterior estimation error upon reaching the goal, and the covariance norm upon reaching the goal. 
Code implemented in MATLAB using iSAM2 efficient methodologies and executed on a Linux machine, with Xeon E3-1241v3 3.5GHz processor with 64GB of memory.
For comparison we perform 20 rollouts (entire mission run), each with a different sampled ground-truth for the prior state, on 10 different, randomly generated, maps presented in Figure~\ref{fig:ibsp:maps}. Each map, contains two goals and between 2 to 150 landmarks. The goals and landmarks location as well as the number of landmarks are all randomly generated for each map. For each of the 200 rollouts, we clock the planning session computation time of both methods for comparison.
Across all randomized maps, the robot, equipped with a stereo camera, has the same mission -  reaching each one of the goals whilst maximizing information gain and minimizing distance to goal using the reward function 
\begin{equation}\label{eq:simulationReward}
	r_{i}() = \alpha \cdot \frac{1}{2}ln \left[ (2\pi e)^n \cdot det(\Lambda_{i})\right] + (1- \alpha) \cdot (D2G_{i-1} - D2G_{i}),
\end{equation}
where $\alpha \in [0,1]$ is a weighting parameter, $n$ represents the dimension of the robot's pose, $\Lambda_{i}$ represents the focused information matrix at time $i$, $D2G_i$ represents distance-to-goal at time $i$. 
\begin{figure}[h]
	\centering
		\subfloat[]{\includegraphics[trim={0 0 0 0},clip, width=0.95\columnwidth]{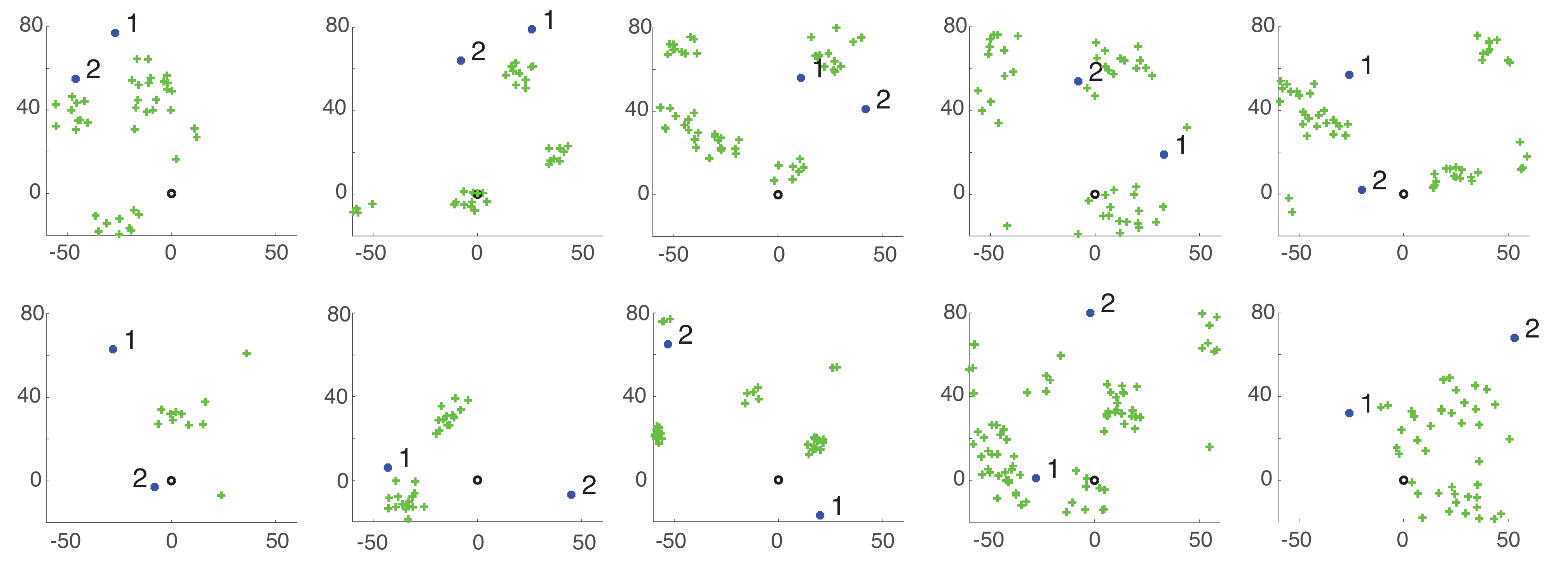}\label{fig:ibsp:maps}}\\
        \subfloat[]{\includegraphics[trim={0 0 0 0},clip, width=0.3\columnwidth]{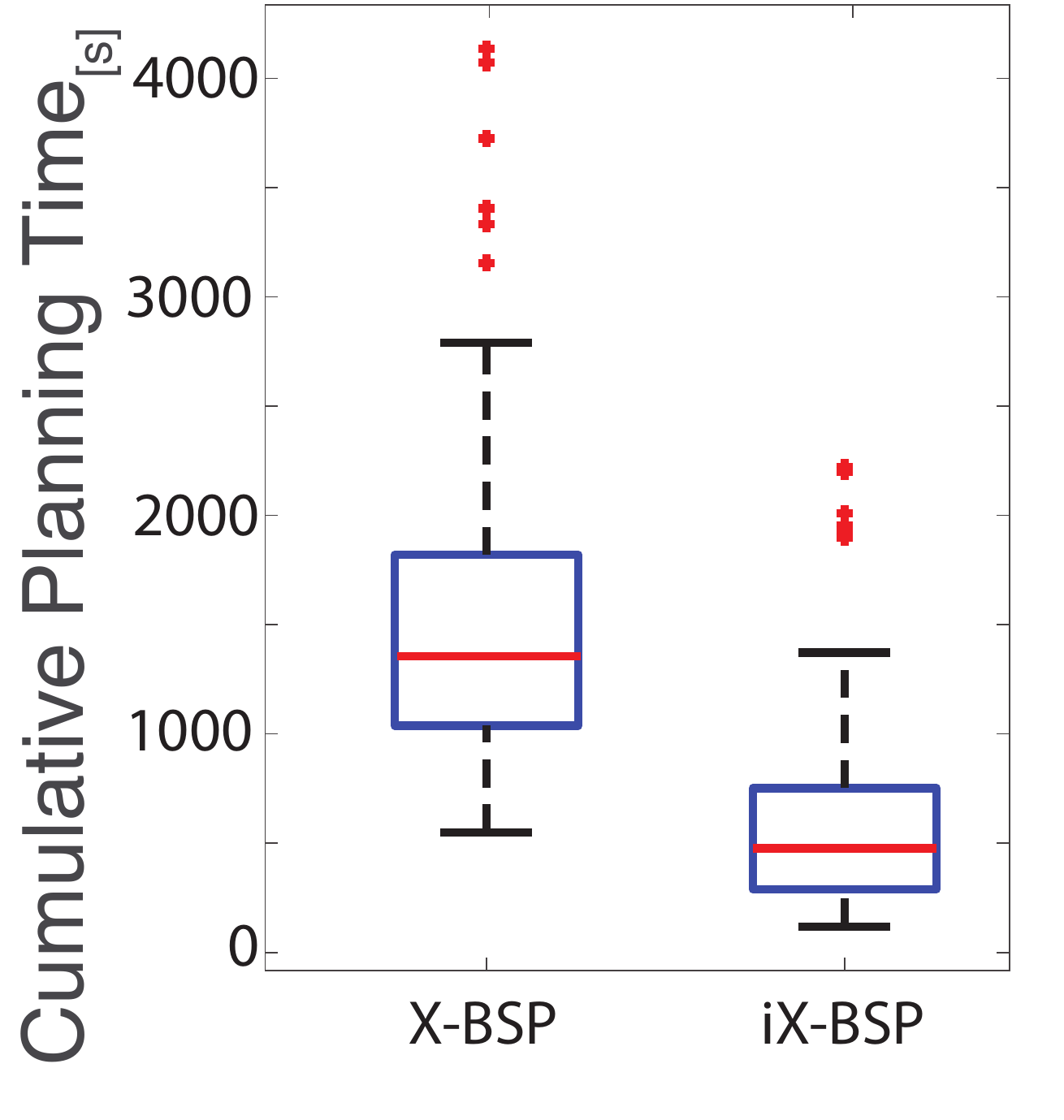}\label{fig:ibsp:time}}
        \subfloat[]{\includegraphics[trim={0 0 0 0},clip, width=0.3\columnwidth]{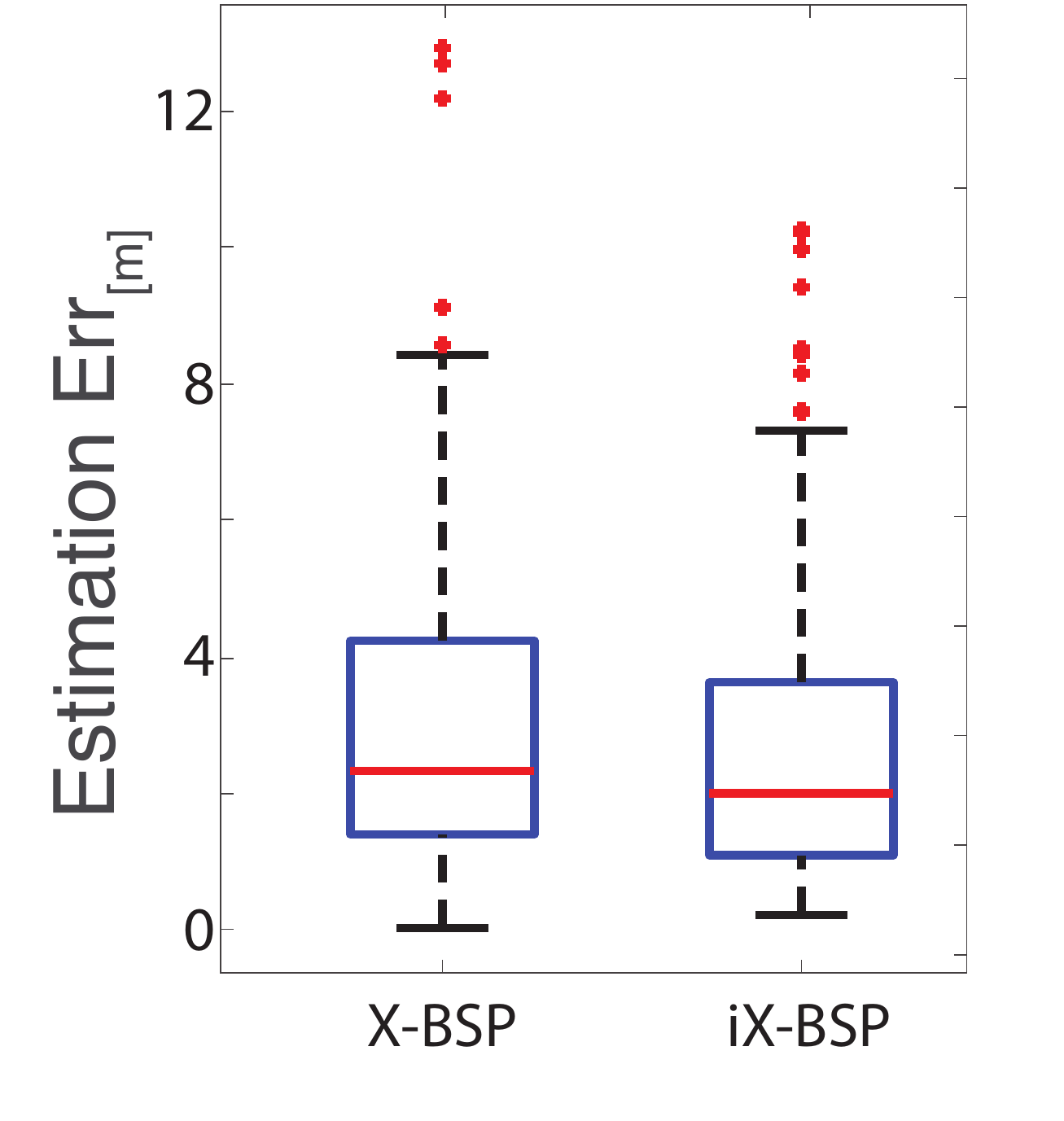}\label{fig:ibsp:acc}}
        \subfloat[]{\includegraphics[trim={0 2 0 0},clip, width=0.3\columnwidth]{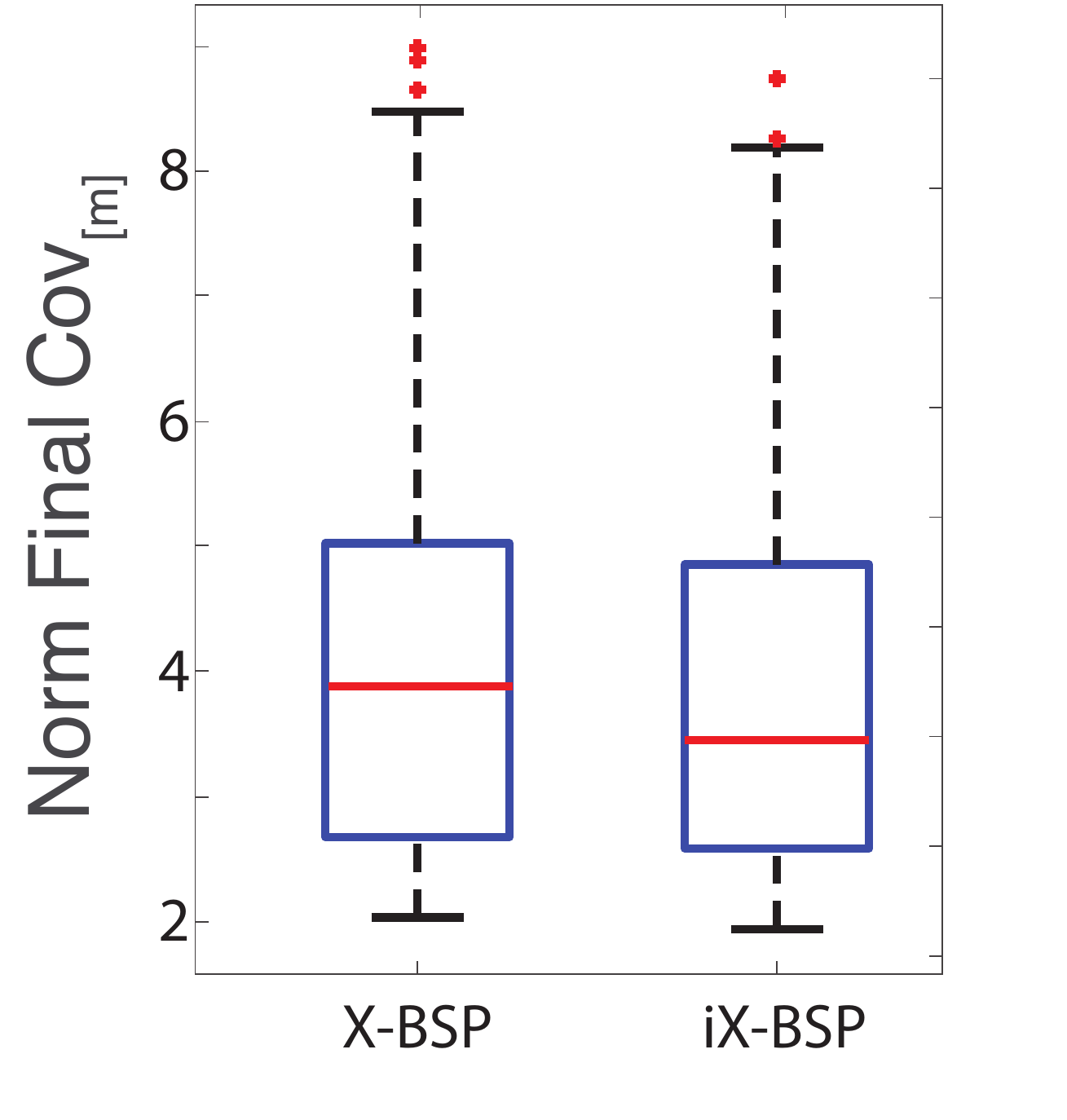}\label{fig:ibsp:cov}}
        \caption{Statistical comparison between \fullbsp and \ibsp:(a) 10 randomly generated maps, used for the statistical comparison. Each with 2 numbered goals denoted by blue dots, and between 2 to 150 landmarks denoted by green crosses. (b) (c) and (d) Box plots of 20 rollouts for planning session timing (b), posterior estimation error upon reaching the goal (c) and the covariance norm upon reaching the goal.}
        \label{fig:ibsp}
\end{figure}

Figure~\ref{fig:ibsp:time} presents a box-plot for the timing data of all 200 rollouts, each with 6 outliers, where the computation time advantage in favor of \ibsp is easily noticed. The significant reduction in computation time is originated in the fact that \ibsp performs inference update in a more efficient way, computation wise, compared to \fullbsp. By forcing previously sampled measurements as part of the objective estimation, \ibsp is able to utilize previously solved beliefs from a precursory planning session, and efficiently update them, instead of performing inference from scratch as done in \fullbsp.

Since we claim to provide a more efficient paradigm to the general problem of \fullbsp, we also verify how \ibsp favors in estimation results. Figure~\ref{fig:ibsp:acc} presents a box plot of the estimation error upon reaching the goal for each of the methods. The estimation error was calculated using the normalized distance between the last pose estimation and the last pose ground truth value, i.e.
\begin{equation}\label{eq:accCalc}
	Estimation \ err = \parallel \overset{\wedge}{x}_{final} - x^{gt}_{final} \parallel.
\end{equation}
As can be seen in Figure~\ref{fig:ibsp:acc} both methods average around an estimation error of $2_{[m]}$, while \fullbsp is with 5 outliers and \ibsp is with 9 outliers. In $49.5\%$ of the rollouts, \ibsp provided with a better estimation error than \fullbsp. The large variance that can be seen in Figure~\ref{fig:ibsp:acc} is probably the result of using a small number of samples for estimating the objective. Nonetheless the empiric estimation variance of both methods can be considered as statistically identical for all practical purposes. Of course a more rigorous examination in required, by analytically comparing the estimation variance, we leave this for future work. We push further and compare the covariance norm of the final pose. As can be seen in Figure~\ref{fig:ibsp:cov}, they average around $3.4_{[m]}$ for \ibsp and $3.8_{[m]}$ for \fullbsp, with only 2 and 3 outliers respectively, and can be considered as statistically identical for all practical purposes. As suggested in Section~\ref{subsec:results:pre_ibsp}, relaxing the simplifying assumption of adequately representative samples, in fact resulted with a better match between \fullbsp and \ibsp as can be seen from comparing Figure~\ref{fig:first:gtErr} to Figure~\ref{fig:ibsp:acc}.
   
\subsection{\ibsp with the \wf Condition}\label{subsec:results:wf}
In this section we examine the effects of \wf over the \ibsp paradigm. 
We compare \ibsp with and without \wf (see Section~\ref{ssubsec:wf:performance}) in the sense of planning-session computation time, the posterior estimation error upon reaching the goal and the covariance norm upon reaching the goal. We do so on the exact scenario used in Section~\ref{subsec:results:ibsp}, with the single exception of using \wf. We also preform a sensitivity analysis for the \wf threshold value $\epsilon_{wf}$ (see Section~\ref{ssubsec:wf:sensitivity}), in order to check its effect over the objective value.
\begin{table}
	\caption{Parameters for Section~\ref{ssubsec:wf:performance} following Alg.~\ref{alg:ibsp}} 
\begin{center} \label{table:param_wf}
  \begin{tabular}{ l | c }
    \hline \hline
    Prior belief  standard deviation & $\begin{bmatrix}
    	1^o \cdot I_{3 x 3} & 0 \\
    	0 & 5_{[m]}\cdot I_{3 x 3}
    \end{bmatrix}$ \\ \hline
    Motion Model standard deviation & $\begin{bmatrix}
    	0.5^o \cdot I_{3 x 3} & 0 \\
    	0 & 0.5_{[m]}\cdot I_{3 x 3}
    \end{bmatrix}$  \\ \hline
    Observation Model standard deviation & $\begin{bmatrix}
    	3_{[px]}  & 0 \\
    	0 & 3_{[px]}
    \end{bmatrix}$  \\ \hline
    Camera Aperture & $90^o$  \\ \hline
    Camera acceptable Sensing Range & between $2_{[m]}$ and $40_{[m]}$ \\ \hline
    \texttt{useWF} & \texttt{true} \\ \hline
    $\epsilon_c$ & 250 \\ \hline
    $\epsilon_{wf}$ & 2 \\ \hline
    $\beta_{\sigma}$ & 1.5 \\ \hline
    $n_u$ & 3 \\ \hline
    $n_x$ & 5 \\ \hline
    $n_z$ & 1 \\ \hline
    action primitives & left, right and forward with $1_{[m]}$ translation and  $\pm 90^o$ rotations\\ \hline
    $\mathbb{D}$ & \JD \\ \hline
    \hline
  \end{tabular}
\end{center}
\end{table} 
\subsubsection{\wf effect on performance}\label{ssubsec:wf:performance}
We compare \ibsp with and without the use of \wf under the same scenario as in Section~\ref{subsec:results:ibsp} (see parameters in Table~\ref{table:param_wf}). We use the same code and the same Linux machine to perform 20 rollouts, each with a different sampled ground-truth for the prior state, on the same 10 maps presented in Figure~\ref{fig:ibsp:maps}, and with the same reward function (\ref{eq:simulationReward}).
Figure~\ref{fig:ibspwf:time} presents a box-plot of the accumulated planning time of all 200 rollouts, where the computation time advantage in favor of \wf usage is easily noticed, on average \wf saved $90\%$ off \ibsp accumulated planning time. 
While \ibsp favors in computation time over \fullbsp due to a more efficient belief update, the significant reduction in computation time when using \wf is originated in the fact that for "close enough" beliefs we refrain from updating the belief altogether. 
As previously mentioned, the use of \wf with a non-zero threshold is an approximation of \ibsp and essentially of \fullbsp. As such, we would expect that using \wf will affect the estimation accuracy and covariance, but as seen in Figures~\ref{fig:ibspwf:acc}-\ref{fig:ibspwf:cov}, there is no significant toll on estimation. The reason for this supposedly "free lunch" hides behind the choice of the \wf threshold, as the objective value error due to the use of \wf is directly related to the choice of \wf threshold (as seen in Section~\ref{subsec:wf}). For a small enough \wf threshold the same action is chosen, and so the impact over the estimation is practically unnoticeable.     
It is worth stressing out that this will not always be the case, as \wf is in fact an approximation, and one should treat it as such when choosing to invoke it. Moreover \wf can be seen as breaking the MPC framework, whenever the newly obtained information of time $k$ is "close enough" to the most relevant prediction of time $k$ we use the prediction rather than updating it with the new information.
We continue with empirically testing the impact of the \wf threshold over the objective value in Section~\ref{ssubsec:wf:sensitivity}. 
\begin{figure}[H]
	\centering
		\subfloat[]{\includegraphics[trim={0 0 0 0},clip, width=0.3\columnwidth]{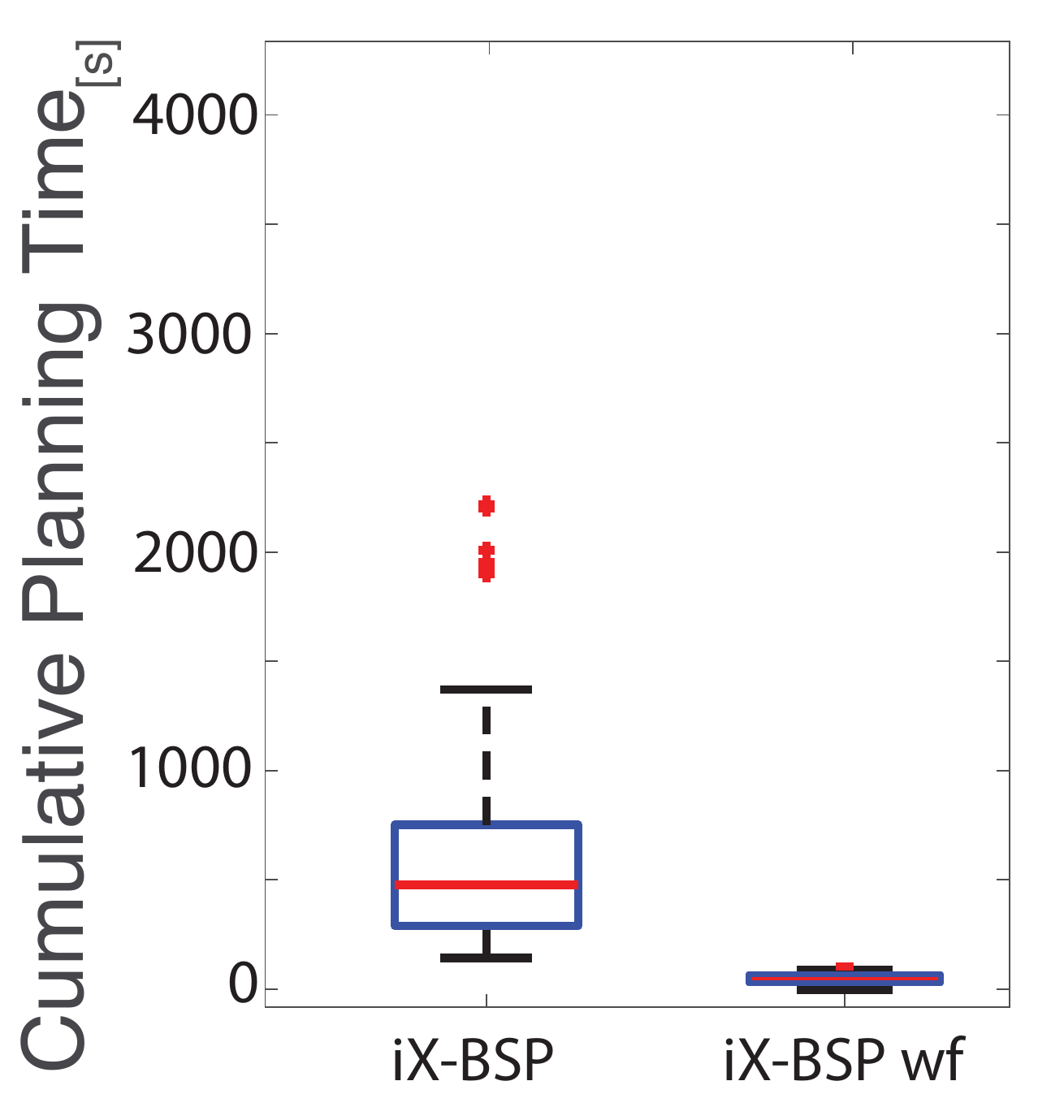}\label{fig:ibspwf:time}}
        \subfloat[]{\includegraphics[trim={0 0 0 0},clip, width=0.3\columnwidth]{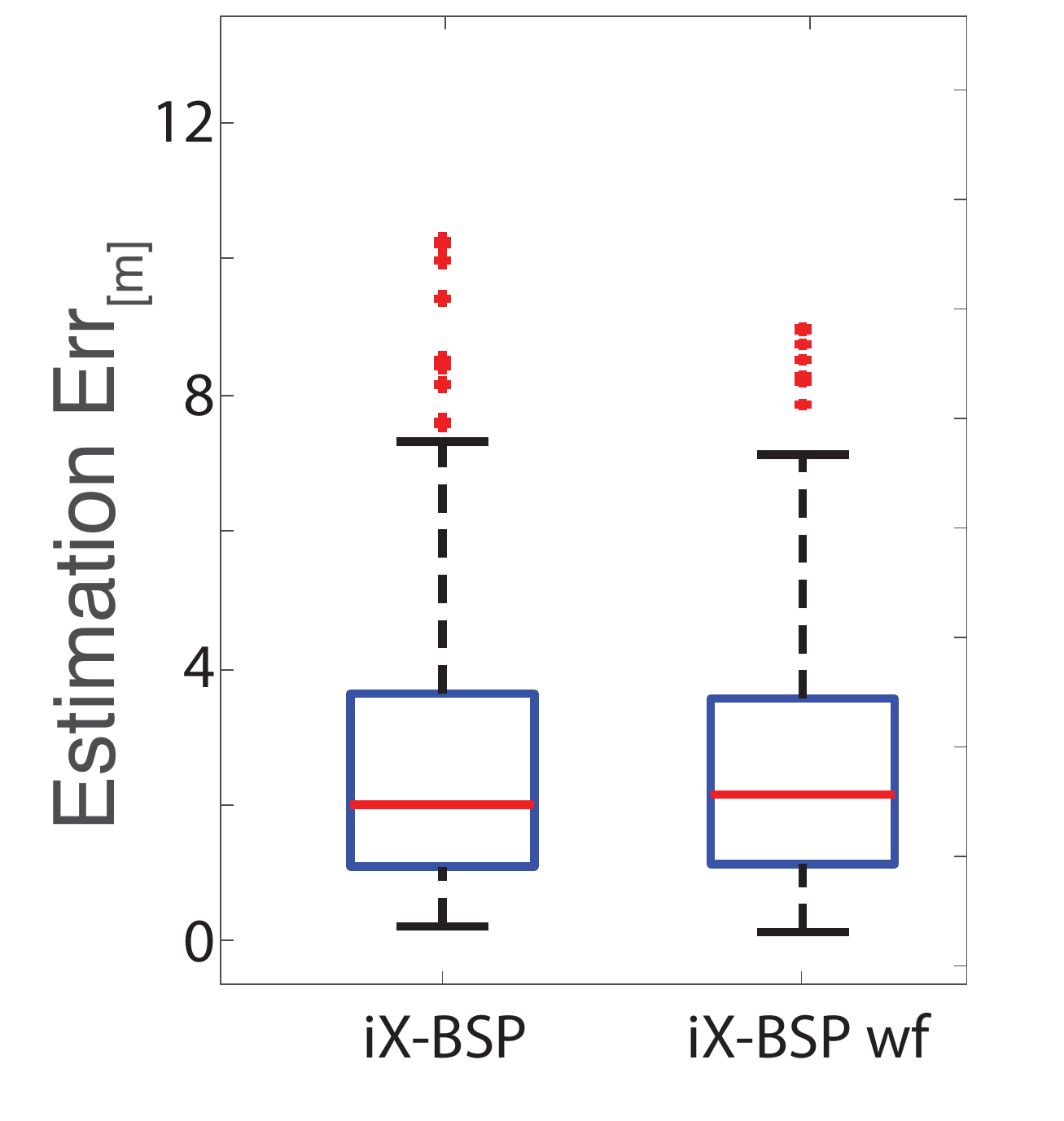}\label{fig:ibspwf:acc}}
        \subfloat[]{\includegraphics[trim={0 2 0 0},clip, width=0.3\columnwidth]{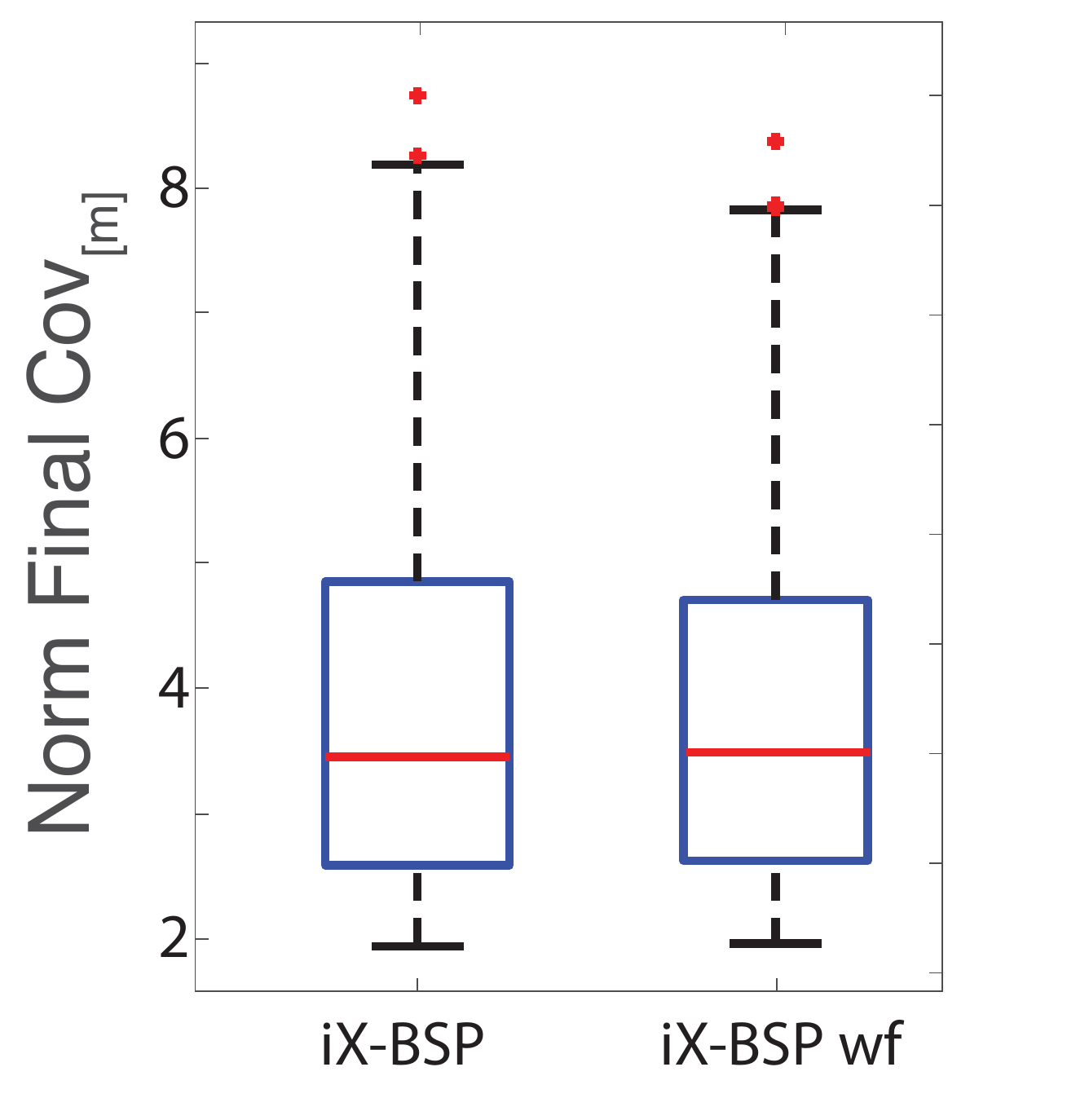}\label{fig:ibspwf:cov}}
        \caption{Statistical comparison between \ibsp with and without \wf over the same scenario as in Figure~\ref{fig:ibsp} of 10 randomly generated maps each with 2 goals and between 2 to 150 landmarks. (a) (b) and (c) Box plots of 200 rollouts for planning session timing (a), posterior estimation error upon reaching the goal (b) and the covariance norm upon reaching the goal (c).}
        \label{fig:ibspwf}
\end{figure}

\subsubsection{\wf threshold - Sensitivity Analysis}\label{ssubsec:wf:sensitivity}

Complementary to the bounds provided in Section~\ref{ssubsec:wfBounds}, we provide an empiric analysis for the impact the \wf threshold holds over the objective value, for the non-Liphschitz reward function (\ref{eq:simulationReward}). Code implemented in MATLAB using iSAM2 efficient methodologies and executed on the same Linux machine. The relevant parameters are summarized in Table~\ref{table:wfSensitivity}.
  \begin{table}
	\caption{Parameters for Section~\ref{ssubsec:wf:sensitivity} following \fullbsp} 
\begin{center} \label{table:wfSensitivity}
  \begin{tabular}{ l | c }
    \hline \hline
    Prior belief  standard deviation & $\begin{bmatrix}
    	1^o \cdot I_{3 x 3} & 0 \\
    	0 & 5_{[m]}\cdot I_{3 x 3}
    \end{bmatrix}$ \\ \hline
    Motion Model standard deviation & $\begin{bmatrix}
    	0.5^o \cdot I_{3 x 3} & 0 \\
    	0 & 0.5_{[m]}\cdot I_{3 x 3}
    \end{bmatrix}$  \\ \hline
    Observation Model standard deviation & $\begin{bmatrix}
    	3_{[px]}  & 0 \\
    	0 & 3_{[px]}
    \end{bmatrix}$  \\ \hline
    Camera Aperture & $90^o$  \\ \hline
    Camera acceptable Sensing Range & between $2_{[m]}$ and $40_{[m]}$ \\ \hline
    $n_u$ & 3 \\ \hline
    $n_x$ & 5 \\ \hline
    $n_z$ & 1 \\ \hline
    action primitives & left, right and forward with $1_{[m]}$ translation and  $\pm 90^o$ rotations\\ \hline \hline
  \end{tabular}
\end{center}
\end{table} 
\begin{figure}[H]
	\centering
		\subfloat[]{\includegraphics[trim={0 0 0 0},clip, width=0.6\columnwidth]{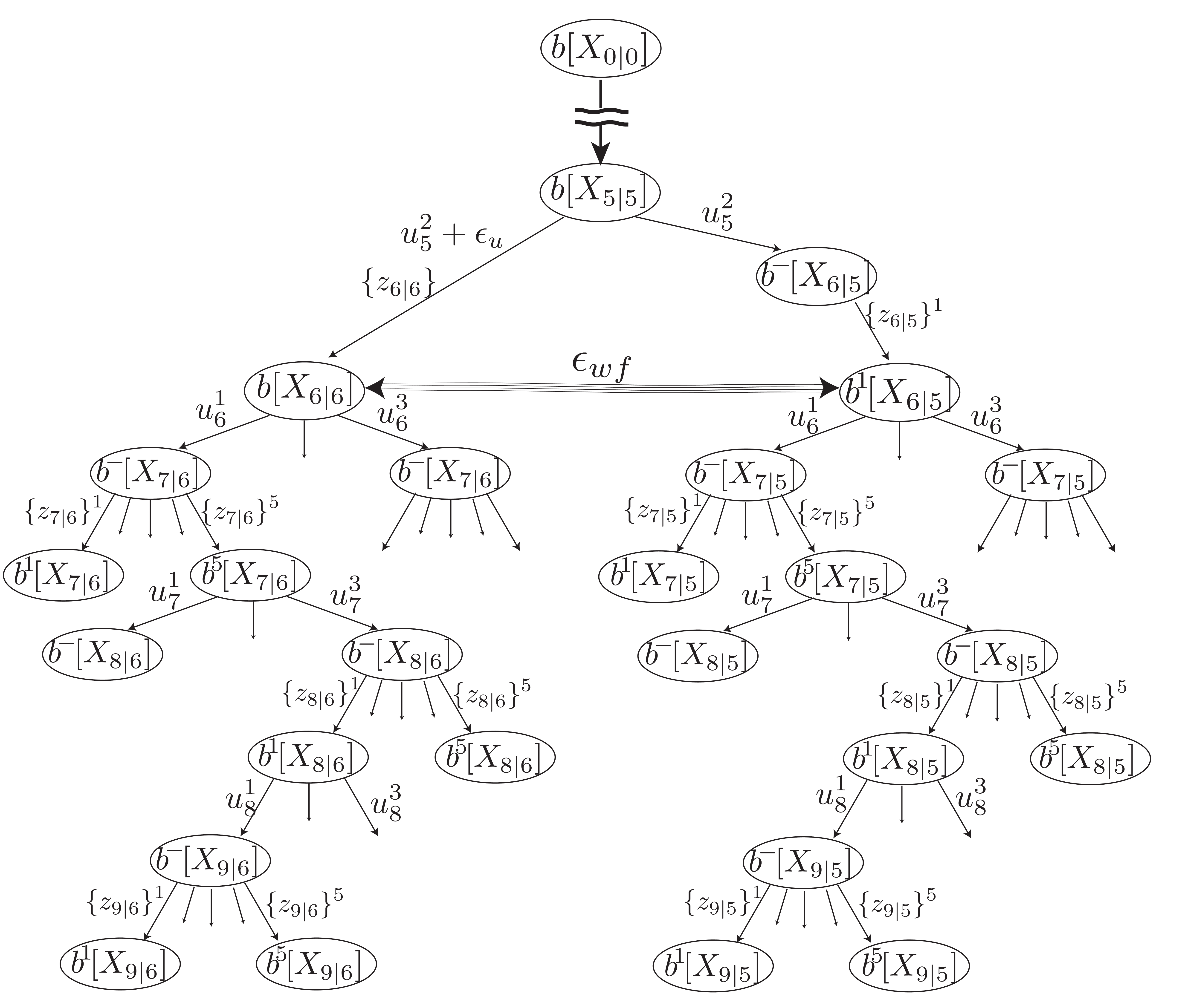}\label{fig:wf_bounds:trees}}
		\subfloat[]{\includegraphics[trim={0 0 0 0},clip, width=0.4\columnwidth]{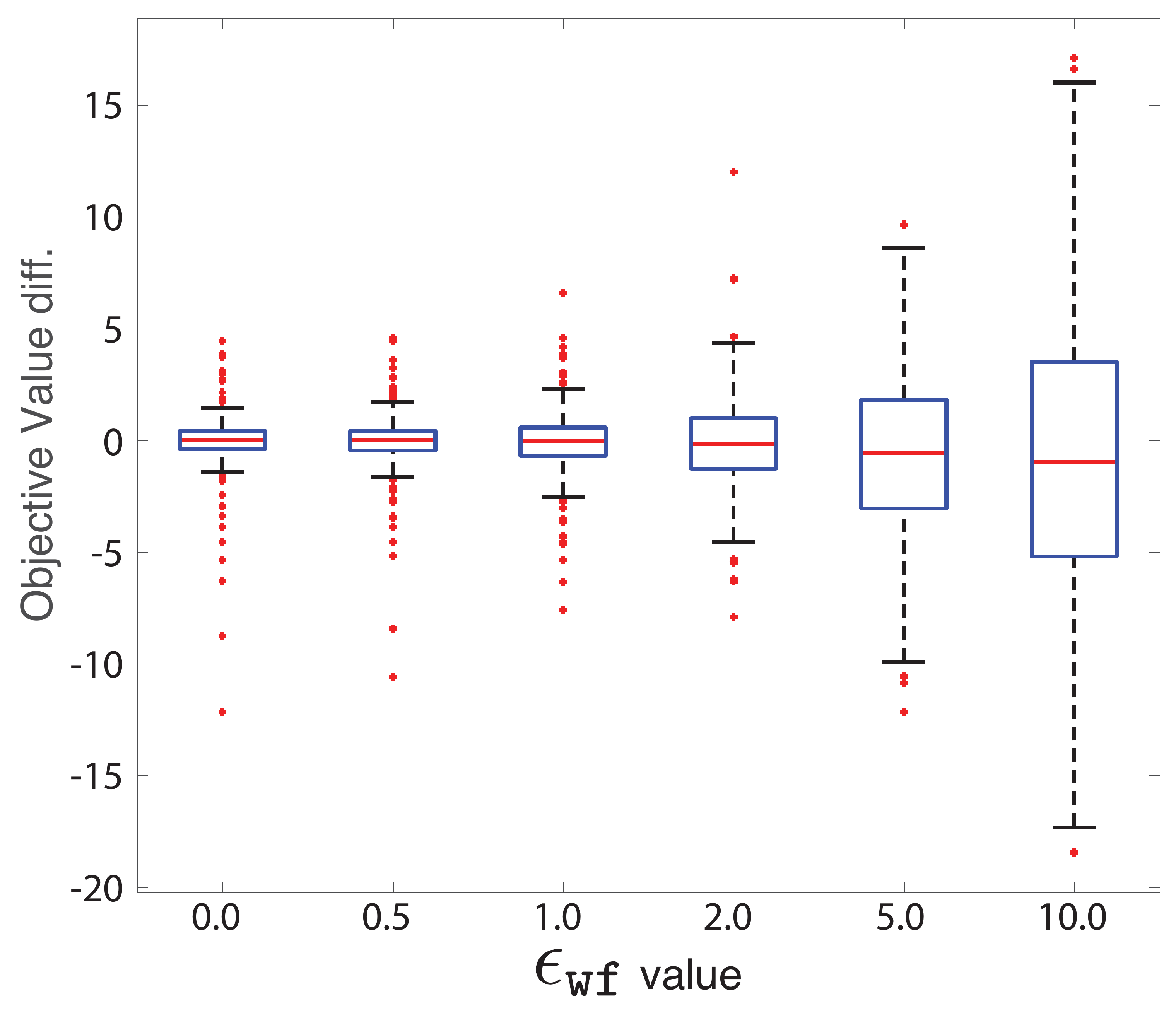}\label{fig:wf_bounds:res}}
        \caption{Statistical evaluation of wildfire threshold effect over Objective value. (a) The belief trees used for the statistical evaluation: the shared history up to time $t=5$; the planning tree at time $t=5$ on the right; the planning tree at time $t=6$ on the left. (b) Box plot of the objective value difference between corresponding action sequences as a function of the forced distance between $b[X_{6|6}]$ and $b[X^5_{6|5}]$.}
        \label{fig:wf_bounds}
\end{figure}

In order to perform such analysis, we would like to compare the objective values resulting from planning over two posterior beliefs, while the \JD distance between them equals different \wf threshold values. 
The scenario used for this analysis is illustrated in Figure~\ref{fig:wf_bounds:trees}. For six different values of $\epsilon_wf=\{0, 0.5, 1, 2, 5, 10\}$ and each of the 10 maps in Figure~\ref{fig:ibsp:maps} we perform 20 repetitions of the following.
Starting from the same prior belief $b[X_{0|0}]$, and the same ground-truth, we advance the robot 5 steps forward and obtain the belief $b[X_{5|5}]$. 
Using $b[X_{5|5}]$ we obtain two different beliefs, the first by propagating $b[X_{5|5}]$ with the primitive action $u^2$ and the resulting \emph{sampled} measurements $\{ z_{6|5}\}^1$ - denoted as $b^1[X_{6|5}]$, and the second by calculating the posterior $b[X_{6|6}]$ resulting from advancing the robot using the same primitive action $u^2$ along with some additive noise $\epsilon_u$ and obtaining the corresponding measurements $\{z_{6|6}\}$. 
We now have two beliefs, $b^1[X_{6|5}]$ and $b[X_{6|6}]$, both representing the belief at time $t=6$ but with partially different history, the same we would get from performing \ibsp under MPC framework. 
The additive noise $\epsilon_u$ is chosen such that it would cause the \JD distance between $b^1[X_{6|5}]$ and $b[X_{6|6}]$ to equal the desired \wf threshold value - $\epsilon_{wf}$. 
We then perform \fullbsp over both $b^1[X_{6|5}]$ and $b[X_{6|6}]$, with the same parameters, and obtain corresponding objective values for each of the candidate action sequences.
As both planning trees share the same action sequences we can compare the corresponding objective values per candidate action sequence. We denote the difference between the aforementioned as the Objective Value difference, which form the vertical axes in Figure~\ref{fig:wf_bounds:res}.

Figure~\ref{fig:wf_bounds:res} presents the results of the described analysis, the objective value difference as a function of the \JD distance between $b^1[X_{6|5}]$ and $b[X_{6|6}]$ denoted as $\epsilon_{wf}$. 
We can clearly see the variance of the different objective value differences increasing with the \wf threshold values, thus reflecting a connection between the objective value difference and the \wf threshold value over non-Liphschitz reward. 
Moreover, for the results in Figure~\ref{fig:wf_bounds:res}, the relation between the objective value difference variance and the \wf threshold values appears to be somewhat linear, but more work is needed before anything can be deduced.

\subsection{\imlbsp}\label{subsec:results:iml}
In this section we compare \mlbsp and \imlbsp (see parameters in Table~\ref{table:iml}) in the sense of planning-session computation time, the posterior estimation error upon reaching the goal, and the covariance norm upon reaching the goal. For comparison we performed 1000 rollouts (entire mission run), each with a different sampled ground-truth for the prior state. Figure~\ref{fig:iml:map} presents the scenario on which all rollouts were performed, along with the estimation results of an arbitrary rollout. 
\begin{table}[h!]
	\caption{Parameters for Section~\ref{subsec:results:iml} following \imlbsp} 
\begin{center} \label{table:iml}
  \begin{tabular}{ l | c }
    \hline \hline
    Prior belief  standard deviation & $\begin{bmatrix}
    	1^o \cdot I_{3 x 3} & 0 \\
    	0 & 5_{[m]}\cdot I_{3 x 3}
    \end{bmatrix}$ \\ \hline
    Motion Model standard deviation & $\begin{bmatrix}
    	0.5^o \cdot I_{3 x 3} & 0 \\
    	0 & 0.5_{[m]}\cdot I_{3 x 3}
    \end{bmatrix}$  \\ \hline
    Observation Model standard deviation & $\begin{bmatrix}
    	3_{[px]}  & 0 \\
    	0 & 3_{[px]}
    \end{bmatrix}$  \\ \hline
    Camera Aperture & $90^o$  \\ \hline
    Camera acceptable Sensing Range & between $2_{[m]}$ and $40_{[m]}$ \\ \hline
    \texttt{useWF} & \texttt{false} \\ \hline
    $\epsilon_c$ & 250 \\ \hline
    $\beta_{\sigma}$ & 1.5 \\ \hline
    $n_u$ & 3 \\ \hline
    $n_x$ & 1 \\ \hline
    $n_z$ & 1 \\ \hline
    action primitives & left, right and forward with $1_{[m]}$ translation and  $\pm 90^o$ rotations\\ \hline
    $\mathbb{D}$ & \JD \\ \hline
    \hline
  \end{tabular}
\end{center}
\end{table} 

A robot equipped with a stereo camera, is tasked with reaching two goals, numbered and denoted by blue dots in Figure~\ref{fig:iml:map}, in a world with 45 randomly placed landmarks, denoted by green crosses in Figure~\ref{fig:iml:map}, while considering the reward function (\ref{eq:simulationReward}). Same as in Section~\ref{subsec:results:ibsp}, both \mlbsp and \imlbsp consider 6 DOF robot pose, 3 DOF landmarks, a joint state comprised of both robot pose and landmarks and three candidate actions for each step (left, right, forward), hence for the considered horizon of 3 look ahead steps, there are 27 candidate action sequences for each planning session. Under the ML assumption, both methods, \mlbsp and \imlbsp, consider a single measurement per action per look ahead step. While for \mlbsp, this action is always the most likely measurement, hence zero innovation, for \imlbsp this measurement is usually not the most likely measurement, hence we retain some innovation along the look ahead steps. 

Similarly to \fullbsp and \ibsp, when considering a plan-act-infer system, all differences between \mlbsp and \imlbsp are confined within the planning block, hence the computation time of the planning process is adequate for fair comparison.

Figure~\ref{fig:iml:time} presents a box-plot for the timing data of all 1000 rollouts, with 3 and 2 outliers for \mlbsp and \imlbsp respectively, where the computation time advantage is in favor of \imlbsp by a factor of 5. The significant reduction in computation time is originated in the fact that \imlbsp performs inference update in a more efficient way, computation wise, compared to \mlbsp. By considering previously calculated beliefs, and utilizing one of which for efficient inference update instead of performing inference from scratch as done in \mlbsp.

Since we claim to provide a more efficient paradigm than \mlbsp, we also verify how \imlbsp favors in estimation results. Figure~\ref{fig:ibsp:acc} presents a box plot of the estimation error upon reaching the goal for each of the methods. The estimation error was calculated using (\ref{eq:accCalc}).

As can be seen in Figure~\ref{fig:iml:acc} both methods average around an estimation error of $3.5_{[m]}$, while only \mlbsp is with a single outlier. In $51.0\%$ of the rollouts, \imlbsp provided with a better estimation error than \mlbsp. The large estimation variance that can be seen in Figure~\ref{fig:iml:acc} is similar to the one obtained in Figure~\ref{fig:ibsp:acc}, so it provides some assurance regarding the assumption it is the result of using a small number of samples for estimating the objective. Nonetheless, also here, the empiric estimation variance of both methods can be considered as statistically identical for all practical purposes. Of course a more rigorous examination in required, by analytically comparing the estimation variance, we leave this for future work as well. We push further and compare the covariance norm of the final pose. As can be seen in Figure~\ref{fig:iml:cov}, they average around $5.0_{[m]}$, without any outliers, and can be clearly considered as statistically identical for all practical purposes.

\begin{figure}[h!]
	\centering
		\subfloat[]{\includegraphics[trim={0 0 0 0},clip, width=0.25\columnwidth]{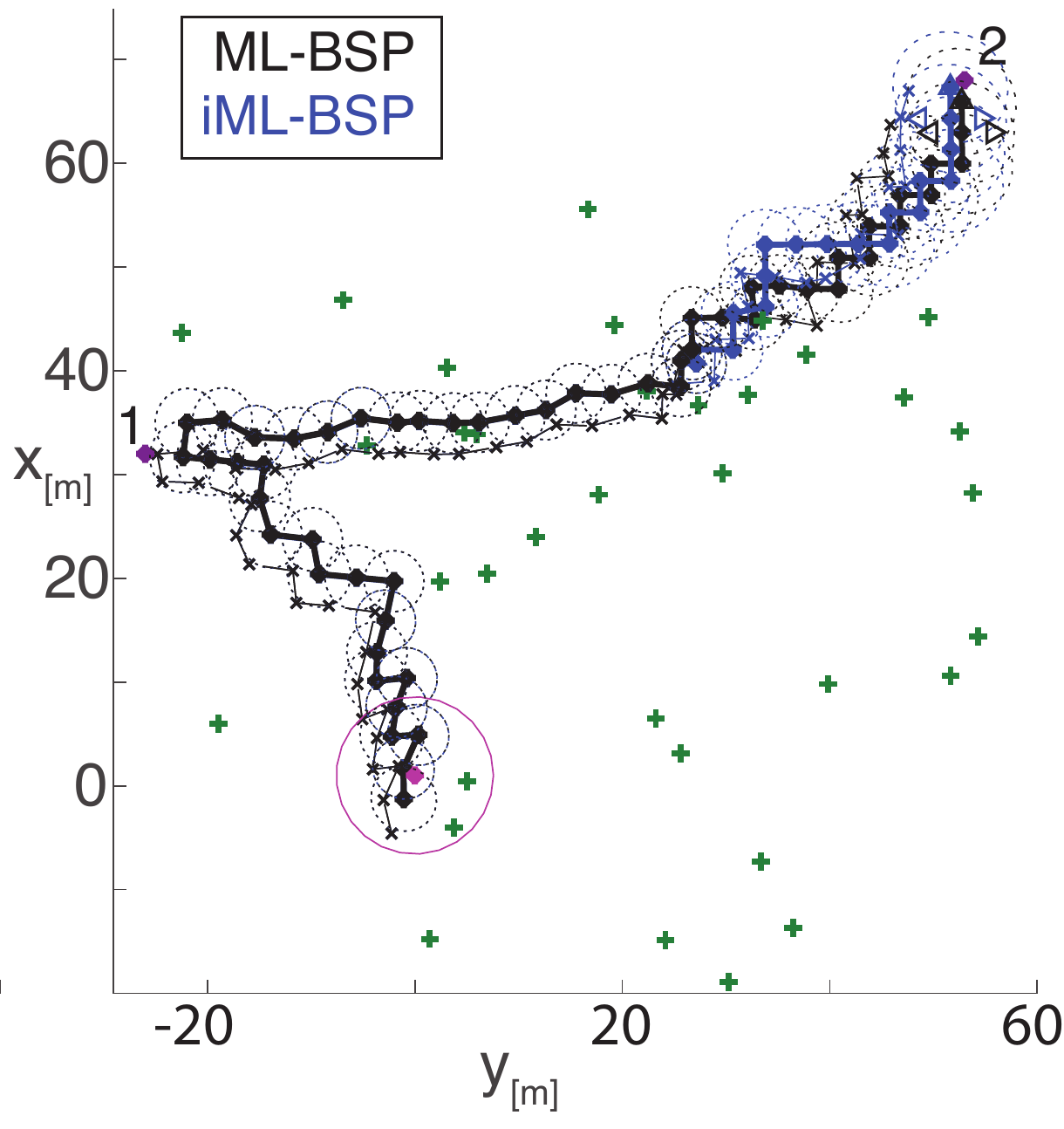}\label{fig:iml:map}}
        \subfloat[]{\includegraphics[trim={0 0 0 0},clip, width=0.25\columnwidth]{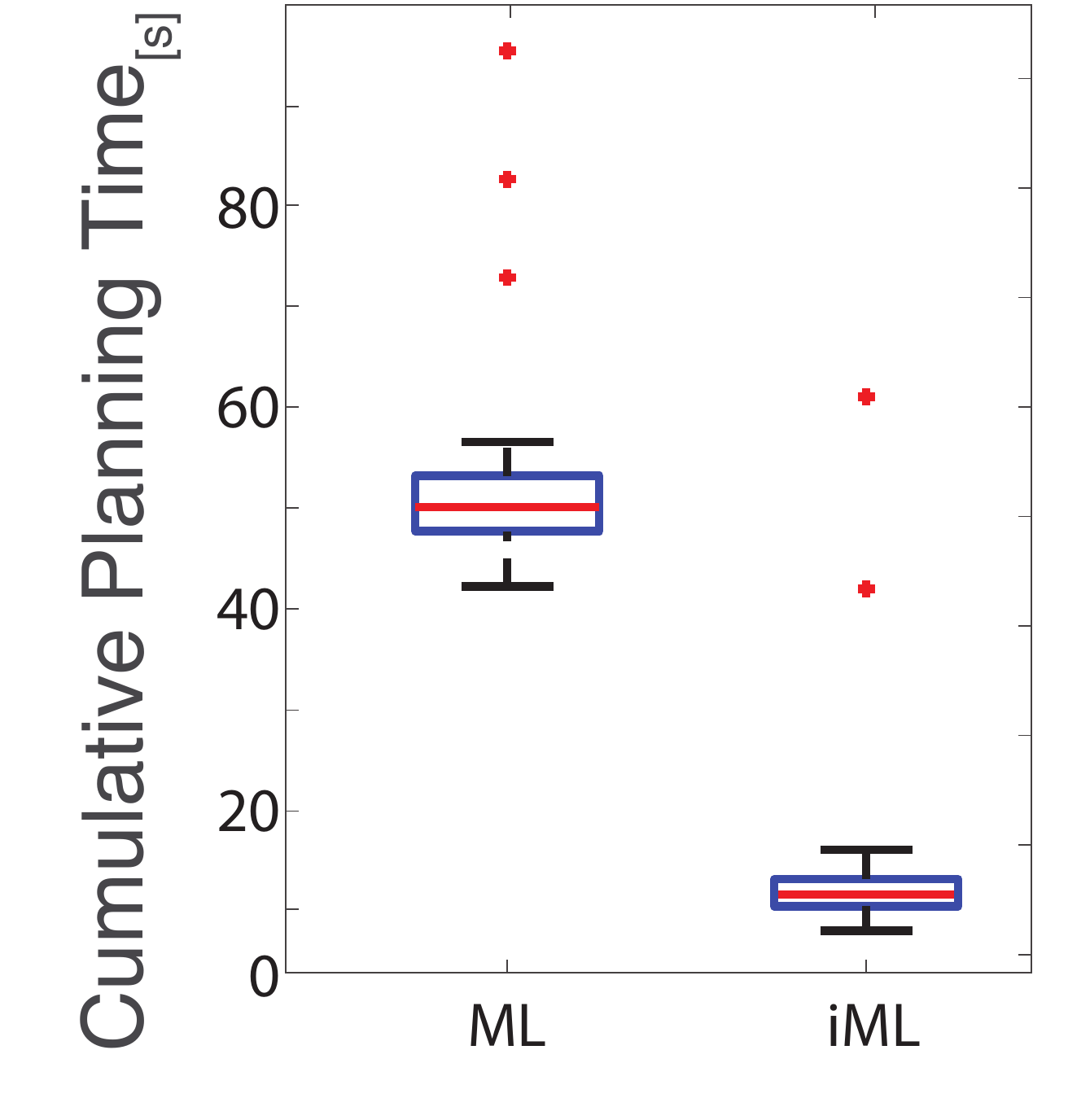}\label{fig:iml:time}}
        \subfloat[]{\includegraphics[trim={0 0 0 0},clip, width=0.25\columnwidth]{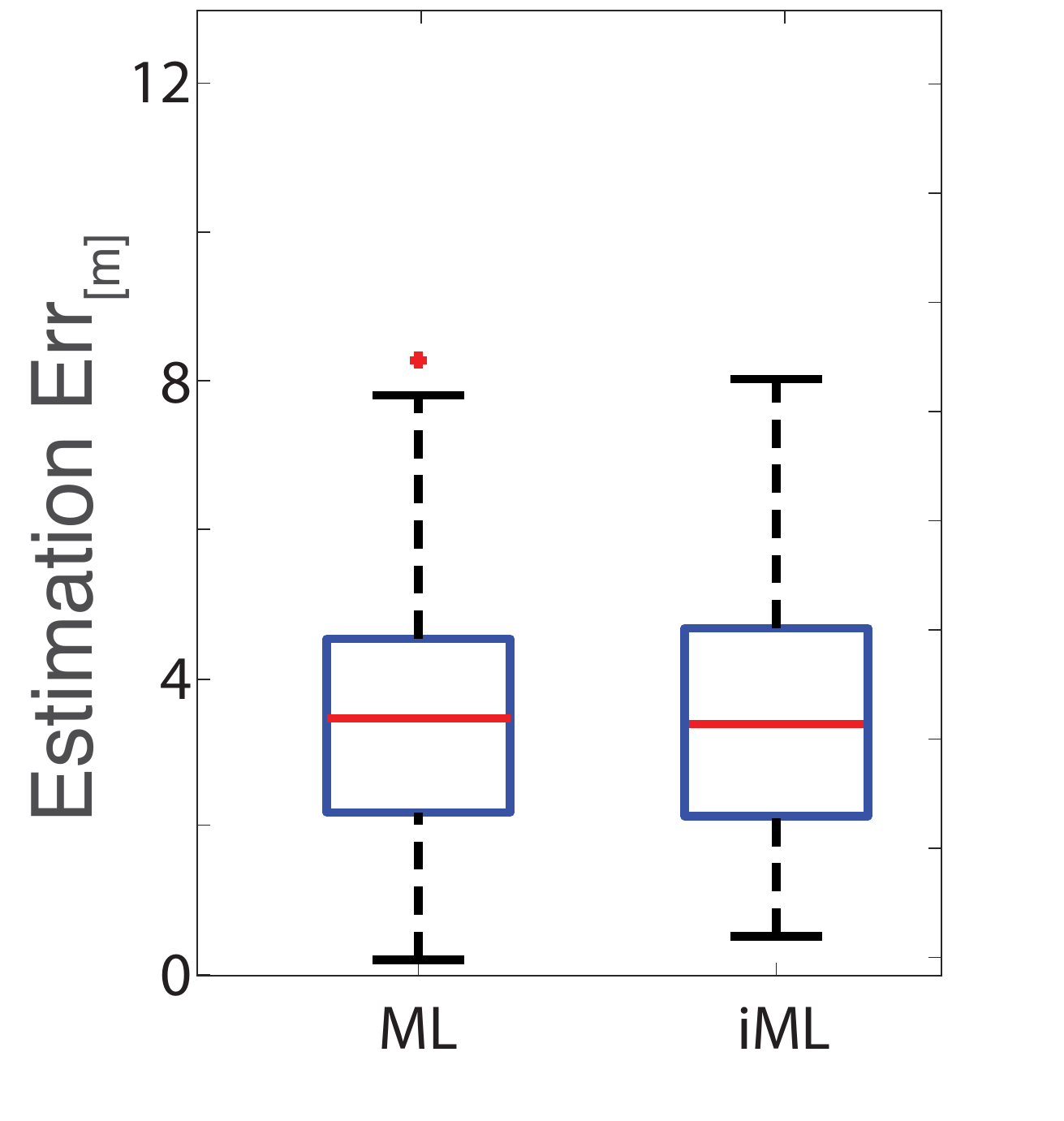}\label{fig:iml:acc}}
        \subfloat[]{\includegraphics[trim={0 0 0 0},clip, width=0.25\columnwidth]{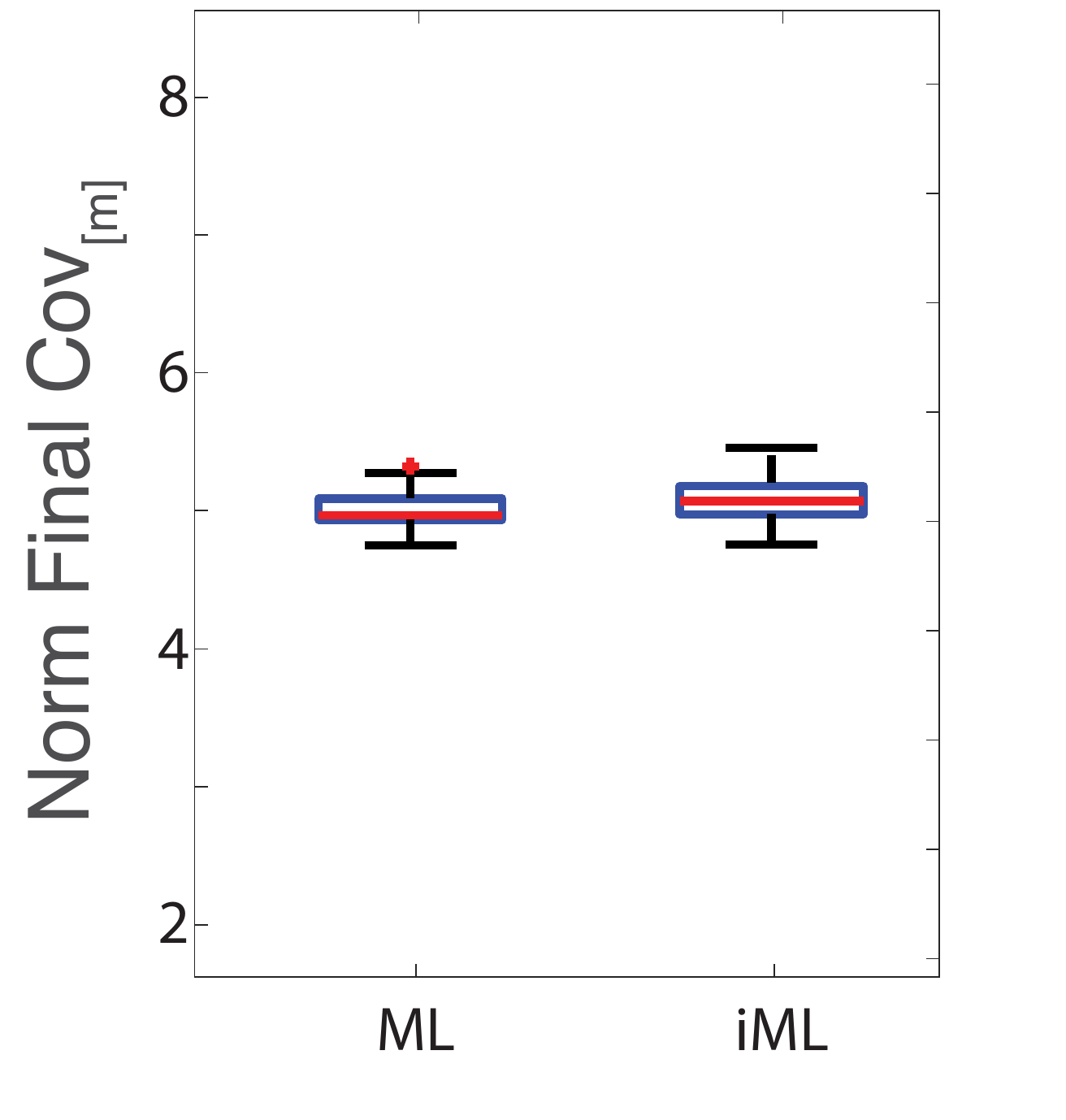}\label{fig:iml:cov}}
        \caption{Statistical comparison between \mlbsp and \imlbsp:(a) The scenario used for the statistical comparison, with 2 numbered goals denoted by blue dots, and 45 landmarks denoted by green crosses. Showcasing the estimation results for one of the 1000 performed rollouts. (b) (c) and (d) Box plots of 1000 rollouts for planning session timing (b), posterior estimation error upon reaching the goal (c) and the covariance norm upon reaching the goal.}
        \label{fig:iml}
\end{figure}

\subsection{\imlbsp in a Real world Experiment}\label{subsec:results:real_iml}
In this section we compare \mlbsp and \imlbsp in a real-world setting. In the following we describe the scenario on which we ran these experiments (Section~\ref{ssubsec:live:scenario}), as well as the results of the two live experiments (Section~\ref{ssubsec:live:results}). All relevant parameters used for these experiments are summarized in Table~\ref{table:live_param}.
\begin{table}
	\caption{Parameters for Section~\ref{subsec:results:real_iml} following \imlbsp} 
\begin{center} \label{table:live_param}
  \begin{tabular}{ l | c }
    \hline \hline
    Prior belief  standard deviation & $\begin{bmatrix}
    	1^o \cdot I_{3 x 3} & 0 \\
    	0 & 5_{[m]}\cdot I_{3 x 3}
    \end{bmatrix}$ \\ \hline
    Motion Model standard deviation & $\begin{bmatrix}
    	0.5^o \cdot I_{3 x 3} & 0 \\
    	0 & 0.5_{[m]}\cdot I_{3 x 3}
    \end{bmatrix}$  \\ \hline
    Observation Model standard deviation & $\begin{bmatrix}
    	3_{[px]}  & 0 \\
    	0 & 3_{[px]}
    \end{bmatrix}$  \\ \hline
    Camera Aperture & $90^o$  \\ \hline
    Camera acceptable Sensing Range & between $2_{[m]}$ and $40_{[m]}$ \\ \hline
    \texttt{useWF} & \texttt{false} \\ \hline
    $\epsilon_c$ & 250 \\ \hline
    $\beta_{\sigma}$ & 1.5 \\ \hline
    planning horizon & 4 \\ \hline
    $n_u$ & 3 \\ \hline
    $n_x$ & 1 \\ \hline
    $n_z$ & 1 \\ \hline
    action primitives & left, right and forward with $1_{[m]}$ translation and  $\pm 45^o$ rotations\\ \hline
    $\mathbb{D}$ & \JD \\ \hline
    ORBSLAM2 & default parameters \\ \hline
    \hline
  \end{tabular}
\end{center}
\end{table} 

\subsubsection{The Scenario}\label{ssubsec:live:scenario}
For these experiment we used the Pioneer 3AT robot, equipped with a ZED stereo camera, Hokuyo UTM-30LX Lidar, and a Linux machine with octa-core i7-6820HQ 2.7GHz processor and 32GB of memory (see Figure~\ref{fig:live:robot}). 
\begin{figure}[h]
	\centering
		\includegraphics[trim={0 0 0 0},clip, width=0.6\columnwidth]{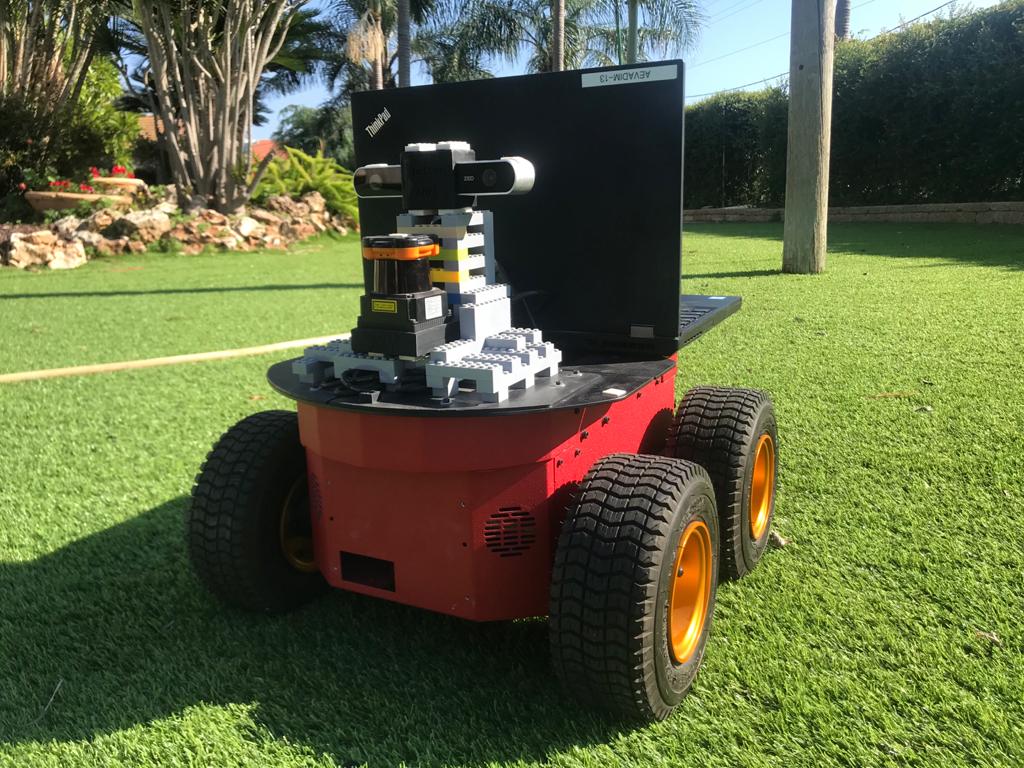}
        \caption{Pioneer 3AT robot used for the live experiments. The robot is equipped with a ZED stereo camera, Hokuyo UTM-30LX Lidar, and a Linux machine with octa-core i7-6820HQ 2.7GHz processor and 32GB of memory.}
        \label{fig:live:robot}
\end{figure}

The robot received a number of goals, his mission - reaching each one of the goals while maximizing information gain and minimizing distance to goal using the reward function (\ref{eq:simulationReward}). The robot state was comprised of both its own poses and environment landmarks. The robot did not possess any prior information over the environment, nor did it make use of any offline calculations. The robot did receive a prior over its initial pose (as stated in Table~\ref{table:live_param}). The environment used for these experiments was the surrounding garden and path-way of a cottage house, laced with both still obstacles (e.g. cars, trees, rocks, children toys, garden furniture) and dynamic obstacles (e.g. people, children, dogs). 

The robot uses a plan-act-infer architecture, similar to the one presented in Figure~\ref{fig:HighLevel_Algo}. The planner is the exact same MATLAB code used for the aforementioned simulations where the rest of the code is C++ based. 
The Hokuyo UTM-30LX Lidar was used for collision avoidance and odometry while the ZED camera was used strictly for its rectified stereo images output.  
We also made use of ORBSLAM2 as the vision pipeline, feeding the rectified stereo images to ORBSLAM2 and extracting the appropriate factors created by ORBSLAM2 as output. 

Future landmark observations are generated by considering only landmarks projected within the camera field of view using posterior estimates for landmark positions and camera pose. As in this work the planning phase considers only the already-mapped landmarks, without reasoning about expected new landmarks, each new landmark observation in inference would essentially mean facing a factor that can not be re-used.

As this is not a simulated environment, where the uncertainties can be replicated, in order to provide a fair comparison between \mlbsp and \imlbsp each planning session the robot performs both \mlbsp and \imlbsp sequentially, using the same posterior information. The planning duration is timed for comparison, the optimal action given by \imlbsp is being used as the next action, and the optimal action given by \mlbsp is matched against former for comparison.

\subsubsection{Real World Results}\label{ssubsec:live:results}
In this section we cover the results of live experiments done over two different sets of goals following the scenario described above. Differently than the simulation related results presented so far, we compare the computation time of the entire planning horizon and do not omit the last horizon step. 
Figure~\ref{fig:live:tiny:map} presents the estimated route of the first experiment, coursing through 4 goals along 35 meters for under both \mlbsp and \imlbsp. As the two methods chose the same optimal action sequences throughout the mission, they have an identical estimation (up to some machine noise). Figure~\ref{fig:live:tiny:time:tot} presents a bar plot of the cumulative planning time for the entire mission for both methods, where each bar is divided into the contribution of the first three horizon steps (denoted in gray) and the contribution of the last horizon step (denoted in black). 
This division is meant to help the reader compare the live experiment to the simulation results which omitted the shared last horizon step computation time, while still assessing the overall reduction in computation time. 
The percentage in Figure~\ref{fig:live:tiny:time:tot} represents the relative contribution of the two segments to each of the cumulative planning times. 
Although the last horizon remains unchanged, we can clearly see the computation time reduction in \imlbsp when compared to \mlbsp from forming $59\%$ of the computation time to only $23\%$.
Figure~\ref{fig:live:tiny:time:perStep} presents the planning time per planning session. As expected, in the first planning session considering a new goal \imlbsp performs a regular \mlbsp planning hence both computation times are identical and no factors are re-used. 

Figure~\ref{fig:live:tiny:fac} suggests some insight on the timing result of \imlbsp by comparing the number of factors involved in the computation of each method. 
Figure~\ref{fig:live:tiny:fac:count} the sum of added factors per planning session. In blue the number of factors added at time $k+1|k+1$ as part of standard Bayesian inference update. In red the portion of aforementioned factors that are already part of the state prior to the inference update, and as in this work we do not make use of any mechanism to predict new states it also represents the upper bound for the number of factors we can hope to re-use.
In orange the number of factors that were originally calculated in a previous planning session and where reused.
The difference between the orange and blue lines represents the number of factors needed to be added to the re-used planning tree in order to match the posterior at $k+1$ and the black line in Figure~\ref{fig:live:tiny:fac:rem} represents the number of factors needed to be removed from the re-used planning tree. 
While the orange line is quite close to the red upper bound, there are still a lot of factors needed to be removed in order to match the posterior at $k+1$, which contributes to DA update computation time in \imlbsp. 
Using a more sophisticated prediction mechanism for future factors might reduce this overhead in removed factors and save more valuable computation time without introducing an approximation, we leave this for future work.
\begin{figure}[H]
	\centering
		\includegraphics[trim={0 0 0 0},clip, width=0.95\columnwidth]{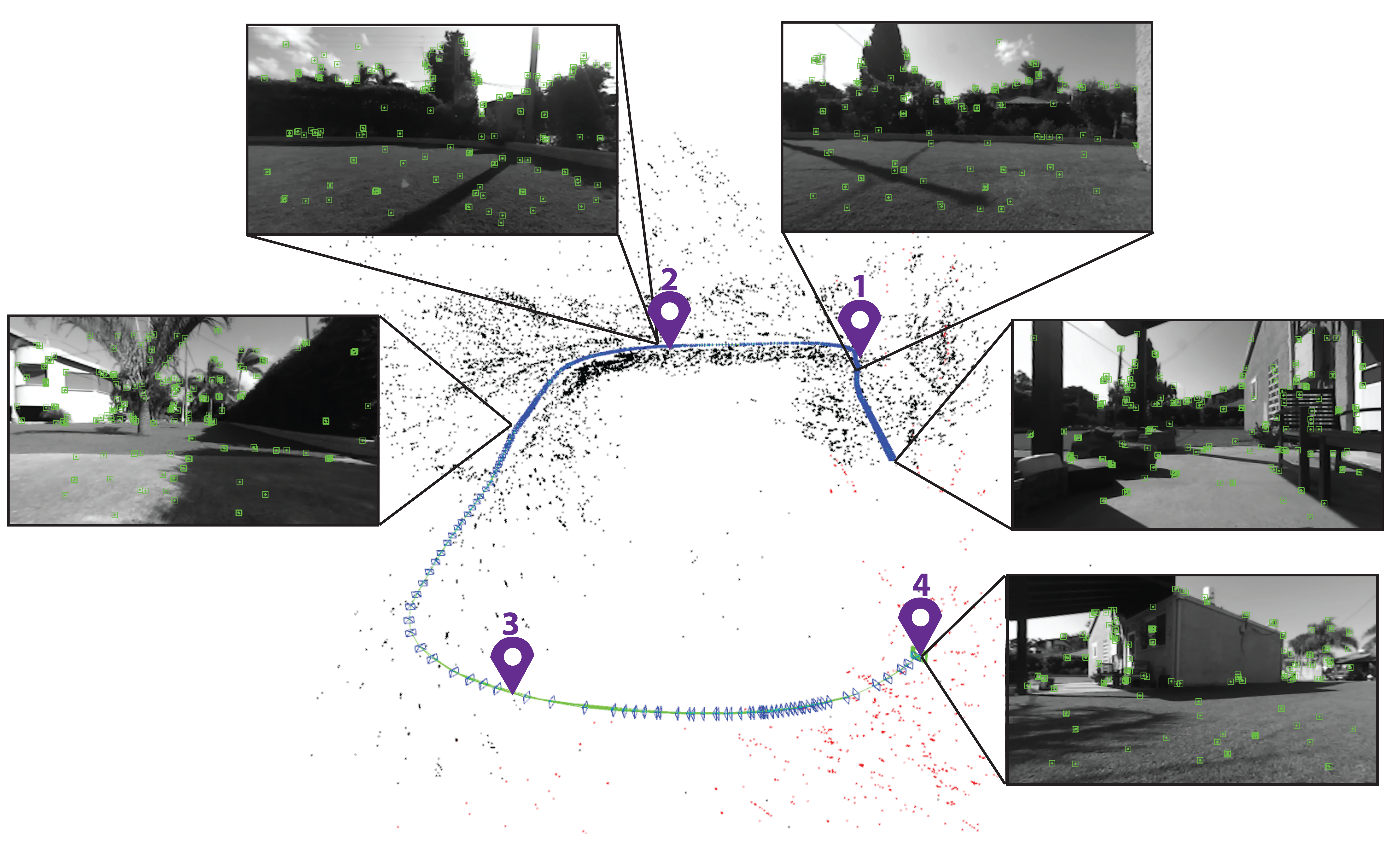}
        \caption{The first live experiment stretching across a 35 meter course, where the Pioneer robot was given 4 goals (numbered and denoted in purple) to reach. The state estimation of the robot in the form of keyframes (denoted in blue frames), robot trajectory (denoted in green) landmarks (denoted in black and red), and few snapshots along the rout containing the selected features. }
        \label{fig:live:tiny:map}
\end{figure}
\begin{figure}[H]
	\centering
		\subfloat[]{\includegraphics[trim={0 0 0 0},clip, width=0.42\columnwidth]{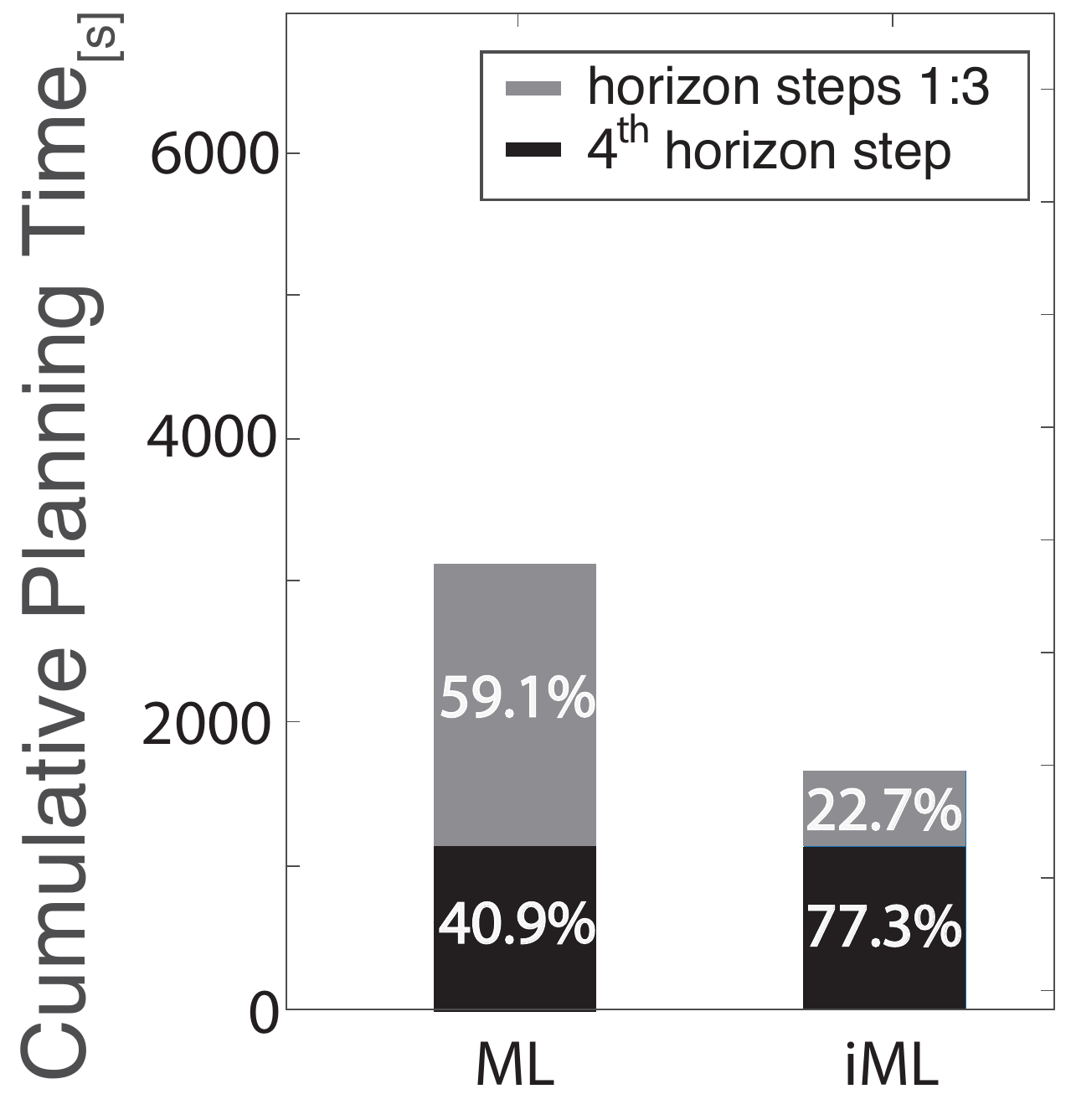}\label{fig:live:tiny:time:tot}}
		\subfloat[]{\includegraphics[trim={0 0 0 0},clip, width=0.58\columnwidth]{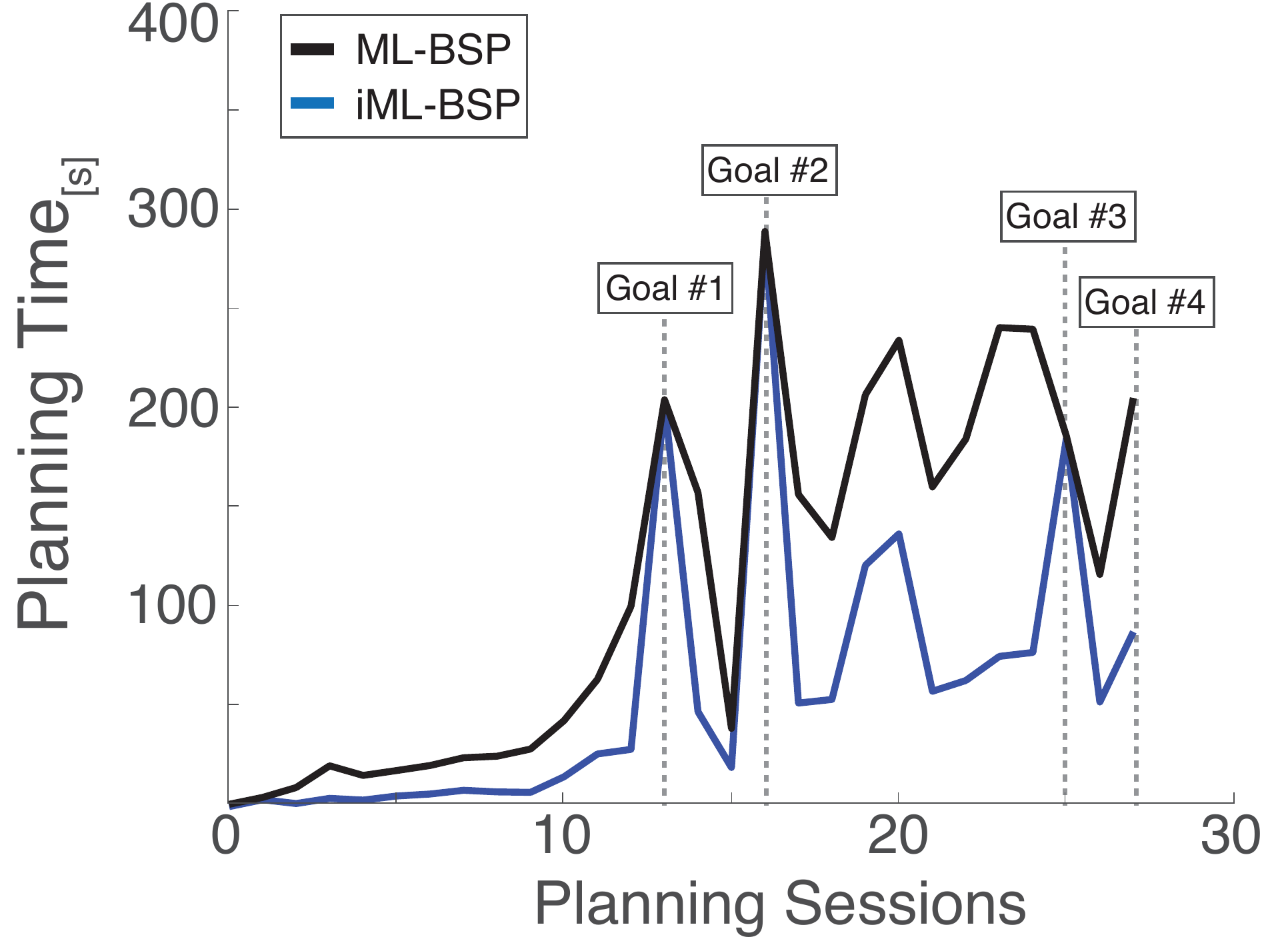}\label{fig:live:tiny:time:perStep}}
        \caption{Planning time results of the first live experiment using the Pioneer robot (a) The cumulative planning time of both \mlbsp and \imlbsp, divided into relative contributions of the first three planning horizon steps (denoted in gray) and the last planning horizon step (denoted in black). (b) Per planning session comparison of the computation time. The goals are marked over the planning session performed after reaching them.}
        \label{fig:live:tiny:time}
\end{figure}
\begin{figure}[H]
	\centering
		\subfloat[]{\includegraphics[trim={0 0 0 0},clip, width=0.5\columnwidth]{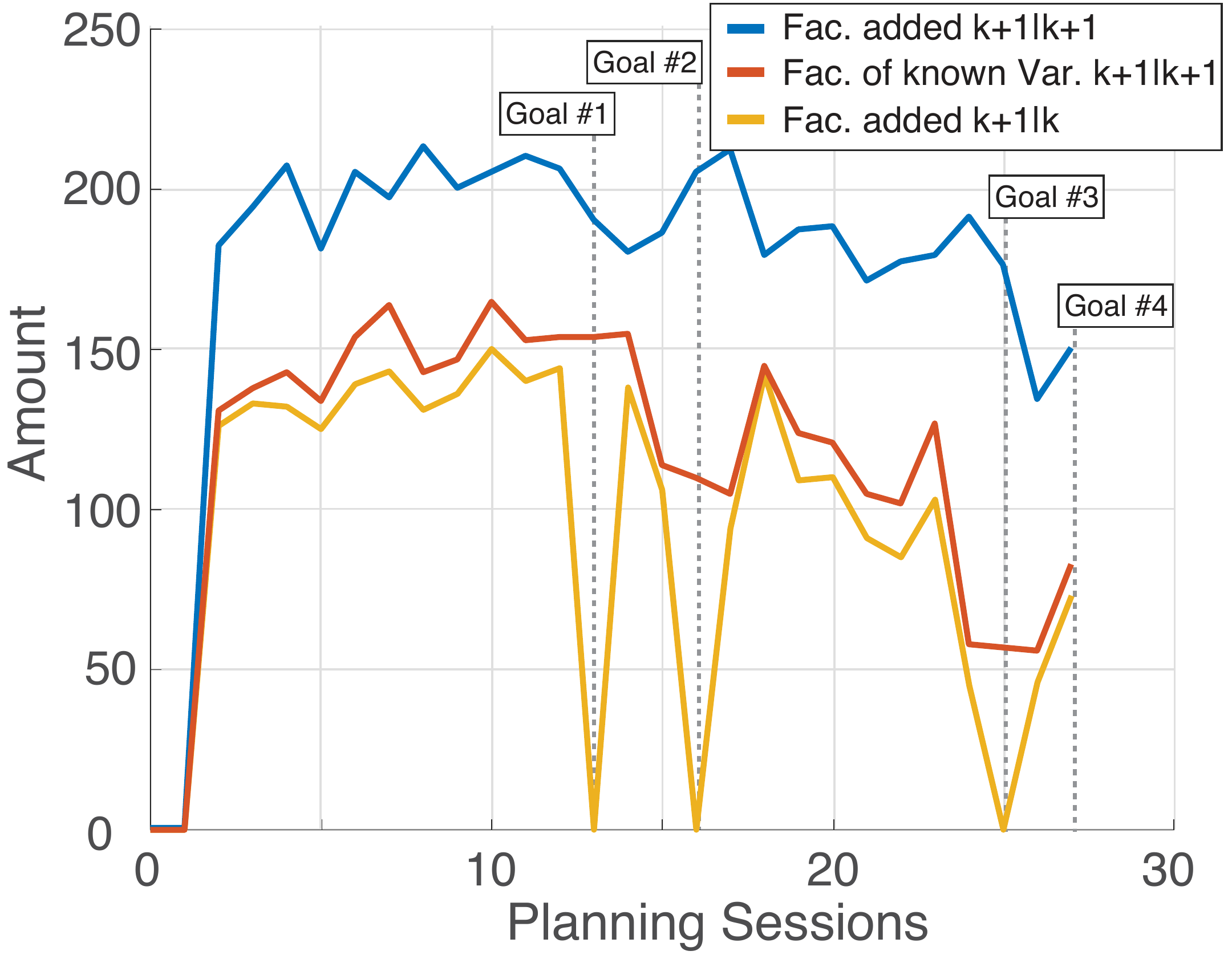}\label{fig:live:tiny:fac:count}}
		\subfloat[]{\includegraphics[trim={0 0 0 0},clip, width=0.5\columnwidth]{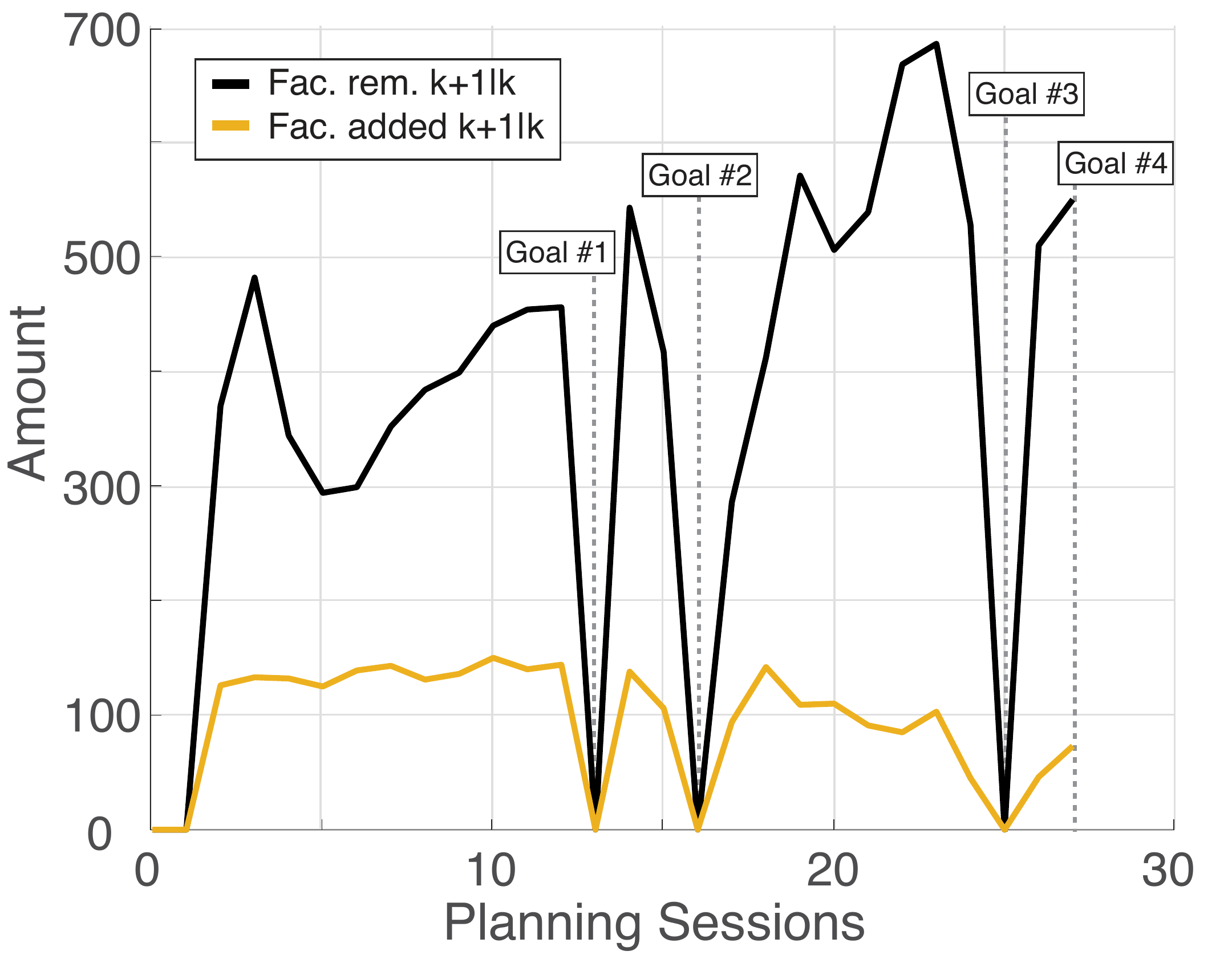}\label{fig:live:tiny:fac:rem}}
        \caption{Number of factors involved in each planning session of the first live experiment. (a) The number of factors added in the last inference update at $k+1|k+1$, denoted in blue. The portion of aforementioned factors which relate to existing states, denoted in red. The number of factors re-used from previously calculated planning tree, denoted in orange. (b) The number of factors re-used from previously calculated planning tree, denoted in orange. The number of factors removed from the previously calculated planning tree, denoted in black.}
        \label{fig:live:tiny:fac}
\end{figure}

Figure~\ref{fig:live:large:map} presents the estimated route of the second experiment, coursing through 3 goals along 148 meters under both \mlbsp and \imlbsp. In this experiment, like in the former, both \mlbsp and \imlbsp chose the same optimal action sequence at each planning session. Similar to Figure~\ref{fig:live:tiny:time}, Figure~\ref{fig:live:large:time} presents the timing results for the second experiment. Same as before we can see that while the computation time related to the last horizon step is identical between \mlbsp and \imlbsp, there is a considerable time reduction in the computation time related to the first three horizon steps, from constituting $54\%$ of the cumulative planning time in \mlbsp to just $25\%$ in \imlbsp.

In a similar manner Figure~\ref{fig:live:large:fac}, like Figure~\ref{fig:live:tiny:fac}, presents the sum of factors related to the second experiment. As before, we can see that the orange line is quite close to the red line (i.e. upper bound for factor re-use), but there is a considerable number of factor to remove each planning step (black line in Figure~\ref{fig:live:large:fac:rem}). As the second experiment provides us with the same insights over the comparison between \mlbsp and \imlbsp it essentially validates the results of the first experiment as well as the insights derived from it. 

\begin{figure}[H]
	\centering
		\includegraphics[trim={0 0 0 0},clip, width=0.95\columnwidth]{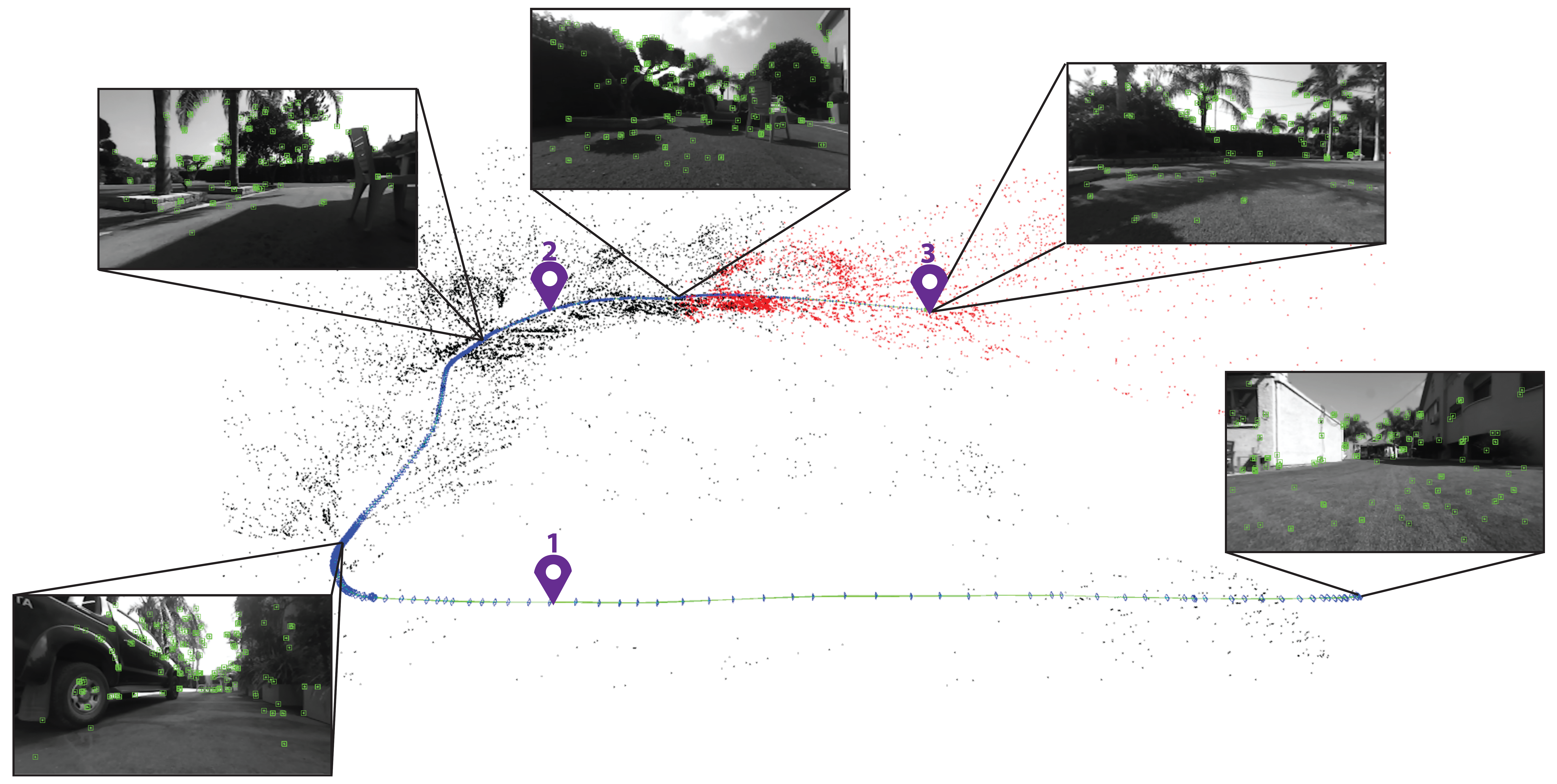}
        \caption{The second live experiment stretching across a 148 meter course, where the Pioneer robot was given 3 goals (numbered and denoted in purple) to reach. The state estimation of the robot in the form of keyframes (denoted in blue frames), robot trajectory (denoted in green) landmarks (denoted in black and red), and few snapshots along the rout containing the selected features. }
        \label{fig:live:large:map}
\end{figure}
\begin{figure}[H]
	\centering
		\subfloat[]{\includegraphics[trim={0 0 0 0},clip, width=0.42\columnwidth]{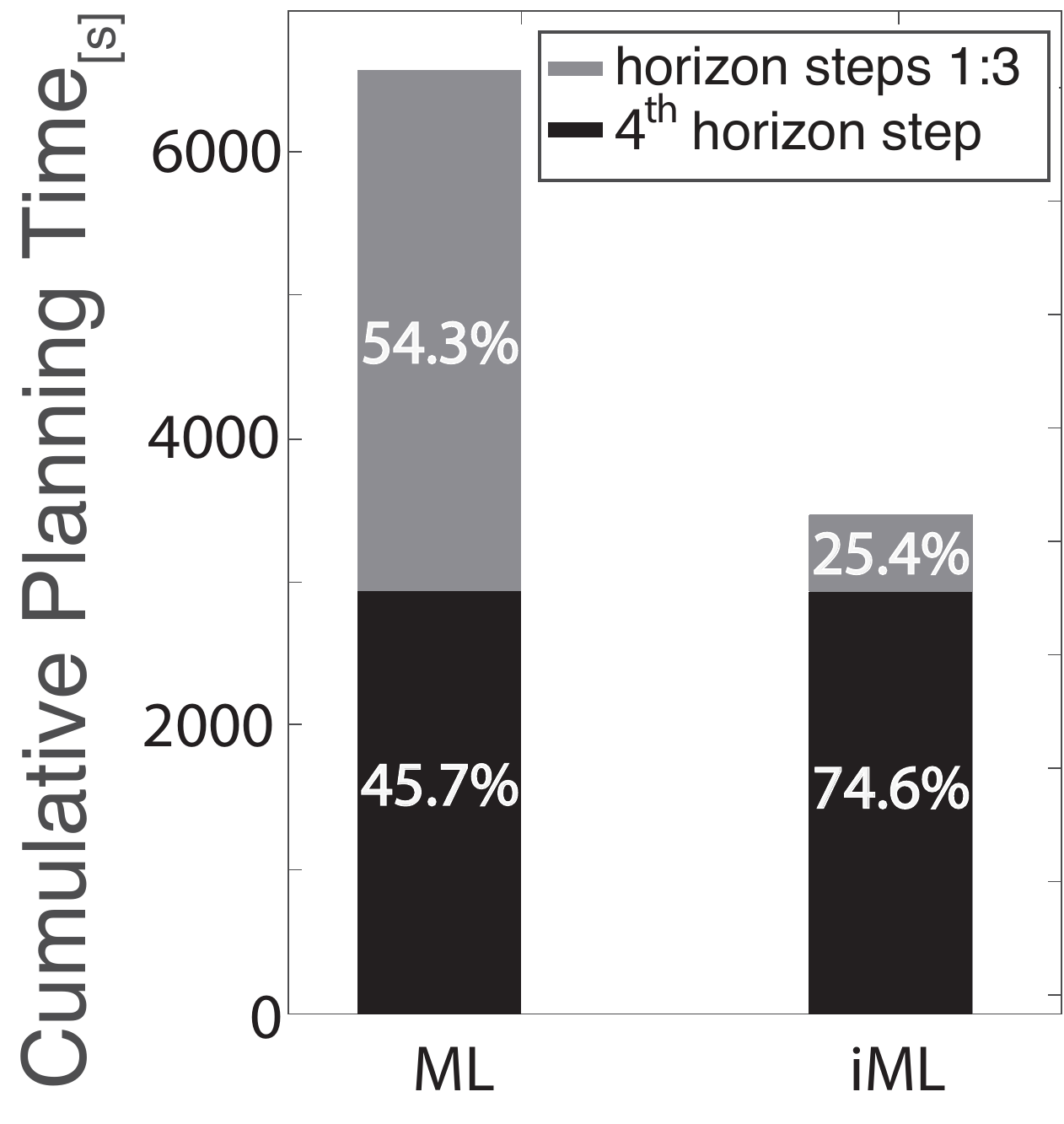}\label{fig:live:large:time:tot}}
		\subfloat[]{\includegraphics[trim={0 0 0 0},clip, width=0.58\columnwidth]{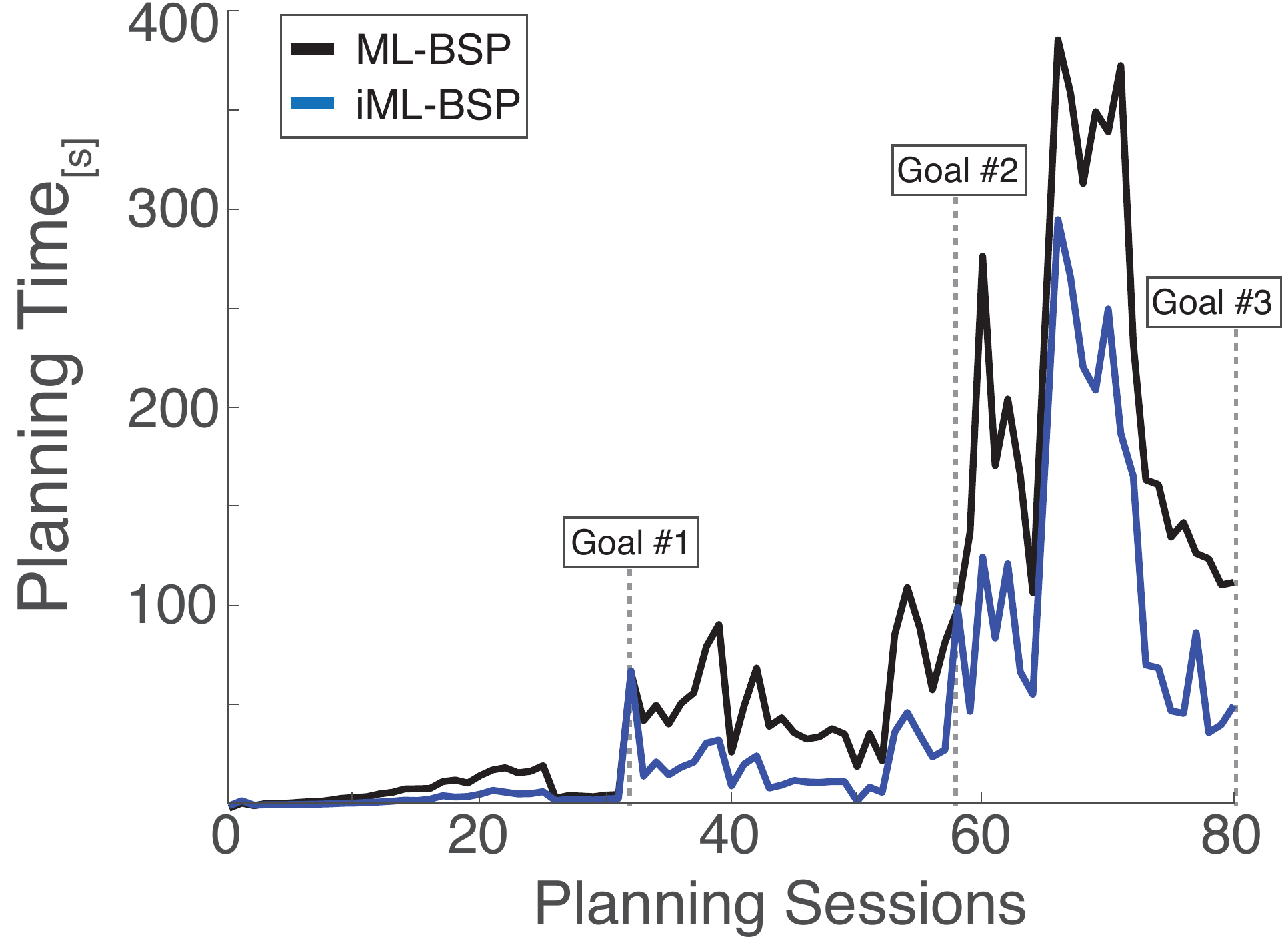}\label{fig:live:large:time:perStep}}
        \caption{Planning time results of the second live experiment using the Pioneer robot (a) The cumulative planning time of both \mlbsp and \imlbsp, divided into relative contributions of the first three planning horizon steps (denoted in gray) and the last planning horizon step (denoted in black). (b) Per planning session comparison of the computation time. The goals are marked over the planning session performed after reaching them.}
        \label{fig:live:large:time}
\end{figure}
\begin{figure}[H]
	\centering
		\subfloat[]{\includegraphics[trim={0 0 0 0},clip, width=0.5\columnwidth]{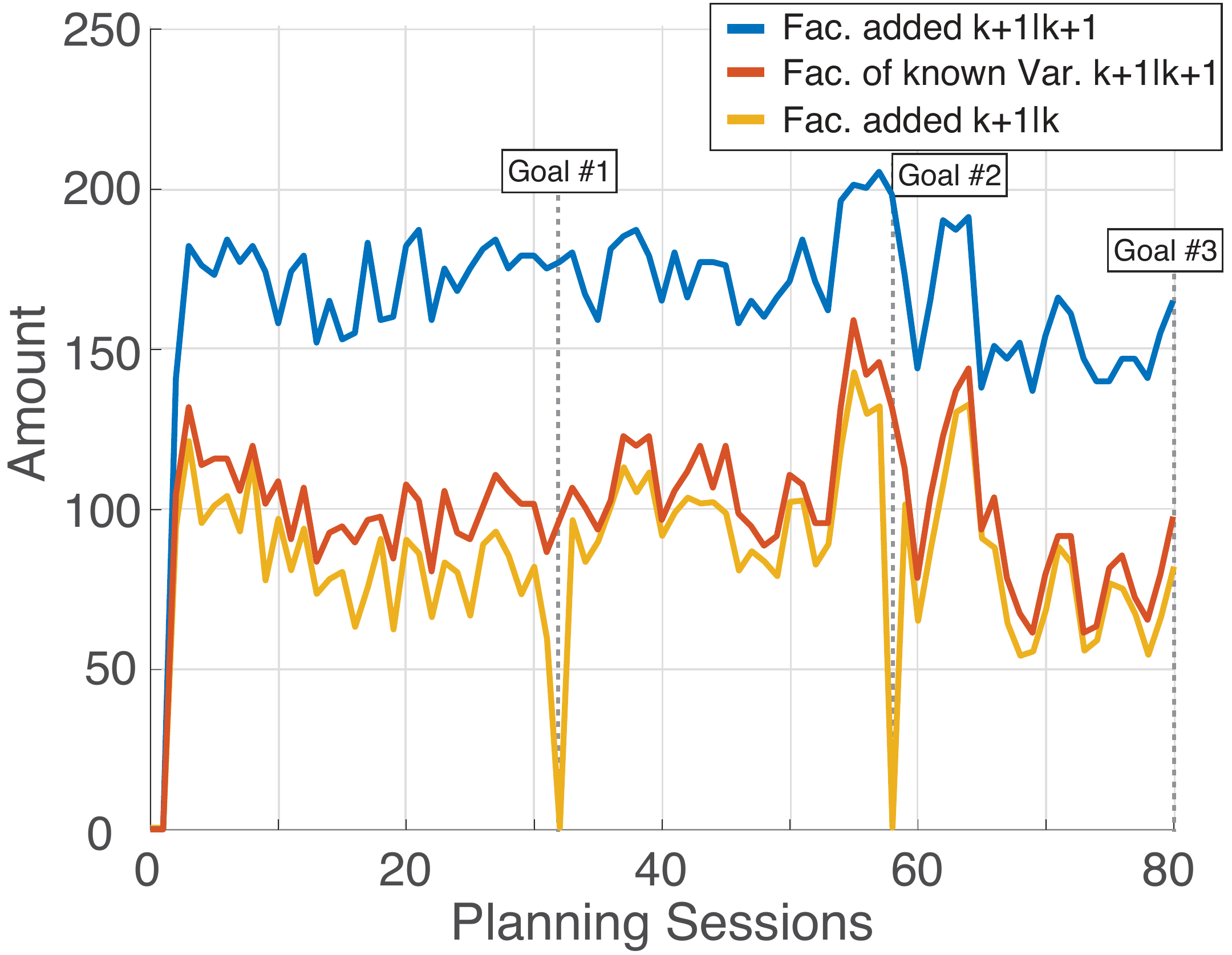}\label{fig:live:large:fac:count}}
		\subfloat[]{\includegraphics[trim={0 0 0 0},clip, width=0.5\columnwidth]{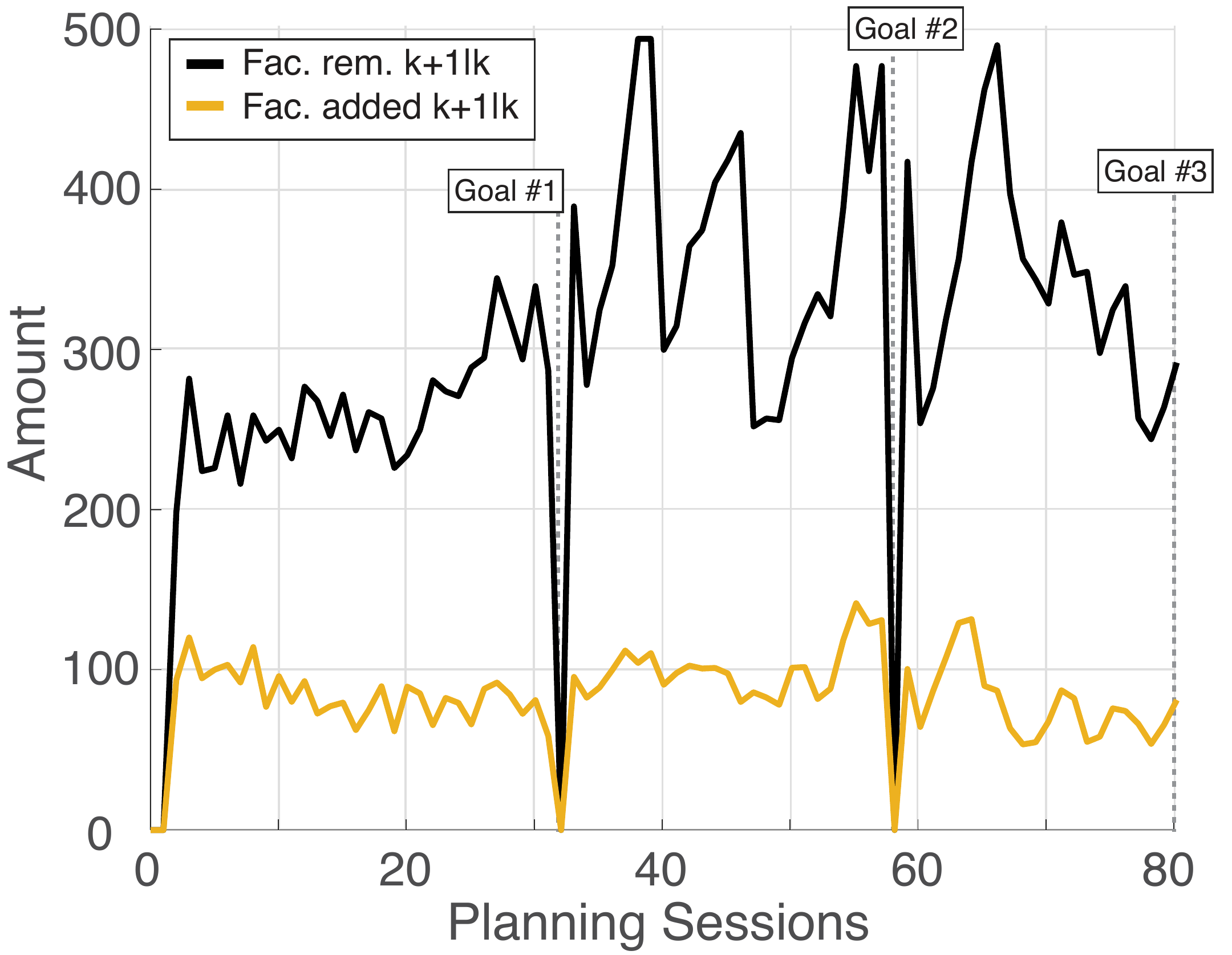}\label{fig:live:large:fac:rem}}
        \caption{Number of factors involved in each planning session of the second live experiment. (a) The number of factors added in the last inference update at $k+1|k+1$, denoted in blue. The portion of aforementioned factors which relate to existing states, denoted in red. The number of factors re-used from previously calculated planning tree, denoted in orange. (b) The number of factors re-used from previously calculated planning tree, denoted in orange. The number of factors removed from the previously calculated planning tree, denoted in black.}
        \label{fig:live:large:fac}
\end{figure}


\vspace{-5pt}
\section{Conclusions}
\label{sec:conclusions}
State of the art approaches under \fullbsp paradigm (BSP with expectation) calculate each planning session from scratch. In this paper we suggest to re-formulate the general problem of BSP using expectation, \fullbsp. We present \ibsp, incrementally calculating the expectation by utilizing previous planning sessions in order to solve the current planning session with a reduced computational effort and without affecting the solution accuracy.

By selectively re-using sampled measurements from previous planning sessions, we are able to avert from standard Bayesian inference as part of reward(cost) values thus reducing the computational effort. 
As the expectation in \ibsp is potentially considered over a set of samples taken from multiple different measurement likelihood distributions, we reformulate \fullbsp as a Multiple Importance Sampling (MIS) problem, thus statistically maintaining the solution accuracy. 
Considering their stochastic nature, we evaluate \ibsp against \fullbsp in simulation considering active-SLAM as application, comparing both cumulative planning computation time and estimation accuracy upon reaching the goal. By considering different sampled ground-truth prior states we are able to show that \ibsp is statistically equal to \fullbsp whilst providing shorter computation time.

In addition to providing with the full formulation of \ibsp, we introduce a non-integral approximation denoted as \wf, enabling one to trade accuracy for computation time by putting a threshold defining beliefs as "close enough" to be considered as identical. We also analyze, analytically and empirically, the effect \wf holds over the resulting objective value, as well as demonstrate using \wf in \ibsp.
Because \ibsp changes the solution approach of the original, un-approximated, problem (\fullbsp), we believe it can be utilized to also reduce computation time of existing approximations of \fullbsp. To support this claim we push further and show how it can be utilized to benefit an existing common approximation of \fullbsp. We present \imlbsp which provides reduced computation time compared to \mlbsp, while obtaining the same estimation accuracy. We demonstrate the performance advantage \imlbsp holds over \mlbsp in both simulation and real-world experiments considering vision-based active-SLAM in previously unknown uncertain environment and high dimensional state space.

In contrast to common research directions dealing with approximations of the \fullbsp problem, this paper tackles the un-approximated formulation of \fullbsp and suggests to improve it by considering calculation re-use across planning sessions, thus enabling to reduce \fullbsp computation time without affecting accuracy. As \ibsp is equivalent to \fullbsp, we claim the existing approximations of \fullbsp could benefit from the \ibsp formulation, as demonstrated by considering the ML approximation \mlbsp and the resulting \imlbsp.
While demonstrated here only on data from a precursory planning session, using the same formulation it can be easily shown that \ibsp can selectively re-use any supplied data, from a set of offline calculated beliefs to planning sessions of other agents, all while maintaining the same estimation accuracy as well as the computational advantage.
The performance of \ibsp can be further improved without introducing any approximations, e.g. by using a more sophisticated prediction mechanism to reduce the number of removed factors, or by using a mechanism to predict factors related to new states.

\vspace{-10pt}
\bibliographystyle{IEEEtran}
\bibliography{paper.bib}
\appendix
\section{Multiple Importance Sampling}
\label{app:IS}
Let us assume we wish to express expectation over some function $f(x)$ with respect to distribution $p(x)$, by sampling $x$ from a different distribution $q(x)$,
\begin{equation}\label{eq:basicIM}
	\mathbb{E}_p f(x) = \int f(x)\cdot p(x) dx = \int \frac{f(x)\cdot p(x)}{q(x)}q(x) dx = \mathbb{E}_q \left( \frac{f(x)\cdot p(x)}{q(x)} \right).
\end{equation}
Eq.~(\ref{eq:basicIM}) presents the basic importance sampling problem, where $\mathbb{E}_q$ denotes expectation for $x \sim q(x)$. The probability ratio between the nominal distribution $p$ and the importance sampling distribution $q$ is usually referred to as the likelihood ratio. Our problem is more complex, since our samples are potentially taken from $M$ different distributions while $M \in [1,(n_x \!\cdot\! n_z)^L]$, i.e. a multiple importance sampling problem 
\begin{equation}\label{eq:MIM}
	\mathbb{E}_p f(x) = \tilde{\mu}(x) \sim \sum^{M}_{m=1} \frac{1}{n_m} \sum^{n_m}_{i=1} w_m(x_{im})\frac{f(x_{im}) p(x_{im})}{q_m(x_{im})},
\end{equation}
where $w_m(.)$ are weight functions satisfying $\Sigma^{M}_{m=1}w_m(x)=1 $, $n_m$ denotes the number of samples from the $M_{th}$ distribution. For ${q_m(x)>0}$ whenever $w_m(x)p(x)f(x)\neq0$, Eq.~(\ref{eq:MIM}) forms an unbiased estimator
\begin{equation}\label{eq:unbiasedMIM}
	\mathbb{E} \left[\tilde{\mu}(x) \right]  = \sum^{M}_{m=1} \mathbb{E}_{q_m} \left[\frac{1}{n_m} \sum^{n_m}_{i=1} w_m(x_{im})\frac{f(x_{im}) p(x_{im})}{q_m(x_{im})} \right] = \tilde{\mu}(x).
\end{equation}
Although there are numerous options for weight functions satisfying $\Sigma^{M}_{m=1}w_m(x)=1 $, we chose to consider the Balance Heuristic \cite{Veach95siggraph}, considered to be nearly optimal in the sense of estimation variance \citep[Theorem~1]{Veach95siggraph},
\begin{equation}\label{eq:balanceHeuristic}
	w_m(x) = w^{BH}_m(x) = \frac{n_m q_m(x)}{\Sigma^{M}_{s=1} n_s q_s(x)}. 
\end{equation}
Using (\ref{eq:balanceHeuristic}) in (\ref{eq:MIM}) produces the multiple importance sampling with the balance heuristic
\begin{equation}\label{eq:MIMBH}
	\mathbb{E}_p f(x) \sim \frac{1}{n} \sum^{M}_{m=1} \sum^{n_m}_{i=1} \frac{ p(x_{im})}{\Sigma^{M}_{s=1} \frac{n_s}{n} q_s(x_{im})}f(x_{im}).
\end{equation}
\section{The \JD Distance}
\label{app:Distance}
In this work we are required to make use of a probability density function (pdf) distance. After some consideration we chose to use \JD, which is a variant of the Jeffreys divergence $\mathbb{D}_J$ first suggested in \cite{jeffreys46rs},
\begin{equation}\label{eq:metricDef}
	\JD(P,Q) = \sqrt{\frac{1}{2}\mathbb{D}_J } = \sqrt{\frac{1}{2}\mathbb{D}_{KL}(P||Q) + \frac{1}{2}\mathbb{D}_{KL}(Q||P)},
\end{equation}
where $P$ and $Q$ are probability density functions and $\mathbb{D}_{KL}(P||Q)$ is the Kullback-Leibler(KL) divergence.

The Kullback-Leibler(KL) divergence, sometime referred to as relative entropy, measures how well some distribution $Q$ approximates distribution $P$, or in other words how much information will be lost if one considers distribution $Q$ instead of $P$. The KL divergence is not a metric (asymetric) and is given by
\begin{equation}\label{eq:kldDef}
	\mathbb{D}_{KL}(P||Q) = \int P \cdot log\frac{P}{Q} = \mathbb{E}_{P} \left[ log P - log Q\right].
\end{equation}
From a view point of Bayesian Inference, as explained in \cite{Endres03tit}, the $\mathbb{D}_{KL}(P||Q)$ metric can be interpreted as twice the expected information gain when deciding between $P$ and $Q$ given a uniform prior over them.

For the special case of Gaussian distributions, we can express $\mathbb{D}_{KL}(P||Q)$ and consequently $\JD(P,Q)$ in terms of means and covariances. Let us consider two multivariate Gaussian distributions $P\sim\mathcal{N}(\mu_p,\Sigma_p)$ and $Q\sim\mathcal{N}(\mu_q,\Sigma_q)$ in $\mathbb{R}^d$.  
\begin{align}
	\nonumber
	\mathbb{D}_{KL}(P||Q) &= \mathbb{E}_{P} \left[ log P - log Q\right] \\
	\nonumber
	&= \frac{1}{2}\mathbb{E}_{P} \left[ -log \vert \Sigma_p \vert -  (x - \mu_p)^T \Sigma_p^{-1} (x - \mu_p) + log \vert \Sigma_q \vert +  (x - \mu_q)^T \Sigma_q^{-1} (x - \mu_q) \right]\\
	\nonumber
	&= \frac{1}{2}log\frac{\vert \Sigma_q \vert}{\vert \Sigma_p \vert} + \frac{1}{2}\mathbb{E}_{P} \left[ -  (x - \mu_p)^T \Sigma_p^{-1} (x - \mu_p) +  (x - \mu_q)^T \Sigma_q^{-1} (x - \mu_q) \right]\\
	\nonumber
	&= \frac{1}{2}log\frac{\vert \Sigma_q \vert}{\vert \Sigma_p \vert} + \frac{1}{2}\mathbb{E}_{P} \left[ - tr( \Sigma_p^{-1}(x - \mu_p)^T(x - \mu_p)) +  tr( \Sigma_q^{-1}(x - \mu_q)^T(x - \mu_q)) \right]\\
	\nonumber
	&= \frac{1}{2}log\frac{\vert \Sigma_q \vert}{\vert \Sigma_p \vert} - \frac{1}{2}tr( \Sigma_p^{-1}\Sigma_p) + \frac{1}{2}\mathbb{E}_{P} \left[ tr\left( \Sigma_q^{-1}(xx^T-2x\mu_q^T+\mu_q\mu_q^T)\right) \right]\\
	\nonumber
	&= \frac{1}{2}log\frac{\vert \Sigma_q \vert}{\vert \Sigma_p \vert} - \frac{1}{2}d_p + \frac{1}{2} tr\left( \Sigma_q^{-1}(\Sigma_p + \mu_p\mu_p^T - 2\mu_p\mu_q^T + \mu_q\mu_q^T)\right)\\
	\nonumber
	&= \frac{1}{2} \left[ log\frac{\vert \Sigma_q \vert}{\vert \Sigma_p \vert} - d_p +  tr\left( \Sigma_q^{-1}\Sigma_p\right) + tr\left( \Sigma_q^{-1}(\mu_p - \mu_q)^T(\mu_p - \mu_q) \right)  \right]\\
	\label{eq:gaussKLD}
	&= \frac{1}{2} \left[ log\frac{\vert \Sigma_q \vert}{\vert \Sigma_p \vert} - d_p +  tr\left( \Sigma_q^{-1}\Sigma_p\right) + (\mu_p - \mu_q)^T\Sigma_q^{-1}(\mu_p - \mu_q) \right].
\end{align}
Substituting Eq.~(\ref{eq:gaussKLD}) in Eq.~(\ref{eq:metricDef}) we get the \JD representation for the multivariate Gaussian case, 
 \begin{multline*}
	\JD(P,Q)= \\ \sqrt{\frac{1}{4} \left[ log\frac{\vert \Sigma_q \vert}{\vert \Sigma_p \vert} - d_p-d_q +  tr\left( \Sigma_q^{-1}\Sigma_p\right) + (\mu_p - \mu_q)^T\left[\Sigma_q^{-1} + \Sigma_p^{-1} \right](\mu_p - \mu_q) + log\frac{\vert \Sigma_p \vert}{\vert \Sigma_q \vert} +  tr\left( \Sigma_p^{-1}\Sigma_q\right)  \right]}
\end{multline*}
 \begin{equation}\label{eq:gaussPQ}
	= \frac{1}{2}\sqrt{ (\mu_p - \mu_q)^T\left[\Sigma_q^{-1} + \Sigma_p^{-1} \right](\mu_p - \mu_q) +  tr\left( \Sigma_q^{-1}\Sigma_p\right) +  tr\left( \Sigma_p^{-1}\Sigma_q\right) - d_p-d_q }.
\end{equation}

\section{Proof of Theorem \ref{th:rewardBound}.}
\label{app:rewardBound}
\begin{lemma}\label{lm:L1toKL} 
For any two distributions $P$ and $Q$, and $\alpha \in (0,1]$ the $L_1^{\alpha}$ distance is bounded by the $KL$ divergence in the following manner
	\begin{equation}
		 \| P- Q \|_1^{\alpha} \leq \left[ 2 \cdot ln2 \cdot \mathbb{D}_{KL}(P\|Q) \right]^{\frac{\alpha}{2}}
	\end{equation}
\end{lemma}
\begin{proof}
	Following Lemma 11.6.1 in \cite{cover06book},
	\begin{equation}
		\frac{1}{2 \cdot lan2} \| P- Q \|_1^2 \leq \mathbb{D}_{KL}(P\|Q), 
	\end{equation}
	multiplying both sides by the positive constant $2 \cdot ln2$ and raising to the $\frac{\alpha}{2}$ power we get
	\begin{equation}
		 \| P- Q \|_1^{\alpha} \leq \left[ 2 \cdot ln2 \cdot \mathbb{D}_{KL}(P\|Q) \right]^{\frac{\alpha}{2}}.
	\end{equation}
\end{proof}

\begin{forceTheorem}{\ref{th:rewardBound}}[Bounded reward difference]
	Let $r(b,u)$ be \aH continuous with $\lambda_{\alpha}$ and $\alpha \in (0,1]$. Let $b$ and $b'$ denote two beliefs. Then the difference between $r(b,u)$ and $r(b',u)$ is bounded by
	\begin{equation}
		\mid r(b,u) - r(b',u) \mid \leq \left(4 \cdot ln 2\right)^{\frac{\alpha}{2}} \cdot \lambda_{\alpha} \cdot  \JD^{\alpha}(b,b').
	\end{equation}
	where 
	\begin{equation}
		\JD(b,b') = \sqrt{\frac{1}{2}\mathbb{D}_{KL}(b||b') + \frac{1}{2}\mathbb{D}_{KL}(b'||b)},
	\end{equation}
	and $\mathbb{D}_{KL}(.)$ is the KL divergence.
\end{forceTheorem}
\begin{proof}
	The reward function $r(b,u)$ is \aH continuous with $\lambda_{\alpha}$ and $\alpha$ so following Eq.~(4.4) in \cite{Gilbarg15book},  
	\begin{equation}
		\mid r(b,u) - r(b',u) \mid \leq \lambda_{\alpha} \cdot \| b - b' \|_1^{\alpha},
	\end{equation} 
	using Lemma~\ref{lm:L1toKL} we can rewrite the bound as
	\begin{equation}
		\mid r(b,u) - r(b',u) \mid \leq \left(4 \cdot ln 2\right)^{\frac{\alpha}{2}} \cdot \lambda_{\alpha} \cdot \left(\frac{1}{2}\mathbb{D}_{KL}(b\|b')\right)^{\frac{\alpha}{2}},
	\end{equation}
	adding a non-negative scalar to the right-most expression yields 
	%
	\begin{equation}
		\mid r(b,u) - r(b',u) \mid \leq \left(4 \cdot ln 2\right)^{\frac{\alpha}{2}} \cdot \lambda_{\alpha} \cdot  \left(\frac{1}{2}\mathbb{D}_{KL}(b\|b') + \frac{1}{2}\mathbb{D}_{KL}(b'\|b)\right)^{\frac{\alpha}{2}},
	\end{equation}
	and finally using Eq.~(\ref{eq:metricDef}) we get
	\begin{equation}
		\mid r(b,u) - r(b',u) \mid \leq \left(4 \cdot ln 2\right)^{\frac{\alpha}{2}} \cdot \lambda_{\alpha} \cdot  \JD^{\alpha}(b,b').
	\end{equation}
\end{proof}

\section{Proof of Corollary \ref{cor:weightedBound}.}
\label{app:weightedBound}
\begin{forceCorollary}{\ref{cor:weightedBound}}[of Theorem~\ref{th:rewardBound}]
	Let $r(b,u)$ be \aH continuous with $\lambda_{\alpha}$ and $\alpha \in (0,1]$. Let $b$ and $b'$ denote two beliefs. Let $\omega_i$ denote a positive weight, such that $0 \leq \omega_i \leq 1 \ , \ i \in \{1,2\}$. Then the weighted difference between $r(b,u)$ and $r(b',u)$ is given by
	\begin{equation}
		(\omega_1-\omega_2) r(b') - \omega_1 \lambda_{\alpha}\left(4 \cdot ln 2\right)^{\frac{\alpha}{2}} \JD^{\alpha}(b,b') \leq  \omega_1 r(b) - \omega_2r(b') \leq \omega_1 \lambda_{\alpha}\left(4 \cdot ln 2\right)^{\frac{\alpha}{2}} \JD^{\alpha}(b,b')  + (\omega_1-\omega_2) r(b')
 	\end{equation}	
\end{forceCorollary}
\begin{proof}
	Following Theorem~\ref{th:rewardBound}, 
	\begin{equation}
		\mid r(b,u) - r(b',u) \mid \leq \lambda_{\alpha}\left(4 \cdot ln 2\right)^{\frac{\alpha}{2}} \JD^{\alpha}(b,b')
	\end{equation}
	\begin{equation}
		- \lambda_{\alpha}\left(4 \cdot ln 2\right)^{\frac{\alpha}{2}} \JD^{\alpha}(b,b') \leq  r(b) - r(b') \leq \lambda_{\alpha}\left(4 \cdot ln 2\right)^{\frac{\alpha}{2}} \JD^{\alpha}(b,b')
	\end{equation}
	\begin{equation}
		- \omega_1 \lambda_{\alpha}\left(4 \cdot ln 2\right)^{\frac{\alpha}{2}} \JD^{\alpha}(b,b') \leq  \omega_1 r(b) - \omega_1 r(b') \leq \omega_1 \lambda_{\alpha}\left(4 \cdot ln 2\right)^{\frac{\alpha}{2}} \JD^{\alpha}(b,b')
	\end{equation}
	\begin{equation}
		- \omega_1 \lambda_{\alpha}\left(4 \cdot ln 2\right)^{\frac{\alpha}{2}} \JD^{\alpha}(b,b') \leq  \omega_1 r(b) - \omega_2r(b') + (\omega_2-\omega_1)r(b') \leq \omega_1 \lambda_{\alpha}\left(4 \cdot ln 2\right)^{\frac{\alpha}{2}} \JD^{\alpha}(b,b')
	\end{equation}
	\begin{equation}
		(\omega_1-\omega_2) r(b') - \omega_1 \lambda_{\alpha}\left(4 \cdot ln 2\right)^{\frac{\alpha}{2}} \JD^{\alpha}(b,b') \leq  \omega_1 r(b) - \omega_2r(b') \leq \omega_1 \lambda_{\alpha}\left(4 \cdot ln 2\right)^{\frac{\alpha}{2}} \JD^{\alpha}(b,b')  + (\omega_1-\omega_2) r(b')
	\end{equation}
\end{proof}

\section{Proof of Theorem \ref{th:JacobBound}.}
\label{app:JacobBound}
\begin{forceTheorem}{\ref{th:JacobBound}}
Let r(b,u) be \aH continuous with $\lambda_{\alpha}$ and $\alpha \in (0,1]$. Let $J_{k+l|k+l}$ and $J_{k+l|k}$ be objective values of the same time step $k+l$, calculated based on information up to time $k+l$ and $k$ respectively. Let $L$ be a planning horizon such that $L \geq k+l+1$. Let $n_i$ be the number of samples used to estimate the expected reward at lookahead step $i$.  Let $\omega_i^j$ be non-negative weights such that $0 \leq \omega_i^j \leq 1$ and $\sum_{j=1}^{n_i} \omega_i^j = 1$. Then the difference $(J_{k+l|k+l} -J_{k+l|k})$ is bounded by
	\begin{equation}
		\sum_{i = k+l+1}^{L} \sum_{j = 1}^{n_i} \omega_{i|k+l}^{j} \left[r^j_{i|k} - \mathcal{D}_{i}^{j} \right] - J_{k+l|k} \leq J_{k+l|k+l} - J_{k+l|k} \leq \sum_{i = k+l+1}^{L} \sum_{j = 1}^{n_i} \omega_{i|k+l}^{j} \left[ r^j_{i|k} + \mathcal{D}_{i}^{j}\right] - J_{k+l|k}
	\end{equation}
	where
	\begin{equation}
		\mathcal{D}_{i}^{j} = \lambda_{\alpha}\left(4 \cdot ln 2\right)^{\frac{\alpha}{2}} \JD^{\alpha}(b^j[X_{i|k+l}],b^j[X_{i|k}]). 
	\end{equation}
\end{forceTheorem}
\begin{proof}
By definition,
	\begin{equation}
		J_{k+l|k+l} - J_{k+l|k} = \sum_{i = k+l+1}^{L} \left[ \Expec{r_i(b[X_{i|k+l}],u)} - \Expec{r_i(b[X_{i|k}],u)} \right],
	\end{equation}
	assuming the measurement likelihood is not explicitly available, we approximate the expectation using samples, 
	\begin{equation}
		J_{k+l|k+l} - J_{k+l|k} \approx \sum_{i = k+l+1}^{L} \left[ \sum_{j = 1}^{n_i} \omega_{i|k+l}^{j} r^j_i(b^j[X_{i|k+l}],u) - \sum_{j = 1}^{n_i} \omega_{i|k}^{j} r^j_i(b^j[X_{i|k}],u) \right]
	\end{equation}
	\begin{equation}
		= \sum_{i = k+l+1}^{L} \sum_{j = 1}^{n_i}\left[  \omega_{i|k+l}^{j} r^j_{i|k+l} - \omega_{i|k}^{j} r^j_{i|k} \right].
	\end{equation}
	Using Corollary~\ref{cor:weightedBound}, for specific $i,j$ we can write
	\begin{equation}\label{eq:th:specific_ij}
		(\omega_{i|k+l}^{j} - \omega_{i|k}^{j})r^j_{i|k} -\omega_{i|k+l}^{j}\mathcal{D}_{i}^{j} \leq \omega_{i|k+l}^{j} r^j_{i|k+l} - \omega_{i|k}^{j} r^j_{i|k} \leq \omega_{i|k+l}^{j}\mathcal{D}_{i}^{j} + (\omega_{i|k+l}^{j} - \omega_{i|k}^{j})r^j_{i|k}
	\end{equation}
	where
	\begin{equation}
		\mathcal{D}_{i}^{j} = \lambda_{\alpha}\left(4 \cdot ln 2\right)^{\frac{\alpha}{2}} \JD^{\alpha}(b^j[X_{i|k+l}],b^j[X_{i|k}]). 
	\end{equation}
	So following Eq.~(\ref{eq:th:specific_ij}) 
	\begin{equation}
		\sum_{i = k+l+1}^{L} \sum_{j = 1}^{n_i}\left[(\omega_{i|k+l}^{j} - \omega_{i|k}^{j})r^j_{i|k} -\omega_{i|k+l}^{j}\mathcal{D}_{i}^{j} \right] \leq J_{k+l|k+l} - J_{k+l|k} \leq \sum_{i = k+l+1}^{L} \sum_{j = 1}^{n_i}\left[ (\omega_{i|k+l}^{j} - \omega_{i|k}^{j})r^j_{i|k} +\omega_{i|k+l}^{j}\mathcal{D}_{i}^{j}\right]
	\end{equation}
	\begin{equation}
		\sum_{i = k+l+1}^{L} \sum_{j = 1}^{n_i} \omega_{i|k+l}^{j} \left[r^j_{i|k} - \mathcal{D}_{i}^{j} \right] - J_{k+l|k} \leq J_{k+l|k+l} - J_{k+l|k} \leq \sum_{i = k+l+1}^{L} \sum_{j = 1}^{n_i} \omega_{i|k+l}^{j} \left[ r^j_{i|k} + \mathcal{D}_{i}^{j}\right] - J_{k+l|k}.
	\end{equation}
\end{proof}

We can simplify this further by assuming that $J_{k+l|k+l}$ is estimated using samples from the nominal measurement likelihood, so the $\omega_{i|k+l}^j$ weights are simply given by $\frac{1}{n_i}$ $\forall j$,
	\begin{equation}
		\sum_{i = k+l+1}^{L}\frac{1}{n_i} \sum_{j = 1}^{n_i} \left[r^j_{i|k} - \mathcal{D}_{i}^{j} \right] - J_{k+l|k} \leq J_{k+l|k+l} - J_{k+l|k} \leq \sum_{i = k+l+1}^{L} \frac{1}{n_i} \sum_{j = 1}^{n_i} \left[ r^j_{i|k} + \mathcal{D}_{i}^{j}\right] - J_{k+l|k}.
	\end{equation}

\section{Proof of Theorem \ref{th:JacobBoundProb}.}
\label{app:JacobBoundProb}
\begin{forceTheorem}{\ref{th:JacobBoundProb}}
	Let r(b,u) be \aH continuous with $\lambda_{\alpha}$ and $\alpha \in (0,1]$. Let $J_{k+l|k+l}$ and $J_{k+l|k}$ be objective values of the same time step $k+l$, calculated based on information up to time $k+l$ and $k$ respectively. Let $L$ be a planning horizon such that $L \geq l+1$. Then the difference $(J_{k+l|k+l} -J_{k+l|k})$ is bounded by
	\begin{equation}
		\phi - \psi \leq J_{k+l|k+l} - J_{k+l|k} \leq \phi + \psi,
	\end{equation}
	where
	\begin{align}
		\label{eq:th:phi}
		\phi &\triangleq \sum_{i = k+l+1}^{k+L} \Expec_{z \sim p_k} (\omega-1) r_{i|k}   \\ 
		\label{eq:th:psi}
		\psi &\triangleq \lambda_{\alpha}\left(4 \cdot ln 2\right)^{\frac{\alpha}{2}} \left[ (L-l)\epsilon_{wf}^{\alpha} + \sum_{i = k+l+1}^{k+L} \left( \sum_{j=k+l+1}^{i} \Expec_{z \sim p_{k+l}} \Delta_j \right)^{\frac{\alpha}{2}} \right],  
	\end{align}
    and
    \begin{eqnarray}
    	\JD(b[X_{i|k+l}],b[X_{i|k}]) &=& \sqrt{ \JD^2(b[X_{i-1|k+l}],b[X_{i-1|k}]) + \Delta_{i}},
    	\\
    	\epsilon_{wf} &=& \JD(b[X_{k+l|k+l}],b[X_{k+l|k}]),	
    	\\
    		\omega &=& \frac{\prob{z_{k+l+1:k+L}|H_{k+l|k+l},u_{k+l:k+L-1}}}{\prob{z_{k+l+1:k+L}|H_{k+l|k},u_{k+l:k+L-1}}} \triangleq \frac{p_{k+l}}{p_{k}}.
    \end{eqnarray}
\end{forceTheorem}
\begin{proof}
By definition
	\begin{equation}
		J_{k+l|k+l} - J_{k+l|k} = \sum_{i = k+l+1}^{k+L} \left[ \Expec_{z_{k+l+1:k+L|k+l}} \!\!\!\!\!\!\!\!\!\!\!\!r_i(b[X_{i|k+l}],u) - \Expec_{z_{k+l+1:k+L|k}} \!\!\!\!\!\!\!\!\!\! r_i(b[X_{i|k}],u) \right],
	\end{equation}
	where each expectation is over a different measurement likelihood. In order to use a single expectation over both rewards we use importance sampling (see Eq.~(\ref{eq:basicIM})). For simplicity we define $\omega$ as the likelihood ratio
	\begin{equation}
		\omega = \frac{\prob{z_{k+l+1:k+L}|H_{k+l|k+l},u_{k+l:k+L-1}}}{\prob{z_{k+l+1:k+L}|H_{k+l|k},u_{k+l:k+L-1}}} = \frac{p_{k+l}}{p_{k}},
	\end{equation}
	and can now re-write the objective difference under a single expectation
	\begin{equation}\label{eq:th:jDiff}
		J_{k+l|k+l} - J_{k+l|k} = \sum_{i = k+l+1}^{k+L} \Expec_{z \sim p_k} \left[ \omega r_{i|k+l} - r_{i|k} \right].
	\end{equation}
	Following Corollary~\ref{cor:weightedBound} we can bound the reward difference in (\ref{eq:th:jDiff}) with
	\begin{equation}
		(\omega-1) r_{i|k} - \omega \mathcal{D}_i \leq \omega r_{i|k+l} - r_{i|k} \leq \omega \mathcal{D}_i + (\omega-1) r_{i|k},	
	\end{equation}
	where
	\begin{equation}
		\mathcal{D}_i = \lambda_{\alpha}\left(4 \cdot ln 2\right)^{\frac{\alpha}{2}} \JD^{\alpha}(b[X_{i|k+l}],b[X_{i|k}]).
	\end{equation}
	So the objective value difference is bounded by
	\begin{equation}
		\sum_{i = k+l+1}^{k+L} \Expec_{z \sim p_k} \left[ (\omega-1) r_{i|k} - \omega \mathcal{D}_i\right] \leq J_{k+l|k+l} - J_{k+l|k} \leq 	\sum_{i = k+l+1}^{k+L} \Expec_{z \sim p_k} \left[ \omega \mathcal{D}_i + (\omega-1) r_{i|k} \right]
	\end{equation}
	\begin{equation}
		\sum_{i = k+l+1}^{k+L} \Expec_{z \sim p_k} (\omega-1) r_{i|k} - \Expec_{z \sim p_k} \omega \mathcal{D}_i \leq J_{k+l|k+l} - J_{k+l|k} \leq 	\sum_{i = k+l+1}^{k+L} \Expec_{z \sim p_k} \omega \mathcal{D}_i + \Expec_{z \sim p_k} (\omega-1) r_{i|k},
	\end{equation}
	or in a more compact manner
	\begin{equation}\label{eq:th:cleanBound}
		\mid J_{k+l|k+l} - J_{k+l|k} - \!\!\! \sum_{i = k+l+1}^{k+L}\Expec_{z \sim p_k} (\omega-1) r_{i|k} \mid \leq 	\sum_{i = k+l+1}^{k+L} \Expec_{z \sim p_k} \omega \mathcal{D}_i.
	\end{equation}

	Let us define the delta distance between two consecutive lookahead steps as, 
	\begin{equation}
		\JD(b[X_{i|k+l}],b[X_{i|k}]) = \sqrt{ \JD^2(b[X_{i-1|k+l}],b[X_{i-1|k}]) + \Delta_{i}},
	\end{equation}
	so two sequential yet non consecutive lookahead steps can be written as
	\begin{equation}\label{eq:th:distDelta}
		\JD(b[X_{i+1|k+l}],b[X_{i+1|k}]) = \sqrt{ \JD^2(b[X_{i-1|k+l}],b[X_{i-1|k}]) + \Delta_{i} + \Delta_{i+1}}.
	\end{equation}
	We will now continue with simplifying the RHS of (\ref{eq:th:cleanBound}),
	\begin{equation}
		\sum_{i = k+l+1}^{k+L} \Expec_{z \sim p_k} \omega \mathcal{D}_i = \sum_{i = k+l+1}^{k+L} \Expec_{z \sim p_k} \omega \lambda_{\alpha}\left(4 \cdot ln 2\right)^{\frac{\alpha}{2}} \JD^{\alpha}(b[X_{i|k+l}],b[X_{i|k}])
	\end{equation}
	\begin{equation}
		= \lambda_{\alpha}\left(4 \cdot ln 2\right)^{\frac{\alpha}{2}} \sum_{i = k+l+1}^{k+L} \Expec_{z \sim p_k} \omega \JD^{\alpha}(b[X_{i|k+l}],b[X_{i|k}]),
	\end{equation}
	using the connection given by (\ref{eq:th:distDelta}) as well as moving back to expectation over $p_{k+l}$ we get
	\begin{equation}
		= \lambda_{\alpha}\left(4 \cdot ln 2\right)^{\frac{\alpha}{2}} \sum_{i = k+l+1}^{k+L} \Expec_{z \sim p_{k+l}} \left[ \JD^2(b[X_{k+l|k+l}],b[X_{k+l|k}]) + \sum_{j=k+l+1}^{i} \Delta_j \right]^{\frac{\alpha}{2}},
	\end{equation}
	\begin{equation}
		\leq \lambda_{\alpha}\left(4 \cdot ln 2\right)^{\frac{\alpha}{2}} \sum_{i = k+l+1}^{k+L} \Expec_{z \sim p_{k+l}} \left[ \epsilon_{wf}^2 + \sum_{j=k+l+1}^{i} \Delta_j \right]^{\frac{\alpha}{2}},
	\end{equation}
	where $\frac{\alpha}{2} \in (0,\frac{1}{2}]$ so $(.)^{\frac{\alpha}{2}}$ is a concave function, and $-(.)^{\frac{\alpha}{2}}$ is convex and following Jensen inequality we get
	\begin{equation}\label{eq:th:firUpper}
		\sum_{i = k+l+1}^{k+L} \Expec_{z \sim p_k} \omega \mathcal{D}_i \leq \lambda_{\alpha}\left(4 \cdot ln 2\right)^{\frac{\alpha}{2}} \left[ (L-l)\epsilon_{wf}^{\alpha} + \sum_{i = k+l+1}^{k+L} \left( \sum_{j=k+l+1}^{i} \Expec_{z \sim p_{k+l}} \Delta_j \right)^{\frac{\alpha}{2}} \right].
	\end{equation}
	Going back to (\ref{eq:th:cleanBound}),
	\begin{equation*}
		\mid J_{k+l|k+l} - J_{k+l|k} - \!\!\! \sum_{i = k+l+1}^{k+L} \Expec_{z \sim p_k} (\omega-1) r_{i|k} \mid \leq \sum_{i = k+l+1}^{k+L} \Expec_{z \sim p_k} \omega \mathcal{D}_i
	\end{equation*}
	\begin{equation}\label{eq:tg:rawBound}
		\leq \lambda_{\alpha}\left(4 \cdot ln 2\right)^{\frac{\alpha}{2}} \left[ (L-l)\epsilon_{wf}^{\alpha} + \sum_{i = k+l+1}^{k+L} \left( \sum_{j=k+l+1}^{i} \Expec_{z \sim p_{k+l}} \Delta_j \right)^{\frac{\alpha}{2}} \right].
	\end{equation}
	Using the definitions of $\phi$ and $\psi$ given respectively by Eq.(\ref{eq:th:phi}) and Eq.(\ref{eq:th:psi}), we can reformulate (\ref{eq:tg:rawBound}) into
	\begin{equation}
		\phi - \psi \leq J_{k+l|k+l} - J_{k+l|k} \leq \phi + \psi.
	\end{equation}
	\end{proof}
	For the special case where $\omega = 1$ we get
	\begin{equation}
		\mid J_{k+l|k+l} - J_{k+l|k} \mid \leq \left(4 \cdot ln 2\right)^{\frac{\alpha}{2}} \cdot \lambda_{\alpha} \cdot\left[ (L-l)\cdot \epsilon_{wf}^{\alpha} +  \sum_{i= k+l+1}^{k+L} \left( \sum_{j=k+l+1}^{i} \Expec \Delta_j \right)^{\frac{\alpha}{2}} \right].
	\end{equation}

\section{Proof of Lemma \ref{lm:incDist}}
\label{app:iDistance}
\begin{forceLemma}{\ref{lm:incDist}}[Incremental $\mathbb{D}_{PQ}$ distance]
	Let $b_1$ and $b_2$ be two gaussian beliefs with $(\mu_1 , \Sigma_1)$ and $(\mu_{2} , \Sigma_2)$ respectively, and their two (differently) propagated counterparts $b_{1p}$ and $b_{2p}$ with $(\mu_{1p} , \Sigma_{1p})$ and $(\mu_{2}^+ , \Sigma_{2p})$. When the propagated mean and covariance are defined as
	\begin{equation}\label{eq:appB:def}
		\mu_{ip} = \mu_i + \zeta_i \quad , \quad 
		\Sigma_{ip} = \left( \Sigma_i^{-1} + A_i^T A_i \right)^{-1} \quad , \quad 
		i \in [1,2]
	\end{equation}
	Then the squared $\mathbb{D}_{PQ}$ distance between the propagated beliefs can be written as 
	\begin{equation}
		\JD^2\left( b_{1p},b_{2p}\right) = \JD^2\left( b_1,b_2\right) + \Delta,
	\end{equation}
where 
\begin{multline}
	\Delta = \frac{1}{4}(\mu_2 - \mu_1)^T \left[ A_2^T A_2 + A_1^T A_1 \right](\mu_2 - \mu_1) + 
	\frac{1}{2}(\mu_2 - \mu_1)^T \left[ {\Sigma_{1p}}^{-1} + {\Sigma_{2p}}^{-1} \right](\zeta_2 - \zeta_1) \\
	+ \frac{1}{4}(\zeta_2 - \zeta_1)^T \left[ {\Sigma_{1p}}^{-1} + {\Sigma_{2p}}^{-1} \right](\zeta_2 - \zeta_1)
	+ \frac{1}{4}tr\left(  A_2^T A_2 \Sigma_{1p} - \Sigma_2^{-1} \Sigma_1 A_1^T ( I + A_1 \Sigma_1 A_1^T)^{-1} A_1 \Sigma_1 \right) \\ 
	+ \frac{1}{4}tr\left(  A_1^T A_1 \Sigma_{2p} - \Sigma_1^{-1} \Sigma_2 A_2^T ( I + A_2 \Sigma_2 A_2^T)^{-1} A_2 \Sigma_2 \right) - \frac{1}{2}\left( d_{p} - d\right).
\end{multline}
\end{forceLemma}
\begin{proof}
The $\mathbb{D}_{PQ}$ distance between $b_1$ and $b_2$ is thus given by
\begin{equation}\label{eq:appB:dist}
	\JD^2\left( b_1,b_2\right) = \frac{1}{4} \left[ (\mu_2 - \mu_1)^T \left[ \Sigma_1^{-1} + \Sigma_2^{-1}\right](\mu_2 - \mu_1) + tr(\Sigma_2^{-1}\Sigma_1) + tr(\Sigma_1^{-1}\Sigma_2) -d_1 - d_2 \right], 
\end{equation}
when $d_i$ represents the dimension of the un-zero-padded $\Sigma_i$. Equivalently the distance between $b_{1p}$ and $b_{2p}$ is given by
\begin{equation}\label{eq:appB:dist+}
	\JD^2\left( b_{1p},b_{2p}\right) = \frac{1}{4} \left[ (\mu_{2p} - \mu_{1p})^T \left[ {\Sigma_{1p}}^{-1} + {\Sigma_{2p}}^{-1}\right](\mu_{2p} - \mu_{1p}) + tr({\Sigma_{2p}}^{-1}\Sigma_{1p}) + tr({\Sigma_{1p}}^{-1}\Sigma_{2p}) -d_{1p} -d_{2p} \right] 
\end{equation}
where $d^+_i$ represents the dimension of the un-zero-padded $\Sigma_i^+$.

We would like to show that 
\begin{equation}
	\JD^2\left( b_{1p},b_{2p}\right) = \JD^2\left( b_1,b_2\right) + \Delta.
\end{equation}
We start from substituting Eq.~(\ref{eq:appB:def}) in the trace expression from Eq.~(\ref{eq:appB:dist+})
\begin{equation*}
	tr({\Sigma_{2p}}^{-1}\Sigma_{1p}) = 	tr(\left( \Sigma_2^{-1} + A_2^T A_2 \right)\left( \Sigma_1^{-1} + A_1^T A_1 \right)^{-1})
\end{equation*}
using Woodbury matrix identity \cite{Woodbury1950mr} we can rewrite it as
\begin{equation*}
	tr({\Sigma_{2p}}^{-1}\Sigma_{1p}) = 	tr(\left( \Sigma_2^{-1} + A_2^T A_2 \right)\left( \Sigma_1 - \Sigma_1 A_1^T( I + A_1\Sigma_1 A_1^T)^{-1}A_1 \Sigma_1  \right) ).
\end{equation*} 
After some simple manipulations we get
\begin{equation}\label{eq:appB:trace21}
	tr({\Sigma_{2p}}^{-1}\Sigma_{1p}) = tr(\Sigma_2^{-1}\Sigma_1) + tr\left(  A_2^T A_2 \Sigma_{1p} - \Sigma_2^{-1} \Sigma_1 A_1^T ( I + A_1 \Sigma_1 A_1^T)^{-1} A_1 \Sigma_1 \right).
\end{equation} 
In a similar manner we can get an expression for the symmetric trace expression in Eq.~(\ref{eq:appB:dist+})
\begin{equation}\label{eq:appB:trace12}
	tr({\Sigma_{1p}}^{-1}\Sigma_{2p}) = tr(\Sigma_1^{-1}\Sigma_2) + tr\left(  A_1^T A_1 \Sigma_{2p} - \Sigma_1^{-1} \Sigma_2 A_2^T ( I + A_2 \Sigma_2 A_2^T)^{-1} A_2 \Sigma_2 \right).
\end{equation} 
We are left with the square root Mahalanobis distance expression in Eq.~(\ref{eq:appB:dist+}). By substituting Eq.~(\ref{eq:appB:def}) in the aforementioned we get 
\begin{multline}\label{eq:appB:mahal}
	(\mu_{2p} - \mu_{1p})^T \left[ {\Sigma_{1p}}^{-1} + {\Sigma_{2p}}^{-1}\right](\mu_{2p} - \mu_{1p}) = (\mu_2 + \zeta_2 - \mu_1 - \zeta_1)^T \left[ \Sigma_1^{-1} + A_1^T A_1 + \Sigma_2^{-1} + A_2^T A_2 \right](\mu_2 + \zeta_2 - \mu_1 - \zeta_1) \\
	= (\mu_2 - \mu_1)^T \left[ \Sigma_1^{-1} + \Sigma_2^{-1}\right](\mu_2 - \mu_1) +
	(\mu_2 - \mu_1)^T \left[ A_2^T A_2 + A_1^T A_1 \right](\mu_2 - \mu_1) + \\
	2(\mu_2 - \mu_1)^T \left[ {\Sigma_{1p}}^{-1} + {\Sigma_{2p}}^{-1} \right](\zeta_2 - \zeta_1) + 
	(\zeta_2 - \zeta_1)^T \left[ {\Sigma_{1p}}^{-1} + {\Sigma_{2p}}^{-1} \right](\zeta_2 - \zeta_1).
\end{multline} 

By substituting Eqs.~(\ref{eq:appB:trace21})-(\ref{eq:appB:mahal}) in Eq.~(\ref{eq:appB:dist+}) we receive 
\begin{equation}
	\JD^2\left( b_{1p},b_{2p}\right) = \JD^2\left( b_1,b_2\right) + \Delta,
\end{equation}
where 
\begin{multline}\label{eq:appB:delta}
	\Delta = \frac{1}{4}(\mu_2 - \mu_1)^T \left[ A_2^T A_2 + A_1^T A_1 \right](\mu_2 - \mu_1) + 
	\frac{1}{2}(\mu_2 - \mu_1)^T \left[ {\Sigma_{1p}}^{-1} + {\Sigma_{2p}}^{-1} \right](\zeta_2 - \zeta_1) \\
	+ \frac{1}{4}(\zeta_2 - \zeta_1)^T \left[ {\Sigma_{1p}}^{-1} + {\Sigma_{2p}}^{-1} \right](\zeta_2 - \zeta_1)
	+ \frac{1}{4}tr\left(  A_2^T A_2 \Sigma_{1p} - \Sigma_2^{-1} \Sigma_1 A_1^T ( I + A_1 \Sigma_1 A_1^T)^{-1} A_1 \Sigma_1 \right) \\ 
	+ \frac{1}{4}tr\left(  A_1^T A_1 \Sigma_{2p} - \Sigma_1^{-1} \Sigma_2 A_2^T ( I + A_2 \Sigma_2 A_2^T)^{-1} A_2 \Sigma_2 \right) - \frac{1}{2}\left( d_{p} - d\right).
\end{multline}
\end{proof}

When considering future beliefs, which are a function of future measurements, the belief solution is a random variable depending on the future measurement. As such $\Delta$ in Eq.~(\ref{eq:appB:delta}), which is a function of belief solution, is in-fact a random variable.

\section{Proof of Lemma \ref{lm:deltaAsGQ}.}
\label{app:deltaAsGQ}
\begin{lemma}[state estimation increment as a random variable]\label{lm:zetaDist}
	 Let $\mu$ denote the state estimate of some gaussian belief $b = \mathcal{N}(\mu_0, \Sigma_0)$. Let $b_p$ denote the gaussian belief resulting from propagating $b$ with some action $u$ and some measurements $z$ using the linear Gaussian models (\ref{eq:wf:linMotionModel})-(\ref{eq:wf:linMeasurementModel}).  Let $\mu_p$ denote the state estimate $b_p$ and define a new random variabble $\zeta = \mu_p - \mu $. 
	Then  
	\begin{equation}
		\zeta \sim \mathcal{N}(\mu_{\zeta},\Sigma_{\zeta}),
	\end{equation}
	where
	\begin{equation}
		\mu_{\zeta} = \sigma_{21} \left(\Sigma_0^{-1} \mu_0 - \mathcal{F}^T \Sigma_w^{-1}\mathcal{J}u \right) + \sigma_{22}\left( \Sigma_w^{-1}\mathcal{J}u + \mathcal{H}^T\Sigma_v^{-1}(\mathcal{H}\mathcal{F}\mu_0 + \mathcal{H}\mathcal{J}u) \right) - \mu_0 
	\end{equation}
	\begin{equation}
		\Sigma_{\zeta} = \sigma_{22}\mathcal{H}^T\Sigma_v^{-1}\mathcal{H}\mathcal{F}\Sigma_0\mathcal{F}^T\mathcal{H}^T\Sigma_v^{-1}\mathcal{H}\sigma_{22} + \sigma_{22}\mathcal{H}^T\Sigma_v^{-1}\mathcal{H}\Sigma_w\mathcal{H}^T\Sigma_v^{-1}\mathcal{H}\sigma_{22} + \sigma_{22}\mathcal{H}^T\Sigma_v^{-1}\mathcal{H}\sigma_{22}
	\end{equation}
	\begin{equation}
		\begin{bmatrix}
 		\sigma_{11} & \sigma_{12} \\
 		\sigma_{21} & \sigma_{22}	
 		\end{bmatrix} = \begin{bmatrix}
 		\Sigma_0^{-1} + \mathcal{F}^T \Sigma_w^{-1}\mathcal{F} & -\mathcal{F}^T \Sigma_w^{-1} \\
 		-\Sigma_w^{-1} \mathcal{F} & \Sigma_w^{-1} + \mathcal{H}^T \Sigma_v^{-1}\mathcal{H}	
 		\end{bmatrix}^{-1}.
	\end{equation}

\end{lemma}
\begin{proof}
	see Appendix~\ref{app:zetaDist}.
\end{proof}

\begin{forceLemma}{\ref{lm:deltaAsGQ}}[Incremental \JD distance as Gaussian Quadratic]
		Let $b_1$ and $b_2$ be two Gaussian beliefs  $\mathcal{N}(\mu_1 , \Sigma_1)$ and $\mathcal{N}(\mu_{2} , \Sigma_2)$, respectively with state dimension $d$, and their two propagated counterparts $b_{1p}$ and $b_{2p}$ with $\mathcal{N}(\mu_{1p} , \Sigma_{1p})$  $\mathcal{N}(\mu_{2p} , \Sigma_{2p})$ and with state dimension $d_{p}$. 
		There exist $\zeta_i $ and $A_i$ such that the propagated mean and covariance are given by, 
		\begin{equation}
			\mu_{ip} = \mu_i + \zeta_i \quad , \quad 
			\Sigma_{ip} = \left( \Sigma_i^{-1} + A_i^T A_i \right)^{-1} \quad , \quad 
			i \in [1,2].
		\end{equation}
		Then the incremental \JD distance $\Delta \triangleq \JD^2 \left( b_{1p},b_{2p}\right) - \JD^2 \left( b_1,b_2\right)$ is a quadratic form of a gaussian variable.
\end{forceLemma}
\begin{proof}
	The incremental \JD distance under the assumption of linear Gaussian models is given by Lemma~\ref{lm:incDist}. Let us define a new variable $\mathcal{S}$ such that
	\begin{equation}
		\mathcal{S} = \zeta_2 - \zeta_1.
	\end{equation}
	We can now re-write the result for $\Delta$ given by Lemma~\ref{lm:incDist} in terms of $\mathcal{S}$,
	\begin{equation}\label{eq:lm:quadraticForm}
		\Delta(\mathcal{S}) = \mathcal{S}^T C \mathcal{S} + c^T \mathcal{S} + y,
	\end{equation} 
	where
	\begin{equation}
		C = \frac{1}{4} \left[ {\Sigma_{1p}}^{-1} + {\Sigma_{2p}}^{-1} \right]
	\end{equation}
	\begin{equation}
		c = \frac{1}{2}(\mu_2 - \mu_1)^T \left[ {\Sigma_{1p}}^{-1} + {\Sigma_{2p}}^{-1} \right]
	\end{equation}
	\begin{multline}
		y = \frac{1}{4}(\mu_2 - \mu_1)^T \left[ A_2^T A_2 + A_1^T A_1 \right](\mu_2 - \mu_1) 
		+ \frac{1}{4}tr\left(  A_2^T A_2 \Sigma_{1p} - \Sigma_2^{-1} \Sigma_1 A_1^T ( I + A_1 \Sigma_1 A_1^T)^{-1} A_1 \Sigma_1 \right) \\ 
		+ \frac{1}{4}tr\left(  A_1^T A_1 \Sigma_{2p} - \Sigma_1^{-1} \Sigma_2 A_2^T ( I + A_2 \Sigma_2 A_2^T)^{-1} A_2 \Sigma_2 \right) - \frac{1}{2}\left( d_{p} - d\right).
	\end{multline}
	Eq~\ref{eq:lm:quadraticForm} presents $\Delta$ as a quadratic form of $\mathcal{S}$ leaving us to find how $\mathcal{S}$ is distributed. 

	Following Lemma~\ref{lm:zetaDist}, we know the distribution of $\zeta_i$, thus from being a linear combination of Gaussian variables we know the distribution of $\mathcal{S}$ to be also Gaussian 
	\begin{equation}
		\mathcal{S} \sim \mathcal{N}(\mu_{\mathcal{S}}, \Sigma_{\mathcal{S}})
	\end{equation}
	where
	\begin{equation}
		\mu_{\mathcal{S}} = \mu_{\zeta_2} - \mu_{\zeta_1},
	\end{equation}
	\begin{equation}
		\Sigma_{\mathcal{S}} = \Sigma_{\zeta_2} + \Sigma_{\zeta_1} + 2\Sigma_{\zeta_1\zeta_2}. 
	\end{equation}
	So $\Delta$ is a quadratic form of the Gaussian variable $\mathcal{S}$.
\end{proof}


\section{Proof of Lemma \ref{lm:zetaDist}.}
\label{app:zetaDist}
\begin{forceLemma}{\ref{lm:zetaDist}}[state estimation increment as a random variable]
	 Let $\mu$ denote the state estimate of some gaussian belief $b = \mathcal{N}(\mu_0, \Sigma_0)$. Let $b_p$ denote the gaussian belief resulting from propagating $b$ with some action $u$ and some measurements $z$ using the linear Gaussian models (\ref{eq:wf:linMotionModel})-(\ref{eq:wf:linMeasurementModel}).  Let $\mu_p$ denote the state estimate $b_p$ and define a new random variabble $\zeta = \mu_p - \mu $. 
	Then  
	\begin{equation}
		\zeta \sim \mathcal{N}(\mu_{\zeta},\Sigma_{\zeta}),
	\end{equation}
	where
	\begin{equation}
		\mu_{\zeta} = \sigma_{21} \left(\Sigma_0^{-1} \mu_0 - \mathcal{F}^T \Sigma_w^{-1}\mathcal{J}u \right) + \sigma_{22}\left( \Sigma_w^{-1}\mathcal{J}u + \mathcal{H}^T\Sigma_v^{-1}(\mathcal{H}\mathcal{F}\mu_0 + \mathcal{H}\mathcal{J}u) \right) - \mu_0 
	\end{equation}
	\begin{equation}
		\Sigma_{\zeta} = \sigma_{22}\mathcal{H}^T\Sigma_v^{-1}\mathcal{H}\mathcal{F}\Sigma_0\mathcal{F}^T\mathcal{H}^T\Sigma_v^{-1}\mathcal{H}\sigma_{22} + \sigma_{22}\mathcal{H}^T\Sigma_v^{-1}\mathcal{H}\Sigma_w\mathcal{H}^T\Sigma_v^{-1}\mathcal{H}\sigma_{22} + \sigma_{22}\mathcal{H}^T\Sigma_v^{-1}\mathcal{H}\sigma_{22}
	\end{equation}
	\begin{equation}
		\begin{bmatrix}
 		\sigma_{11} & \sigma_{12} \\
 		\sigma_{21} & \sigma_{22}	
 		\end{bmatrix} = \begin{bmatrix}
 		\Sigma_0^{-1} + \mathcal{F}^T \Sigma_w^{-1}\mathcal{F} & -\mathcal{F}^T \Sigma_w^{-1} \\
 		-\Sigma_w^{-1} \mathcal{F} & \Sigma_w^{-1} + \mathcal{H}^T \Sigma_v^{-1}\mathcal{H}	
 		\end{bmatrix}^{-1}.
	\end{equation}
\end{forceLemma}
\begin{proof}
	Because the belief $b_p$ resulted from propagating the belief $b$ with motion and measurements, using Bayes rule we can write it as proportional to 
	\begin{equation}
		b_p \propto b \cdot \prob{x' | x, u}\prob{z|x'},
	\end{equation}
	without affecting generality let us assume a single measurement is considered. 
	Denoting $X =\begin{bmatrix} x \\x' \end{bmatrix}$, the inference solution of $b_p$ can be obtained through MAP estimation
	\begin{equation}
		\hat{X} = \argmax_{X} b \cdot \prob{x' | x, u}\prob{z|x'},	
	\end{equation}
	taking the negative log yields the following LS problem
	\begin{equation}
		\hat{X} = \argmin_{X} \|x - \mu_0\|_{\Sigma_0}^2	+ \| x' - \mathcal{F}x - \mathcal{J}u) \|_{\Sigma_w}^2 + \| z - \mathcal{H}x \|_{\Sigma_v}^2,
	\end{equation}
	where (\ref{eq:wf:linMotionModel})-(\ref{eq:wf:linMeasurementModel}) were used for linear motion and measurement models with zero mean Gaussian noises
	\begin{equation*}
		x' = \mathcal{F}x + \mathcal{J}u + w \quad , \quad  w\sim \mathcal{N}(0,\Sigma_w),
	\end{equation*}
	\begin{equation*}
		z = \mathcal{H}x + v\quad , \quad  v\sim \mathcal{N}(0,\Sigma_v).
	\end{equation*}
	We can further reformulate the problem into a LS form
	\begin{equation}
		\hat{X} = \argmin_{X} \|A \cdot X - b\|^2,
	\end{equation}
	where
	\begin{equation}
		A = \begin{bmatrix}
 		\Sigma_0^{-\frac{1}{2}} & 0 \\
 		-\Sigma_w^{-\frac{1}{2}} \mathcal{F} & \Sigma_w^{-\frac{1}{2}} \\
 		0 & \Sigma_v^{-\frac{1}{2}} \mathcal{H} 	
 		\end{bmatrix} \ , \ 
 		b = \begin{bmatrix}
 		\Sigma_0^{-\frac{1}{2}} \mu_0 \\
 		\Sigma_w^{-\frac{1}{2}} \mathcal{J}u \\
 		\Sigma_v^{-\frac{1}{2}} z 	
 		\end{bmatrix}.
	\end{equation}
	The solution to the inference problem is given by
	\begin{equation}
		X \sim \mathcal{N}\left( \Sigma A^Tb , \Sigma\right)	
	\end{equation}
	where the joint covariance matrix is given by
	\begin{equation}
		\Sigma_j = (A^TA)^{-1} = \left(\begin{bmatrix}
 		\Sigma_0^{-\frac{1}{2}} & -\mathcal{F}^T\Sigma_w^{-\frac{1}{2}}  & 0 \\
 		0 & \Sigma_w^{-\frac{1}{2}} & \mathcal{H}^T\Sigma_v^{-\frac{1}{2}} 	
 		\end{bmatrix}\cdot \begin{bmatrix}
 		\Sigma_0^{-\frac{1}{2}} & 0 \\
 		-\Sigma_w^{-\frac{1}{2}} \mathcal{F} & \Sigma_w^{-\frac{1}{2}} \\
 		0 & \Sigma_v^{-\frac{1}{2}} \mathcal{H} 	
 		\end{bmatrix} \right)^{-1} 
 	\end{equation}
	\begin{equation}
		\Sigma_j = \begin{bmatrix}
 		\Sigma_0^{-1} + \mathcal{F}^T \Sigma_w^{-1}\mathcal{F} & -\mathcal{F}^T \Sigma_w^{-1} \\
 		-\Sigma_w^{-1} \mathcal{F} & \Sigma_w^{-1} + \mathcal{H}^T \Sigma_v^{-1}\mathcal{H}	
 		\end{bmatrix}^{-1} = \begin{bmatrix}
 		\sigma_{11} & \sigma_{12} \\
 		\sigma_{21} & \sigma_{22}	
 		\end{bmatrix},
	\end{equation}
	for the readers convenience we denote the block matrices of $\Sigma_j$ by $\sigma_{i,j}$, where they can be calculated explicitly through the Schur complement. So the joint state estimation is given by 
	\begin{equation}
		\hat{X} = \begin{bmatrix}
 		\sigma_{11} & \sigma_{12} \\
 		\sigma_{21} & \sigma_{22}	
 		\end{bmatrix} \begin{bmatrix}
 		\Sigma_0^{-\frac{1}{2}} & -\mathcal{F}^T\Sigma_w^{-\frac{1}{2}}  & 0 \\
 		0 & \Sigma_w^{-\frac{1}{2}} & \mathcal{H}^T\Sigma_v^{-\frac{1}{2}} 	
 		\end{bmatrix} \begin{bmatrix}
 		\Sigma_0^{-\frac{1}{2}} \mu_0 \\
 		\Sigma_w^{-\frac{1}{2}} \mathcal{J}u \\
 		\Sigma_v^{-\frac{1}{2}} z 	
 		\end{bmatrix},
	\end{equation}
	and we can write the state estimation of $x'$ explicitly as
	\begin{equation}
		\mu_p = \hat{x}' = \sigma_{21} \left(\Sigma_0^{-1} \mu_0 - \mathcal{F}^T \Sigma_w^{-1}\mathcal{J}u \right) + \sigma_{22}\left( \Sigma_w^{-1}\mathcal{J}u + \mathcal{H}^T\Sigma_v^{-1}z \right).
	\end{equation}	

	Now that we have developed the state estimation, we are in position to express $\zeta$ in terms of the state estimations
	\begin{equation}\label{th:zetaVsz}
		\zeta = \mu_p - \mu_0 = \sigma_{21} \left(\Sigma_0^{-1} \mu_0 - \mathcal{F}^T \Sigma_w^{-1}\mathcal{J}u \right) + \sigma_{22}\left( \Sigma_w^{-1}\mathcal{J}u + \mathcal{H}^T\Sigma_v^{-1}z \right) - \mu_0
	\end{equation}
	through the measurement and motion models we know 
	\begin{equation}
		z = \mathcal{H}\mathcal{F}x + \mathcal{H}\mathcal{J}u + \mathcal{H}w + v.
	\end{equation}
	Following $z$ is a linear transformation of a gaussian variable, it is also a gaussian variable with
	\begin{equation}\label{th:zProb}
		z \sim \mathcal{N}(\mathcal{H}\mathcal{F}\mu_0 + \mathcal{H}\mathcal{J}u , \mathcal{H}\mathcal{F}\Sigma_0\mathcal{F}^T\mathcal{H}^T + \mathcal{H}\Sigma_w\mathcal{H}^T + \Sigma_v),
	\end{equation}
	the same result can be obtained by explicitly calculating the pdf of the measurement likelihood function as presented in Appendix~\ref{app:measLikeProb}.
	Using (\ref{th:zProb}) and the connection provided in (\ref{th:zetaVsz}) we can say that as a linear transformation of a gaussian variable, $\zeta$ too is a gaussian variable with
	\begin{equation}
		\zeta \sim \mathcal{N}(\mu_{\zeta} ,\Sigma_{\zeta})
	\end{equation}
	where
	\begin{equation}
		\mu_{\zeta} = \sigma_{21} \left(\Sigma_0^{-1} \mu_0 - \mathcal{F}^T \Sigma_w^{-1}\mathcal{J}u \right) + \sigma_{22}\left( \Sigma_w^{-1}\mathcal{J}u + \mathcal{H}^T\Sigma_v^{-1}(\mathcal{H}\mathcal{F}\mu_0 + \mathcal{H}\mathcal{J}u) \right) - \mu_0 
	\end{equation}
	\begin{equation}
		\Sigma_{\zeta} = \sigma_{22}\mathcal{H}^T\Sigma_v^{-1}\mathcal{H}\mathcal{F}\Sigma_0\mathcal{F}^T\mathcal{H}^T\Sigma_v^{-1}\mathcal{H}\sigma_{22} + \sigma_{22}\mathcal{H}^T\Sigma_v^{-1}\mathcal{H}\Sigma_w\mathcal{H}^T\Sigma_v^{-1}\mathcal{H}\sigma_{22} + \sigma_{22}\mathcal{H}^T\Sigma_v^{-1}\mathcal{H}\sigma_{22}
	\end{equation}
\end{proof}

\section{Proof of Corollary~\ref{cor:linModelBounds}.}
\label{app:linModelJacobBaound}
\begin{lemma}[Moments of Gaussian Quadratic]\label{lm:momentsGQ}
	 Let $\Delta$ be quadratic expression $\Delta = \mathcal{S}^T C \mathcal{S} + c^T \mathcal{S} + y$, where $C^T = C$, $\mathcal{S} \sim \mathcal{N}(\mu, \Sigma)$, and $\Sigma>0$.
	Then the first two moments of $\Delta$ are given by
	\begin{equation}
		\Expec [\Delta] = tr(\Sigma^{\frac{1}{2}} C \Sigma^{\frac{1}{2}}) + \mu^T C \mu + c^T\mu + y
	\end{equation}
	\begin{equation}
		\Expec [\Delta]^2 = 2tr(\Sigma^{\frac{1}{2}} C \Sigma C  \Sigma^{\frac{1}{2}}) + \|c + 2C\mu\|_{\Sigma}^2
	\end{equation}
	where $\| . \|^2_{\Sigma}$ is the Mahalanobis distance.
\end{lemma}
\begin{proof}
	This is Theorem 3.2b.3 in \cite{Mathai92quadratic}.
\end{proof}
Interestingly enough, there are specified conditions under-which a Gaussian quadratic form is distributed as a non-central $\chi^2$ distribution, although not required for our proof we supply these conditions in Appendix~\ref{app:chiSquared} for completion. 

\begin{forceCorollary}{\ref{cor:linModelBounds}}[of Theorem~\ref{th:JacobBoundProb}]
	Let r(b,u) be \aH continuous with $\lambda_{\alpha}$ and $\alpha \in (0,1]$. Let $J_{k+l|k+l}$ and $J_{k+l|k}$ be objective values of the same time step $k+l$, calculated based on information up to time $k+l$ and $k$ respectively. Let $L$ be a planning horizon such that $L \geq l+1$. Let the motion and measurement models be linear with additive Gaussian noise (\ref{eq:wf:linMotionModel})-(\ref{eq:wf:linMeasurementModel}).  Then the bound of $(J_{k+l|k+l} -J_{k+l|k})$ can be explicitly calculated.
	\end{forceCorollary}
\begin{proof}
	Following Theorem~\ref{th:JacobBoundProb}, the difference $(J_{k+l|k+l} -J_{k+l|k})$ is bounded by 
	\begin{equation}\label{eq:wf:boundProb}
		\mid J_{k+l|k+l} - J_{k+l|k}  - \!\!\! \sum_{i = k+l+1}^{k+L} \Expec_{z \sim p_k} (\omega-1) r_{i|k} \mid \leq \left(4 \cdot ln 2\right)^{\frac{\alpha}{2}} \cdot \lambda_{\alpha} \cdot\left[ (L-l)\cdot \epsilon_{wf}^{\alpha} +  \sum_{i= k+l+1}^{k+L} \left( \sum_{j=k+l+1}^{i} \Expec \Delta_j \right)^{\frac{\alpha}{2}} \right],
	\end{equation}
	were 
	\begin{equation*}
		\epsilon_{wf} = \JD(b[X_{k+l|k+l}],b[X_{k+l|k}]).	
	\end{equation*}
	\begin{equation*}
		\Delta_j = \JD^2(b[X_{j|k+l}], b[X_{j|k}]) - \JD^2(b[X_{j-1|k+l}], b[X_{j-1|k}]).
	\end{equation*}
	In order to explicitly calculate (\ref{eq:wf:boundProb}) we are left with calculating the expected value of $\Delta_j$ $\forall j$, where the rewards $r_{i_k}$ are given from previously calculated planning session, and just need to be re-weighted using ($\omega -1$).
	Following Lemma~\ref{lm:deltaAsGQ} we know $\Delta_j$ to be a quadratic expression of the Gaussian multinomial variable $\mathcal{S}_j \sim \mathcal{N}\left(\mu_{\mathcal{S}_j} , \Sigma_{\mathcal{S}_j}\right)$ 
	\begin{equation*}
		\Delta_j = \mathcal{S}_j^T C_j \mathcal{S}_j + c_j^T \mathcal{S}_j + y_j	
	\end{equation*}

	As such following Lemma~\ref{lm:momentsGQ}, the first moment of $\Delta_j$ is readily given by
	\begin{equation}
		\Expec \Delta_j = tr(\Sigma_{\mathcal{S}_j}^{\frac{1}{2}} C \Sigma_{\mathcal{S}_j}^{\frac{1}{2}}) + \mu_{\mathcal{S}_j}^T C_j \mu_{\mathcal{S}_j} + c_j^T\mu_{\mathcal{S}_j} + y_j.	
	\end{equation}

\end{proof}

\section{Measurement Likelihood Under Linear Gaussian Models.}
\label{app:measLikeProb}
In this section we calculate the measurement likelihood probability under the assumption of linear gaussian models. We will show that under these assumptions the measurement likelihood is gaussian. Our proof starts with introducing the state into the measurement likelihood, followed by some manipulations to marginalize the state out and be left with a gaussian function over the measurement.

We start by introducing the state into the measurement likelihood, which allows us to get an equivalent expression with the measurement and motion models,
\begin{equation}
	\prob{z|H^-} = \int \prob{z|x} \prob{x|H^-} dx
\end{equation}
where 
\begin{equation}
	\prob{z|x} = \mathcal{N}(Hx,\Sigma_v) \quad , \quad \prob{x|H^-} = \mathcal{N}(\mu_p, \Sigma_p)
\end{equation}
\begin{equation}
	\mu_p = F \mu_0 + Ju \quad ,\quad \Sigma_p = \Sigma_w + F\Sigma_0 F^T.
\end{equation}

\begin{equation}\label{eq:probdev:explicit}
	\prob{z|H^-} = \int \frac{1}{\sqrt{(2\pi)^{d_z} \mid \Sigma_v\mid}}e^{-\frac{1}{2} \|z-Hx \|^2_{\Sigma_v}} \frac{1}{\sqrt{(2\pi)^{d_x} \mid \Sigma_p\mid}}e^{-\frac{1}{2} \|x-\mu_p \|^2_{\Sigma_p}} dx.
\end{equation}
Let us look only at the argument of the exponent (divided by $-\frac{1}{2}$)
\begin{equation}
		\|z-Hx \|^2_{\Sigma_v} + \|x-\mu_p \|^2_{\Sigma_p} = (z-Hx)^T \Sigma_v^{-1}(z-Hx) + (x-\mu_p)^T \Sigma_p^{-1}(x-\mu_p)
\end{equation}
\begin{equation}
		= z^T \Sigma_v^{-1}z -2x^T H^T \Sigma_v^{-1} z + x^T H^T \Sigma_v^{-1} H x + x^T\Sigma_p^{-1}x - 2x^T\Sigma_p^{-1}\mu_p + \mu_p^T\Sigma_p^{-1}\mu_p
\end{equation}
\begin{equation}\label{eq:probdev:1}
		= x^T(H^T \Sigma_v^{-1}H + \Sigma_p^{-1})x - 2x^T(H^T\Sigma_v^{-1}z + \Sigma_p^{-1}\mu_p) +   z^T \Sigma_v^{-1}z + \mu_p^T\Sigma_p^{-1}\mu_p.
\end{equation}
Let us define
\begin{equation}\label{eq:probdev:sig_x}
	\Sigma_x \doteq (H^T \Sigma_v^{-1}H + \Sigma_p^{-1})^{-1} = \Sigma_p - \Sigma_p H^T(\Sigma_v + H\Sigma_pH^T)^{-1}H\Sigma_p 
\end{equation}
\begin{equation}\label{eq:probdev:mu_x}
	\mu_x \doteq (H^T \Sigma_v^{-1}H + \Sigma_p^{-1})^{-1}(H^T\Sigma_v^{-1}z + \Sigma_p^{-1}\mu_p) = \Sigma_x(H^T\Sigma_v^{-1}z + \Sigma_p^{-1}\mu_p)
\end{equation}
now we can reformulate (\ref{eq:probdev:1}) as
\begin{equation}
		= x^T\Sigma_x^{-1}x - 2x^T\Sigma_x^{-1}\mu_x + \mu_x^T\Sigma_x^{-1}\mu_x - \mu_x^T\Sigma_x^{-1}\mu_x + z^T \Sigma_v^{-1}z + \mu_p^T\Sigma_p^{-1}\mu_p
\end{equation}
\begin{equation}
		= \|x- \mu_x \|^2_{\Sigma_x}- \mu_x^T\Sigma_x^{-1}\mu_x + z^T \Sigma_v^{-1}z + \mu_p^T\Sigma_p^{-1}\mu_p.
\end{equation}
Going back to (\ref{eq:probdev:explicit}) we get
\begin{equation}
	\prob{z|H^-} = \sqrt{\frac{\mid\Sigma_x \mid}{(2\pi)^{d_z} \mid \Sigma_p\mid \mid \Sigma_v\mid }}e^{-\frac{1}{2}\left(z^T \Sigma_v^{-1}z + \mu_p^T\Sigma_p^{-1}\mu_p - \mu_x^T\Sigma_x^{-1}\mu_x \right) } \int \frac{1}{\sqrt{(2\pi)^{d_x} \mid \Sigma_x\mid}}e^{-\frac{1}{2} \|x-\mu_x \|^2_{\Sigma_x}} dx
\end{equation}
\begin{equation}\label{eq:probdev:wo_x}
	 = \sqrt{\frac{\mid\Sigma_x \mid}{(2\pi)^{d_z} \mid \Sigma_p\mid \mid \Sigma_v\mid}}e^{-\frac{1}{2}\left(z^T \Sigma_v^{-1}z + \mu_p^T\Sigma_p^{-1}\mu_p - \mu_x^T\Sigma_x^{-1}\mu_x \right) }.
\end{equation}

We now try to reformulate the exponent argument in (\ref{eq:probdev:wo_x}) into a quadratic form
\begin{equation}
	z^T \Sigma_v^{-1}z + \mu_p^T\Sigma_p^{-1}\mu_p - \mu_x^T\Sigma_x^{-1}\mu_x = 
\end{equation}
\begin{equation}
	= z^T \Sigma_v^{-1}z + \mu_p^T\Sigma_p^{-1}\mu_p - (H^T\Sigma_v^{-1}z + \Sigma_p^{-1}\mu_p)^T\Sigma_x(H^T\Sigma_v^{-1}z + \Sigma_p^{-1}\mu_p) 
\end{equation}
\begin{equation}\label{eq:probdev:pre_z}
	= z^T \left[ \Sigma_v^{-1} - \Sigma_v^{-1} H \Sigma_x H^T \Sigma_v^{-1} \right] z - 2z^T\Sigma_v^{-1} H \Sigma_x \Sigma_p^{-1}\mu_p - \mu_p^T\Sigma_p^{-1}\Sigma_x\Sigma_p^{-1}\mu_p + \mu_p^T\Sigma_p^{-1}\mu_p
\end{equation}
Let us define 
\begin{equation}\label{eq:probdev:sig_z}
	\Sigma_z \doteq \left[ \Sigma_v^{-1} - \Sigma_v^{-1} H \Sigma_x H^T \Sigma_v^{-1} \right]^{-1} = \left[ \Sigma_v^{-1} - \Sigma_v^{-1} H (H^T \Sigma_v^{-1}H + \Sigma_p^{-1})^{-1} H^T \Sigma_v^{-1} \right]^{-1}
\end{equation}
\begin{equation}
	\Sigma_z \doteq \Sigma_v + H\Sigma_p H^T
\end{equation}
\begin{equation}\label{eq:probdev:mu_z}
	\mu_z \doteq  \left( \Sigma_v + H \Sigma_p H^T \right)\Sigma_v^{-1} H \Sigma_x \Sigma_p^{-1}\mu_p
\end{equation}
now we can reformulate (\ref{eq:probdev:pre_z}) as
\begin{equation}
	= z^T \Sigma_z^{-1} z - 2z^T\Sigma_z^{-1} \mu_z + \mu_z^T\Sigma_z^{-1} \mu_z - \mu_z^T\Sigma_z^{-1} \mu_z - \mu_p^T\Sigma_p^{-1}\Sigma_x\Sigma_p^{-1}\mu_p + \mu_p^T\Sigma_p^{-1}\mu_p
\end{equation}
\begin{equation}\label{eq:probdev:z_arg}
	= \| z - \mu_z\|^2_{\Sigma_z} - \mu_z^T\Sigma_z^{-1} \mu_z - \mu_p^T\Sigma_p^{-1}\Sigma_x\Sigma_p^{-1}\mu_p + \mu_p^T\Sigma_p^{-1}\mu_p.
\end{equation}
Using (\ref{eq:probdev:z_arg}) in (\ref{eq:probdev:wo_x}) we get
\begin{equation}
	\prob{z|H^-} = \sqrt{\frac{\mid\Sigma_x \mid \mid\Sigma_z \mid }{\mid\Sigma_p \mid \mid \Sigma_v\mid}}e^{-\frac{1}{2}\left(- \mu_z^T\Sigma_z^{-1} \mu_z - \mu_p^T\Sigma_p^{-1}\Sigma_x\Sigma_p^{-1}\mu_p + \mu_p^T\Sigma_p^{-1}\mu_p \right) }\frac{1}{\sqrt{(2\pi)^{d_z} \mid \Sigma_z\mid}}e^{-\frac{1}{2} \| z -\mu_z \|_{\Sigma_z}^{2}}.
\end{equation}
Now let us use the fact that $\prob{z|H^-}$ is a valid pdf over $z$

\begin{equation}
	\int \prob{z|H^-} dz = \int \sqrt{\frac{\mid\Sigma_x \mid \mid\Sigma_z \mid }{\mid\Sigma_p \mid \mid \Sigma_v\mid}}e^{-\frac{1}{2}\left(- \mu_z^T\Sigma_z^{-1} \mu_z - \mu_p^T\Sigma_p^{-1}\Sigma_x\Sigma_p^{-1}\mu_p + \mu_p^T\Sigma_p^{-1}\mu_p \right) }\frac{1}{\sqrt{(2\pi)^{d_z} \mid \Sigma_z\mid}}e^{-\frac{1}{2} \| z -\mu_z \|_{\Sigma_z}^{2}} dz
\end{equation}
\begin{equation}
	 = \sqrt{\frac{\mid\Sigma_x \mid \mid\Sigma_z \mid }{\mid\Sigma_p \mid \mid \Sigma_v\mid}}e^{-\frac{1}{2}\left(- \mu_z^T\Sigma_z^{-1} \mu_z - \mu_p^T\Sigma_p^{-1}\Sigma_x\Sigma_p^{-1}\mu_p + \mu_p^T\Sigma_p^{-1}\mu_p \right) }\int \frac{1}{\sqrt{(2\pi)^{d_z} \mid \Sigma_z\mid}}e^{-\frac{1}{2} \| z -\mu_z \|_{\Sigma_z}^{2}} dz
\end{equation}
\begin{equation}
	 = \sqrt{\frac{\mid\Sigma_x \mid \mid\Sigma_z \mid }{\mid\Sigma_p \mid \mid \Sigma_v\mid}}e^{-\frac{1}{2}\left( \mu_p^T(\Sigma_p^{-1} - \Sigma_p^{-1}\Sigma_x\Sigma_p^{-1})\mu_p - \mu_z^T\Sigma_z^{-1} \mu_z \right) } = 1
\end{equation}
using the matrix determinant lemma over $\mid \Sigma_x \mid$ we get
\begin{equation}
	\sqrt{\frac{\mid\Sigma_x \mid \mid\Sigma_z \mid }{\mid\Sigma_p \mid \mid \Sigma_v\mid}} = 
	\sqrt{\frac{\mid \Sigma_v + H \Sigma_p H^T \mid \mid \Sigma_p^{-1} + H^T \Sigma_v^{-1} H \mid ^{-1} }{\mid\Sigma_p \mid \mid \Sigma_v\mid}} = \sqrt{\frac{\mid \Sigma_p^{-1} + H^T \Sigma_v^{-1} H \mid  \mid\Sigma_p \mid \mid \Sigma_v\mid}{ \mid \Sigma_p^{-1} + H^T \Sigma_v^{-1} H \mid  \mid\Sigma_p \mid \mid \Sigma_v\mid}}	= 1,
\end{equation}
meaning the following has to hold
\begin{equation}\label{eq:openQ}
	 \mu_p^T(\Sigma_p^{-1} - \Sigma_p^{-1}\Sigma_x\Sigma_p^{-1})\mu_p = \mu_z^T\Sigma_z^{-1} \mu_z. 
\end{equation}
And so
\begin{equation}
	\prob{z|H^-} = \mathcal{N} (\mu_z, \Sigma_z)
\end{equation}
where
\begin{equation}
	\Sigma_z \doteq \Sigma_v + H\Sigma_p H^T = \Sigma_v + \mathcal{H}\Sigma_w\mathcal{H}^T + \mathcal{H}\mathcal{F}\Sigma_0\mathcal{F}^T\mathcal{H}^T 
\end{equation}
\begin{equation}\label{eq:probdev:mu_z}
	\mu_z \doteq  \left( \Sigma_v + H \Sigma_p H^T \right)\Sigma_v^{-1} H \Sigma_x \Sigma_p^{-1}\mu_p = \mathcal{H}\mathcal{F}\mu_0 + \mathcal{H}\mathcal{J}u.
\end{equation}
%


\section{Gaussian Quadratic as $\chi^2$}
\label{app:chiSquared}
\begin{lemma}
	 Let $\Delta$ be quadratic expression $\Delta = \mathcal{S}^T C \mathcal{S} + c^T \mathcal{S} + y$, where $C^T = C$, $\mathcal{S} \sim \mathcal{N}(\mu, \Sigma)$, and $\Sigma>0$. Then the set of necessary and sufficient conditions for $\Delta$ to be distributed as non-central $\chi^2$ with non-centrality parameter $\delta^2$ and degrees of freedom r is that
	 \begin{equation}
	 	C\Sigma C = C
	 \end{equation}
	\begin{equation}
	 	r = tr(C\Sigma)
	 \end{equation}
	 \begin{equation}
	 	c = C\Sigma c \ , \ y = \frac{1}{4}c^T\Sigma c \ , \ \delta^2 = \mu^T C \mu + \mu^Tc+y
	 \end{equation}
\end{lemma}
\begin{proof}
	This is Theorem 5.1.4 in \cite{Mathai92quadratic}.
\end{proof}


\end{document}